\documentclass[twoside,11pt]{article}

\usepackage{preamble}

\usepackage[abbrvbib]{jmlr2e}

\hypersetup{hidelinks}

\usepackage{lastpage}
\jmlrheading{27}{2026}{1-\pageref{LastPage}}{6/24; Revised
1/26}{4/26}{24-0880}{Achraf Azize and Marc Jourdan and Aymen Al Marjani and Debabrota Basu}
\ShortHeadings{Differentially Private Best-Arm Identification}{Azize, Jourdan, Al Marjani and Basu}
\firstpageno{1}

\begin{document}

\title{Differentially Private Best-Arm Identification}

\author{\name Achraf Azize\thanks{Equal Contribution} \thanks{This work was done when Achraf Azize was a PhD student in the Scool team of Inria Lille.} \email achraf.azize@ensae.fr \\
       \addr FairPlay Joint Team, CREST, ENSAE Paris \\
       \AND
       \name \name Marc Jourdan$^*$\thanks{This work was done when Marc Jourdan was a PhD student in the Scool team of Inria Lille.} \email marc.jourdan@epfl.ch \\
       \addr EPFL, Lausanne, Switzerland \\
       \AND
       \name Aymen Al Marjani\thanks{This work was done when Aymen Al Marjani was a PhD student at ENS Lyon.} \email almarjan@amazon.lu \\
       \addr Amazon  \\
       \AND
       \name Debabrota Basu \email debabrota.basu@inria.fr \\
       \addr Univ. Lille, Inria, CNRS, Centrale Lille, UMR 9189 - CRIStAL, F-59000 Lille, France\\}
       
\editor{Raef Bassily}

\maketitle

\begin{abstract}Best Arm Identification (BAI) problems are progressively used for data-sensitive applications, such as designing adaptive clinical trials, tuning hyper-parameters, and conducting user studies. 
Motivated by the data privacy concerns invoked by these applications, we study the problem of BAI with fixed confidence in both the local and central models, i.e. under $\epsilon$-local and $\epsilon$-global Differential Privacy (DP). 
First, to quantify the cost of privacy, we derive lower bounds on the sample complexity of any $\delta$-correct BAI algorithm satisfying $\epsilon$-global DP or $\epsilon$-local DP. 
Our lower bounds suggest the existence of two privacy regimes.
In the high-privacy regime, the hardness depends on a coupled effect of privacy and novel information-theoretic quantities involving the Total Variation distance. 
In the low-privacy regime, the lower bounds reduce to the non-private lower bounds.  
We propose $\epsilon$-local DP and $\epsilon$-global DP variants of a Top Two algorithm, namely CTB-TT and AdaP-TT$^\star$, respectively.
For $\epsilon$-local DP, CTB-TT is asymptotically optimal by plugging in a private estimator of the means based on Randomised Response.
For $\epsilon$-global DP, our private estimator of the mean runs in arm-dependent adaptive episodes and adds Laplace noise to ensure a good privacy-utility trade-off.
By adapting the transportation costs, the expected sample complexity of AdaP-TT$^\star$ reaches the asymptotic lower bound in the asymptotic high-privacy regime, up to a small multiplicative constant.
\end{abstract}

\begin{keywords}
  differential privacy, multi-armed bandits, best arm identification, fixed confidence, top two algorithm
\end{keywords}

\section{Introduction} 
\label{sec:introduction}

We study the stochastic multi-armed \textit{bandit} problem~\citep{lattimore2018bandit}, which allows us to reflect on fundamental information-utility trade-offs involved in interactive sequential learning. 
Specifically, in a bandit problem, a learning \textit{agent} is exposed to interact with $K$ unknown probability distributions $\lbrace \nu_1, \ldots, \nu_K \rbrace$ with bounded expectations, referred to as the \textit{reward distributions} (or \textit{arms}). 
$\boldsymbol{\nu} \triangleq \lbrace \nu_1, \ldots, \nu_K \rbrace$ is called \textit{a bandit instance}. 
At every step $n>0$, the agent chooses to interact with one of the reward distributions $\nu_{a_n}$ for an arm $a_n \in [K]$, and obtains a sample (or \textit{reward}) $r_n$ from it. 
The goal of the agent can be of two types: (a) maximise the reward accumulated over time, or equivalently to minimise the regret, and (b) to find the reward distribution (or arm) with the highest expected reward. 
The first problem is called the regret-minimisation problem~\citep{auer2002finite}, while the second one is called the \textit{Best Arm Identification (BAI)} problem~\citep{kaufmann2016complexity}. 
In this paper, we focus on the BAI problem, i.e. to compute $a^\star(\boldsymbol{\nu}) \triangleq \argmax_{a \in [K]} \underset{r\sim \nu_{a}}{\mathbb{E}}[r] \triangleq \argmax_{a \in [K]} \mu_{a}$.

With its advent in 1950s~\citep{bechhofer1954single,bechhofer1958sequential} and recent resurgence~\citep{mannor2004sample,gabillon2012best,jamieson2014lil,kaufmann2016complexity}, BAI has been extensively studied with different structural assumptions: Fixed-confidence~\citep{jamieson2014best}; Fixed-budget~\citep{carpentier2016tight}; Non-stochastic~\citep{jamieson2016non}; Best-of-both-worlds~\citep{abbasi2018best}; Linear~\citep{soare2014best}. 
In this paper, we specifically investigate the \textit{Fixed Confidence BAI problem}, in brief FC-BAI, that yields a $\delta$-correct recommendation $\widehat{a} \in [K]$, i.e. the probability that the algorithm stops and returns $\widehat{a} \ne a^\star(\boldsymbol{\nu})$ is upper bounded by $\delta$. 
FC-BAI is increasingly deployed for different applications, such as clinical trials~\citep{aziz2021multi}, hyper-parameter tuning~\citep{li2017hyperband}, communication networks~\citep{lindstaahl2022measurement}, online advertisement~\citep{chen2014combinatorial}, crowd-sourcing~\citep{zhou2014optimal}, user studies~\citep{losada2022day}, and pandemic mitigation~\citep{libin2019bayesian} to name a few. All of these applications often involve the sensitive and personal data of users, which raises serious data privacy concerns~\citep{tucker2016protecting}, as illustrated in Example~\ref{ex:1}. 

\begin{example}[Adaptive dose finding trial]\label{ex:1}
In a dose-finding trial, one physician decides $K$ possible dose levels of a medicine based on preliminary studies---$K \in \{3, \ldots, 10\}$ in practice~\citep{aziz2021multi}. 
At each step $n$, a patient is chosen from a local pool of volunteers and a dose level $a_n \in [K]$ is applied to the patient. 
Following that, the effectiveness of the dose on the patient, i.e. $r_n \in \R$ is observed. 
The goal of the physician is to recommend after the trial, which dose level is most effective on average, i.e. the dose level $a^\star$ that maximises the expected reward.
Here, \emph{every application of a dose level and the patient's reaction to it exposes information regarding the medical conditions of the patient}. 
Additionally, at each step $n$ of an adaptive sequential trial, the physician can use an FC-BAI algorithm that observes the previous history of dose levels $\{a_t\}_{t < n}$ and their effectiveness $\{r_t\}_{t < n}$ to decide on the next dose level $a_{n}$ to test. 
When releasing the experimental findings of the trial to health authorities, the physician should thoroughly detail the experimental protocol. This includes the dose allocated to each patient $\{a_t\}_{t \le n}$ and the final recommended dose level $a^\star$. 
Thus, even if the sequence of reactions to doses $\{r_t\}_{t \le n}$ is kept secret, \emph{publishing the sequence of chosen dose levels $\{a_t\}_{t \le n}$ and the final recommended dose level $a^\star$ computed using the history can leak information regarding patients involved in the trial}. 
\end{example} 

This example demonstrates the need for privacy in best-arm identification. In this paper, we investigate \textit{privacy-utility trade-offs for a privacy-preserving algorithm in FC-BAI}. 
Specifically, we use the celebrated Differential Privacy (DP)~\citep{dwork2014algorithmic} as the framework to preserve data privacy. DP ensures that an algorithm's output is unaffected by changes in input by a single data point. 
By limiting the amount of sensitive information that an adversary can deduce from the output, DP renders an individual corresponding to a data point `\textit{indistinguishable}'. 
Popular ways to achieve DP include Randomised Response~\citep{warner1965randomized} or injecting a calibrated amount of noise, from a Laplace~\citep{dwork2014algorithmic} or Gaussian distribution~\citep{dong2022gaussian}, into the algorithm. 
The scale of the noise is set to be proportional to the algorithm’s sensitivity and inversely proportional to the privacy budget $\epsilon$. 
Specifically, we study \textit{$\epsilon$-local DP}, where users do not trust the data curator, and \textit{$\epsilon$-global DP}, where users trust the centralised decision-maker with access to the raw sensitive rewards. 
For example, in an adaptive dose-finding trial, the patients could trust the physician conducting the trial. 
In that case, at any time $n$, she has access to all the true history $\{a_t, r_t\}_{t < n}$, and it is her duty to design an algorithm such that \textit{publishing $\{a_t\}_{t \leq n}$ and the recommended optimal dose $a^\star$ obeys $\epsilon$-global DP given the sensitive input}, i.e. the effectiveness of the dose levels on the patients $\{r_t\}_{t \le n}$. 
Without this trust from the user, she has only access to a perturbed history $\{a_t, \tilde r_t\}_{t < n}$, where $\tilde r_t$ is a perturbed observation of the true observation $r_t$ which ensures $\epsilon$-local DP.
We define the notions of \textit{$\epsilon$-local DP} and \textit{$\epsilon$-global DP} for BAI rigorously in Section~\ref{sec:background}.

For different settings of bandits, the costs of $\epsilon$-local DP or $\epsilon$-global DP and optimal algorithm design techniques are widely studied in the regret-minimisation problem~\citep{Mishra2015NearlyOD,tossou2016algorithms,dpseOrSheffet,shariff2018differentially,neel2018mitigating,basu2019differential,azize2022privacy,azizeconcentrated}. 
Recently, a problem-dependent lower bound on regret of stochastic multi-armed bandits with $\epsilon$-global DP and an algorithm matching the regret lower bound is proposed by~\citet{azize2022privacy}. 
In contrast, DP is meagerly studied in the FC-BAI problem of bandits~\citep{dpseOrSheffet,kalogerias2020best}.
Though \textit{efficient} algorithm design in FC-BAI literature is traditionally propelled by deriving tight lower bounds, we do not have any explicit sample complexity lower bound for FC-BAI satisfying $\epsilon$-local DP or $\epsilon$-global DP. 
By ``efficient'' algorithm, we refer to the FC-BAI algorithms that aim to minimise the expected number of samples required (i.e. \textit{sample complexity}) to find a $\delta$-correct recommendation.
Presently, we know neither the minimal cost in terms of sample complexity for ensuring DP in FC-BAI, nor the feasibility of efficient algorithm design to achieve the minimal cost.

\subsection{Contributions}
\label{ssec:contributions}

Motivated by this gap in the literature, this paper answers the following two questions:\\
\indent \textit{A.} \textit{How many additional samples a BAI strategy must need to ensure $\epsilon$-local DP?}\\
\indent \textit{B.} \textit{How many additional samples a BAI strategy must need to ensure $\epsilon$-global DP?}\\
\marc{Gathered in Appendix~\ref{app:notation} for convenience, the notation are introduced in context.}

\subsubsection{Lower Bounds} 
First, we derive a lower bound on the expected sample complexity of any $\delta$-correct FC-BAI algorithm to ensure either $\epsilon$-local DP (Theorem~\ref{thm:local_lower_bound} and Corollary~\ref{lem:local_more_explicit_lower_bound}) or $\epsilon$-global DP (Theorem~\ref{thm:global_lower_bound} and Corollary~\ref{lem:more_explicit_lower_bound}).
Due to the $\delta$-correctness and the DP constraints, each of the lower bounds corresponds to the minimum of two characteristic times.
The first one is the $\mathrm{KL}$ characteristic time $T^{\star}_{\mathrm{KL}}(\boldsymbol{\nu})$ of the non-private FC-BAI~\citep{kaufmann2016complexity} (Lemma~\ref{lem:garivier2016optimal}).
The second one depends on the privacy $\epsilon$ and novel information-theoretic quantities depending on the Total Variation (TV) distance: the $\mathrm{TV}^2$ characteristic time $T^{\star}_{\mathrm{TV}^2}(\boldsymbol{\nu})$ for $\epsilon$-local DP and the $\mathrm{TV}$ characteristic time $T^{\star}_{\mathrm{TV}}(\boldsymbol{\nu})$ for $\epsilon$-global DP.
As for $\epsilon$-global DP regret minimisation~\citep{azize2022privacy}, the lower bound indicates that there are two regimes of hardness depending on $\epsilon$ and the aforementioned characteristic times. 
For lower levels of privacy (i.e. higher $\epsilon$), the expected sample complexity matches the non-private FC-BAI. 
However, for higher levels of privacy (i.e. lower $\epsilon$), the expected sample complexity depends both on the privacy budget $\epsilon$ and the $\mathrm{TV}^2$ or $\mathrm{TV}$ characteristic time.
To derive the lower bound for $\epsilon$-global DP, we provide an $\epsilon$-global DP version of the ``change-of-measure'' lemma~\citep{kaufmann2013information} (Lemma~\ref{lem:chg_env_dp}), which we prove using a sequential coupling argument. 

\subsubsection{Algorithm Design}
We propose algorithms which are $\delta$-correct, and either $\epsilon$-local DP or $\epsilon$-global DP.
While most existing asymptotically optimal FC-BAI algorithms can be modified to tackle DP, we consider the class of Top Two algorithms~\citep{russo2016simple} due to their good empirical performances, low computational cost, and easy implementation.
As a case study, we consider the \hyperlink{TTUCB}{TTUCB} meta-algorithm based on the work of~\citep{jourdan2022non}.
Table~\ref{tab:algorithm_design} summarises the different instances that we propose.
We highlight that our private wrappers could be used for other FC-BAI algorithms.

\textit{A.}
\textit{For $\epsilon$-local DP,} we propose the \locptt{} algorithm.
\locptt{} plugs in the \hyperlink{CTB}{CTB}$(\epsilon)$ estimator of the means, which is $\epsilon$-local DP~\citep[Lemma~\ref{lem:lem5_ren2020multi}]{ren2020multi}, into the TTUCB algorithm for $\sigma$-sub-Gaussian distributions, which is $\delta$-correct.

\textit{B.} 
\textit{For $\epsilon$-global DP,} we propose the \hyperlink{DAF}{DAF}$(\epsilon)$ estimator of the means, which is $\epsilon$-global DP (Lemma~\ref{lem:DAF_private}).
It relies on three ingredients: \textit{adaptive episodes with doubling per arm}, \textit{forgetting}, and \textit{adding calibrated Laplacian noise}. 
Using the \hyperlink{DAF}{DAF}$(\epsilon)$ estimator in the \hyperlink{TTUCB}{TTUCB} meta-algorithm, we propose the \adaptt{} and the \adapttt{} algorithms.
As a plug-in approach, \adaptt{} uses the non-private transportation costs both in the TC challenger and in the \marc{generalised likelihood ratio (GLR)} stopping rule, which is shown to be $\delta$-correct by adding a privacy term to the stopping threshold (Lemma~\ref{lem:delta_correct_threshold_TTUCB_global_v1}).
As a lower bound based approach, \adapttt{} adapts the transportation costs to account for $\epsilon$-global privacy both in the TC challenger and in the GLR stopping rule, which is shown to be $\delta$-correct by modifying the privacy term in the stopping threshold (Lemma~\ref{lem:delta_correct_threshold_TTUCB_global_v2}).

\subsubsection{Upper Bounds}
We show that the proposed algorithms exhibit upper bounds that match the lower bounds up to multiplicative constants.
We highlight that our generic asymptotic analysis can be applied to any Top Two algorithms, since it builds on the one of~\citet{jourdan2022top}.

\textit{A.}
As \locptt{} is equivalent to running TTUCB on a modified Bernoulli instance $\boldsymbol{\nu}_{\epsilon}$, it recovers the asymptotic and non-asymptotic upper bounds on the expected sample complexity derived in~\citet{jourdan2022non}.
The asymptotic upper bound matches the asymptotic lower bound up to a constant multiplicative term, $2$ when $\epsilon \to + \infty$ and $4$ when $\epsilon \to \marc{0}$. 
Our experiments confirm the good performance of \locptt{}, and the existence of two hardness regimes for $\epsilon$-local DP (Section~\ref{sec:expe_localDP}).

\textit{B.} 
Using the \hyperlink{DAF}{DAF}$(\epsilon)$ estimator yields a batched algorithm with adaptive and data-dependent changes of episodes.
While our analysis is inspired by the one of~\citet{jourdan2022top}, studying \adaptt{} and \adapttt{} requires carefully quantifying the effects of doubling, forgetting, and adding noise.
We derive an asymptotic upper bound on the expected sample complexity of \adaptt{} (Theorem~\ref{thm:sample_complexity_TTUCB_global_v1}) and \adapttt{} (Theorem~\ref{thm:sample_complexity_TTUCB_global_v2}).
In the non-private regime of $\epsilon \to + \infty$, both algorithms recover the asymptotic lower bound \marc{for Gaussian distributions,} up to multiplicative constants ($16$ and $8$ respectively), with solely $\mathcal O \left( K \log_2 (T^\star_{\mathrm{KL}}(\boldsymbol{\nu})\log(1/\delta)) \right)$ rounds of adaptivity.
When $\epsilon \to 0$, \adapttt{} achieves the asymptotic lower bound up to a multiplicative constant $48$, while \adaptt{} only recovers it for instances where the mean gaps have the same order of magnitude.
Our experiments show the good performance of our algorithms compared to DP-SE~\citep{dpseOrSheffet}, which can be adapted for FC-BAI (see Section~\ref{ssec:lower_bound_based_TTUCB_global} for a detailed comparison).
They confirm the existence of two hardness regimes for $\epsilon$-global DP, as well as the empirical superiority of \adapttt{} over \adaptt{} when $\epsilon \to 0$ (Section~\ref{sec:expe_globalDP}).

\subsection{Outline}

After presenting Differential Privacy and Best-Arm Identification in the Fixed-Confidence Setting in Section~\ref{sec:background}, we formulate the problem of private best-arm identification.
We present lower bounds and matching upper bounds for $\epsilon$-local DP FC-BAI (Section~\ref{sec:local_DP}) and $\epsilon$-global DP FC-BAI (Section~\ref{sec:global_DP}).
Our algorithms are studied empirically in Section~\ref{sec:experiments}.
For clarity, standalone pseudocodes of our algorithms are given in Appendix~\ref{app:algo_pseudocode}.

\section{Differential Privacy and Best-Arm Identification}
\label{sec:background}

In this section, we provide relevant background information on Differential Privacy (DP) in Section~\ref{ssec:background_DP}, and Best-Arm Identification in the Fixed-Confidence Setting (FC-BAI) in Section~\ref{ssec:background_FC_BAI}. 
Then, we formulate the problem of private best-arm identification (FC-BAI with DP) in Section~\ref{ssec:background_pb_statement}, under both the local and global trust models.

\subsection{Background: Differential Privacy}
\label{ssec:background_DP}

Differential Privacy (DP) ensures the protection of an individual's sensitive information when her data is used for analysis. 
A randomised algorithm satisfies DP if the output of the algorithm stays almost the same, regardless of whether any single individual's data is included in or excluded from the input. 
One way of achieving DP is by adding controlled noise to the algorithm's output.

\begin{definition}[$(\epsilon, \delta)$-DP,~\citealp{Dwork_Calibration}]
	A randomised algorithm $\alg$ is $(\epsilon, \delta)$-DP (Differential Privacy) if for any two neighbouring data sets $\mathcal D$ and $\mathcal D'$ that differ only in one entry, i.e. $\dham(\mathcal D, \mathcal D') = 1$, and for all sets of output $\mathcal{O} \subseteq \mathrm{Range}(\alg)$,
	\begin{equation*}
		\operatorname{Pr}[\alg(\mathcal D) \in \mathcal{O}] \leq e^{\epsilon} \operatorname{Pr}\left[\alg\left(\mathcal D'\right) \in \mathcal{O}\right] + \delta \: ,
	\end{equation*}
	where the probability space is over the coin flips of the mechanism $\alg$, and $(\epsilon, \delta)\in \real^{\geq0} \times \real^{\geq0}$. 
	If $\delta = 0$, we say that $\alg$ satisfies \emph{$\epsilon$-DP}. A lower privacy budget $\epsilon$ implies higher privacy.
\end{definition}

The Laplace mechanism~\citep{dwork2010differential, dwork2014algorithmic} ensures $\epsilon$-DP by injecting controlled random noise into the output of the algorithm, which is sampled from a calibrated Laplace distribution (as specified in Theorem~\ref{thm:laplace}). 
We use $Lap(b)$ to denote the Laplace distribution with mean 0 and variance $2b^2$. 

\begin{theorem}[Laplace mechanism, Theorem 3.6,~\citealp{dwork2014algorithmic}]\label{thm:laplace}Let $f: \mathcal{X} \rightarrow \real^d$ be an algorithm with sensitivity $s(f) \defn \underset{\substack{\mathcal{D}, \mathcal{D'} \text{ s.t }|\mathcal{D} - \mathcal{D'}|_{\mathrm{Hamming}} = 1}}{\max} \onenorm{f(\mathcal{D}) - f(\mathcal{D'})}$, where $\onenorm{\cdot}$ is the $L_1$ norm. 
 If samples $\setof{N_i}_{i=1}^d$ are generated independently from $Lap\left(\frac{s(f)}{\epsilon}\right)$, then the output injected with the noise, i.e. $f(\mathcal{D}) + [N_1, \ldots, N_d]$, satisfies $\epsilon$-DP.
\end{theorem}

We also study the setting of \textit{local differential privacy}, where users do not trust the data curator, i.e. the entity collecting the data. 
Local DP is one of the oldest formulations of privacy, dating back to~\citet{warner1965randomized}, who advocated it as a solution to what he called ``evasive answer bias" in survey sampling.

\begin{definition}[$\epsilon$-local DP,~\citealp{dinur2003revealing, evfimievski2003limiting}]
    A \\randomised algorithm $\mathcal{M}$ satisfies $\epsilon$-local DP if for any pair of input values $x, x' \in \mathcal D$, and for all sets of output $\mathcal{O} \subseteq \mathrm{Range}(\mathcal{M})$,
	\begin{equation*}\operatorname{Pr}[\mathcal{M}(x) \in \mathcal{O}] \leq e^{\epsilon} \operatorname{Pr}\left[\mathcal{M}(x') \in \mathcal{O}\right],
	\end{equation*}
	where the probability space is over the coin flips of the mechanism $\mathcal{M}$, and for some $\epsilon \in \real^{\geq0}$. The perturbation mechanism $\mathcal{M}$ is applied to each user record independently.
\end{definition}

For binary attributes, the Randomised Response (RR) mechanism~\citep{warner1965randomized} is a popular way to achieve $\epsilon$-local DP. The idea is to output the true value of a user's response with probability $e^\epsilon/(e^\epsilon+1)$ and output the opposite value with probability $1/(e^\epsilon + 1)$. 
To make it suitable for larger discrete domains, a Generalised Randomised Response (GRR) is proposed in~\citet{kairouz2016discrete}. 
For continuous numerical data statistics, adding Laplace noise to each data record achieves local DP as well.

\subsection{Background: Best Arm Identification in the Fixed-Confidence Setting}
\label{ssec:background_FC_BAI}
In this section, we first present the Best-arm identification (BAI) problem, a BAI strategy and $\delta$-correctness. Then, we present a lower bound on the sample complexity of any $\delta$-correct BAI strategy. Finally, we discuss algorithms in the BAI literature which match the sample complexity lower bound. We focus on the Top Two family of algorithms since they enjoy both theoretical optimality and good empirical performance. 

\subsubsection{The Best Arm Identification Problem}

Best-arm identification (BAI) is a pure exploration problem that aims to identify the optimal arm.
It has been studied in two major theoretical frameworks \citep{Bubeck10BestArm,gabillon2012best,jamieson2014best,garivier2016optimal}: the fixed-confidence and fixed-budget setting. 
In the fixed-budget setting, the objective is to minimise the probability of misidentifying a correct answer with a fixed number of samples $T$.
We consider the fixed-confidence setting (FC-BAI), in which the learner aims at minimising the number of samples used to identify a correct answer with confidence $1 - \delta \in (0,1)$.~\footnote{We remind not to confuse risk level $\delta$ with the $\delta$ of $(\epsilon, \delta)$-DP. Hereafter, we consider $\epsilon$-global DP as the privacy definition, and $\delta$ always represents the risk (or probability of mistake) of the BAI strategy.}
To achieve this, the learner defines an FC-BAI strategy to interact with the bandit instance $\boldsymbol{\nu} = \{\nu_a\}_{a \in [\arms]} \in \mathcal F^{K}$, consisting of $K$ arms with finite means $\{\mu_a\}_{a \in [\arms]} \in (0,1)^{K}$. 
We assume that there is a unique best arm $a^\star (\boldsymbol{\nu})$ defined as $a^\star (\boldsymbol{\nu}) = \argmax_{a \in [\arms]} \mu_a$.
The set of distributions $\mathcal F$ will depend on the considered result, e.g. Bernoulli distributions, bounded distributions on $[0,1]$ or $\sigma$-sub-Gaussian distributions. 
A distribution $\kappa$ is $\sigma$-sub-Gaussian if it satisfies $\bE_{X \sim \kappa}[e^{\lambda(X - \bE_{X \sim \kappa}[X])}] \le e^{\sigma^2\lambda^2/2}$ for all $\lambda \in \R$.

We denote the action played at step $n$ by $a_n$, and the corresponding observed reward by $r_n \sim \nu_{a_n}$. 
The $\sigma$-algebra $\mathcal{H}_n=\sigma\left(a_1, r_1, \ldots, a_n, r_n\right)$ is the history of actions played and rewards collected at the end of time $n$. 
We augment the action set by a \textit{stopping action} $\top$, and write $a_n=\top$ to denote that the algorithm has stopped before step $n$. 
A FC-BAI strategy $\pi$ is composed of

\textit{i. A pair of sampling and stopping rules} $\left(\mathrm{S}_n: \mathcal{H}_{n-1} \rightarrow \mathcal{P}([|1, K|] \cup\{\top\})\right)_{n \ge 1}$. 
For an action $a \in$ $[K],$ $ \mathrm{S}_n\left(a \mid \mathcal{H}_{n-1}\right)$ denotes the probability of playing action $a$ given history $\mathcal{H}_{n-1}$. 
On the other hand, $\mathrm{S}_n\left(\top \mid \mathcal{H}_{n-1}\right)$ is the probability of the algorithm halting given $\mathcal{H}_{n-1}$. 
For any history $\mathcal{H}_{n-1}$, a consistent sampling and stopping rule $\mathrm{S}_n$ satisfies $\mathrm{S}_n\left(\top \mid \mathcal{H}_{n-1}\right)=1$ if $\top$ has been played before $n$.

\textit{ii. A recommendation rule} $\left(\operatorname{Rec}_n: \mathcal{H}_{n-1} \rightarrow \mathcal{P}([|1, K|])\right)_{n>1}$. 
A recommendation rule dictates $\operatorname{Rec}_n\left(a \mid \mathcal{H}_{n-1}\right)$, i.e. the probability of returning action $a$ as a guess for the best action given $\mathcal{H}_{n-1}$.

We denote by $\tau_{\delta}$ the \textbf{stopping time} (or \textbf{sample complexity}) of the algorithm, i.e. the first step $n$ demonstrating $a_n = \top$.
A FC-BAI strategy $\pi$ is called \textbf{$\delta$-correct} for a class of bandit instances $\mathcal{M} \subseteq \mathcal F^{K}$, if for every instance $\boldsymbol{\nu} \in \mathcal{M}$, $\pi$ recommends $\widehat{a}$ as the optimal action $a^\star (\boldsymbol{\nu})$ with probability at least $1-\delta$, i.e. $\mathbb{P}_{\boldsymbol{\nu}} ( \tau_{\delta} <+\infty, \widehat{a} = a^\star(\boldsymbol{\nu}) ) \geq 1 - \delta$.

\subsubsection{Lower Bound on the Expected Sample Complexity}

Being $\delta$-correct imposes a lower bound on the expected sample complexity on any instance.

\begin{lemma}[\citealt{garivier2016optimal}] \label{lem:garivier2016optimal}
	For all $\delta$-correct FC-BAI strategy and all instances $\bm{\nu} \in \mathcal M$ \Modif{with unique best arm}, we have that $\mathbb{E}_{\boldsymbol{\nu}}[\tau_{\delta}] \ge T^{\star}_{\mathrm{KL}}(\boldsymbol{\nu}) \log (1 / (2.4 \delta))$ with 
	\begin{equation} \label{eq:characteristic_time}
	T^{\star}_{\textbf{d}}(\boldsymbol{\nu})^{-1} \defn \sup _{\omega \in \Sigma_K} \inf _{\boldsymbol{\lambda} \in \operatorname{Alt}(\boldsymbol{\nu})}  \sum_{a \in [K]} \omega_a \textbf{d}(\nu_a, \lambda_a)  \: ,
	\end{equation}
	where the probability simplex is denoted by $\Sigma_K \triangleq \{ \omega \in [0,1]^\arms \mid \sum_{a=1}^\arms \omega_a = 1 \} $ and the set of alternative instances is $\operatorname{Alt}(\boldsymbol{\nu})\triangleq \left\{\boldsymbol{\lambda} \in \mathcal M  \mid a^{\star}(\boldsymbol{\lambda}) \neq a^{\star}(\boldsymbol{\nu})\right\}$, i.e. the bandit instances with a different optimal arm than $\boldsymbol{\nu}$. 
	For two probability distributions $\mathbb{P}, \mathbb{Q}$ on $(\Omega, \mathcal{F})$, the KL divergence is $\KL{\mathbb{P}}{\mathbb{Q}} \triangleq \int \log\left(\frac{\dd\mathbb{P}}{\dd\mathbb{Q}} (\omega) \right) \dd \mathbb{P} (\omega)$, when $\mathbb{P} \ll \mathbb{Q}$, and $+\infty$ otherwise.
\end{lemma}

\marc{The characteristic time $T^{\star}_{\textbf{d}}(\boldsymbol{\nu})$ in the lower bound is the value of a two-player zero-sum game. 
The MIN player plays instances $\bm \lambda \in \mathcal M$ having an answer different from $a^{\star}(\boldsymbol{\nu})$, while staying close to $\bm \nu$ in terms of allocation-reweighted sum of $\textbf{d}(\nu_a, \lambda_a)$, in order to confuse the MAX player.
The MAX player plays an arm allocation $w \in \Sigma_K$ to distinguish between $\bm \nu$ and $\bm \lambda$.
The measure $\textbf{d}$ captures the ``distinguishability'' between instances.
In the non-private lower bounds, this is captured by the $\mathrm{KL}$ divergence for parametric distributions (Lemma~\ref{lem:garivier2016optimal}) and by the Kinf (i.e., $\inf \mathrm{KL}$ under mean constraint) for non-parametric distributions~\citep{Agrawal20GeneBAI}.}

Early FC-BAI algorithms failed to reach the lower bound of Theorem~\ref{lem:garivier2016optimal}, e.g. Successive Elimination (SE) based algorithms~\citep{even2006action} or confidence bounds based algorithms, e.g. LUCB~\citep{kalyanakrishnan2012pac} or lil'UCB~\citep{jamieson2014lil}.
Inspired by this lower bound, many algorithms have been designed to tackle FC-BAI.
The Track-and-Stop algorithm~\citep{garivier2016optimal} is the first algorithm to reach asymptotic optimality, by sequentially solving the optimisation problem $T^{\star}_{\mathrm{KL}}(\boldsymbol{\nu_n})$ and tracking the associated optimal weights.
To reduce the computational cost of Track-and-Stop, several asymptotically optimal algorithms have been proposed recently: online optimisation-based approach, e.g. game-based algorithm~\citep{degenne2019non} or FWS~\citep{wang2021fast}, and Top Two algorithms~\citep{russo2016simple}.
While most algorithms can be modified to tackle $\epsilon$-DP, we consider the Top Two algorithms due to their great empirical performance and easy implementation.
At every step, a Top Two sampling rule selects the next arm to sample from among two candidate arms, a leader and a challenger. 
In recent years, numerous variants of Top Two algorithms have been analysed and shown to be asymptotically optimal~\citep{russo2016simple,qin2017improving,shang2020fixed,jourdan2022top,you2023information,jourdan2024varepsilon}.
In particular, we consider one particular case study, i.e. the TTUCB algorithm~\citep{jourdan2022non}, but our approach can be directly adapted to any other Top Two algorithms. 

\begin{remark}
    The private estimators of the means and the private stopping rules presented in Sections~\ref{sec:local_DP} and~\ref{sec:global_DP} could be used with most existing existing FC-BAI algorithm.
    The resulting algorithms will be private, $\delta$-correct and have near-optimal asymptotic sample complexity.
    In this work, we consider and rigorously analyse the Top Two algorithms since they are simple algorithms enjoying both strong theoretical guarantees and empirical performance.
\end{remark}

\begin{algorithm}[t!]
         \caption{\protect\hypertarget{TTUCB}{TTUCB} Meta-algorithm}
         \label{algo:TTUCB}
	\begin{algorithmic}[1]
 		\State {\bfseries Input:} \Modif{setting parameter $\delta \in (0,1)$, algorithmic hyperparameter $\beta \in (0,1)$, e.g., $\beta=1/2$}, confidence bonuses $(b_{a})_{a \in [K]}$, transportation costs $(W_{a,b})_{(a,b) \in [K]^2}$, estimator mechanisms $(\text{ESTIMATOR}_{a})_{a \in [K]}$ and stopping conditions $(\text{STOP}_{a,b})_{(a,b) \in [K]^2}$.
        \State {\bfseries Initialisation:} Observe $r_{a} \sim \nu_{a}$ for all arms $a \in [K]$; Initialise the estimators $\tilde \mu_{n,a} = r_{a}$ and $N_{n, a} = \tilde N_{n, a} = 1$ where $n = K+1$;
        \For{$n > K$}
				\State Set arm $\hat a_n = \argmax_{a \in [K]} \tilde \mu_{n,a}$; \Comment{\Modif{Recommendation rule}}
                \If{$\text{STOP}_{\hat a_n,a}(\tilde \mu_n, \tilde N_n, \delta) \ge 0$ for all $a \neq \hat a_n$} \Comment{\Modif{Stopping rule}}
                    \State Set $a_n = \top$ and {\bfseries return} $\hat a_n$;
                \EndIf
        		\State Set arm $B_n = \argmax_{a \in [K]} \left\{\tilde \mu_{n,a} + b_{a}(\tilde N_{n})\right\}$; \Comment{\Modif{UCB leader}}
                \State Set arm $C_n = \argmin_{a \ne B_n} W_{B_n,a}(\tilde \mu_n, N_n)$; \Comment{\Modif{TC challenger}}
				\State Set arm $a_n = B_n$ if $N_{n,B_n}^{B_n} \le \beta L_{n+1,B_n}$,  and $a_n = C_n$ otherwise; \Comment{\Modif{$\beta$-tracking}}
				\State Pull $a_n$ and observe $r_{n} \sim \nu_{a_n}$; \Comment{\Modif{Sampling rule}}
                \State Set $N_{n+1, a_n} \gets N_{n, a_n} + 1$, $L_{n+1, B_n} \gets L_{n, B_n} + 1$, $N^{B_n}_{n+1, B_n} \gets N^{B_n}_{n, B_n} + \indi{B_n=a_n}$; 
                \State Get $(\tilde \mu_{n+1,a_{n}}, \tilde N_{n+1,a_{n}}) = \text{ESTIMATOR}_{a_{n}}(\mathcal H_{n})$ and $n \gets n + 1 $; \Comment{Update rule}
        \EndFor
	\end{algorithmic}
\end{algorithm}

\begin{algorithm}[t!]
         \caption{Maximum Likelihood Estimator (\protect\hypertarget{MLE}{MLE})}
         \label{algo:MLE}
    \begin{algorithmic}[1]
 	\State {\bfseries Input:} History $\mathcal H_{n}$, arm $a \in [K]$.
        \State \textbf{Return} $(\hat \mu_{n,a}, N_{n,a})$ where $\hat \mu_{n,a} = N_{n,a}^{-1} \sum_{t \in [n-1]} r_{t} \ind{a_t = a}$;
    \end{algorithmic}
\end{algorithm}

\subsubsection{The TTUCB Meta-algorithm}

Since we will propose several private algorithm building on top of TTUCB, we propose a \hyperlink{TTUCB}{TTUCB} meta-algorithm (Algorithm~\ref{algo:TTUCB}). 
To instantiate it, one should specify: a parameter $\beta \in (0,1)$ (e.g. $\beta = 1/2$), confidence bonuses $(b_{a})_{a \in [K]}$ where $b_{a} : \N^K \to \R_+$, transportation costs $(W_{a,b})_{(a,b) \in [K]^2}$ where $W_{a,b} : \real^{K} \times \N^{K} \to \R_+$, estimator mechanisms $(\text{ESTIMATOR}_{a})_{a \in [K]}$ computing data-dependent estimators $(\mu_{n,a})_{(n,a) \in \N \times [K]}$ based on local counts $(\tilde N_{n,a})_{(n,a) \in \N \times [K]}$, and \marc{generalised likelihood ratio (GLR)} stopping conditions
\begin{equation} \label{eq:GLR_stopping_rule}
    \forall (a,b) \in [K]^2, \quad \text{STOP}_{a,b}(\tilde \mu, \omega, \delta) = W_{a,b}(\tilde \mu, \omega) - c_{a,b}(\omega, \delta) \: ,
\end{equation}
where $(c_{a,b})_{(a,b) \in [K]^2}$ where $c_{a,b} : \N^{K} \times (0,1) \to \R_+$ are stopping thresholds.

For $\sigma$-sub-Gaussian distributions, TTUCB in~\citet{jourdan2022non} is an instance of Algorithm~\ref{algo:TTUCB} using the \hyperlink{MLE}{MLE} (Algorithm~\ref{algo:MLE}) and
\begin{align} \label{eq:TTUCB_Gaussian}
	&W^{G}_{a,b}(\tilde \mu, \omega) =  \frac{(\tilde \mu_{a} - \tilde \mu_{b})^2_{+}}{2\sigma^2(1/\omega_{a}+ 1/\omega_{b})} \: \text{ and } \: b^{G}_{a}(\omega) = \sqrt{\frac{2 \sigma^2 \alpha (1+s) \log \|\omega\|_{1}}{\omega_{a}}} \: \text{with} \: s,\alpha > 1  \: .
\end{align} 
In practice, they take $s=\alpha=1.2$.
\marc{The standalone pseudocode of TTUCB is detailed in Algorithm~\ref{algo:TTUCB} (Appendix~\ref{app:algo_pseudocode}).}
The GLR stopping rule has to ensure $\delta$-correctness.
This is done by choosing the stopping threshold as
\begin{equation} \label{eq:threshold_non_private}
	c^{G}_{a,b}(\omega,\delta) = 2 \mathcal{C}_{G}(\log\left((K-1)/\delta\right)/2) + 2 \log (4 + \ln \omega_{a}) + 2 \log (4 + \ln \omega_{b}) \: ,
\end{equation}
where the function $\cC_{G}$ is defined in Eq.~\eqref{eq:def_C_gaussian_KK18Mixtures}. 
It satisfies $\cC_{G}(x) \approx x + \ln(x)$.
For bounded distributions on $[0,1]$ such as Bernoulli, we take $\sigma = 1/2$.

At each step, a Top Two algorithm selects two arms called leader and challenger, and samples one arm among them.
TTUCB uses a UCB-based leader and a Transportation Cost (TC) challenger.
The theoretical motivation behind the TC challenger comes from the theoretical lower bound in FC-BAI (Lemma~\ref{lem:garivier2016optimal}), which involves the KL-characteristic time $T^{\star}_{\mathrm{KL}}(\boldsymbol{\nu}) = \min_{\beta \in (0,1)} T^{\star}_{\mathrm{KL},\beta}(\boldsymbol{\nu})$. 
For Gaussian distributions $\nu_{a} = \mathcal N(\mu_a, \sigma^2)$, it writes as
\begin{equation}\label{eq:T_KL_Gaussian}
    T^{\star}_{\mathrm{KL},\beta}(\boldsymbol{\nu})^{-1} = \max_{\omega \in \Sigma_K, \omega_{a^\star} = \beta} \frac{(\mu_{a^\star} - \mu_{a})^2}{2\sigma^2(1/\beta + 1/\omega_{a})} \quad \text{and} \quad T^{\star}_{\mathrm{KL},1/2}(\boldsymbol{\nu}) \le 2 T^{\star}_{\mathrm{KL}}(\boldsymbol{\nu}) \: .
\end{equation}
Note that $H (\boldsymbol{\nu}) \le T^{\star}_{\mathrm{KL}}(\boldsymbol{\nu}) \le 2 H (\boldsymbol{\nu})$ where $H (\boldsymbol{\nu}) = 2\sigma^2\sum_{a \in [K]} \Delta_{a}^{-2}$ with $\Delta_{a} = \mu_{a^\star} - \mu_{a}$ for all $a \ne a^\star$ and $\Delta_{a^\star} = \Delta_{\min} = \min_{a \ne a^\star} ( \mu_{a^\star} - \mu_{a} )$ for all $a \ne a^\star$.
The maximiser of~\eqref{eq:T_KL_Gaussian} is denoted by $\omega^{\star}_{\mathrm{KL},\beta}(\boldsymbol{\nu})$, and is further referred to as the $\beta$-optimal allocation as it is unique. 
Let $N_{n,b}^{a}$ denote the number of times arm $b$ was pulled when $a$ was the leader, and $L_{n,a}$ denotes the number of times arm $a$ was the leader.
In order to select the next arm to sample $a_n$, TTUCB relies on $K$ tracking procedures, i.e. set $a_n = B_n$ if $N_{n,B_n}^{B_n} \le  \beta L_{n+1,B_n}$, else $a_n = C_n$.
This ensures that $\max_{a \in [K], n>K}|N_{n,a}^{a} - \beta L_{n,a}| \le 1$ \citep{degenne2020structure}.

\subsection{Problem Statement: FC-BAI with DP}
\label{ssec:background_pb_statement}

Now, we formally extend DP to BAI.
We consider two trust models: (1) $\epsilon$-local DP BAI, where each user sends her reward to the BAI strategy, using an $\epsilon$-local DP perturbation mechanism, and (2) $\epsilon$-global DP BAI, where the BAI strategy, a.k.a. the centralised decision maker, is trusted with all the intermediate rewards. 
We summarise the BAI strategy-Users interaction in Algorithm~\ref{prot:bai}, under global DP and local DP.

\subsubsection{Local DP FC-BAI}
We represent each user $u_t$ by the vector $\textbf{x}_t \defn (x_{t, 1}, \dots, x_{t, \arms}) \in \real^\arms$, where $x_{t, a}$ represents the \textbf{potential} reward observed, if action $a$ was recommended to user $u_t$. Due to the bandit feedback, only $r_t = x_{t, a_t} \sim \nu_{a_t}$ is observed at step $t$. The user observes the real reward $r_t = \textbf{x}_{t, a_t}$ but only sends a noisy version $z_t$ to the BAI strategy, by sampling $z_t$ from the perturbation mechanism, i.e. $z_t \sim \mathcal{M}(r_t)$. The BAI strategy only has access to the noisy rewards $(z_t)$ to make its decisions. 

\begin{definition}[$\epsilon$-local DP for BAI]\label{def:local_dp}
A pair $(\mathcal{M}, \pi)$ of perturbation mechanism and BAI strategy satisfies \textbf{$\epsilon$-local DP}, if they satisfy

\noindent (a) The perturbation mechanism $\mathcal{M}$ is $\epsilon$-local DP with respect to each reward record, i.e. for all $\horizon$, all rewards $r_t, r'_t$ and all noisy outputs $z_t$, $\operatorname{Pr}[\mathcal{M}(r_t) = z_t] \leq e^{\epsilon} \operatorname{Pr}\left[\mathcal{M}(r'_t) = z_t\right]$.

\noindent (b) The BAI strategy only has access to the noisy rewards $z_t \sim \mathcal{M}(r_t)$ to make its decisions. 
\end{definition}

For a pair $(\mathcal{M}, \pi)$ to be $\delta$-correct with respect to an environment $\nu$, under a local DP interaction protocol, the pair should verify: (a) the perturbation mechanism $\mathcal{M}$ should not change the identity of the optimal arm, i.e. $a^\star(\nu) = a^\star(\nu^\mathcal{M})$ and (b) the BAI strategy $\pi$ should be $\delta$-correct for the noisy environment $\nu^\mathcal{M}$. The goal in $\epsilon$-local DP FC-BAI is to design a $\delta$-correct $\epsilon$-local DP pair $(\mathcal{M}, \pi)$ of perturbation mechanism and BAI strategy, with $\bE[\tau_{\delta}]$ as small as possible. 

\setlength{\textfloatsep}{8pt}\begin{algorithm}[t!]
\caption{Sequential Interaction Between a BAI Strategy and Users}\label{prot:bai}
\begin{algorithmic}[1]
\State {\bfseries Input:} A BAI strategy $\pi$, Users $\{u_t\}_{n \ge 1}$ represented by the table $\underline{\textbf{d}}$ and a perturbation mechanism $\mathcal{M}$
\State {\bfseries Output:} A stopping time $\tau$, a sequence of samples actions $\underline{a}^\tau = (a_1, \dots, a_\tau)$ and a recommendation $\hat{a}$ satisfying $\epsilon$-DP
\For{$t = 1, \dots$} 
\State $\pi$ recommends action $a_t \sim S_t(. \mid a_1, z_1, \dots, a_{t - 1}, z_{t - 1})$
\If{$a_t = \top$}
    \State Halt. Return $\tau = t$ and $ \hat{a} \sim \operatorname{Rec}_t(. \mid a_1, z_1, \dots, a_{t - 1}, z_{t - 1})$
\Else
\If{Global DP}
\State $u_t$ observes the \textbf{sensitive} reward $r_t \defn \underline{\textbf{d}}_{t, a_t}$
\State $u_t$ sends the \textbf{sensitive} reward $z_t \defn r_t$ to $\pi$
\Else{~Local DP}
\State $u_t$ observes the \textbf{sensitive} reward $r_t \defn \underline{\textbf{d}}_{t, a_t}$
\State $u_t$ sends the \textbf{noisy} reward $z_t \sim \mathcal{M}(r_t)$ to $\pi$
\EndIf
\EndIf
\EndFor
\end{algorithmic}
\end{algorithm}

\subsubsection{Global DP BAI}
Again, we represent each user $u_t$ by the vector $\textbf{x}_t \defn (x_{t, 1}, \dots, x_{t, \arms}) \in \real^\arms$, where $x_{t, a}$ represents the \textbf{potential} reward observed, if action $a$ was recommended to user $u_t$. Due to the bandit feedback, only $r_t = x_{t, a_t} \sim \nu_{a_t}$ is observed at step $t$. We use an underline to denote any sequence. Thus, we denote the sequence of sampled actions until $\horizon$ as $\underline{a}^\horizon = (a_1, \dots, a_\horizon)$. We further represent a set of users $\lbrace u_t\rbrace_{t=1}^{\horizon}$ until $\horizon$ by \textbf{the table of potential rewards} $ \underline{\textbf{d}}^\textbf{\horizon} \defn \lbrace \textbf{x}_1, \dots, \textbf{x}_\horizon\rbrace \in (\real^\arms)^\horizon$.
First, we observe that $\underline{\textbf{d}}^\horizon$ is the sensitive input data set to be made private, and $(\underline{a}^\horizon, \widehat{a}, \horizon)$ is the output of the BAI strategy. Hence, we define the probability that the BAI strategy $\pi$ samples the action sequence $\underline{a}^\horizon$, recommends the action $\widehat{a}$, and halts at time $\horizon$, as
\begin{equation}\label{eq:prob_event}
    \pi(\underline{a}^\horizon, \widehat{a}, \horizon \mid \underline{\textbf{d}}^\horizon) \defn \operatorname{Rec}_{T+1}\left(\widehat{a} \mid \mathcal{H}_{\horizon}\right) \mathrm{S}_{T+1}\left(\top \mid \mathcal{H}_{\horizon} \right) \prod_{t\in [T]} \mathrm{~S}_t\left(a_t \mid \mathcal{H}_{t - 1}\right) \: ,
\end{equation}
where $\horizon$ users under interaction are represented by the table of potential rewards $\underline{\textbf{d}}^\horizon$.
A BAI strategy satisfies $\epsilon$-global DP if the probability in Eq.~\eqref{eq:prob_event} is similar when the BAI strategy interacts with two neighbouring tables of rewards differing by one user (i.e. a row in $\underline{\textbf{d}}^\horizon$).
Definition~\ref{def:global_dp} can be seen as a BAI counterpart of the $\epsilon$-global DP definition proposed in~\citet{azize2022privacy} for regret minimisation. 

\begin{definition}[$\epsilon$-global DP for BAI]\label{def:global_dp}
A BAI strategy satisfies \textbf{$\epsilon$-global DP}, if for all $\horizon \geq 1$, all neighbouring table of rewards $\underline{\textbf{d}}^\horizon$ and $\underline{\textbf{d}'}^\horizon$, i.e. $\dham(\underline{\textbf{d}}^\horizon, \underline{\textbf{d}'}^\horizon) = 1$, all sequences of sampled actions $\underline{a}^\horizon \in [\arms]^\horizon$ and recommended actions $\widehat{a} \in [\arms]$ we have that
\begin{equation*}
    \pi(\underline{a}^\horizon, \widehat{a}, \horizon \mid \underline{\textbf{d}}^\horizon)  \leq e^\epsilon \pi(\underline{a}^\horizon, \widehat{a}, \horizon \mid \underline{\textbf{d}'}^\horizon) \: .
\end{equation*}
\end{definition}

The goal in $\epsilon$-global DP FC-BAI is to design a $\delta$-correct $\epsilon$-global DP BAI strategy $\pi$, with $\bE[\tau_{\delta}]$ as small as possible. 

\begin{remark}
    It is possible to consider that the output of a BAI strategy is \textit{only} the final recommended action $\hat{a}$, i.e. not publishing the intermediate actions $\underline{a}^\horizon$. This gives a weaker definition of privacy compared to Definition~\ref{def:global_dp}, since the latter defends against adversaries that may look inside the execution of the BAI strategy, i.e. pan-privacy~\citep{dwork2010pan}. Also, Definition~\ref{def:global_dp} is needed in practice. For example, in the case of dose-finding (Example~\ref{ex:1}), the experimental protocol, i.e. the intermediate actions, needs to be published too.
\end{remark}

\section{Local Differentially Private Best-Arm Identification}
\label{sec:local_DP}

In this section, we answer the following question: \textit{How many additional samples a BAI strategy must select to ensure $\epsilon$-local DP?}
We provide a lower bound on the expected sample complexity of any $\delta$-correct $\epsilon$-local DP pair of perturbation mechanism and BAI strategy. We complement the sample complexity lower bound with a matching upper bound.

\subsection{Lower Bound on the Expected Sample Complexity}

We derive a lower bound on the expected sample complexity in $\epsilon$-local DP FC-BAI, which features problem-dependent characteristic times as in the FC-BAI setting.
\begin{theorem}\label{thm:local_lower_bound} Let $\delta \in (0,1) $ and $\epsilon>0$. 
For any $\delta$-correct $\epsilon$-local DP pair $(\mathcal{M}, \pi)$ of perturbation mechanism and BAI strategy, \Modif{for all instance $\boldsymbol{\nu}$ with unique best arm}, we have $\mathbb{E}_{\boldsymbol{\nu}}[\tau_{\delta}] \geq T^{\star}_{\ell}\left(\boldsymbol{\nu} ; \epsilon\right) \log (1 / (2.4 \delta))$ with
\begin{align*}
    T^{\star}_{\ell}\left(\boldsymbol{\nu} ; \epsilon\right)^{-1} \defn \sup\limits_{\omega \in \Sigma_K} \inf\limits_{\boldsymbol{\lambda} \in \operatorname{Alt}(\boldsymbol{\nu})}&\sum_{a \in [K]} \omega_a  \min  \left\{\KL{\nu_a}{\lambda_a}, c(\epsilon) \left(\TV{\nu_a}{\lambda_a} \right)^2 \right\} \: ,
\end{align*}
where $c(\epsilon) \defn \min\{4, e^{2 \epsilon} \} \left(e^\epsilon - 1 \right)^2$ is a privacy term.
For two probability distributions $\mathbb{P}, \mathbb{Q}$ on the measurable space $(\Omega, \mathcal{F})$, the TV distance is $ \TV{\mathbb{P}}{\mathbb{Q}} \triangleq \sup_{A \in \mathcal{F}} \{\mathbb{P}(A) - \mathbb{Q}(A)\}$.
\end{theorem}

\noindent \textit{Proof sketch.} To prove this theorem, \marc{we first define  the ``noisy'' environment $\boldsymbol{\nu}^\mathcal{M} \triangleq \{\nu_a^\mathcal{M} : a \in [\arms]\} $ \textit{induced} by the perturbation mechanism $\mathcal{M}$, where 
\begin{align*}
    \nu_a^\mathcal{M}(Z) = \int_{r \in \real} \mathcal{M}(Z \mid r) \dd \nu_a(r) dr
\end{align*}
is the marginal over the noisy rewards of arm $a$.
Then, we use the KL-decomposition from Lemma 1 of~\citet{garivier2016optimal} applied
to the ``noisy'' environment to get}
\begin{align*}
         \marc{\sum_{a = 1}^\arms \expect \left [N_{\tau_{\delta},a} \right] \KL{\nu^\mathcal{M}_{a}}{\lambda^\mathcal{M}_{a}} \geq \mathrm{kl}(1-\delta, \delta) \: ,}
\end{align*}
\marc{where $\mathrm{kl}\Modif{(x,y)} \defn x \log\frac{x}{y} + (1 - x) \log \frac{1 - x}{1 - y}$ for $x,y \in (0,1)$.}
Then, Theorem 1 of~\citet{duchi2013local} is applied to relate the KL of rewards in the \marc{``}noisy'' bandit environment to the original environment to get
\begin{align*}
        \marc{\KL{\nu^\mathcal{M}_{a}}{\lambda^\mathcal{M}_{a}}
        \leq c(\epsilon) (\TV{\nu_a}{\lambda_a})^2 \: .}
    \end{align*}
\marc{This bound shows that the perturbation mechanism $\mathcal{M}$ acts as a contraction on the space of probability measures. The rest of the proof is recovered by observing that $ \expect[\tau_{\delta}] = \sum_{a = 1}^\arms \expect \left [N_{\tau_{\delta},a}  \right]  $ and taking the infinimum over all alterative environments.} 
In Appendix~\ref{app:loc_lb_proof}, we formally define the bandit canonical model under local DP, and provide a complete proof of the theorem.$\qed$

Similar to the lower bound for the non-private BAI \marc{(Lemma~\ref{lem:garivier2016optimal}, see~\citealt{garivier2016optimal})}, the lower bound of Theorem~\ref{thm:local_lower_bound} is the value of a two-player zero-sum game between a MIN player and MAX player. 
For $\epsilon$-local DP, the measure of ``distinguishability'' between instances is captured by $\min\{\mathrm{KL}, c(\epsilon)  \mathrm{TV}^2\}$. 
It interpolates between the $\mathrm{KL}$ stemming from the $\delta$-correctness constraint (i.e., ``distinguishability'' measure in the non-private lower bound) and the squared $\mathrm{TV}$ coming from the $\epsilon$-local DP constraint, scaled by $c(\epsilon)$.

\marc{Corollary~\ref{lem:local_more_explicit_lower_bound} gives a lower bound on $T^{\star}_{\ell}$ as the maximum between the non-private characteristic time $T^{\star}_{\mathrm{KL}}$ and a privacy-rescaled characteristic time $T^{\star}_{\mathrm{TV}^2}$. }
\begin{corollary}[Relaxing the local DP lower bound] \label{lem:local_more_explicit_lower_bound}
Let $T^{\star}_{\ell}\left(\boldsymbol{\nu} ; \epsilon\right)$ as in Theorem~\ref{thm:local_lower_bound} and $T^{\star}_{\textbf{d}}(\boldsymbol{\nu})$ as in Eq.~\eqref{eq:characteristic_time}.
Then, we have 
\begin{equation} \label{eq:localDP_lower_bound}
 T^{\star}_{\ell}\left(\boldsymbol{\nu} ; \epsilon\right) \geq \max\left\{ T^{\star}_{\mathrm{KL}}(\boldsymbol{\nu}) , c(\epsilon)^{-1} T^{\star}_{\mathrm{TV}^2}\left(\boldsymbol{\nu}\right)  \right\} \quad \text{with} \quad c(\epsilon)  \defn \min\{4, e^{2 \epsilon} \} \left(e^\epsilon - 1 \right)^2\: .
\end{equation}
Let $\boldsymbol{\nu}_{\mathrm{G}}$ be the Gaussian instance with unit variances and the same means as the Bernoulli instance $\boldsymbol{\nu}$.
Then, we have $T^{\star}_{\mathrm{TV}^2}\left(\boldsymbol{\nu}\right)\ge 2 T^{\star}_{\mathrm{KL}}(\boldsymbol{\nu})$ and $T^{\star}_{\mathrm{TV}^2}\left(\boldsymbol{\nu}\right) =  T^{\star}_{\mathrm{KL}}(\boldsymbol{\nu}_{\mathrm{G}})/2$.
\end{corollary}
\begin{proof}
	The first part is true since $T^{\star}_{\ell}\left(\boldsymbol{\nu} ; \epsilon\right)\geq T^{\star}_{\mathrm{KL}}(\boldsymbol{\nu})$ and $T^{\star}_{\ell}\left(\boldsymbol{\nu} ; \epsilon\right) \geq \frac{ T^{\star}_{\mathrm{TV}^2}\left(\boldsymbol{\nu}\right)}{\min\{4, e^{2 \epsilon} \} \left(e^\epsilon - 1 \right)^2}$.
    The second part uses that $\KL{\mathcal{N}(p, 1)}{\mathcal{N}(q, 1)} = \frac{1}{2} |p - q|^2 = \frac{1}{2} \TV{\text{Ber}(p)}{\text{Ber}(q)}^2$.
\end{proof}

Corollary~\ref{lem:local_more_explicit_lower_bound} relates the $\mathrm{TV}^2$ characteristic time for Bernoulli to the $\mathrm{KL}$ characteristic times for Bernoulli and Gaussian.
The sample complexity of FC-BAI with local DP on Bernoulli instances is reduced to the characteristic time of the non-private FC-BAI on Gaussian instances, up to a multiplicative factor which only depends on $\epsilon$.

\textit{Two privacy regimes.} The sample complexity lower bound in Eq.~\eqref{eq:localDP_lower_bound} suggests the existence of two hardness regimes depending on $\epsilon$, $T^{\star}_{\mathrm{KL}}(\boldsymbol{\nu})$ and $T^{\star}_{\mathrm{TV}^2}(\boldsymbol{\nu})$. In the high privacy regime, as $\epsilon \to 0$, the lower bound reduces to $\epsilon^{-2} T^{\star}_{\mathrm{TV}^2}\left(\boldsymbol{\nu}\right)$. In the low privacy regime, as $\epsilon \to \infty$, the lower bound reduces to the non-private complexity $T^{\star}_{\mathrm{KL}}(\boldsymbol{\nu})$. The switch between the low and the high privacy regimes happens at the $\epsilon$ verifying ${\min\{4, e^{2 \epsilon} \} \left(e^\epsilon - 1 \right)^2} = \frac{ T^{\star}_{\mathrm{TV}^2}\left(\boldsymbol{\nu}\right)}{T^{\star}_{\mathrm{KL}}(\boldsymbol{\nu})}$. For example, for environments where the Pinsker inequality is tight, i.e. $T^{\star}_{\mathrm{TV}^2}\left(\boldsymbol{\nu}\right) \approx 2 T^{\star}_{\mathrm{KL}}\left(\boldsymbol{\nu}\right)$, then the switch happens at $\epsilon \approx 0.582$.

\subsection{A Plug-In Approach: the \locptt{} Algorithm}
\citet{ren2020multi} proposed the Convert-To-Bernoulli (CTB, Algorithm~\ref{algo:CTB}) estimator of the means, which relies on the Randomised Response mechanism to ensure $\epsilon$-local DP on $[0, 1]$.

\begin{algorithm}[t!]
         \caption{Convert-To-Bernoulli$(\epsilon)$ Estimator (\protect\hypertarget{CTB}{CTB})~\citep{ren2020multi}}
         \label{algo:CTB}
	\begin{algorithmic}[1]
 	      \State {\bfseries Input:} History $\mathcal H_{n}$ with past perturbations $(\tilde r_{t})_{t \in [n-2]}$, arm $a \in [K]$.
        \State Observe $\tilde r_{n-1} \sim \text{Ber}\left( \frac{r_{n-1}(e^{\epsilon} - 1) + 1}{e^{\epsilon} + 1} \right)$; \Comment{Randomised Response}
		\State \textbf{Return} $(\tilde \mu_{n,a}, N_{n,a})$ with $\tilde \mu_{n,a} = \frac{1}{N_{n,a}}\sum_{t=1}^{n-1} \tilde r_{t} \ind{a_t = a}$; 
	\end{algorithmic}
\end{algorithm}
\begin{lemma}[\Modif{Lemma 5 in} \citealt{ren2020multi}] \label{lem:lem5_ren2020multi}
    \hyperlink{CTB}{CTB}$(\epsilon)$ ensures $\epsilon$-local DP on $[0, 1]$\Modif{, and the returned value follows the Bernoulli distribution with mean $\mu_{\epsilon,a} \defn (2 \mu_{a} - 1) \frac{e^{\epsilon}-1}{2(e^{\epsilon}+1)} + 1/2$.}
\end{lemma}

When the reward $r$ is generated using a Bernoulli of parameter $\mu_a$, and $r'$ is the result of the Randomised Response mechanism applied to $r$, \ie $r' \sim \text{Ber}\left( \frac{r (e^{\epsilon} - 1) + 1}{e^{\epsilon} + 1} \right)$, then the marginal distribution of $r'$ is Bernoulli of parameter $\mu_{\epsilon,a}$ defined in Lemma~\ref{lem:lem5_ren2020multi}.

\textit{\locptt{} algorithm.}
To solve $\epsilon$-local DP FC-BAI, we propose the \locptt{} algorithm\marc{, whose standalone pseudocode is detailed in Algorithm~\ref{algo:CTB_TT_full} (Appendix~\ref{app:algo_pseudocode}).}
\locptt{} is an instance of Algorithm~\ref{algo:TTUCB} using the \hyperlink{CTB}{CTB}$(\epsilon)$ estimator (Algorithm~\ref{algo:CTB}), $(W^{G}_{a,b}, b^{G}_{a})$ as in Eq.~\eqref{eq:TTUCB_Gaussian} with $\sigma = 1/2$ and $c^{G}_{a,b}$ as in Eq.~\eqref{eq:threshold_non_private}.

Using Lemma~\ref{lem:lem5_ren2020multi}, the \locptt{} algorithm is $\epsilon$-local DP and is equivalent to running the non-private TTUCB algorithm on a modified bandit instance $\boldsymbol{\nu_{\epsilon}}$, where $\nu_{\epsilon,a}  \defn  \text{Ber}(\mu_{\epsilon,a})$ with $\mu_{\epsilon,a}  \defn  (2 \mu_{a} - 1) \frac{e^{\epsilon}-1}{2(e^{\epsilon}+1)} + 1/2$ for all $a \in [K]$. 
While the analysis in~\citet{jourdan2022non} is written for Gaussian distributions with unit variance, their Section 3.2 shows that the same results can be obtained for $\sigma$-sub-Gaussian distributions.
As such, the theoretical guarantees obtained in~\citet{jourdan2022non} apply to our algorithm.
In particular, \locptt{} is $\delta$-correct and satisfies that, for all $\boldsymbol{\nu} \in \mathcal M$ such that $\min_{a \ne b}|\mu_{a} - \mu_{b}| > 0 $, 
\begin{equation*} \limsup_{\delta \to 0} \frac{\mathbb E_{\boldsymbol{\nu}}[\tau_{\delta}]}{\log (1/\delta)} \le T^{\star}_{\mathrm{KL},\beta}(\boldsymbol{\nu}_\epsilon) = \left(1 + \frac{2}{e^\epsilon -1}\right)^2 T^{\star}_{\mathrm{KL},\beta}(\boldsymbol{\nu})  \: ,  \end{equation*}
where $T^{\star}_{\mathrm{KL},\beta}$ as in Eq.~\eqref{eq:T_KL_Gaussian} for $\sigma = 1/2$.
For $\beta=1/2$, combining Lemma~\ref{lem:local_more_explicit_lower_bound} and~\eqref{eq:T_KL_Gaussian} yields $\limsup_{\delta \to 0} \mathbb E_{\boldsymbol{\nu}}[\tau_{\delta}]/\log (1/\delta) \le \left(1 + 2/(e^\epsilon -1)\right)^2 T^{\star}_{\mathrm{TV^2}}(\boldsymbol{\nu}) $.
On top of its asymptotic guarantees, \locptt{} enjoys guarantees on its expected sample complexity at any confidence level (non-asymptotic regime).
\Modif{Taking $s=\alpha=1.2$ and $\beta=1/2$ as algorithmic parameters yields}, for all $\delta \in (0,1)$ and all $\boldsymbol{\nu} \in \mathcal M$ such that $|a^\star(\boldsymbol{\nu})|=1$,
\[
	\mathbb E_{\boldsymbol{\nu}}[\tau_{\delta}] = \mathcal O \left(  \left( H(\boldsymbol{\nu}_\epsilon) \log H (\boldsymbol{\nu}_\epsilon) \right)^{1.2}\right) \quad \text{with} \quad H (\boldsymbol{\nu}_\epsilon) = \left(1 + 2/(e^\epsilon -1)\right)^2 H (\boldsymbol{\nu}) \: .
\]
The notation $\mathcal O$ gives the dominating term when $H(\boldsymbol{\nu}) \to + \infty$.

In the non-private regime where $\epsilon \to +\infty$, our upper bound recovers the result of~\citet{jourdan2022non}.
It matches the non-private lower bound for Gaussian distributions $T_{\mathrm{KL}}^\star(\boldsymbol{\nu})$ up to a multiplicative factor $2$.
Our upper bound matches the lower bound of Theorem~\ref{thm:local_lower_bound} up to a multiplicative factor of $(e^{\epsilon}+1)^2\min\{4,e^{2\epsilon}\}$, whose limit is $4$ when $\epsilon \to 0$.
Instead of a fixed design $\beta$, we could use the optimal design IDS~\citep{you2023information} which sets $\beta_n$ adaptively, i.e. $\beta_n = \frac{N_{n,C_n}}{N_{n,C_n} + N_{n,B_n}}$ for Gaussian distributions.
Since this modification yields $T^{\star}_{\mathrm{KL}}(\boldsymbol{\nu}_\epsilon)$ as an asymptotic upper bound, it shaves a multiplicative factor $2$.
In the limit of $\epsilon \to 0$, it leaves a multiplicative gap of $2$ between the lower and the upper bound.
Closing this gap is an interesting direction for future research.

\begin{remark}[Extension to $(\epsilon, \gamma)$-Local DP] For approximate\footnote{We use $(\epsilon, \gamma)$-DP notation for approximate DP, since $\delta$ is used throughout the paper for correctness.} $(\epsilon, \gamma)$-local DP,~\citet[Algorithm 1]{zheng2020locally} uses the Gaussian mechanism~\citep{dwork2014algorithmic} to send noisy versions of the reward to the policy. Specifically, if the rewards are in $[0,1]$, at each step $t$, the noisy reward is $\tilde{r}_t = r_t + \mathcal{N}(0, \sigma_{\epsilon, \gamma}^2)$, where $\sigma_{\epsilon, \gamma} \defn \frac{\sqrt{2 \log(1.25/\gamma)}}{\epsilon}$. Then, the noisy rewards are fed to some non-private algorithm $\mathcal{A}$.
Similar to our analysis of CTB-TT, ~\citet[Theorem 12]{zheng2020locally} shows that the sample complexity of the local DP algorithm (Algorithm 1) is equivalent to the sample complexity of the non-private algorithm $\mathcal{A}$ run on an instance of variance $\sigma_{\epsilon, \gamma}^2 + \frac{1}{4}$.
When combined with TTUCB, their Algorithm 1 inherits from the known guarantees of TTUCB on these modified Gaussian instances.
\end{remark}

\section{Global Differentially Private Best-Arm Identification}
\label{sec:global_DP}

The central question that we address in this section is: \textit{How many additional samples a BAI strategy must select for ensuring $\epsilon$-global DP?}
In response, we prove a lower bound on the expected sample complexity of any $\delta$-correct $\epsilon$-global DP BAI strategy (Section~\ref{ssec:global_DP_lower_bound}). 
\marc{To design an $\epsilon$-global DP BAI algorithm}, we first propose a private mean estimator (Section~\ref{ssec:global_DP_private_mean}) \marc{based on arm-dependent doubling and forgetting. Then, in Sections~\ref{ssec:plug_in_based_TTUCB_global} and~\ref{ssec:lower_bound_based_TTUCB_global}, we plug this estimator in TTUCB to get \adaptt{} and \adapttt{}. These two algorithms differ in how they account for the noise addition used in the mean estimation part. Specifically, \adaptt{} accounts for the noise addition by adapting the stopping threshold and UCB index of the leader. In addition to this, \adapttt{} also changes the transport used in TTUCB, and bases it on the lower bound of Section~\ref{ssec:global_DP_lower_bound}. This provides a tighter stopping rule and different challenger choosing rule that depends on the privacy regime.}

\subsection{Lower Bound on the Expected Sample Complexity}
\label{ssec:global_DP_lower_bound}

\marc{To prove BAI lower bounds with privacy, it is important to translate the privacy constraint to an upper bound on the KL between the marginals over the outputs, when the inputs are stochastically generated. In the following, we use coupling techniques to generate these upper bounds on the KL between marginals. We first explore the batch setting, where the data-generating distributions are product distributions. Then, we adapt the same techniques to the sequential setting of BAI.}

\textit{Batch Setting with Product Distributions.} Let $\mech$ be a mechanism that takes as input 
data set $D \in \mathcal{X}^n$, and outputs $o \in \mathcal{O}$. Let $\cP_1$ and $\cP_2$ be two data-generating distributions over $\CX^n$. We define the marginals $M_1$ and $M_2$ over the output of the mechanism $\mech$ as
\begin{equation*}\label{eq:marginal}
 M_\nu(A) \deffn \int_{D \in \CX^n } \mech_D\left(A \right) \dd\cP_\nu\left(D\right),
\end{equation*}
when the inputs are generated from $\cP_\nu$ for $\nu \in \{1, 2\}$ and $ A$ an event in the output space. The goal in this section is to provide an upper bound on the quantity $\KLL{M_1}{M_2}$ when the mechanism $\mech$ satisfies $\epsilon$-DP.
Theorem~\ref{thm:kl_bound} uses coupling and optimal transport to provide an upper bound on the quantity $\KLL{M_1}{M_2}$.

\begin{theorem}[KL Upper Bound as a Transport Problem]\label{thm:kl_bound} If $\mech$ is $\epsilon$-pure DP, then 
\begin{equation*}
 \KLL{M_1}{M_2} \leq \epsilon \inf_{\cC \in \Pi(\cP_1, \cP_2)} \expect_{(D,D') \sim \cC} [\dham(D, D')] \: ,
\end{equation*}
where $\Pi(\cP_1, \cP_2)$ is the set of all couplings between $\cP_1$ and $\cP_2$.
\end{theorem}

\noindent \textit{Proof sketch.} To prove this, the main idea is to think about $M_1$ as the marginal over outputs when a pair of data sets $(D, D')$ is generated through the coupling $\mathcal{C}$ and the channel is $\mathcal{M}( \mid D)$, i.e. applying the mechanism only to the first data set. On the other hand,  $M_2$ is the marginal over outputs when $(D, D') \sim \mathcal{C}$ but the channel is $\mathcal{M}( \mid D')$. Combining the fact that the KL between marginals is smaller than the expected KL between the channels, and that the KL of between the channels is controlled by group privacy, the proof is concluded. The complete proof is presented in Appendix~\ref{app:kl_trans}

Deriving the sharpest upper bound for the KL requires solving the transport problem
\begin{equation}\label{eq:trs_eps}
    \marc{\inf_{\cC \in \Pi(\cP_1, \cP_2)} \expect_{(D,D') \sim \cC} [\dham(D, D')] \: .}
\end{equation}
As a proxy, we use maximal couplings.

\noindent \textit{Product Distributions.} Now, suppose that $\cP_1$ and $\cP_2$ are two product distributions over $\cX^n$, \ie $\cP_1 = \bigotimes_{i = 1}^n p_{1,i}$ and $\cP_2 = \bigotimes_{i = 1}^n p_{2,i}$, where $p_{\nu,i}$ for $\nu \in \{1, 2\}$ and $i \in [1, n]$ are distributions over $\cX$. Let $c_\infty^i$ be a maximal coupling between $p_{1,i}$ and $p_{2,i}$ for all $i \in [1,n]$. We define the coupling $\cC_\infty \deffn \bigotimes_{i = 1}^n c_\infty^i$. Then $\cC_\infty$ is a coupling of $\cP_1$ and $\cP_2$. Using the $\cC_\infty$ coupling between the product distributions $\cP_1$ and $\cP_2$ as a proxy to solve the transport problem of Equation~\eqref{eq:trs_eps}, we show Corollary~\ref{thm:kl_deco_prod}.

\begin{corollary}[KL Decomposition for Product Distributions]\label{thm:kl_deco_prod} 
If $\mech$ is $\epsilon$-pure DP, $\cP_1 = \bigotimes_{i = 1}^n p_{1,i}$ and $\cP_2 = \bigotimes_{i = 1}^n p_{2,i}$ are product distributions, then 
\[
    \KLL{M_1}{M_2} \leq \epsilon  \sum_{i =1 }^n t_i \: ,
\] 
with $t_i \deffn \TV{p_{1,i}}{p_{2,i}}$.
\end{corollary}
\begin{proof}
 Since $\dham(D, D') = \sum_{i=1}^n \ind{d_{i} \neq d'_{i} }$, we have $\dham(D, D') \sim \sum_{i = 1}^n \text{Bernoulli}( t_i )$ for $(D, D') \sim \cC_\infty \deffn \bigotimes_{i = 1}^n c_\infty^i$, where $t_i \deffn \TV{p_{1,i}}{p_{2,i}}$, and the terms in the sum are mutually independent.
 This further yields that $\expect_{(D,D') \sim \cC_\infty} [\dham(D, D')] =  \sum_{i =1 }^n t_i$.
\end{proof}

Corollary~\ref{thm:kl_deco_prod} can be seen as a stochastic generalisation of the group privacy property of DP. Specifically, the results from Theorem~\ref{thm:kl_deco_prod} suggest that two random data sets $D$ and $D'$ sampled from $\cP_1 = \bigotimes_{i = 1}^n p_{1,i}$ and $\cP_2 = \bigotimes_{i = 1}^n p_{2,i}$ respectively could be thought of as $(\sum_{i= 1}^n t_i)$-neighboring data sets ``in expectation'', where $t_i = \TV{p_{1,i}}{p_{2,i}}$.

\noindent \textit{Relation to similar results in the literature.} Lemma 6.1 in~\cite{KarwaVadhan} shows that, for any event E, $M_1(E) \leq e^{6 \epsilon n \TV{p_1}{p_2}} M_2(E)$, when the mechanism is $\epsilon$-pure DP, and the data-generating distributions are i.i.d from $p_1$ or $p_2$, \ie $\cP_\nu = \bigotimes_{i = 1}^n p_{\nu}$ for $\nu \in \{ 1, 2\}$. The Karwa Vadhan is a stronger result than Theorem~\ref{thm:kl_deco_prod} since it controls the multiplicative difference between the marginals at each event. This gives the following direct KL upper bound $\KLL{M_1}{M_2} \leq 6 \epsilon n \TV{p_1}{p_2}$ for i.i.d distributions. Also, the Karwa Vadhan lemma builds explicitly the maximal coupling in their proof. Our result generalises this upper bound to product distributions and improves the dependence of factor $6$ there. Also, similar coupling ideas have been developed in~\cite{lalanne2022statistical} to derive DP and zCDP variants of LeCam and Fano inequalities, and in~\cite{azizeconcentrated} to derive zCDP regret lower bounds for bandits.

\noindent \textit{BAI setting.} Adapting similar coupling ideas from the batch setting, we derive an $\epsilon$-global DP version of the ``change-of-measure'' lemma.
\begin{lemma}[Change-of-measure lemma under $\epsilon$-global DP]\label{lem:chg_env_dp}
    Let $\delta \in (0,1) $ and $\epsilon>0$. Let $\boldsymbol{\nu}$ be a bandit instance and $\lambda \in \operatorname{Alt}(\boldsymbol{\nu})$. For any $\delta$-correct $\epsilon$-global DP BAI strategy, 
    \begin{equation*}
        \epsilon  \sum_{a \in [K]} \expect_{\boldsymbol{\nu}}  [N_{\tau_{\delta},a} ] \TV{\nu_{a}}{\lambda_{a}}  \ge        \mathrm{kl}(1-\delta, \delta) \: .
    \end{equation*}
\end{lemma}

\noindent \textit{Proof Sketch.} The main technical challenge in this proof is to extend Corollary~\ref{thm:kl_deco_prod} to bandit distributions using the idea of ``coupled environments''. Then, using a classic data-processing inequality concludes the proof. The complete proof is presented in Appendix~\ref{app:chg_env_proof}.

Combining the non-private and $\epsilon$-global DP change of measure lemmas gives the lower bound on the sample complexity of any $\delta$-correct $\epsilon$-global DP BAI algorithm.

\begin{theorem}\label{thm:global_lower_bound}
Let $(\epsilon,\delta) \in \R_{+}^\star \times (0,1) $. 
For any $\delta$-correct and $\epsilon$-global DP FC-BAI algorithm, \Modif{for all instances $\bm{\nu} \in \mathcal M$ with unique best arm}, we have  $\mathbb{E}_{\boldsymbol{\nu}}[\tau_{\delta}] \geq T^{\star}_g\left(\boldsymbol{\nu} ; \epsilon\right) \log (1/(2.4 \delta))$ with
\[
T^{\star}_g\left(\boldsymbol{\nu} ; \epsilon\right)^{-1 } \defn \sup\limits_{\omega \in \Sigma_K} \inf\limits_{\boldsymbol{\lambda} \in \operatorname{Alt}(\boldsymbol{\nu})} \min  \bigg\{\sum_{a \in [K]} \omega_a \KL{\nu_a}{\lambda_a}, \epsilon \sum_{a \in [K]} \omega_a  \TV{\nu_a}{\lambda_a} \bigg\} \: .
\]
\end{theorem}

As for the non-private BAI \marc{(Lemma~\ref{lem:garivier2016optimal}, see~\citealt{garivier2016optimal})}, Theorem~\ref{thm:global_lower_bound} is the value of a two-player zero-sum game between a MIN player and MAX player.
On top of the KL divergence present in the non-private lower bound, our bound features the TV distance that appears naturally when incorporating the $\epsilon$-global DP constraint. 
\marc{For $\epsilon$-global DP, our proposed measure of ``distinguishability'' between instances is captured by a minimum between the instance-wise non-private measure $\sum_{a \in [K]} \omega_a \KL{\nu_a}{\lambda_a}$ and an instance-wise private measure $\sum_{a \in [K]} \omega_a  \TV{\nu_a}{\lambda_a} $, scaled by $\epsilon$. 
While it trades off between the $\delta$-correctness constraint and the $\epsilon$-global DP constraint, this interpolation occurs at the level of the instance instead of being at the level of individual arms, as it does in Theorem~\ref{thm:local_lower_bound} for $\epsilon$-local DP.
Deriving a measure of ``distinguishability'' at the arm's level is an interesting open-problem that we leave for furture work.}

\marc{Corollary~\ref{lem:more_explicit_lower_bound} gives a lower bound on $T^{\star}_{g}$ as the maximum between the non-private characteristic time $T^{\star}_{\mathrm{KL}}$ and a privacy-rescaled characteristic time $T^{\star}_{\mathrm{TV}}$.
Compared to Corollary~\ref{lem:local_more_explicit_lower_bound}, the ``distinguishability'' measure for global DP is $\mathrm{TV}^2$ instead of $\mathrm{TV}$ for local DP.
The TV characteristic time $T^{\star}_{\text{TV}}(\boldsymbol{\nu})$ serves as the BAI counterpart to the TV-distinguishability gap ($t_{\inf}$) in the problem-dependent regret lower bound for bandits with $\epsilon$-global DP as in \citet[Theorem 3]{azize2022privacy}. }
\begin{corollary}[Relaxing the global DP lower bound] \label{lem:more_explicit_lower_bound}
Let $T^{\star}_g(\boldsymbol{\nu} ; \epsilon )$ \marc{be the $\epsilon$-global DP characteristic time} as in Theorem~\ref{thm:global_lower_bound} and $T^{\star}_{\textbf{d}}(\boldsymbol{\nu})$ as in Eq.~\eqref{eq:characteristic_time} \marc{be the characteristic time for the ``distinguishability'' measure $\textbf{d}$ between probability distributions, e.g., $\mathrm{KL}$ or $\mathrm{TV}$}.
Then, 
	\begin{equation} \label{eq:global_explicit_lower_bound}
		T^{\star}_g(\boldsymbol{\nu} ; \epsilon ) \ge \max \{ T^{\star}_{\mathrm{KL}}(\boldsymbol{\nu}) ,  T^{\star}_{\mathrm{TV}}(\boldsymbol{\nu})/\epsilon \} \: .
	\end{equation}

\marc{\textbf{Bernoulli instances.}} Let $\boldsymbol{\nu}$ be a Bernoulli instance with \marc{unique best arm $a^\star$,} mean gaps $\Delta_{a} = \mu_{a^\star} - \mu_{a}$ for all $a \ne a^\star$ and $\Delta_{a^\star} = \Delta_{\min} = \min_{a \ne a^\star} ( \mu_{a^\star} - \mu_{a} )$ for all $a \ne a^\star$.
\marc{Let $\boldsymbol{\nu}_{G,\epsilon}$ be a Gaussian instance with variance $\sigma=1/2$ and means $\mu_{\epsilon,a} = \mu_{a}$ for all $a \in \{a^\star\} \cup \{a \ne a^\star \mid \Delta_{a} \le \epsilon/2\}$, and $\mu_{\epsilon,a} = \mu_{a^\star} - \sqrt{\epsilon \Delta_{a}/2}$ otherwise.}
Then, we have 
	\begin{align} \label{eq:global_explicit_relaxation}
        &T^{\star}_{\mathrm{TV}}(\boldsymbol{\nu}) = \sum_{a \in [K]} \Delta_a^{-1} \quad \text{and} \quad T^{\star}_g (\boldsymbol{\nu} ; \epsilon ) \le \marc{T^\star_{\mathrm{KL}}(\bm{\nu}_{G,\epsilon})}  \le \marc{2 H(\bm{\nu}_{G,\epsilon})} \: ,\quad \text{where we recall that} \nonumber \\
	&\marc{T^\star_{\mathrm{KL}}(\bm{\nu}_{G,\epsilon})}^{-1} \defn \max\limits_{\omega \in \Sigma_K} \min_{a \ne a^\star}  \frac{2\Delta_{a} \min\{\epsilon/2, \Delta_{a}\}}{1/\omega_{a^\star}+1/\omega_{a}}  \: \text{ , } \: \marc{2 H(\bm{\nu}_{G,\epsilon})} = \sum_{a \in [K]} (\Delta_a \min\{\Delta_a, \epsilon/2\})^{-1}  .
	\end{align}
\end{corollary}

\begin{remark}
    In Section~\ref{ssec:global_DP_lower_bound}, only the results of Paragraph ``\textbf{Bernoulli instances}" are specific to Bernoulli distributions. All the other results of Section~\ref{ssec:global_DP_lower_bound} are true for any class of distributions, e.g. Gaussians, sub-Gaussians, Exponentials, etc.
\end{remark}

\begin{proof}
	The first part is a direct consequence of the definition of $T^{\star}_g (\boldsymbol{\nu}; \epsilon )$.
    The second part uses $\TV{\text{Ber}(p)}{\text{Ber}(q)} = |p - q|$ to solve the optimisation problem and is detailed in Appendix~\ref{app:ssec_proof_prop_tv}. 
    The last part is obtained by using Pinsker's inequality.
\end{proof}

\marc{Corollary~\ref{lem:more_explicit_lower_bound} also upper bounds $T^{\star}_{g}$ by the non-private characteristic time $T^\star_{\mathrm{KL}}$ on a Gaussian instance $\bm{\nu}_{G,\epsilon}$ with privacy-aware means.
In this modified instance, the non-private mean gaps are clipped when the privacy budget is small enough.}

\textit{Two privacy regimes.} 
The sample complexity lower bound in Eq.~\eqref{eq:global_explicit_lower_bound} suggests the existence of \textit{two hardness regimes depending on $\epsilon$, $T^{\star}_{\mathrm{KL}}(\boldsymbol{\nu})$ and $T^{\star}_{\mathrm{TV}}(\boldsymbol{\nu})$}. 
(1) \textit{Low-privacy regime}: When $\epsilon > T^{\star}_{\mathrm{TV}}(\boldsymbol{\nu})/( T^{\star}_{\mathrm{KL}}(\boldsymbol{\nu}))$, the lower bound retrieves the non-private lower bound, i.e. $T^{\star}_{\mathrm{KL}}(\boldsymbol{\nu})$, and thus, \textbf{privacy can be achieved for free}. 
(2) \textit{High-privacy regime:} When $\epsilon < T^{\star}_{\mathrm{TV}}(\boldsymbol{\nu})/(T^{\star}_{\mathrm{KL}}(\boldsymbol{\nu}))$, the lower bound becomes $ T^{\star}_{\mathrm{TV}}(\boldsymbol{\nu}) / \epsilon$ and $\epsilon$-global DP $\delta$-BAI requires more samples than non-private ones. 
Using Pinsker's inequality, one can connect the TV and KL characteristic times by $T^\star_{\text{TV}}(\boldsymbol{\nu}) \ge \sqrt{2T^\star_{\mathrm{KL}}(\boldsymbol{\nu})}$.

The global trade-off between low and high privacy regimes at the instance level given by~\eqref{eq:global_explicit_lower_bound} does not give any information at the level of a specific arm.
For each sub-optimal arm, the transition from low to high privacy is better understood by considering~\eqref{eq:global_explicit_relaxation}, even though it only upper bounds $T^{\star}_{g} (\boldsymbol{\nu}; \epsilon )$. 
For any arm $a \ne a^\star(\boldsymbol{\nu})$, the high-privacy regime corresponds to a mean gap such that $\epsilon < 2\Delta_{a} $, and the low-privacy regime to $\epsilon > 2\Delta_{a} $.

\subsection{Private Mean Estimator}
\label{ssec:global_DP_private_mean}

To define a sequence of mean estimators, we propose the \hyperlink{DAF}{DAF}$(\epsilon)$ update (Algorithm~\ref{algo:doubling_forgetting}) which relies on three ingredients: \textit{adaptive episodes with doubling}, \textit{forgetting}, and \textit{adding calibrated Laplacian noise}. 
(1) \hyperlink{DAF}{DAF} maintains $K$ episodes, i.e. one per arm. 
The private empirical estimate of the mean of an arm is only updated at the end of an episode, that means when the number of times that a particular arm was played doubles. 
(2) For each arm $a$, \hyperlink{DAF}{DAF} forgets rewards from previous phases of arm $a$, i.e. the private empirical estimate of arm $a$ is only computed using the rewards collected in the last phase of arm $a$. 
This assures that the means of each arm are estimated using a non-overlapping sequence of rewards. 
(3) Thanks to this \textit{doubling} and \textit{forgetting}, \hyperlink{DAF}{DAF} is $\epsilon$-global DP as soon as each empirical mean is made $\epsilon$-DP, and thus, avoiding any use of privacy composition. 
This is achieved by adding Laplace noise. 
We formalise this intuition in Lemma~\ref{lem:privacy} of Appendix~\ref{app:privacy_proof}.

\begin{lemma} \label{lem:DAF_private}
	Any algorithm relying solely on the \hyperlink{DAF}{DAF}$(\epsilon)$ update is $\epsilon$-global DP on $[0, 1]$.
\end{lemma}
\begin{proof}
	A change in one user \textit{only affects} the empirical mean calculated at one episode of an arm, which is made private using the Laplace Mechanism and Lemma~\ref{lem:privacy}. 
	Since the sampled actions, recommended action, and stopping time are computed only using the private empirical means, the algorithm satisfies $\epsilon$-global DP thanks to the post-processing lemma. 
\end{proof}

\marc{\noindent \textit{Batching and forgetting for BAI}. For stochastic bandits, it is crucial to effectively use a combination of batching and forgetting~\citep{dpseOrSheffet, azize2022privacy, chowdhurydistributed}. These techniques help to use the parallel composition property of DP, and thus avoid using sequential composition, where more noise is needed to achieve DP. However, these batching and forgetting techniques should be adapted, depending on the setting and the ``accuracy'' guarantee. For example, for BAI, it is important that the batching is arm-dependent. For example, having a global doubling scheme would not work, i.e. update the means when the global count $n$ double. Also, it is important that the actions sampled by the Top Two algorithm during a fixed arm-phase ensure exploration of all the arms. This is different from the arm-dependent batching of bandits under regret~\citep{azize2022privacy}, where the \emph{same arm} is always chosen during a phase. In contrast, for BAI, only the mean estimators are updated when an arm counts doubles, but the Top Two algorithm still samples arms between the fixed leader and the potentially varying challenger. Intuitively, the challenger will evolve to ensure the equality of the transportation costs for the global counts, which is key to obtaining asymptotic optimality. In contrast to the regret minimisation algorithms, BAI algorithms require the empirical counts to be ``synchronised'' with respect to the same global time, as the empirical proportions should converge towards an allocation $\omega$ with dense support.}

\begin{algorithm}[t]
         \caption{Doubling-And-Forgetting$(\epsilon)$ Estimator (\protect\hypertarget{DAF}{DAF})}
         \label{algo:doubling_forgetting}
	\begin{algorithmic}[1]
 	      \State {\bfseries Input:} History $\mathcal H_{n}$, arm $a \in [K]$.
		\State {\bfseries Initialisation:} $T_{1}(a) = \arms+1$ and $k_{K+1,a}=1$; 
		\If{$N_{n,a} \ge 2 N_{T_{k_{n,a}}(a),a}$} \Comment{Per-arm doubling update grid}
            \State Change phase $k_{n,a} \gets k_{n,a} + 1$ for arm $a$;
            \State Set $T_{k_{n,a}}(a) = n$ and $\tilde N_{k_{n,a}, a} =  N_{T_{k_{n,a}}(a), a} - N_{T_{k_{n,a} - 1}(a), a}$; 
            \State Set $\hat \mu_{k_{n,a},a} = \tilde N_{k_{n,a},a}^{-1} \sum_{t = T_{k_{n,a}-1}(a)}^{T_{k_{n,a}}(a) - 1} r_{t} \ind{a_t = a}$; \Comment{\Modif{Forgetting past observations}}
            \State Set $\tilde \mu_{k_{n,a}, a} =  \hat \mu_{k_{n,a},a} +  Y_{k_{n,a},a}$ where $Y_{k_{n,a},a} \sim \text{Lap}((\epsilon\tilde N_{k_{n,a},a} )^{-1} )$; \Comment{\Modif{Private estimator}}
        \EndIf
		\State \textbf{Return} $(\tilde \mu_{n,a}, \tilde N_{n,a})$;
	\end{algorithmic}
\end{algorithm}

\subsection{A Plug-In Approach: the \adaptt{} Algorithm} 
\label{ssec:plug_in_based_TTUCB_global}

A natural approach is to simply plug in the private mean estimator in the non-private algorithm.
The Plug-In approach is successful for $\epsilon$-local DP FC-BAI (Section~\ref{sec:local_DP}) and for $\epsilon$-global DP regret minimisation \citep{azize2022privacy}.

\textit{\adaptt{} algorithm.}
To solve $\epsilon$-global DP FC-BAI, we propose the \adaptt{} algorithm\marc{, whose standalone pseudocode is detailed in Algorithm~\ref{algo:AdaPTT_full} (Appendix~\ref{app:algo_pseudocode}).}
\adaptt{} is an instance of Algorithm~\ref{algo:TTUCB} using the \hyperlink{DAF}{DAF}$(\epsilon)$ estimator (Algorithm~\ref{algo:doubling_forgetting}), \marc{$W^{G}_{a,b}$ as in Eq.~\eqref{eq:TTUCB_Gaussian} for $\sigma = 1/2$, }
\begin{equation}  \label{eq:TTUCB_global_v1}
	 b^{G,\epsilon}_{a}(\omega) = \sqrt{\frac{k(\omega_a)}{\omega_{a}}} + \frac{k(\omega_a)}{\epsilon \omega_{a}} \quad \text{with} \quad k(x) = \log_2 \marc{(} x \marc{)}+ 2 \: ,
\end{equation}
and $c^{G,\epsilon}_{a,b}$ as in Eq.~\eqref{eq:threshold_TTUCB_global_v1} which yields $\delta$-correctness for any sampling rule (Lemma~\ref{lem:delta_correct_threshold_TTUCB_global_v1}).
\begin{lemma}\label{lem:delta_correct_threshold_TTUCB_global_v1} 
    Let $\delta \in (0,1)$, $\epsilon > 0$.
    Let $c^{G}_{a,b}$ as in Eq.~\eqref{eq:threshold_non_private} and $k(x) = \log_2 x + 2 $.
    Given any sampling rule, combining the \hyperlink{DAF}{DAF}$(\epsilon)$ estimator with the GLR stopping rule as in Eq.~\eqref{eq:GLR_stopping_rule} with $W^{G}_{a,b}$ as in Eq.~\eqref{eq:TTUCB_Gaussian} and the stopping threshold
	\begin{equation} \label{eq:threshold_TTUCB_global_v1}
		c^{G,\epsilon}_{a,b}(\omega,\delta) = 2c^{G}_{a,b}(\omega,\delta (  k(\omega_{a})^2 k(\omega_{b})^2 \pi^4/18)^{-1}) + \frac{1}{\epsilon^2\sigma^2} \sum_{c \in \{a,b\}} \frac{1}{\omega_{c}} \left(\log \frac{\pi^2 K  k(\omega_{c})^{2} }{3\delta} \right)^2  \: ,
	\end{equation}
    yields a $\delta$-correct algorithm \Modif{for any $\sigma$-sub-Gaussian distributions with unique best arm}. 
\end{lemma}
\noindent Asymptotically, our threshold is $c^{G,\epsilon}_{a,b}(\omega,\delta) \approx_{\delta \to 0} 2\log (1/\delta) + ( 1/\omega_{a} + 1/\omega_{b} ) \log (1/\delta)^2/(\epsilon^2 \sigma^2)$.
\begin{proof}
\marc{The proof of $\delta$-correctness of a GLR stopping rule is standard, by leveraging concentration results.
Compared to the non-private proof, we also need to control the Laplace ``noise'', which results in additive terms in the stopping threshold.
Our proof technique can be used to study any Gaussian GLR stopping rules facing additive known pertubations.}
Specifically, we start by decomposing the failure probability $\bP_{\mu}\left( \tau_{\delta} < + \infty, \hat a \neq  a^\star \right)$ into a non-private and a private part using the basic property of $\bP(X + Y  \ge a+b) \le \bP(X \ge a) + \bP(Y \ge b)$. 
The two-factor in front of $c^{G}_{a,b}$ originates from the looseness of this decomposition, and we improve on it in Section~\ref{ssec:lower_bound_based_TTUCB_global}.
We conclude using concentration results from $\sigma$-sub-Gaussian and Laplace random variables.
The proof is detailed in Appendix~\ref{app:ssec_proof_delta_correct_threshold_TTUCB_global_v1}.
\end{proof}

\marc{\textit{Algorithmic intuition}.
 Thanks to the implicit exploration of the UCB leader, $B_n$ and $\tilde a_n$ converge towards $a^\star$.
 Moreover, the TC challenger selects the alternative arm to balance the information between the other arms, i.e., an arm $a \ne a^\star$ stops being selected as challenger when its empirical allocation $N_{n,a}/(n-1)$ overshoots an allocation for which the empirical transportation costs $W^{G}_{a^\star,a}$ are all equals (Lemma~\ref{lem:starting_phase_overshooting_challenger_implies_not_sampled_next}).
 The per-arm geometric update grid ensures that the local empirical proportions $\tilde N_{n}/(n-1)$ behave similarly to $N_{n}/(n-1)$, up to a factor at most $4$ due to doubling and forgetting.
 This synchronization of the local counts with the global time is key to studying the GLR stopping rule.
 While $\epsilon$-global DP regret minimization algorithms pull arms in batches, this approach fails in BAI as it would prevent convergence of $\tilde N_{n}/(n-1)$ towards a fixed allocation. }

\marc{\textit{Asymptotic upper bound}. \adaptt{} achieves an asymptotic upper bound on its expected sample complexity as $T_{\mathrm{KL},\beta}^\star$ rescaled by a privacy-aware cost (Theorem~\ref{thm:sample_complexity_TTUCB_global_v1}).}
\begin{theorem}\label{thm:sample_complexity_TTUCB_global_v1}
	Let $(\delta, \beta) \in (0,1)^2$ and $\epsilon > 0$.
        The \adaptt{} algorithm is $\epsilon$-global DP, $\delta$-correct and satisfies that, for any Bernoulli $\boldsymbol{\nu}$ instance with distinct means $\mu \in (0,1)^{K}$,
	\[
		\limsup_{\delta \to 0} \frac{\bE_{\boldsymbol{\nu}}[\tau_{\delta}]}{\log(1/\delta)} \le 4 T_{\mathrm{KL},\beta}^\star(\boldsymbol{\nu}) \left( 1+ \sqrt{1 +  \frac{\Delta_{\max}^2}{2\sigma^4\epsilon^2}} \right) \: ,
	\]    
    where $T^{\star}_{\mathrm{KL},\beta}$ as in Eq.~\eqref{eq:T_KL_Gaussian} for $\sigma = 1/2$.
\end{theorem}
\marc{The asymptotic analysis of Top Two algorithms is standard, by building upon the unified analysis from~\citet{jourdan2022top}.
Compared to the non-private proof, we adapt their argument to cope for the effect of the \hyperlink{DAF}{DAF}$(\epsilon)$ estimator on the expected sample complexity.
This required two main technical novelties. 
First, the proof of sufficient exploration requires to reason about phases per arm, and cope for doubling and forgetting (Appendix~\ref{app:ssec_sufficient_exploration}).
Second, once the convergence of the global count is established, we should argue that the phase switches of the arms happen in a ``round-robin'' fashion, i.e., an arm switches phase for a second time after all other arms first switch their own phases (Lemma~\ref{lem:bounded_phase_length_ratio}).
Our proof technique can be used to study algorithms relying on phases to limit the number of rounds of adaptivity (Remark~\ref{rem:round_adaptivity}).
We sketch other high-level ideas of the proof below, and refer to Appendix~\ref{app:TTUCB_global_upper_bounds} for more details.}

\begin{proof}
(1) The \textit{non-private TTUCB algorithm}~\citep{jourdan2022non} achieves a sample complexity of $T_{\mathrm{KL},\beta}^\star(\boldsymbol{\nu})$ for sub-Gaussian random variables.
The proof relies on showing that the empirical pulling counts are converging towards the $\beta$-optimal allocation $\omega^{\star}_{\mathrm{KL},\beta}(\boldsymbol{\nu})$. 
(2) The \textit{effect of doubling and forgetting} is a multiplicative four-factor, i.e. $4 T_{\mathrm{KL},\beta}^\star(\boldsymbol{\nu})$. 
The first two-factor is due to forgetting since we throw away half of the samples.
The second two-factor is due to doubling since we have to wait for the end of an episode to evaluate the stopping condition. 
(3) The \textit{Laplace noise} only affects the empirical estimate of the mean.
Since the Laplace noise has no bias and a sub-exponential tail, the private means will still converge towards their true values.
Therefore, the empirical counts will also converge to $\omega^{\star}_{\mathrm{KL},\beta}(\boldsymbol{\nu})$ asymptotically.
(4) While the \textit{Laplace noise has little effect on the sampling rule} itself, it \textit{changes the dependency in $\log(1/\delta)$ of the threshold} used in the GLR stopping rule.
The private threshold $c^{G,\epsilon}_{a,b}$ has an extra factor $\bigO(\log^2( 1/\delta ))$ compared to the non-private one $c^{G}_{a,b}$. 
Using the convergence towards $\omega^{\star}_{\mathrm{KL},\beta}(\boldsymbol{\nu})$, the stopping condition is met as soon as $\frac{n}{T^{\star}_{\mathrm{KL},\beta}(\boldsymbol{\nu})} \lesssim 2\log ( 1/\delta ) + \frac{\Delta^2_{\max}}{2\sigma^4\epsilon^2} \frac{T^{\star}_{\mathrm{KL},\beta}(\boldsymbol{\nu})}{n} \log^2( 1/\delta )$.
Solving the inequality for $n$ concludes the proof while adding a multiplicative four-factor.
\end{proof}

\textit{Discussion.}
In the non-private regime where $\epsilon \to +\infty$, our upper bound recovers the non-private lower bound for Gaussian distributions $T_{\mathrm{KL}}^\star(\boldsymbol{\nu})$ up to a multiplicative factor $16$.
For Bernoulli distributions (or bounded distributions in $[0,1]$), there is still a mismatch between the upper and lower bounds due to the mismatch between the KL divergence of Bernoulli distributions and that of Gaussian (e.g. large ratio when the means are close to $0$ or $1$). 
This is in essence, similar to the mismatch between UCB and KL-UCB in the regret-minimisation literature (e.g. Chapter 10 in~\citealt{lattimore2018bandit}). 
To overcome this mismatch, it is necessary to adapt the transportation costs to the family of distributions considered. 
While the Top Two algorithms for Bernoulli distributions (or bounded distributions in $[0,1]$) have been studied in~\citet{jourdan2022top}, the analysis is more involved. 
Therefore, it would obfuscate where and how privacy is impacting the expected sample complexity.

In the asymptotic highly privacy regime where $\epsilon \to 0$, our upper bound becomes $ \cO(T_{\mathrm{KL}}^\star(\boldsymbol{\nu})\Delta_{\max}/\epsilon)$ while the lower bound is $\Omega(T_{\mathrm{TV}}^\star(\boldsymbol{\nu})/\epsilon)$.
Therefore, our upper bound is only asymptotically tight for instances such that $T_{\mathrm{KL}}^\star(\boldsymbol{\nu}) = \cO(T_{\mathrm{TV}}^\star(\boldsymbol{\nu})/\Delta_{\max} )$, e.g. instances where the mean gaps have the same order of magnitude.
In all the other cases, the plug-in approach is sub-optimal due to a problem-dependent gap.
\marc{The major limitation of \adaptt{} lies in the fact that the transportation costs are independent of the privacy budget.
Without accounting for the knowledge that there is an additional Laplace ``noise'', there is little hope to match the asymptotic lower bound in the high-privacy regime. We remedy this issue with \adapttt{} in Section~\ref{ssec:lower_bound_based_TTUCB_global}.}

\marc{\textit{Non-asymptotic confidence regime for \adaptt{}.} 
Studying \adaptt{} for any confidence level requires to adapt the proof of~\citet{jourdan2022non}.
The $\delta$-dependent term in their upper bound is worse than for an asymptotic analysis (i.e., Theorem~\ref{thm:sample_complexity_TTUCB_global_v1}).
Therefore, we focus on their dominating $\delta$-independent term that stems from~\citet[Lemma 3.2]{jourdan2022non}, i.e., $B_n = a^\star$ except for a sublinear number of times. 
While the intution behind their proof still holds, using the \hyperlink{DAF}{DAF}$(\epsilon)$ estimator adds steps to cope for the Laplace ``noise'', doubling and forgetting. 
If $B_n = a$ with $a \ne a^\star$ and empirical means do not deviate too much from the true means, then we have $N_{n,a} = \cO(\log (n) \Delta_{\epsilon,a}^{-2})$ where $\Delta_{\epsilon,a} \defn  \epsilon  (\sqrt{1 + 2\Delta_{a}/\epsilon} - 1)$ by solving for $N_{n,a}/\log n$ the quadratic inequality
\[
    \mu_{a^\star} \lesssim \tilde \mu_{n,a^\star} + b_{a^\star}^{G,\epsilon}(\tilde N_{n}) \le \tilde \mu_{n,a} + b_{a}^{G,\epsilon}(\tilde N_{n}) \lesssim \mu_{a} + 2b_{a}^{G,\epsilon}(\tilde N_{n}) \le \mu_{a} + 2\left(\sqrt{\frac{4 \log n}{N_{n,a}}} + \frac{4 \log n}{\epsilon N_{n,a}}\right)  \: ,
\]
where $\tilde N_{n,a} \ge N_{n,a}/4$ due to doubling and forgetting, and $\lesssim$ stems from the ``heuristic'' choice of $(b^{G,\epsilon}_{a})_{a \in [K]}$, which do not yield valid \Modif{upper confidence bounds}.
Since $N_{n,a}$ is bounded and incremented by one half of the time (tracking for $\beta=1/2$), the event $B_n \ne a^\star$ occurs less than $\cO(H(\boldsymbol{\nu}, \epsilon) \log n)$ times, where $H(\boldsymbol{\nu}, \epsilon) = \sum_{a \in [K]} \Delta_{\epsilon,a}^{-2} $ and $\Delta_{\epsilon,a^\star} = \min_{a \ne a^\star} \Delta_{\epsilon,a}$.
The rest of the proof can be adapted similarly, with extra terms due to the Laplace ``noise'', doubling and forgetting.
Intuitively, it holds thanks to the use of the Gaussian transportation costs $W^{G}_{a,b}$ as in Eq.~\eqref{eq:TTUCB_Gaussian} providing time-uniform separability between the means and the allocations, i.e., ratio of the squared mean gap and the sum of the inverse allocation.
In summary, we conjecture that \adaptt{} satisfies a non-asymptotic upper bound whose dominating $\delta$-independent term scale as $\mathcal O \left(  \left( H(\boldsymbol{\nu}, \epsilon) \log H(\boldsymbol{\nu}, \epsilon) \right)^{\alpha}\right)$ \Modif{for some} $\alpha > 1$\Modif{, where $\alpha$ would be an algorithmic hyperparameter involved in a valid choice for $(b^{G,\epsilon}_{a})_{a \in [K]}$}.
As $\Delta_{\epsilon,a}^2/\epsilon  \to_{\epsilon \to 0} \Delta_{a}$ and $\Delta_{\epsilon,a}^2 \to_{\epsilon \to +\infty} \Delta_{a}^2$, the bound suffers from a suboptimal \Modif{dependency due to the power $\alpha$}.}

\subsection{A \marc{Modified Transportation Cost} Approach: the \texorpdfstring{\adapttt{}}{} Algorithm}
\label{ssec:lower_bound_based_TTUCB_global}

To overcome the limitation of \adaptt{}, one should adapt the transportation costs to reflect the lower bound (Theorem~\ref{thm:global_lower_bound}) instead of ``ignoring'' the privacy constraint by using the transportation costs $W^{G}_{a,b}$ as in Eq.~\eqref{eq:TTUCB_Gaussian} which are tailored for non-private FC-BAI.

\textit{\adapttt{} algorithm.}
We propose the \adapttt{} algorithm\marc{, whose standalone pseudocode is detailed in Algorithm~\ref{algo:AdaPTTstar_full} (Appendix~\ref{app:algo_pseudocode}).}
\marc{It differs from \adaptt{} solely by (i) the use of a modified transportation costs in the GLR stopping rule and the TC challenger, and (ii) an adequate stopping threshold to ensure $\delta$-correctness.}
\adapttt{} is an instance of Algorithm~\ref{algo:TTUCB} using the \hyperlink{DAF}{DAF}$(\epsilon)$ estimator (Algorithm~\ref{algo:doubling_forgetting}), 
\begin{equation} \label{eq:TTUCB_global_v2}
	W^{G,\epsilon}_{a,b}(\tilde \mu, \omega) =  \frac{(\tilde  \mu_{a} - \tilde  \mu_{b})_{+} \min\{\epsilon/2, (\tilde  \mu_{a} - \tilde  \mu_{b})_{+}\} }{2\sigma^2(1/\omega_{a}+ 1/\omega_{b})} \quad \text{with} \quad \sigma = 1/2 \: ,
\end{equation}
\marc{$b^{G,\epsilon}_{a,b}$ as in Eq.~\eqref{eq:TTUCB_global_v1},} and $\tilde c^{G,\epsilon}_{a,b}(\tilde \mu, \omega, \delta)$ as in Eq.~\eqref{eq:threshold_TTUCB_global_v2}, which yields $\delta$-correctness for any sampling rule (Lemma~\ref{lem:delta_correct_threshold_TTUCB_global_v2}).

\begin{lemma}\label{lem:delta_correct_threshold_TTUCB_global_v2} 
    Let $\delta \in (0,1)$, $\epsilon > 0$.
    Let $\overline{W}_{-1}(x)  = - W_{-1}(-e^{-x})$ for all $x \ge 1$, where $W_{-1}$ is the negative branch of the Lambert $W$ function.
    It satisfies $\overline{W}_{-1}(x) \approx x + \log x$.
    Let $c^{G,\epsilon}_{a,b}$ as in Eq.~\eqref{eq:threshold_TTUCB_global_v1}, $k(x) = \log_2 (x) + 2 $ and
    \[
    	h(x,\delta) = \overline{W}_{-1}\left(2 \log(\pi^2 K  k(x)^{2}/(2\delta)) + 4 \log (4 + \log x)+ 1/2 \right)/2 \: .
    \]
    Given any sampling rule, combining the \hyperlink{DAF}{DAF}$(\epsilon)$ estimator with the GLR stopping rule as in Eq.~\eqref{eq:GLR_stopping_rule} with $W^{G,\epsilon}_{a,b}$ as in Eq.~\eqref{eq:TTUCB_global_v2} and the stopping threshold $\tilde c^{G,\epsilon}_{a,b}(\tilde \mu, \omega,\delta)$ which is equal to
	\begin{equation} \label{eq:threshold_TTUCB_global_v2}
		\begin{cases}
		\frac{1}{2} c^{G,\epsilon}_{a,b}(\omega,2\delta/3) + \frac{\sqrt{2}}{\epsilon \sigma}  \sum_{c \in \{a,b\}} \sqrt{\frac{h(\omega_{c},\delta)}{\omega_{c}}} \log \left(\frac{\pi^2 K  k(\omega_{c})^{2} }{2\delta} \right) & \text{if } (\tilde \mu_{a} - \tilde \mu_{b})_{+} < \epsilon/2\\
		\frac{1}{2\sigma^2}  \log \left( \pi^2 K \max_{c \in \{a, b\}}k(w_c)/(2\delta)\right) + \frac{\epsilon}{2\sqrt{2}\sigma}\sum_{c \in \{a,b\}} \sqrt{\omega_c h(\omega_{c},\delta)} &
		\end{cases} \: ,
	\end{equation}
    yields a $\delta$-correct algorithm for any $\sigma$-sub-Gaussian distributions with unique best arm. 
\end{lemma}
Our threshold is $\frac{1}{2\sigma^2}  \log(1/\delta) + \frac{\epsilon}{2\sqrt{2}\sigma} (\sqrt{\omega_b}+\sqrt{\omega_a}) \sqrt{\log (1/\delta) }$ when $\tilde \mu_{a} - \tilde \mu_b \ge \epsilon/2$, and
\[
	\log (1/\delta) + \frac{1}{2\epsilon^2 \sigma^2}( 1/\omega_{a} + 1/\omega_{b} ) \log (1/\delta)^2 + \frac{\sqrt{2}}{\epsilon \sigma} (\sqrt{1/\omega_{a}} + \sqrt{1/\omega_{b}} )\log (1/\delta)^{3/2}  \quad \text{otherwise.}
\]

\marc{While the proof bares ressemblence to the one of Lemma~\ref{lem:delta_correct_threshold_TTUCB_global_v1} (see Appendix~\ref{app:ssec_proof_delta_correct_threshold_TTUCB_global_v2}), the main novelty in $\tilde c^{G,\epsilon}_{a,b}$ is that it depends on whether $(\tilde \mu_{n,a} - \tilde \mu_{n,b})_{+}$ lies above $\epsilon/2$ or below it.
This is instrumental to switch between the high-privacy regime, where the Laplace ``noise'' dominates, and the low-privacy regime, where the Laplace ``noise'' is ``negligible'' 
In the latter case, we recover $c^{G,\epsilon}_{a,b}$, with an improved factor $1/2$ due to tighter concentration inequalities.
In the former case, the dominating term in the stopping threshold stems from the control of the Laplace ``noise'', while the Gaussian ``noise'' contributes to second-order terms.
To the best of our knownledge, Lemma~\ref{lem:delta_correct_threshold_TTUCB_global_v2} constitutes the first mean-aware stopping thresholds.
This idea might be of independent interest, especially when several sources of randomness coexist in the GLR stopping rule. 
Echoing the discussion in Section~\ref{ssec:global_DP_lower_bound}, once a measure of ``distinguishability'' at the arm’s level is obtained, we conjecture that controlling its empirical version in a mean-agnostic fashion is possible.}

The transportation cost $W^{G,\epsilon}_{a,b}$ is inspired by \marc{$T^\star_{\mathrm{KL}}(\bm{\nu}_{G,\epsilon})$ as in Eq.~\eqref{eq:global_explicit_relaxation}.
Recall that the associated $\beta$-characteristic time $T^\star_{\mathrm{KL},\beta}(\bm{\nu}_{G,\epsilon})$ in Eq.~\eqref{eq:T_KL_Gaussian} is defined as}
\begin{equation} \label{eq:relaxed_characteristic_time}
	\marc{T^\star_{\mathrm{KL},\beta}(\bm{\nu}_{G,\epsilon})}^{-1} \defn \max\limits_{\omega \in \Sigma_K, \omega_{a^\star} = \beta} \min_{a \ne a^\star}  \frac{(\mu_{a^\star} - \mu_{a}) \min\{\epsilon/2, \mu_{a^\star} - \mu_{a}\}}{2\sigma^2(1/\beta+1/\omega_{a})}  \quad \text{with} \quad \sigma = 1/2 \: .
\end{equation}
\marc{The algorithmic intuition behind \adapttt{} is the same as for \adaptt{}.
The main difference lies in the behavior of the TC challenger that now aims at reaching the equality for all the empirical modified transporation costs $W^{G,\epsilon}_{a^\star,a}$ compared to $W^{G}_{a^\star,a}$ for \adaptt{}.
Given that \adapttt{} is designed to reach $T^\star_{\mathrm{KL},\beta}(\bm{\nu}_{G,\epsilon})$ asymptotically, we will not achieve our lower bound, as $T^\star_{\mathrm{KL}}(\bm{\nu}_{G,\epsilon})$ is only an upper bound on $T^{\star}_g(\boldsymbol{\nu} ; \epsilon )$ (Corollary~\ref{lem:more_explicit_lower_bound}). }

\marc{\textit{Asymptotic upper bound}. \adapttt{} achieves an asymptotic upper bound on its expected sample complexity as $T_{\mathrm{KL},\beta}^\star$ evaluated at $\boldsymbol{\nu}$ in the low privacy regime and at $\bm{\nu}_{G,\epsilon}$ in the high-privacy regime, up to multiplicative privacy-aware cost (Theorem~\ref{thm:sample_complexity_TTUCB_global_v2}).}
\begin{theorem}\label{thm:sample_complexity_TTUCB_global_v2}
	Let $(\delta, \beta) \in (0,1)^2$ and $\epsilon > 0$.        
        The \adapttt{} algorithm is $\epsilon$-global DP, $\delta$-correct and satisfies that, for any Bernoulli $\boldsymbol{\nu}$ instance with distinct means $\mu \in (0,1)^{K}$,
        \begin{align*}
		\limsup_{\delta \to 0} \frac{\bE_{\boldsymbol{\nu}}\left[\tau_{\delta}\right]}{\log (1 / \delta)}
&\le  \begin{cases}
	4 T^{\star}_{\mathrm{KL},\beta}(\boldsymbol{\nu})g_{1}\left(\Delta_{\max}/(\sigma^2 \epsilon) \right) & \text{if } \Delta_{\max} <  \epsilon/2 \\
	 2 \marc{T^\star_{\mathrm{KL},\beta}(\bm{\nu}_{G,\epsilon})} g_{2}(\epsilon^2 \marc{T^\star_{\mathrm{KL},\beta}(\bm{\nu}_{G,\epsilon})} \max\{\beta, 1-\beta\}/4)/\sigma^2 & \text{otherwise}
\end{cases} \: ,
	\end{align*}
        \marc{where $T^{\star}_{\mathrm{KL},\beta}$ as in Eq.~\eqref{eq:T_KL_Gaussian} for $\sigma = 1/2$} and $\marc{T^\star_{\mathrm{KL},\beta}(\bm{\nu}_{G,\epsilon})}$ as in Eq.~\eqref{eq:relaxed_characteristic_time}.
	The function $g_{1}(y) =  \sup \left\{ x \mid  x^2 <  x +  y \sqrt{2x}  +  \frac{y^2}{4} \right\}$ is increasing on $[0,12]$ and satisfies that $g_{1}(0) = 1$ and $g_{1}(12) \le 10$.
	The function $g_{2}(y) = 1 + 2(\sqrt{1+1/y}-1)^{-1}$ is increasing on $\R^{\star}_+$ and satisfies that $\lim_{y \to 0} g_{2}(y) = 1$. 
\end{theorem}
\begin{proof}
    \marc{The proof is almost identical to the one of Theorem~\ref{thm:sample_complexity_TTUCB_global_v1}, hence we defer to Appendix~\ref{app:TTUCB_global_upper_bounds} for more details.
    While $W^{G,\epsilon}_{a,b}$ is used instead of $W^{G}_{a,b}$, both arguments rely on similar properties holding for Gaussian-like transportation costs.
    The main technical novelty occurs during the ``inversion'' of the GLR stopping rule since the stopping threshold is mean-aware and includes additional second order terms (Appendix~\ref{app:ssec_asymptotic_upper_bound_sample_complexity}).}
\end{proof}
	
\textit{Discussion.}
When $\Delta_{\max} < \epsilon/2$, our upper bound recovers the non-private lower bound for Gaussian distributions $T_{\mathrm{KL}}^\star(\boldsymbol{\nu})$ up to a multiplicative factor $8 g_{1}(4\Delta_{\max}/\epsilon) \in [8,80]$, whose limit is $8$ in non-private regime where $\epsilon \to +\infty$. 
When $\Delta_{\min} \ge \epsilon/2$, we have $2\marc{T^\star_{\mathrm{KL},\beta}(\bm{\nu}_{G,\epsilon})} \le 4T^{\star}_{\mathrm{TV}}(\boldsymbol{\nu})/\epsilon$.
In the asymptotic highly privacy regime where $\epsilon \to 0$, our upper bound matches the lower bound up to a multiplicative factor $16$.
Therefore, we close the gap left open by the algorithm in Section~\ref{ssec:plug_in_based_TTUCB_global}. 
While the regime $\Delta_{\max} \ge \epsilon/2 > \Delta_{\min}$ is relevant for practical application, it is harder to understand how the different quantities interact in the upper/lower bounds in transitional phases. 
Thus, \marc{we do not} claim optimality in those phases.
Having matching upper and lower bounds \textit{only} for high privacy regimes is an interesting phenomenon that appears in different settings of differential privacy literature, such as regret minimisation~\citep{azize2022privacy}, parameter estimation~\citep{cai2021cost} and hidden probabilistic graphical models~\citep{nikolakakis2019optimal}. 
\marc{For regret minimisation and BAI, we conjecture that deriving matching bounds at any privacy level fundamentally requires a measure of ``distinguishability'' at the arm’s level. This is a promising research direction that would require finer technical tools to optimally merge the privacy and the correctness constraints. }

\textit{Comparison to DP-SE.} 
DP-SE~\citep{dpseOrSheffet} is an $\epsilon$-global DP version of the Successive Elimination algorithm introduced for the regret minimisation setting. 
The algorithm samples active arms uniformly during phases of geometrically increasing length. 
Based on the private confidence bounds, DP-SE eliminates provably sub-optimal arms at the end of each phase. 
Due to its phased-elimination structure, DP-SE can be easily converted into an $\epsilon$-global DP FC-BAI algorithm, where we stop once there is only one active arm left. 
In particular, the proof of Theorem 4.3 of~\citet{dpseOrSheffet} shows that with high probability any sub-optimal arm $a\neq a^\star$ is sampled no more than $\bigO (\Delta_a^{2} + (\epsilon\Delta_a)^{-1})$. 
From this result, it is straightforward to extract a sample complexity upper bound for DP-SE, i.e. $\bigO ( \sum_{a\ne a^\star} \Delta_a^{-2}  +  \sum_{a\ne a^\star} (\epsilon\Delta_a)^{-1}).$ This shows that DP-SE too achieves (ignoring constants) the high-privacy lower bound $T^\star_{\textrm{TV}}(\boldsymbol{\nu})/\epsilon$ for Bernoulli instances. 
However, due to its uniform sampling within the phases, DP-SE is less adaptive than \hyperlink{TTUCB}{TTUCB}. 
Inside a phase, DP-SE continues to sample arms that might already be known to be bad, while \hyperlink{TTUCB}{TTUCB} adapts its sampling rule based on the transportation costs that reflect the amount of evidence collected in favour of the hypothesis that the leader is the best arm. 
Finally, \hyperlink{TTUCB}{TTUCB} has the advantage of being anytime, i.e. its sampling strategy does not depend on the risk $\delta$.

Another adaptation of DP-SE, namely DP-SEQ, is proposed in~\citet{kalogerias2020best} for the problem of privately finding the arm with the highest quantile at a fixed level, hence it is different from BAI.  
For multiple agents,~\citet{rio2023multi} studies privacy for BAI under fixed confidence. 
They propose and analyse the sample complexity of DP-MASE, a multi-agent version of DP-SE. 
They show that multi-agent collaboration leads to better sample complexity than independent agents, even under privacy constraints.
While the multi-agent setting with federated learning allows tackling large-scale clinical trials taking place at several locations simultaneously, we study the single-agent setting, which is relevant for many small-scale clinical trials (see Example~\ref{ex:1}).

\textit{Non-asymptotic confidence regime for \adapttt{}.
\adapttt{} has the same property on the leader as \adaptt{}, i.e., $B_n = a^\star$ except for a sublinear number of times.
However, the modified transportation costs $W^{G,\epsilon}_{a,b}$ as in Eq.~\eqref{eq:TTUCB_global_v2} do not provide time-uniform separability between the means and the allocations.
As it depends on $\min\{\epsilon/2, (\tilde \mu_{a} - \tilde \mu_{b})_{+}\}$ in the numerator, we should justifty that $\tilde \mu_{n,a^\star} - \tilde \mu_{n,C_n}$ and $\mu_{a^\star} - \mu_{C_n}$ lies on the same side as $\epsilon/2$, in order to use the equality at equilibrium of the modified transportation costs to obtain $T^{\star}_{\epsilon, \beta}(\boldsymbol{\nu})^{-1}$.
It is challenging to prove this property non-asymptotically without incurring a larger $\delta$-independent dependency.
Therefore, we conjecture that \adapttt{} suffers from a worse non-asymptotic upper bound based on the proof of~\citet{jourdan2022non}.
While our asymptotic analysis deals with this subtelty, the extra cost vanishes as $\delta \to 0$.}

\begin{remark}[On the number of rounds of adaptivity] \label{rem:round_adaptivity}
    Used on any existing FC-BAI algorithm, the \hyperlink{DAF}{DAF} update yields a batched algorithm, which satisfies $\epsilon$-global DP.
    At the end of the episode of arm $a$ (after updating its mean), it is possible to compute the sequence of all the arms to be pulled before the end of the next episode (for another arm), without taking the collected observations into account.
    In contrast to the classical batched setting where the batch size is fixed, the size of the resulting batches is adaptive and data-dependent.
    In the non-private setting ($\epsilon = + \infty$), we recover Batched Best-Arm Identification (BBAI) in the fixed-confidence setting.
    \adaptt{} and \adapttt{} are asymptotically optimal up to a multiplicative factor $4$ with solely $\mathcal O \left( K \log_2 (T^\star_{\mathrm{KL}}(\boldsymbol{\nu})\log(1/\delta)) \right)$ rounds of adaptivity.
    We refer the reader to Appendix~\ref{app:rounds_adaptivity} for more details, including comparison to existing works.
\end{remark}

\section{Experimental Analysis}
\label{sec:experiments}

We perform experiments for both $\epsilon$-local DP and $\epsilon$-global DP. The code is available at \url{https://github.com/achraf-azize/DP-BAI}.

\subsection{Local DP}
\label{sec:expe_localDP}

We run \locptt{} in different Bernoulli instances as in~\citet{dpseOrSheffet}. 
As a benchmark, we also compare to the non-private TTUCB. 
As for $\epsilon$-global DP, we set the risk $\delta=10^{-2}$, implement all the algorithms in Python (version $3.8$) and run each algorithm $1000$ times. 
We plot the corresponding average and standard deviations of the empirical stopping times in Figure~\ref{fig:experiments_local}. 
We also test the algorithms on other Bernoulli instances and report the results in Appendix~\ref{app:experiment}.

\begin{figure}[t!]
	\centering
	\includegraphics[width=0.49\linewidth]{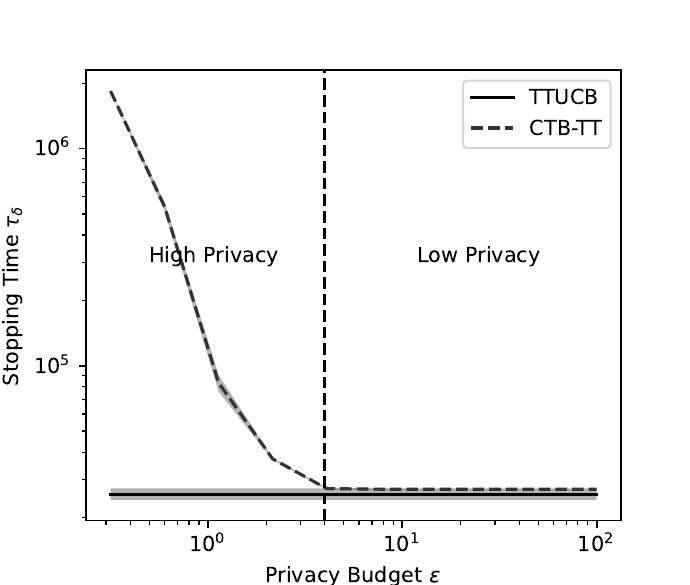}
	\includegraphics[width=0.49\linewidth]{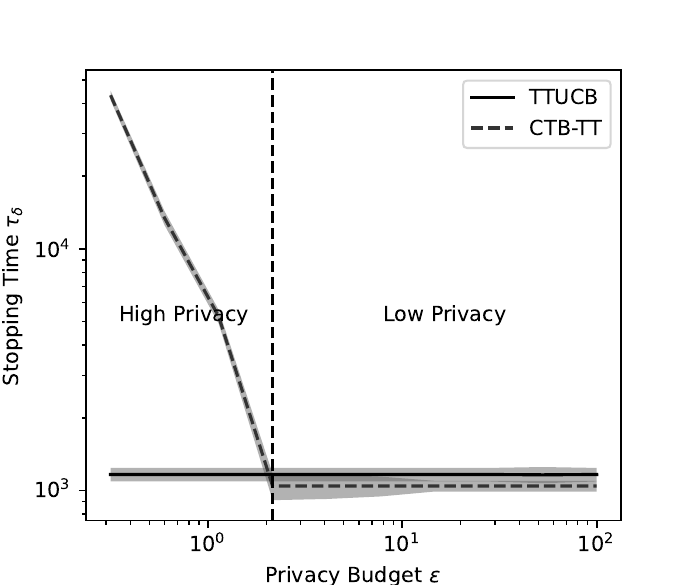}
	\caption{Empirical stopping time $\tau_{\delta}$ (mean $\pm$ std. over 1000 runs, $\delta = 10^{-2}$) with respect to the privacy budget $\epsilon$ for $\epsilon$-local DP on Bernoulli instance $\mu_{1}$ (left) and $\mu_{2}$ (right). The shaded vertical line separates the two privacy regimes.}
	\label{fig:experiments_local}
\end{figure}

Figure~\ref{fig:experiments_local} shows that \locptt{} performance has two regimes. 
In the low privacy regime ($\epsilon > 4$ for $\mu_1$ and $\epsilon > 2$ for $\mu_2$), the \hyperlink{CTB}{CTB}$(\epsilon)$ estimator reduces to the \hyperlink{MLE}{MLE}, and \locptt{} matches exactly the performance of the non-private TTUCB. 
In the high privacy regime ($\epsilon < 4$ for $\mu_1$ and $\epsilon < 2$ for $\mu_2$), the price of privacy on the stopping time is a multiplicative $\epsilon^{-2}$. 
Therefore, the sample complexity is prohibitively large to be computed numerically for $\epsilon < 0.1$. 
The switching value of $\epsilon$ between the low and high privacy regimes is an order of magnitude higher for $\epsilon$-local DP compared to the one for $\epsilon$-global DP. 
This is predictable since local DP provides a ``stronger'' privacy guarantee at the cost of worse performance.

\subsection{Global DP}
\label{sec:expe_globalDP}

We compare the performances of \adaptt{}, \adapttt{} and DP-SE for FC-BAI in different Bernoulli instances as in~\citet{dpseOrSheffet}. 
The first instance has means $\mu_{1} = (0.95, 0.9, 0.9, 0.9, 0.5)$ and the second instance has means $\mu_{2} = (0.75, 0.7, 0.7, 0.7, 0.7)$.
As a benchmark, we also compare to the non-private TTUCB. We set the risk $\delta=10^{-2}$ and implement all the algorithms in Python (version $3.8$). 
We run each algorithm $1000$ times, and plot corresponding average and standard deviations of the empirical stopping times in Figure~\ref{fig:experiments_global}. 
We also test the algorithms on other Bernoulli instances and report the results in Appendix~\ref{app:experiment}.

\begin{figure}[t!]
	\centering
	\includegraphics[width=0.49\linewidth]{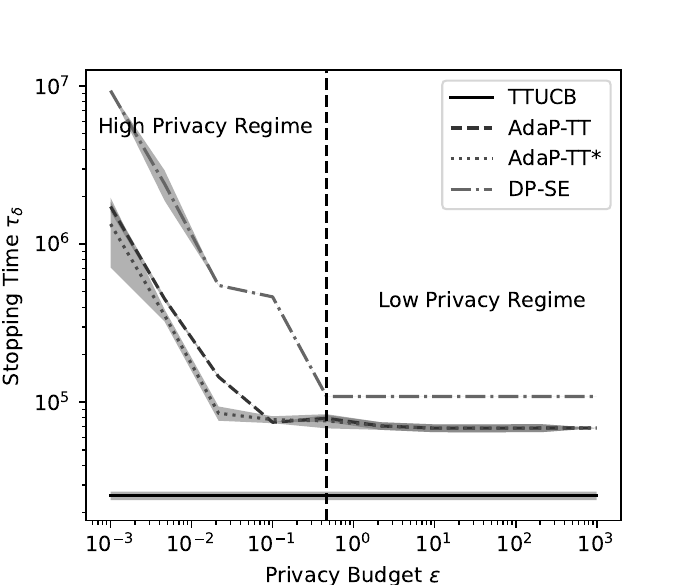}
	\includegraphics[width=0.49\linewidth]{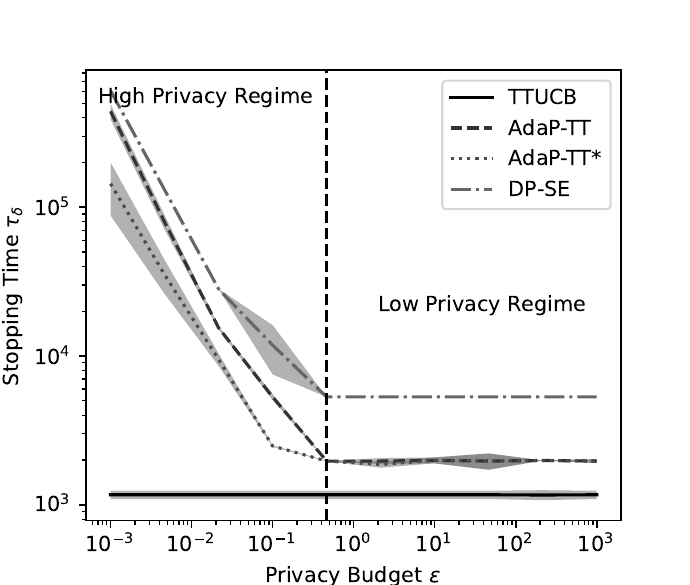}
	\caption{Empirical stopping time $\tau_{\delta}$ (mean $\pm$ std. over 1000 runs) with respect to the privacy budget $\epsilon$ for $\epsilon$-global DP on Bernoulli instance $\mu_{1}$ (left) and $\mu_{2}$ (right). The shaded vertical line separates the two privacy regimes.}
	\label{fig:experiments_global}
\end{figure}

Figure~\ref{fig:experiments_global} shows that: (a) \adaptt{} and \adapttt{} require fewer samples than DP-SE to provide a $\delta$-correct answer, for different values of $\epsilon$ and in all the instances tested. \adaptt{} and \adapttt{} have the same performance in the low privacy regimes, while \adapttt{} improves the sample complexity in the high privacy regime, as predicted theoretically. (b) The experimental performance of \adaptt{} and \adapttt{} demonstrate two regimes. A high-privacy regime (for $\epsilon < 0.1$ for $\mu_1$ and $\epsilon < 0.4$ for $\mu_2$), where the stopping time depends on the privacy budget $\epsilon$, and a low privacy regime (for $\epsilon > 0.1$ for $\mu_1$ and $\epsilon > 0.4$ for $\mu_2$), where the performance of \adaptt{} and \adapttt{} does not depend on $\epsilon$, and is four times the samples required by TTUCB in the worst case, as shown theoretically.

\section{Conclusion and Perspectives}

\begin{table}[t!]
\begin{center}
\begin{tabular}{l l l l l}
\hline
Algorithm & Trust model & Mean estimator & Transportation cost & Stopping threshold \\
\hline
TTUCB  & None &  \protect\hyperlink{MLE}{MLE} (Alg.~\ref{algo:MLE})     & $W^{G}$ as in Eq.~\eqref{eq:TTUCB_Gaussian} & $c^{G}$ as in Eq.~\eqref{eq:threshold_non_private} \\
\locptt{}  & $\epsilon$-local DP &  \protect\hyperlink{CTB}{CTB} (Alg.~\ref{algo:CTB})     & $W^{G}$ as in Eq.~\eqref{eq:TTUCB_Gaussian} & $c^{G}$ as in Eq.~\eqref{eq:threshold_non_private} \\
\adaptt{}  & $\epsilon$-global DP &  \protect\hyperlink{DAF}{DAF} (Alg.~\ref{algo:doubling_forgetting})     & $W^{G}$ as in Eq.~\eqref{eq:TTUCB_Gaussian} & $c^{G,\epsilon}$ as in Eq.~\eqref{eq:threshold_TTUCB_global_v1} \\
\adapttt{}  & $\epsilon$-global DP &   \protect\hyperlink{DAF}{DAF} (Alg.~\ref{algo:doubling_forgetting})     & $W^{G,\epsilon}$ as in Eq.~\eqref{eq:TTUCB_global_v2} & $\tilde c^{G,\epsilon}$ as in Eq.~\eqref{eq:threshold_TTUCB_global_v2} \\
\hline
\end{tabular}
\end{center}
\caption{Instances of the  \protect\hyperlink{TTUCB}{TTUCB} meta-algorithm \marc{defined in Algorithm~\ref{algo:TTUCB}}.}
\label{tab:algorithm_design}
\end{table}

We study FC-BAI with $\epsilon$-local DP and $\epsilon$-global DP. 
In both settings, we derive a lower bound on the expected sample complexity which quantifies the additional samples needed by a $\delta$-correct BAI strategy to ensure DP. 
The lower bounds further suggest the existence of two privacy regimes. 
In the \textit{low-privacy regime}, no additional samples are needed, and \textit{privacy can be achieved for free}. 
For the \textit{high-privacy regime}, the lower bound reduces to $\Omega({\epsilon}^{-2} T^\star_{\text{TV}^2}(\nu))$ for $\epsilon$-local DP, and to $\Omega({\epsilon}^{-1} T^\star_{\text{TV}}(\nu))$ for $\epsilon$-global DP.
To match those lower bounds up to multiplicative constants, we propose $\epsilon$-local DP and $\epsilon$-global DP variants of a Top Two algorithm.
For $\epsilon$-local DP, the \locptt{} algorithm reaches asymptotic optimality by plugging in a private estimator of the means based on Randomised Response.
For $\epsilon$-global DP, our private estimator of the mean runs in \textit{arm-dependent adaptive episodes} and adds \textit{Laplace noise} to ensure a good privacy-utility trade-off.
By solely plugging in this estimator, the AdaP-TT algorithm fails to recover the asymptotic lower bound for instances with highly different mean gaps.
The AdaP-TT$^\star$ algorithm overcomes this limitation by adapting the transportation costs\marc{, hence reaching the asymptotic lower bound in the asymptotic high-privacy regime (up to a small multiplicative constant).}

The upper bound matches the lower bound by a multiplicative constant in the high privacy regime, and is also loose in some instances in the low privacy regime, due to the mismatch between the KL divergence of Bernoulli distributions and that of Gaussian. 
One possible direction to solve this issue is to use transportation costs tailored to Bernoulli for both the Top Two Sampling and the stopping. 
Since our bounds only give a clear picture in the high and low privacy regimes, it would be interesting to provide better insights for the regime in-between where both the $\delta$-correctness and the DP constraints are of the same order.
An interesting direction would be to extend the proposed technique to other variants of pure DP, namely $(\epsilon, \delta)$-DP and R\'enyi-DP~\citep{mironov2017renyi}, or other trust models, e.g. shuffle DP~\citep{cheu2021differential, girgis2021renyi}.

\acks{This work has been partially supported by the THIA ANR program ``AI\_PhD@Lille". 
A. Al-Marjani acknowledges the support of the Chaire SeqALO (ANR-20-CHIA-0020). 
D. Basu acknowledges the Inria-Kyoto University Associate Team ``RELIANT'' for supporting the project, the ANR JCJC for the REPUBLIC project (ANR-22-CE23-0003-01), and the PEPR project FOUNDRY (ANR23-PEIA-0003). 
We thank Emilie Kaufmann and Aurélien Garivier for the interesting conversations. 
We also thank Philippe Preux for his support.}

\newpage
\appendix

\section{Outline}\label{app:outline}

The appendices are organised as follows:
\begin{itemize}
    \item \marc{Notation are summarized in Appendix~\ref{app:notation}.}
    \item \marc{For clarity, we explicit the algorithms as standalone pseudocode in Appendix~\ref{app:algo_pseudocode}.}
    \item In Appendix~\ref{app:lb}, we detail the proofs of our lower bounds for $\epsilon$-local DP (Theorem~\ref{thm:local_lower_bound} and Corollary~\ref{lem:local_more_explicit_lower_bound}) and $\epsilon$-global DP (Theorem~\ref{thm:global_lower_bound} and Corollary~\ref{lem:more_explicit_lower_bound}).
    \item In Appendix~\ref{app:privacy_proof}, we show that \adaptt{} and \adapttt{} are $\epsilon$-global DP since they use the \hyperlink{DAF}{DAF}$(\epsilon)$ estimator of the means.
    \item In Appendix~\ref{app:private_stopping_rule}, we prove that the $\epsilon$-global DP GLR stopping rules yields $\delta$-correctness regardless of the sampling rule, both when using non-private transportation costs (Lemma~\ref{lem:delta_correct_threshold_TTUCB_global_v1}) and adapted transportation costs (Lemma~\ref{lem:delta_correct_threshold_TTUCB_global_v2}).
    \item In Appendix~\ref{app:TTUCB_global_upper_bounds}, we detail the proofs of the asymptotic upper bound on the expected sample complexity of \adaptt{} (Theorem~\ref{thm:sample_complexity_TTUCB_global_v1}) and \adapttt{} (Theorem~\ref{thm:sample_complexity_TTUCB_global_v2}).
    \item In Appendix~\ref{app:rounds_adaptivity}, we discuss in more details the number of rounds of adaptivity.  
    \item Extended experiments are presented in Appendix~\ref{app:experiment}.
\end{itemize}

\section{Notation}\label{app:notation}

We recall some commonly used notation:
the set of integers $[n] \defn \{1, \cdots, n\}$,
the $\ell_{1}$-norm $\|\cdot\|_1$,
the complement $X^{\complement}$ and interior $\mathring X$ of a set $X$,
the indicator function $\indi{X}$ of an event,
the probability $\bP_{\bm \nu}$ and the expectation $\bE_{\bm \nu}$ taken over the randomness of the observations from $\bm \nu$ and the randomness of the algorithm,
Landau's notation $o$, $\cO$, $\Omega$ and $\Theta$,
the $(K-1)$-dimensional probability simplex $\Sigma_K \defn \left\{\omega \in \R_{+}^{K} \mid \omega  \geq 0 , \: \sum_{a \in [K]} \omega_a = 1 \right\}$.
Lap$(b)$ be the Laplace distribution with mean/variance $(0,2b^2)$, and $\textrm{Ber}(p)$ Bernoulli of parameter $p$.
The functions $(x)_{+} = \max\{0,x\}$,
$k(x) \defn \log_{2}(x) + 2$,
$\overline{W}_{-1}$ in Lemma~\ref{lem:property_W_lambert},
$\cC_{G}$ as in Eq.~\eqref{eq:def_C_gaussian_KK18Mixtures},
$\zeta$ is the Riemann $\zeta$ function,
$h$ in Lemma~\ref{lem:delta_correct_threshold_TTUCB_global_v2},
$g_1$ and $g_2$ in Theorem~\ref{thm:sample_complexity_TTUCB_global_v2}.
The concentration parameters $s > 1 $ and $ \alpha > 1$, e.g., $s = \alpha = 1.2$.
Moreover, we recall the definitions:
$T^{\star}_{\mathrm{KL}}(\boldsymbol{\nu})$, $T^{\star}_{\mathrm{TV}}(\boldsymbol{\nu})$, $T^{\star}_{\mathrm{TV}^2}(\boldsymbol{\nu})$ as in Eq.~\eqref{eq:characteristic_time} for the KL divergence, the TV distance and the squared TV distance, 
$H(\bm{\nu}) = 2 \sigma^2 \sum_{a \in [K]}\Delta_{a}^{-2}$, 
$H(\bm{\nu}, \epsilon) = 2 \sigma^2 \sum_{a \in [K]}\Delta_{\epsilon, a}^{-2}$ where $\Delta_{\epsilon,a} \defn  \epsilon  (\sqrt{1 + 2\Delta_{a}/\epsilon} - 1)$,
$\nu_{\epsilon,a}  \defn  \text{Ber}(\mu_{\epsilon,a})$ with $\mu_{\epsilon,a}  \defn  (2 \mu_{a} - 1) \frac{e^{\epsilon}-1}{2(e^{\epsilon}+1)} + 1/2$ for all $a \in [K]$.
The estimator mechanisms $(\text{ESTIMATOR}_{a})_{a \in [K]}$, e.g., \hyperlink{MLE}{MLE} (Alg.~\ref{algo:MLE}), \hyperlink{CTB}{CTB} (Alg.~\ref{algo:CTB}), \hyperlink{DAF}{DAF} (Alg.~\ref{algo:doubling_forgetting}), see Table~\ref{tab:algorithm_design}.
The stopping conditions $(\text{STOP}_{a,b})_{(a,b) \in [K]^2}$ as in Eq.~\eqref{eq:GLR_stopping_rule}. A BAI strategy is denoted by $\pi$, and is composed of a sequence of sampling and stopping rules $S_n$, and recommendation rule $\operatorname{Rec}_n$. The user interacting with the BAI strategy at step $t$ is $u_t$, has potential reward vector $\bm{x}_t$, which are combined in the table of potential rewards $\underline{d} = (\bm{x}_1, \dots)$. For local DP, we denote by $\mech$ the perturbation mechanism that each user $u_t$ uses to send their noisy rewards.
While Table~\ref{tab:notation_table_setting} gathers problem-specific notation, Table~\ref{tab:notation_table_algorithms} groups notation for the algorithms.

\begin{table}
\begin{center}
\begin{tabular}{c c l}
	\toprule
Notation & Type & Description \\
\midrule
$K$ & $\N$ & Number of arms \\
$\cF$ & $\subseteq \cP([0,1])$ & Class of distributions, e.g., Bernoulli \\
$\sigma$ & $\R_{+}^{\star}$ & $\sigma$-sub-Gaussian, e.g., $\sigma=1/2$ for Bernoulli \\
$a$ & $[K]$ & Arm  \\
$\nu_a$ & $\cF$ & Distribution of arm $a \in [K]$ \\
$\bm \nu$ & $\cF^K$ & Vector of distributions, $\nu \defn (\nu_a)_{a \in [K]}$ \\
$\cM$ & $\subseteq \cF^K$ & Class of bandit instances \\
$\mu_a$ & $(0,1)$ & Mean of arm $a \in [K]$, $\mu_a \defn  \bE_{X \sim \nu_a}[X]$ \\
$\mu$ & $(0,1)^K$ & Vector of means, $\mu \defn (\mu_a)_{a \in [K]}$ \\
$\bm \nu_{G}$ &  & Gaussian instance $\mathcal N(\mu,I_{K})$ \\
$a^\star(\bm \nu)$ & $\subseteq [K]$ & Set of best arms, $ a^\star(\bm \nu) \defn  \argmax_{a \in [K]} \mu_a$ \\
$a^\star$ & $[K]$ & Unique best arm, i.e., $ a^\star(\mu) = \{a^\star\}$ \\
$\Delta_a$ & $(0,1)$ & Mean gap arm $a \in [K]$, i.e., $\Delta_a \defn \mu_{a^\star} - \mu_{a}$,\\ 
& & $\Delta_{a^\star} = \Delta_{\min} \defn \min_{a \ne a^\star}\Delta_a$ and $\Delta_{\max} \defn \max_{a \ne a^\star}\Delta_a$ \\
$\epsilon$ & $\R_{+}^{\star}$ & Privacy budget, e.g., $\epsilon$-global or $\epsilon$-local DP \\
$\delta$ & $(0,1)$ & Risk for $\delta$-correctness, i.e., confidence $1-\delta$ \\
$\operatorname{Alt}(\bm \nu)$ & $\subseteq \cM$ & Alternative instances with different best arms \\
$\Sigma_K$ & $\subseteq \R_{+}^{K}$ & $(K-1)$-dimensional probability simplex \\
$\textbf{d}$ & & Measure between distributions, e.g., $\mathrm{KL}$, $\mathrm{TV}$ or $\mathrm{TV}^2$ \\
$\beta$ & $(0,1)$ & Fixed proportion, e.g., $\beta=1/2$ \\
$T^{\star}_{\textbf{d}}(\boldsymbol{\nu})$, $T^{\star}_{\textbf{d}, \beta}(\boldsymbol{\nu})$ & $\R_{+}^{\star}$ & ($\beta$-)Characteristic time for $\bm \nu$ given $\textbf{d}$, see Eq.~\eqref{eq:characteristic_time} \\
$\omega^{\star}_{\textbf{d}}(\boldsymbol{\nu})$, $\omega^{\star}_{\textbf{d}, \beta}(\boldsymbol{\nu})$ & $\Sigma_K$ & ($\beta$-)Optimal allocation, maximiser of Eq.~\eqref{eq:characteristic_time} \\
$T^{\star}_{\ell}\left(\boldsymbol{\nu} ; \epsilon\right)$ & $\R_{+}^{\star}$ & $\epsilon$-Local lower bound for $\bm \nu$, see Theorem~\ref{thm:local_lower_bound}\\
$T^{\star}_{g}\left(\boldsymbol{\nu} ; \epsilon\right)$ & $\R_{+}^{\star}$ & $\epsilon$-Global lower bound for $\bm \nu$, see Theorem~\ref{thm:global_lower_bound} \\
	\bottomrule
\end{tabular}
\end{center}
\caption{Notation for the setting.}
\label{tab:notation_table_setting}
\end{table}

\begin{table}
\begin{center}
\begin{tabular}{c c l}
	\toprule
Notation & Type & Description \\
\midrule
$n$ & $\N$ & Time \\
$a_{n}$ & $[K]$ & Arm sampled at time $n$ \\
$r_{n}$ & $[0,1]$ & Sample observed at the end of time $n$, i.e. $r_{n} \sim \nu_{a_{n}}$ \\
$\cH_n$ &  & History after $n$, $\cH_n = \sigma(a_1,r_1,\cdots, a_n,r_n)$ \\
$\top$ & & Stopping action \\
$\tau_{\delta}$ & $\N$ & Sample complexity (i.e., stopping time) \\
$\hat a_{n}$ & $[K]$ & Arm recommended before time $n$ \\
$\widehat a$ & $[K]$ & Arm recommended when stopping \\
$B_{n}$ & $[K]$ & (UCB) Leader at time $n$ \\
$C_{n}$ & $[K]$ & (TC) Challenger at time $n$ \\
$b_{a}$ & $\N^{K}  \to \R^{\star}_{+}$ & Confidence bonuses for arm $a$, e.g., \\
& & $b^{G}_{a}$ and $b^{G,\epsilon}_{a}$ as in Eq.~\eqref{eq:TTUCB_Gaussian} and~\eqref{eq:TTUCB_global_v1} \\
$W_{a,b}$ & $\R^{K} \times \N^{K}  \to \R^{\star}_{+}$ & Transportation costs for $(a,b)$, e.g., \\ 
& & $W^{G}_{a,b}$ and $W^{G,\epsilon}_{a,b}$ as in Eq.~\eqref{eq:TTUCB_Gaussian} and~\eqref{eq:TTUCB_global_v2}\\
$c_{a,b}$ & $ \N^{K} \times (0,1) \to \R^{\star}_{+}$ & Stopping threshold function for $(a,b)$, e.g.,\\
& & $c^{G}_{a,b}$, $c^{G,\epsilon}_{a,b}$ and $\tilde c^{G,\epsilon}_{a,b}$ as in Eq.~\eqref{eq:threshold_non_private},~\eqref{eq:threshold_TTUCB_global_v1} and~\eqref{eq:threshold_TTUCB_global_v2} \\
$\tilde r_{n}$ & $[0,1]$ & Randomised Response mechanism ($\epsilon$-local DP) \\
$N_{n,a}$ & $\N$ & Number of pulls of arm $a$ before time $n$ \\
$k_{n,a}$ & $\N$ & Current phase of arm $a$ at time $n$ \\
$\tilde N_{k_{n,a},a}$ & $\N$ & Local count of arm $a$ before time $n$ \\
$\hat \mu_{n,a}$ & $\N$ & Non-private estimator based on $\tilde N_{n,a}$ observations \\
$\tilde \mu_{n,a}$ & $\N$ & Private estimator based on $\tilde N_{n,a}$ observations \\
$T_{k}(a)$ & $\N$ & Time $n$ where the arm $a$ changes to phase $k$ \\
$Y_{k, a}$ & $\R$ & Laplace ``noise'' at phase $k$ for arm $a$ ($\epsilon$-global DP)\\
$L_{n,a}$ & $\N$ & Counts of $B_t = a$ before time $n$ \\
$N^{a}_{n,a}$ & $\N$ & Counts of $(B_t, a_t) = (a,a)$ before time $n$ \\
	\bottomrule
\end{tabular}
\end{center}
\caption{Notation for the algorithm.}
\label{tab:notation_table_algorithms}
\end{table}
 
\section{Algorithm Standalone Pseudocodes}
\label{app:algo_pseudocode}

For clarity, we state the algorithms as standalone pseudocodes: TTUCB algorithm~\citep{jourdan2022non} in Algorithm~\ref{algo:TTUCB_full}, \locptt{} algorithm in Algorithm~\ref{algo:CTB_TT_full}, \Modif{$(\epsilon, \gamma)$-DP TTUCB algorithm in Algorithm~\ref{algo:approx_TT_full}}, \adaptt{} algorithm in Algorithm~\ref{algo:AdaPTT_full} and \adapttt{} algorithm in Algorithm~\ref{algo:AdaPTTstar_full}.

\begin{algorithm}[H]
         \caption{TTUCB Algorithm~\citep{jourdan2022non}}
         \label{algo:TTUCB_full}
	\begin{algorithmic}[1]
 		\State {\bfseries Input:} setting parameter $\delta \in (0,1)$, algorithmic hyperparameters $(\beta,s,\alpha) \in (0,1)$, e.g., $(\beta ,s,\alpha)=(1/2,1.2,1.2)$, stopping threshold $(c_{a,b}^{G})_{(a,b) \in [K]^2}$ as in Eq.~\eqref{eq:threshold_non_private}, confidence bonuses $(b_{a}^{G})_{a \in [K]}$ and transportation costs $(W_{a,b}^{G})_{(a,b) \in [K]^2}$ as in Eq.~\eqref{eq:TTUCB_Gaussian}.
        \State {\bfseries Initialisation:} Observe $r_{a} \sim \nu_{a}$ for all $a \in [K]$; Initialise the estimators $\hat \mu_{n,a} = r_{a}$, $N_{n, a} = 1$, $L_{n,a} = N_{n,a}^{a} = 0$ where $n = K+1$ ;
        \For{$n > K$}
			\State Set arm $\hat a_n = \argmax_{a \in [K]} \hat \mu_{n,a}$ ; \Comment{\Modif{Recommendation rule}}
                \If{$W_{\hat a_n,a}^{G}(\hat \mu_{n}, N_{n}) \ge c_{\hat a_n,a}^{G}(N_{n}, \delta)$ for all $a \neq \hat a_n$} \Comment{\Modif{GLR stopping rule}}
                \State Set $a_n = \top$ and \textbf{return} $\hat a_n$ ;
                \EndIf
        	\State Set arm $B_n = \argmax_{a \in [K]} \left\{\hat \mu_{n,a} + b_{a}^{G}(N_{n})\right\}$;  \Comment{\Modif{UCB leader}}
            \State Set arm $C_n = \argmin_{a \ne B_n} W_{B_n,a}^{G}(\hat \mu_n, N_n)$; \Comment{\Modif{TC challenger}}
			\State Set arm $a_n = B_n$ if $N_{n,B_n}^{B_n} \le \beta L_{n+1,B_n}$, and $a_n = C_n$ otherwise; \Comment{\Modif{$\beta$-tracking}}
			\State Pull $a_n$ and observe $r_{n} \sim \nu_{a_n}$; \Comment{\Modif{Sampling rule}}
            \State Set $N_{n+1, a_n} \gets N_{n, a_n} + 1$, $L_{n+1, B_n} \gets L_{n, B_n} + 1$, $N^{B_n}_{n+1, B_n} \gets N^{B_n}_{n, B_n} + \indi{B_n=a_n}$ ; 
            \State Set $\hat \mu_{n+1,a} = N_{n+1,a}^{-1} \sum_{t \in [n]} r_{t} \ind{a_t = a}$ and $n \gets n + 1 $ ; \Comment{\Modif{Update rule}}
        \EndFor
	\end{algorithmic}
\end{algorithm}

\begin{algorithm}[H]
         \caption{\locptt{} Algorithm}
         \label{algo:CTB_TT_full}
	\begin{algorithmic}[1]
 		\State {\bfseries Input:} setting parameters $(\epsilon,\delta) \in \R_{+}^{\star} \times (0,1)$, algorithmic hyperparameters $(\beta,s,\alpha) \in (0,1)$, e.g., $(\beta ,s,\alpha)=(1/2,1.2,1.2)$, stopping threshold $(c_{a,b}^{G})_{(a,b) \in [K]^2}$ as in Eq.~\eqref{eq:threshold_non_private}, confidence bonuses $(b_{a}^{G})_{a \in [K]}$ and transportation costs $(W_{a,b}^{G})_{(a,b) \in [K]^2}$ as in Eq.~\eqref{eq:TTUCB_Gaussian} with $\sigma=1/2$.
        \State {\bfseries Initialisation:} Observe $ \tilde{r}_a \sim \textrm{Ber}\left( \frac{r_{a}(e^{\epsilon} - 1) + 1}{e^{\epsilon} + 1} \right)$ where $r_{a} \sim \nu_{a}$ for all $a \in [K]$; Initialise the estimators $\tilde \mu_{n,a} = \tilde{r}_a$, $N_{n, a} = 1$, $L_{n,a} = N_{n,a}^{a} = 0$ where $n = K+1$ ;
        \For{$n > K$}
			\State Set $\tilde a_n = \argmax_{a \in [K]} \tilde \mu_{n,a}$ ;\Comment{\Modif{Recommendation rule}}
                \If{$W_{\tilde a_n,a}^{G}(\tilde \mu_{n}, N_{n}) \ge c_{\tilde a_n,a}^{G}(N_{n}, \delta)$ for all $a \neq \tilde a_n$} \Comment{\Modif{GLR stopping rule}}
                \State Set arm $a_n = \top$ and \textbf{return} $\tilde a_n$ ;
                \EndIf
        	\State Set arm $B_n = \argmax_{a \in [K]} \left\{\tilde \mu_{n,a} + b_{a}^{G}(N_{n})\right\}$; \Comment{\Modif{UCB leader}}
            \State Set arm $C_n = \argmin_{a \ne B_n} W_{B_n,a}^{G}(\tilde \mu_n, N_n)$; \Comment{\Modif{TC challenger}}
            \State Set arm $a_n = B_n$ if $N_{n,B_n}^{B_n} \le \beta L_{n+1,B_n}$, and $a_n = C_n$ otherwise; \Comment{\Modif{$\beta$-tracking}}
			\State Pull $a_n$, observe and store $\tilde r_{n} \sim \textrm{Ber}\left( \frac{r_{n}(e^{\epsilon} - 1) + 1}{e^{\epsilon} + 1} \right)$ where $r_{n} \sim \nu_{a_n}$;\Comment{\Modif{Sampling rule}}
            \State Set $N_{n+1, a_n} \gets N_{n, a_n} + 1$, $L_{n+1, B_n} \gets L_{n, B_n} + 1$, $N^{B_n}_{n+1, B_n} \gets N^{B_n}_{n, B_n} + \indi{B_n=a_n}$ ; 
            \State Set $\tilde \mu_{n+1,a} = N_{n+1,a}^{-1} \sum_{t \in [n]} \tilde r_{t} \ind{a_t = a}$ and $n \gets n + 1 $; \Comment{\Modif{Update rule}}
        \EndFor
	\end{algorithmic}
\end{algorithm}

\begin{algorithm}[H]
         \caption{$(\epsilon, \gamma)$-DP TTUCB Algorithm, based on Algorithm 1 in~\citet{zheng2020locally}}
         \label{algo:approx_TT_full}
	\begin{algorithmic}[1]
 		\State {\bfseries Input:} setting parameters $(\epsilon,\gamma,\delta) \in \R_{+}^{\star} \times (0,1)$, algorithmic hyperparameters $(\beta,s,\alpha) \in (0,1)$, e.g., $(\beta ,s,\alpha)=(1/2,1.2,1.2)$, stopping threshold $(c_{a,b}^{G})_{(a,b) \in [K]^2}$ as in Eq.~\eqref{eq:threshold_non_private}, confidence bonuses $(b_{a}^{G})_{a \in [K]}$ and transportation costs $(W_{a,b}^{G})_{(a,b) \in [K]^2}$ as in Eq.~\eqref{eq:TTUCB_Gaussian} with $\sigma=1/2$.
        \State {\bfseries Initialisation:} Set $\sigma_{\epsilon, \gamma} \defn \frac{\sqrt{2 \log(1.25/\gamma)}}{\epsilon}$; Observe $\tilde{r}_a \sim  \mathcal{N}(r_a, \sigma_{\epsilon, \gamma}^2)$ where $r_{a} \sim \nu_{a}$ for all $a \in [K]$; Initialise the estimators $\tilde \mu_{n,a} = \tilde{r}_a $, $N_{n, a} = 1$, $L_{n,a} = N_{n,a}^{a} = 0$ where $n = K+1$;
        \For{$n > K$}
			\State Set $\tilde a_n = \argmax_{a \in [K]} \tilde \mu_{n,a}$ ;\Comment{\Modif{Recommendation rule}}
                \If{$W_{\tilde a_n,a}^{G}(\tilde \mu_{n}, N_{n}) \ge c_{\tilde a_n,a}^{G}(N_{n}, \delta)$ for all $a \neq \tilde a_n$}\Comment{\Modif{GLR stopping rule}}
                \State Set arm $a_n = \top$ and \textbf{return} $\tilde a_n$ ;
                \EndIf
        	\State Set arm $B_n = \argmax_{a \in [K]} \left\{\tilde \mu_{n,a} + b_{a}^{G}(N_{n})\right\}$; \Comment{\Modif{UCB leader}}
            \State Set arm $C_n = \argmin_{a \ne B_n} W_{B_n,a}^{G}(\tilde \mu_n, N_n)$; \Comment{\Modif{TC challenger}}
            \State Set arm $a_n = B_n$ if $N_{n,B_n}^{B_n} \le \beta L_{n+1,B_n}$, and $a_n = C_n$ otherwise; \Comment{\Modif{$\beta$-tracking}}
			\State Pull $a_n$, observe and store $\tilde r_{n} \sim \mathcal{N}(r_n , \sigma_{\epsilon, \gamma}^2)$ where $r_{n} \sim \nu_{a_n}$;\Comment{\Modif{Sampling rule}}
                \State Set $N_{n+1, a_n} \gets N_{n, a_n} + 1$, $L_{n+1, B_n} \gets L_{n, B_n} + 1$, $N^{B_n}_{n+1, B_n} \gets N^{B_n}_{n, B_n} + \indi{B_n=a_n}$ ; 
                \State Set $\tilde \mu_{n+1,a} = N_{n+1,a}^{-1} \sum_{t \in [n]} \tilde r_{t} \ind{a_t = a}$ and $n \gets n + 1 $ ;\Comment{\Modif{Update rule}}
        \EndFor
	\end{algorithmic}
\end{algorithm}

\begin{algorithm}[H]
         \caption{\adaptt{} Algorithm}
         \label{algo:AdaPTT_full}
	\begin{algorithmic}[1]
 	\State {\bfseries Input:} setting parameters $(\epsilon,\delta) \in \R_{+}^{\star} \times (0,1)$, algorithmic hyperparameter $\beta \in (0,1)$, e.g., $\beta =1/2$, stopping threshold $(c_{a,b}^{G,\epsilon})_{(a,b) \in [K]^2}$ as in Eq.~\eqref{eq:threshold_TTUCB_global_v1}, confidence bonuses $(b_{a}^{G,\epsilon})_{a \in [K]}$ as in Eq.~\eqref{eq:TTUCB_global_v1} and transportation costs $(W_{a,b}^{G})_{(a,b) \in [K]^2}$ as in Eq.~\eqref{eq:TTUCB_Gaussian} with $\sigma=1/2$.
        \State {\bfseries Initialisation:} Observe $r_{a} \sim \nu_{a}$ for all $a \in [K]$ and draw  $Y_{1,a} \sim \textrm{Lap}\left(1/\epsilon \right)$; Initialise the estimators $\tilde \mu_{n,a} = r_{a} + Y_{1,a}$, $k_{a}=1$, $T_{1}(a) = \arms+1$, $\tilde N_{n, a} = N_{n, a} = 1$, $L_{n,a} = N_{n,a}^{a} = 0$ where $n = K+1$ ;
        \For{$n > K$}
            \If{$N_{n,a} \ge 2 N_{T_{k_{n,a}}(a),a}$} \Comment{\Modif{Per-arm doubling update grid}}
                \State Change phase $k_{n,a} \gets k_{n,a} + 1$ for arm $a$ ;
                \State Set $T_{k_{n,a}}(a) = n$ and $\tilde N_{k_{n,a}, a} =  N_{T_{k_{n,a}}(a), a} - N_{T_{k_{n,a} - 1}(a), a}$ ;
                \State Set $\hat \mu_{k_{n,a},a} = \tilde N_{k_{n,a},a}^{-1} \sum_{t = T_{k_{n,a}-1}(a)}^{T_{k_{n,a}}(a) - 1} r_{t} \ind{a_t = a}$ ;\Comment{\Modif{Forgetting past observations}}
                \State Set $\tilde \mu_{k_{n,a}, a} =  \hat \mu_{k_{n,a},a} +  Y_{k_{n,a},a}$ where $Y_{k_{n,a},a} \sim \text{Lap}((\epsilon\tilde N_{k_{n,a},a} )^{-1} )$ \Comment{\Modif{Private estimator}}
		\EndIf
			\State Set arm $\tilde a_n = \argmax_{a \in [K]} \tilde \mu_{n,a}$ ;\Comment{\Modif{Recommendation rule}}
                \If{$W_{\tilde a_n,a}^{G}(\tilde \mu_{n}, \tilde N_{n}) \ge c_{\tilde a_n,a}^{G,\epsilon}(\tilde N_{n}, \delta)$ for all $a \neq \tilde a_n$}\Comment{\Modif{GLR stopping rule}}
                \State Set $a_n = \top$ and \textbf{return} $\tilde a_n$ ;
                \EndIf
        	\State Set arm $B_n = \argmax_{a \in [K]} \left\{\tilde \mu_{n,a} + b_{a}^{G,\epsilon}(\tilde N_{n})\right\}$; \Comment{\Modif{UCB leader}}
            \State Set arm $C_n = \argmin_{a \ne B_n} W_{B_n,a}^{G}(\tilde \mu_n, N_n)$ ;\Comment{\Modif{TC challenger}}
            \State Set arm $a_n = B_n$ if $N_{n,B_n}^{B_n} \le \beta L_{n+1,B_n}$, and $a_n = C_n$ otherwise; \Comment{\Modif{$\beta$-tracking}}
			\State Pull $a_n$, observe and store $r_{n} \sim \nu_{a_n}$;\Comment{\Modif{Sampling rule}}
                \State Set $N_{n+1, a_n} \gets N_{n, a_n} + 1$, $L_{n+1, B_n} \gets L_{n, B_n} + 1$, $N^{B_n}_{n+1, B_n} \gets N^{B_n}_{n, B_n} + \indi{B_n=a_n}$ and $n \gets n + 1 $ ; 
        \EndFor
	\end{algorithmic}
\end{algorithm}

\begin{algorithm}[H]
         \caption{\adapttt{} Algorithm}
         \label{algo:AdaPTTstar_full}
	\begin{algorithmic}[1]
 	\State {\bfseries Input:} setting parameters $(\epsilon,\delta) \in \R_{+}^{\star} \times (0,1)$, algorithmic hyperparameter $\beta \in (0,1)$, e.g., $\beta =1/2$, stopping threshold $(\tilde c_{a,b}^{G,\epsilon})_{(a,b) \in [K]^2}$ as in Eq.~\eqref{eq:threshold_TTUCB_global_v2}, confidence bonuses $(b_{a}^{G,\epsilon})_{a \in [K]}$ as in Eq.~\eqref{eq:TTUCB_global_v1} and transportation costs $(W_{a,b}^{G,\epsilon})_{(a,b) \in [K]^2}$ as in Eq.~\eqref{eq:TTUCB_global_v2} with $\sigma=1/2$.
        \State {\bfseries Initialisation:} Observe $r_{a} \sim \nu_{a}$ for all $a \in [K]$ and draw  $Y_{1,a} \sim \textrm{Lap}\left(1/\epsilon \right)$; Initialise the estimators $\tilde \mu_{n,a} = r_{a} + Y_{1,a}$, $k_{a}=1$, $T_{1}(a) = \arms+1$, $\tilde N_{n, a} = N_{n, a} = 1$, $L_{n,a} = N_{n,a}^{a} = 0$ where $n = K+1$ ;
        \For{$n > K$}
            \If{$N_{n,a} \ge 2 N_{T_{k_{n,a}}(a),a}$} \Comment{\Modif{Per-arm doubling update grid}}
                \State Change phase $k_{n,a} \gets k_{n,a} + 1$ for arm $a$ ;
                \State Set $T_{k_{n,a}}(a) = n$ and $\tilde N_{k_{n,a}, a} =  N_{T_{k_{n,a}}(a), a} - N_{T_{k_{n,a} - 1}(a), a}$ ;
                \State Set $\hat \mu_{k_{n,a},a} = \tilde N_{k_{n,a},a}^{-1} \sum_{t = T_{k_{n,a}-1}(a)}^{T_{k_{n,a}}(a) - 1} r_{t} \ind{a_t = a}$ ;\Comment{\Modif{Forgetting past observations}}
                \State Set $\tilde \mu_{k_{n,a}, a} =  \hat \mu_{k_{n,a},a} +  Y_{k_{n,a},a}$ where $Y_{k_{n,a},a} \sim \text{Lap}((\epsilon\tilde N_{k_{n,a},a} )^{-1} )$\Comment{\Modif{Private estimator}}
		\EndIf
			\State Set arm $\tilde a_n = \argmax_{a \in [K]} \tilde \mu_{n,a}$ ;\Comment{\Modif{Recommendation rule}}
                \If{$W_{\tilde a_n,a}^{G,\epsilon}(\tilde \mu_{n}, \tilde N_{n}) \ge \tilde c_{\tilde a_n,a}^{G,\epsilon}(\tilde N_{n}, \delta)$ for all $a \neq \tilde a_n$}\Comment{\Modif{GLR stopping rule}}
                \State Set $a_n = \top$ and \textbf{return} $\tilde a_n$ ;
                \EndIf
        	\State Set arm $B_n = \argmax_{a \in [K]} \left\{\tilde \mu_{n,a} + b_{a}^{G,\epsilon}(\tilde N_{n})\right\}$;\Comment{\Modif{UCB leader}}
            \State Set arm $C_n = \argmin_{a \ne B_n} W_{B_n,a}^{G,\epsilon}(\tilde \mu_n, N_n)$ ;\Comment{\Modif{TC challenger}}
            \State Set arm $a_n = B_n$ if $N_{n,B_n}^{B_n} \le \beta L_{n+1,B_n}$, and $a_n = C_n$ otherwise; \Comment{\Modif{$\beta$-tracking}}
			\State Pull $a_n$, observe and store $r_{n} \sim \nu_{a_n}$;\Comment{\Modif{Sampling rule}}
            \State Set $N_{n+1, a_n} \gets N_{n, a_n} + 1$, $L_{n+1, B_n} \gets L_{n, B_n} + 1$, $N^{B_n}_{n+1, B_n} \gets N^{B_n}_{n, B_n} + \indi{B_n=a_n}$ and $n \gets n + 1 $ ; 
        \EndFor
	\end{algorithmic}
\end{algorithm}

\section{Lower Bounds on the Expected Sample Complexity}\label{app:lb}

In this section, we provide the proofs for the sample complexity lower bounds. First, we present the canonical model for BAI to introduce the relevant quantities. Then, we prove the $\epsilon$-local DP sample complexity lower bound. Finally, for global-DP, we first prove an $\epsilon$-global version of the transportation lemma, i.e. Lemma~\ref{lem:chg_env_dp}. Using this lemma, we prove the $\epsilon$-global DP sample complexity lower bound of Theorem~\ref{thm:global_lower_bound}. We also prove the formula expressing the TV characteristic time for Bernoulli instances.

\subsection{Canonical Model for BAI}
Let $\boldsymbol{\nu} \triangleq \{\nu_a : a \in [\arms]\}$ be a bandit instance, consisting of $\arms$ arms with finite means $\{\mu_a\}_{a \in [\arms]}$. Now, we recall the interaction between a BAI strategy $\pi$ and the bandit instance $\nu$ in the Protocol~\ref{prot:bai}. The BAI strategy $\pi$ halts at $\tau$, samples a sequence of actions $\underline{A}^\tau$, and recommends the action $\hat{A}$. Let $\mathbb{P}_{\boldsymbol{\nu}, \pi}$ be the probability distribution over the triplets $(\tau, \underline{A}^\tau, \hat{A})$, when the BAI strategy $\pi$ interacts with the bandit instance $\nu$.

For a fixed $\horizon > 1$, a sequence of actions $\underline{a}^\horizon = (a_1, \dots, a_\horizon) \in [\arms]^\horizon$ and a recommendation $\hat{a} \in [\arms]$, we define the event $E = \{\tau = \horizon, \underline{A}^\tau = \underline{a}^\horizon,  \hat{A} = \hat{a} \}$. We have that

\begin{align*}
    \mathbb{P}_{\boldsymbol{\nu}, \pi}(E) = \int_{\underline{r}^\horizon = (r_1, \dots, r_\horizon) \in \real^\horizon} \pol(\underline{a}^\horizon, \hat{a}, \horizon \mid  \underline{r}^\horizon ) \prod_{t = 1}^\horizon \dd \nu_{a_t}(r_t) dr_t
\end{align*}

where 
\begin{align*}
     \pi(\underline{a}^\horizon, \widehat{a}, \horizon \mid \underline{r}^\horizon) \defn \operatorname{Rec}_{T+1}\left(\widehat{a} \mid \mathcal{H}_{\horizon}\right) \mathrm{S}_{T+1}\left(\top \mid \mathcal{H}_{\horizon} \right) \prod_{t=1}^T \mathrm{~S}_t\left(a_t \mid \mathcal{H}_{t - 1}\right)
\end{align*}
and $\mathcal{H}_t = (a_1, r_1, \dots, a_t, r_t)$.

\textit{Remark on the bandit feedback for global DP.} 
Let $\pi$ be an $\epsilon$-DP BAI strategy. Let $\horizon \geq 1$, $\underline{a}^\horizon \in [\arms]^\horizon$ a sequence sampled actions and $\widehat{a} \in [\arms]$ a recommended actions. This time, let $\underline{r}^\horizon = \{r_1, \dots, r_\horizon \} \in \real^\horizon$ and $\underline{r'}^\horizon \in \real^\horizon$ two neighbouring sequence of rewards, i.e. $\dham(\underline{r}^\horizon, \underline{r'}^\horizon) \defn \sum_{t = 1}^\horizon \ind{r_t \neq r'_t} = 1$ . Consider the table of rewards $\underline{d}^\horizon$ consisting of concatenating $\underline{r}^\horizon$ colon-wise $ \arms$ times, i.e. $\underline{d}^\horizon_{t, i} = \underline{r}^\horizon_t$ for all $i \in [\arms]$ and all $t \in [\horizon]$ . Define $\underline{d'}^\horizon$ similarly with respect to $\underline{r'}^\horizon$.

In this case, by definition of $\pi$, $\underline{d}^\horizon$ and $\underline{d'}^\horizon$, it is direct that 
\begin{equation*}
   \pi(\underline{a}^\horizon, \widehat{a}, \horizon \mid \underline{r}^\horizon) = \pi(\underline{a}^\horizon, \widehat{a}, \horizon \mid \underline{d}^\horizon)
\end{equation*} 
and $\dham(\underline{d}^\horizon, \underline{d'}^\horizon) = 1$.

Which means that
\begin{equation*}
    \pi(\underline{a}^\horizon, \widehat{a}, \horizon \mid \underline{r}^\horizon)  \leq e^\epsilon \pi(\underline{a}^\horizon, \widehat{a}, \horizon \mid \underline{r'}^\horizon).
\end{equation*} 

In other words, \textit{if $\pi$ is $\epsilon$-pure DP for neighbouring table of rewards $\underline{d}^\horizon$, then $\pi$ is also $\epsilon$-pure DP for neighbouring sequence of observed rewards $\underline{r}^\horizon$.}

\textit{Remark on the local DP canonical model.} 
Let $(\mathcal{M}, \pi)$ be a pair of perturbation mechanism and BAI satisfying $\epsilon$-local DP. Let $\boldsymbol{\nu} \triangleq \{\nu_a : a \in [\arms]\}$ be a bandit instance. In the local DP interaction protocol, the BAI strategy $\pi$ only accesses the noisy rewards from the perturbation mechanism, i.e. $z_t \sim \mathcal{M}(r_t)$, where $r_t \sim \nu_{a_t}$. Thus, we can define an environment $\boldsymbol{\nu}^\mathcal{M} \triangleq \{\nu_a^\mathcal{M} : a \in [\arms]\} $ \textit{induced} by the perturbation mechanism, where 
\begin{align*}
    \nu_a^\mathcal{M}(Z) = \int_{r \in \real} \mathcal{M}(Z \mid r) \dd \nu_a(r) dr
\end{align*}
is the marginal over the noisy rewards of arm $a$.

Thus, the local DP canonical model of the interaction between $(\mathcal{M}, \pi)$ and an environment $\boldsymbol{\nu}$ is equivalent to the ``classical'' canonical model between $\pi$ and the induced environment $\boldsymbol{\nu}^\mathcal{M}$.

\subsection{Expected Sample Complexity Lower Bound under \texorpdfstring{$\epsilon$}{}-local DP}\label{app:loc_lb_proof}

\begin{reptheorem}{thm:local_lower_bound}[Sample complexity lower bound for $\epsilon$-local DP FC-BAI]
Let $\delta \in (0,1) $ and $\epsilon>0$. 
For any $\delta$-correct $\epsilon$-local DP pair $(\mathcal{M}, \pi)$ of perturbation mechanism and BAI strategy, we have that $\mathbb{E}_{\boldsymbol{\nu}}[\tau_{\delta}] \geq T^{\star}_{\ell}\left(\boldsymbol{\nu} ; \epsilon\right) \log (1 / (2.4 \delta))$ with
\begin{align*}
    T^{\star}_{\ell}\left(\boldsymbol{\nu} ; \epsilon\right)^{-1} \defn \sup\limits_{\omega \in \Sigma_K} \inf\limits_{\boldsymbol{\lambda} \in \operatorname{Alt}(\boldsymbol{\nu})}&\sum_{a \in [K]} \omega_a  \min  \left\{\KL{\nu_a}{\lambda_a}, c(\epsilon) \left(\TV{\nu_a}{\lambda_a} \right)^2 \right\} \: ,
\end{align*}
where $c(\epsilon) = \min\{4, e^{2 \epsilon} \} \left(e^\epsilon - 1 \right)^2$ is a privacy term.
For two probability distributions $\mathbb{P}, \mathbb{Q}$ on the measurable space $(\Omega, \mathcal{F})$, the TV divergence is $ \TV{\mathbb{P}}{\mathbb{Q}} \triangleq \sup_{A \in \mathcal{F}} \{\mathbb{P}(A) - \mathbb{Q}(A)\}$.
\end{reptheorem}

\begin{proof}
    Let $(\mathcal{M}, \pi)$ a perturbation mechanism and BAI strategy pair that it $\epsilon$-local DP.

    We suppose that $\pi$ is $\delta$-correct.
    
    Using the remark in the local DP canonical model, $\pi$ is $\delta$-correct with respect to the environment $\boldsymbol{\nu}^\mathcal{M} \triangleq \{\nu_a^\mathcal{M} : a \in [\arms]\} $ induced by the perturbation mechanism $\mathcal{M}$, where 
    \begin{align*}
        \nu_a^\mathcal{M}(Z) = \int_{r \in \real} \mathcal{M}(Z \mid r) \dd \nu_a(r) dr
    \end{align*}
    is the marginal over the noisy rewards of arm $a$.

    Thus using Lemma 1 in~\cite{kaufmann2016complexity}, we get that
    \begin{align*}
        \sum_{a = 1}^\arms \expect \left [N_a(\tau) \right] \KL{\nu^\mathcal{M}_{a}}{\lambda^\mathcal{M}_{a}} \geq \mathrm{kl}(1-\delta, \delta)
    \end{align*}

    for any alternative environment $\boldsymbol{\lambda} \in \operatorname{Alt}(\boldsymbol{\nu})$.

    Using Theorem 1 in~\cite{duchi2013local}, we have that
    \begin{align*}
        \KL{\nu^\mathcal{M}_{a}}{\lambda^\mathcal{M}_{a}} &\leq \KL{\nu^\mathcal{M}_{a}}{\lambda^\mathcal{M}_{a}} + \KL{\lambda^\mathcal{M}_{a}}{\nu^\mathcal{M}_{a}}\\
        &\leq \min\{4, e^{2 \epsilon}\} (e^\epsilon - 1)^2 (\TV{\nu_a}{\lambda_a})^2\\
        &= c(\epsilon) (\TV{\nu_a}{\lambda_a})^2
    \end{align*}
    where $c(\epsilon) \defn \min\{4, e^{2 \epsilon}\} (e^\epsilon - 1)^2$.
    
    On the other hand, using the data-processing inequality, we also have that
    \begin{align*}
        \KL{\nu^\mathcal{M}_{a}}{\lambda^\mathcal{M}_{a}} \leq  \KL{\nu_{a}}{\lambda_{a}}
    \end{align*}

    Thus, combining the two inequalities gives that
    \begin{align*}
    \mathrm{kl}(1- \delta, \delta)  &\leq \inf_{\boldsymbol{\lambda} \in \operatorname{Alt}(\boldsymbol{\nu})} \sum_{a = 1}^\arms \expect \left [N_a(\tau) \right]  \min \left \{\KL{\nu_a}{\lambda_a}  , c(\epsilon)(\TV{\nu_a}{\lambda_a})^2 \right \}\\
    &= \expect[\tau] \inf _{\boldsymbol{\lambda} \in \operatorname{Alt}(\boldsymbol{\nu})} \sum_{a = 1}^\arms \frac{\expect \left [N_a(\tau) \right]}{ \expect[\tau]}  \min \left \{\KL{\nu_a}{\lambda_a}  , c(\epsilon)(\TV{\nu_a}{\lambda_a})^2 \right \}\\
    &\leq \expect[\tau] \left(\sup _{\omega \in \Sigma_K} \inf _{\boldsymbol{\lambda} \in \operatorname{Alt}(\boldsymbol{\nu})} \sum_{a = 1}^\arms \omega_a  \min \left \{\KL{\nu_a}{\lambda_a}  , c(\epsilon)(\TV{\nu_a}{\lambda_a})^2 \right \}\right)\\
    &= \expect[\tau] (T^{\star}_\ell \left(\boldsymbol{\nu}, \epsilon \right))^{-1}
\end{align*}

    The theorem follows by noting that for $\delta \in(0,1), \mathrm{kl}(1-\delta, \delta) \geq \log (1 / 2.4 \delta)$.
\end{proof}

\subsection{Expected Sample Complexity Lower Bound under \texorpdfstring{$\epsilon$}{}-global DP}

\subsubsection{Upper Bound on the KL as a Transport Problem}\label{app:kl_trans}

First, we recall the result that conditioning increases the KL.
\begin{theorem}[Conditioning Increases the KL]\label{thm:cond_incr}
 Let $P_X \stackrel{P_{Y \mid X}}{\longrightarrow} P_Y$ and $P_X \stackrel{Q_{Y \mid X}}{\longrightarrow} Q_Y$.
 Then
    $$
 \mathrm{KL}\left(P_Y \| Q_Y\right) \leq \mathbb{E}_{X \sim P_X}\left[\mathrm{KL}\left(P_{Y \mid X} \| Q_{Y \mid X}\right)\right] .
    $$
\end{theorem}
We apply this result to our problem of upper-bounding the marginals of DP mechanisms.
\begin{reptheorem}{thm:kl_bound}[KL Upper Bound as a Transport Problem] If $\mech$ is $\epsilon$-pure DP,  
\begin{equation*}
 \KLL{M_1}{M_2} \leq \epsilon \inf_{\cC \in \Pi(\cP_1, \cP_2)} \expect_{(D,D') \sim \cC} [\dham(D, D')] \:
\end{equation*}
where $\Pi(\cP_1, \cP_2)$ is the set of all couplings between $\cP_1$ and $\cP_2$.
\end{reptheorem}

\begin{proof}
    Let $\mathcal{C} \in \Pi(\cP_1, \cP_2)$ a coupling between $\cP_1$ and $\cP_2$.\\
    Then, we re-write:
    \begin{align*}
        M_1(A) &\deffn \int_{D \in \CX^n } \mech_D\left(A \right) \dd\cP_1\left(D\right) \\
        &= \int_{D, D' \in \CX^n } \mech_D\left(A \right) \dd\mathcal{C}\left(D, D'\right)
    \end{align*}
    and
    \begin{align*}
        M_2(A) &\deffn \int_{D' \in \CX^n } \mech_{D'}\left(A \right) \dd\cP_2\left(D'\right)\\
        &= \int_{D, D' \in \CX^n } \mech_{D'}\left(A \right) \dd\mathcal{C}\left(D, D'\right)
    \end{align*}
    by the definition of the coupling.\\
    Then, using Theorem~\ref{thm:cond_incr}, we get
    \begin{align*}
        \KLL{M_1}{M_2} \leq \expect_{(D,D') \sim \cC} \left[ \KLL{\mathcal{M}_D}{\mathcal{M}_{D'}} \right]
    \end{align*}
    Finally, using group privacy, we have
    \begin{align*}
        \KLL{\mathcal{M}_D}{\mathcal{M}_{D'}} \leq \epsilon \dham(D, D')
    \end{align*}
    Combining the last two inequalities gives
    \begin{align*}
        \KLL{M_1}{M_2} \leq \epsilon \expect_{(D,D') \sim \cC} [\dham(D, D')]
    \end{align*}
    And since this is true for any coupling $\mathcal{C}$, taking the infimum over couplings concludes the proof.
\end{proof}

\subsubsection{Sequential KL Decomposition for Bandits under DP}

Now, we adapt Theorem~\ref{thm:kl_deco_prod} for the bandit marginals. First, for simplicity, we start with the setting where $T$ the number of rounds in the bandit interaction is fixed.

Let $\nu = \{P_a, a \in [\arms] \}$ and $\nu' =  \{P'_a, a \in [\arms] \}$ be two bandit instances. We recall that, when the policy $\pol$ interacts with the bandit instance $\nu$, it induces a marginal distribution $\mathbb{M}_{\nu \pi}$ over the sequence of actions, where
\begin{equation*}
 m_{\nu \pi} (a_1, \dots, a_\horizon)
 \deffn \int_{r_1, \dots, r_\horizon} \prod_{t=1}^\horizon \pol_t(a_t \mid H_{t-1} ) p_{a_t} (r_t) \dd r_t .
\end{equation*}
and for all $C \in \mathcal{P}([\arms]^\horizon)$,
\begin{equation*}
 \mathbb{M}_{\model \pi} (C) \deffn \sum_{(a_1, \dots, a_\horizon) \in C} m_{\model \pol}(a_1, a_2, \dots, a_\horizon).
\end{equation*}

We define $\mathbb{M}_{\nu' \pi}$ similarly.

The goal is to upper bound the quantity $\KLL{\mathbb{M}_{\model \pi}}{\mathbb{M}_{\model' \pi}}$. The marginals $\mathbb{M}_{\model \pi}$ and $\mathbb{M}_{\model \pi}$ in the sequential setting "look like" marginals generated by "product distributions". However, the hardness of the sequential setting resides in the fact that the data-generating distributions depend on the actions chosen, which are stochastic. Thus, the results of the previous section cannot directly be applied. To adapt the proof ideas of the previous section to the bandit case, we introduce the idea of a coupled bandit instance.

Let $\nu = \{P_a: a \in [K] \}$ and $\nu' = \{P'_a: a \in [K] \}$ be two bandit instances. Define $c_a$ as the maximal coupling between $P_a$ and $P'_a$. Fix a policy $\pol = \{\pol_t \}_{t = 1}^\horizon$. 

Here, we build a coupled environment $\gamma$ of $\nu$ and $\nu'$. The policy $\pol$ interacts with the coupled environment $\gamma$ up to a given time horizon $\horizon$ to produce an augmented history $\lbrace (a_t, r_t, r'_t)\rbrace_{t=1}^{\horizon}$. The iterative steps of this interaction process are:
\begin{enumerate}
\item[1.] The probability of choosing an action $a_t = a$ at time $t$ is dictated only by the policy $\pol_t$ and $a_1, r_1, a_2, r_2, \dots, a_{t-1}, r_{t-1}$, \ie the policy ignores $\{r'_s\}_{s = 1}^{t - 1}$.
\item[2.] The distribution of rewards $(r_t, r'_t)$ is $c_{a_t}$ and is conditionally independent of the previous observed history $\lbrace (a_s, r_s, r'_s)\rbrace_{t=1}^{t - 1}$.
\end{enumerate}

This interaction is similar to the interaction process of policy $\pol$ with the first bandit instance $\nu$, with the addition of sampling an extra $r'_t$ from the coupling of $P_{a_t}$ and $P'_{a_t}$. This, in essence, corresponds to the "up" branch in Theorem~\ref{thm:cond_incr}.

The distribution of the augmented history induced by the interaction of $\pol$ and the coupled environment can be defined as
\begin{align*}
     &p_{\gamma \pol}(a_1 , r_1 , r'_1 \dots , a_\horizon , r_\horizon, r'_\horizon ) \deffn \prod_{t=1}^\horizon \pol_t(a_t \mid a_1 , r_1 , \dots , a_{t-1} , r_{t-1} ) c_{a_t} (r_t, r'_t)
\end{align*}

To simplify the notation, let $\textbf{a} \deffn (a_1, \dots, a_\horizon)$, $\textbf{r} \deffn (r_1, \dots, r_\horizon)$ and $\textbf{r'} \deffn (r'_1, \dots, r'_\horizon)$. Also, let $c_{\textbf{a}}(\textbf{r}, \textbf{r'}) \deffn \prod_{t = 1}^\horizon c_{a_t}(r_t, r'_t)$ and $\pi(\textbf{a} \mid \textbf{r}) \deffn \prod_{t=1}^\horizon \pol_t(a_t \mid a_1 , r_1 , \dots , a_{t-1} , r_{t-1} )$. We put $\textbf{h} \deffn (\textbf{a}, \textbf{r} , \textbf{r'} )$.

With the new notation 
\begin{equation*}
 p_{\gamma \pol}(\textbf{a}, \textbf{r} , \textbf{r'} ) \deffn \pi(\textbf{a} \mid \textbf{r}) c_{\textbf{a}}(\textbf{r}, \textbf{r'})
\end{equation*}

Similarly, we define
\begin{equation*}
 q_{\gamma \pol}(\textbf{a}, \textbf{r} , \textbf{r'} ) \deffn \pi(\textbf{a} \mid \textbf{r'}) c_{\textbf{a}}(\textbf{r}, \textbf{r'})
\end{equation*}
which corresponds to the "down" branch in Theorem~\ref{thm:cond_incr}, where the policy ignores the rewards $r_1, \dots, r_T$ in the interaction.

It follows that $m_{\nu \pol}$ is the marginal of $p_{\gamma \pol}$ when integrated over $(\textbf{r}, \textbf{r'})$, and $m_{\nu' \pol}$ is the marginal of $q_{\gamma \pol}$ when integrated over $(\textbf{r}, \textbf{r'})$, \ie
\begin{equation*}
 m_{\nu \pol}( \textbf{a} ) = \int_{ \textbf{r}, \textbf{r'} } p_{\gamma \pol}(\textbf{a}, \textbf{r} , \textbf{r'} ) \dd \textbf{r} \dd \textbf{r'} 
\end{equation*}
and
\begin{equation*}
 m_{\nu' \pol}(\textbf{a}) = \int_{ \textbf{r}, \textbf{r'} } q_{\gamma \pol}(\textbf{a}, \textbf{r} , \textbf{r'} ) \dd \textbf{r} \dd \textbf{r'}.
\end{equation*}

By the data-processing inequality, we get that
\begin{equation*}
 \KLL{\mathbb{M}_{\model \pi}}{\mathbb{M}_{\model' \pi}} \leq \KLL{p_{\gamma \pol}}{q_{\gamma \pol}}
\end{equation*}

We have that
\begin{align*}
 \KLL{p_{\gamma \pol}}{q_{\gamma \pol}} & \stackrel{(a)}{=} \expect_{\textbf{h} \deffn (\textbf{a}, \textbf{r} , \textbf{r'}) \sim p_{\gamma \pol} } \left[ \log \left( \frac{\pi(\textbf{a} \mid \textbf{r}) c_{\textbf{a}}(\textbf{r}, \textbf{r'})}{\pi(\textbf{a} \mid \textbf{r'}) c_{\textbf{a}}(\textbf{r}, \textbf{r'})} \right) \right]\\
    &= \expect_{\textbf{h} \deffn (\textbf{a}, \textbf{r} , \textbf{r'}) \sim p_{\gamma \pol} } \left[ \log \left( \frac{\pi(\textbf{a} \mid \textbf{r}}{\pi(\textbf{a} \mid \textbf{r'}} \right) \right]
\end{align*}
where:
(a): by definition of $p_{\gamma \pol}$, $q_{\gamma \pol} $ and the KL divergence

\noindent \textbf{Global DP.} Now, if the policy $\pi$ is $\epsilon$-global DP, then by group privacy $ \pi(\textbf{a} \mid \textbf{r}) \leq e^{\epsilon \dham(\textbf{r}, \textbf{r'})} \pi(\textbf{a} \mid \textbf{r'})$ for any sequence of actions, and any two sequence of rewards $ \textbf{r}$ and $\textbf{r'} $. Thus, computing the expectation of $\dham(\textbf{r}, \textbf{r'})$ when $\textbf{r}$ and $\textbf{r'}$ are generated through the coupled environment provides the following theorem.

\begin{theorem}[KL Decomposition for $\epsilon$-global DP]\label{thm:kl_decompo_eps}
 If $\pol$ is $\epsilon$-global DP, then
    \begin{align*}
 \KLL{\mathbb{M}_{\model \pi}}{\mathbb{M}_{\model' \pi}} &\leq \epsilon\expect_{\nu \pi} \left(\sum_{t = 1}^\horizon t_{a_t} \right),
    \end{align*}
 where $t_{a_t} \deffn \TV{P_{a_t}}{P'_{a_t}}$ and $\expect_{\nu \pi}$ is the expectation under $m_{\nu \pol}$.
\end{theorem}

\begin{proof}
 The proof follows by computing
    \begin{align*}
 \expect_{\textbf{h} \deffn (\textbf{a}, \textbf{r} , \textbf{r'}) \sim p_{\gamma \pol} } \left[ \log \left( \frac{\mathcal{V}^\pol_{\textbf{r}}(\textbf{a})}{\mathcal{V}_{\textbf{r'}}(\textbf{a})} \right) \right] &\stackrel{(a)}{\leq} \expect_{\textbf{h} \deffn (\textbf{a}, \textbf{r} , \textbf{r'}) \sim p_{\gamma \pol} } \left[ \epsilon \dham(\textbf{r}, \textbf{r'}) \right]\\
        &\stackrel{(b)}{=} \epsilon \sum_{t = 1}^T \expect_{\textbf{h} \deffn (\textbf{a}, \textbf{r} , \textbf{r'}) \sim p_{\gamma \pol} }  \left[ \ind{r_t \neq r'_t} \right]\\
        &\stackrel{(c)}{=} \epsilon \sum_{t = 1}^T \expect_{\textbf{h} \deffn (\textbf{a}, \textbf{r} , \textbf{r'}) \sim p_{\gamma \pol} }  \left[ \expect_{\textbf{h} \deffn (\textbf{a}, \textbf{r} , \textbf{r'}) \sim p_{\gamma \pol} } \left[ \ind{r_t \neq r'_t} | a_t \right] \right]\\
        &\stackrel{(d)}{=} \epsilon \sum_{t = 1}^T \expect_{\textbf{h} \deffn (\textbf{a}, \textbf{r} , \textbf{r'}) \sim p_{\gamma \pol} }  \left[ t_{a_t} \right]\\
        &\stackrel{(e)}{=} \epsilon \sum_{t = 1}^T \expect_{\textbf{a} \sim m_{\nu \pol} }  \left[ t_{a_t} \right]
    \end{align*}

 where:
 (a) is by group privacy, (b) is by the definition of the hamming distance, (c) is by the towering property of the expectation, (d) is by the definition of the maximal coupling and (e) is because the sum only depends on the sequence of actions, with marginal distribution $m_{\nu \pol}$.
\end{proof}

\noindent \textbf{Comparison to the product distribution setting:} The result of Theorem~\ref{thm:kl_decompo_eps} generalises the result of Theorem~\ref{thm:kl_deco_prod} to the sequential setting under pure DP. Since the actions are stochastic, there is an additional expectation over the generation process of the sequence of actions, sampled from $\mathbb{M}_{\model \pi}$. Also, Theorem~\ref{thm:kl_decompo_eps} can be seen as an $\epsilon$-DP version of the general KL decomposition lemma (Exercice 15.8, (b) in~\cite{lattimore2018bandit}), which recall states that $\KLL{\mathbb{P}_{\nu \pi}}{\mathbb{P}_{\nu' \pi}} = \expect_{\nu \pi} \left(\sum_{t = 1}^\horizon \KLL{P_{a_t}}{P'_{a_t}} \right)$.

\begin{remark}[Improvement of a factor of $6$ for~\citealt{azize2022privacy}]
 Theorem\\ of~\cite{azize2022privacy} states $$\KLL{\mathbb{M}_{\model \pi}}{\mathbb{M}_{\model' \pi}} \leq 6 \epsilon\expect_{\nu \pi} \left(\sum_{t = 1}^\horizon t_{a_t} \right).$$ 
 We used a generalisation of Karwa Vadhan lemma to prove this result for product distributions. Using the coupled environment idea in the new proof, we eliminate the extra $6$ factor in the upper bound. This improves all the regret and sample complexity lower bounds in this manuscript by a factor of $6$ compared to the results in~\cite{azize2022privacy,azize2023PrivateBAI}.
 \end{remark}
 
\begin{remark}[Stopping time version of the KL decomposition under DP]\label{rmk:kl_bai}
 Let $\pi$ \\be a DP BAI strategy. Let $\nu$ and $\lambda$ be two bandit instances. Denote by $\mathbb{M}_{\boldsymbol{\nu}, \pi}$ the marginal distribution of $(\underline{A}, \widehat{A}, \tau)$ , i.e. the sequence of action, final recommendation and stopping time when the BAI strategy $\pi$ interacts with $\nu$. By adapting the proof technique seen before for the canonical bandit setting under FC-BAI, we get that
    \begin{align*}
 \KLL{\mathbb{M}_{\model \pi}}{\mathbb{M}_{\lambda \pi}} \leq
\epsilon \sum_{a = 1}^\arms \expect_{\boldsymbol{\nu}, \pi}  \left [  N_a(\tau) \right] \TV{P_{a}}{P'_{a}} 
    \end{align*}
 where $\tau$ is the stopping time, $t_{a_t} \deffn \TV{P_{a_t}}{P'_{a_t}}$ and $\expect_{\nu \pi}$ is the expectation under $m_{\nu \pol}$.
\end{remark}

\begin{proof}
    For completeness, we recall the canonical bandit setting under FC-BAI, using the notation of Remark~\ref{rmk:kl_bai}, and then adapt the coupling technique to FC-BAI. The main difference with the previous section is that the number of interactions $\tau$ is now a random variable, and we use Wald's lemma to deal with that.

    \textbf{Step 1: Canonical Model under FC-BAI:} Let $\boldsymbol{\nu} \triangleq \{P_a : a \in [\arms]\}$ be a bandit instance, consisting of $\arms$ arms with finite means $\{\mu_a\}_{a \in [\arms]}$. Now, we recall the interaction between a BAI strategy $\pi$ and the bandit instance $\nu$ in the Protocol~\ref{prot:bai}. The BAI strategy $\pi$ halts at $\tau$, samples a sequence of actions $\underline{A}^\tau \defn (a_1, \dots, a_\tau)$, and recommends the action $\hat{A}$. Let $\mathbb{M}_{\boldsymbol{\nu}, \pi}$ be the marginal distribution over the output $(\tau, \underline{A}^\tau, \hat{A})$, when the BAI strategy $\pi$ interacts with the bandit instance $\nu$. Then,

\begin{align*}
    m_{\boldsymbol{\nu}, \pi}(\tau = \horizon, \underline{A}^\tau = \underline{a}^\horizon,  \hat{A} = \hat{a} ) = \int_{\underline{r}^\horizon = (r_1, \dots, r_\horizon) \in \real^\horizon} \pol(\underline{a}^\horizon, \hat{a}, \horizon \mid  \underline{r}^\horizon ) \prod_{t = 1}^\horizon p_{a_t}(r_t) dr_t
\end{align*}

where 
\begin{align*}
     \pi(\underline{a}^\horizon, \widehat{a}, \horizon \mid \underline{r}^\horizon) \defn \operatorname{Rec}_{T+1}\left(\widehat{a} \mid \mathcal{H}_{\horizon}\right) \mathrm{S}_{T+1}\left(\top \mid \mathcal{H}_{\horizon} \right) \prod_{t=1}^T \mathrm{~S}_t\left(a_t \mid \mathcal{H}_{t - 1}\right)
\end{align*}
and $\mathcal{H}_t = (a_1, r_1, \dots, a_t, r_t)$.

    \textbf{Step 2: Coupling technique applied to FC-BAI marginals.} Let $\nu = \{P_a: a \in [K] \}$ and $\lambda = \{P'_a: a \in [K] \}$ be two bandit instances. Define $c_a$ as the maximal coupling between $P_a$ and $P'_a$. Here, we build again a coupled environment $\gamma$ of $\nu$ and $\lambda$. The BAI strategy $\pol$ interacts with the coupled environment $\gamma$, decides to halt at step $\tau$ to produce an augmented history $(\tau, a_1, r_1, r'_1, \dots, a_\tau, r_\tau, r'_\tau, \hat a)$. The iterative steps of this interaction process are:
\begin{enumerate}
\item[1.] The probability of choosing an action $a_t = a$ at time $t$ is dictated only by the sampling rule $S_t$ and $\mathcal{H}_{t - 1} \defn (a_1, r_1, a_2, r_2, \dots, a_{t-1}, r_{t-1})$, \ie the sampling rule ignores $\{r'_s\}_{s = 1}^{t - 1}$.
\item[2.] The distribution of rewards $(r_t, r'_t)$ is $c_{a_t}$ and is conditionally independent of the previous observed history $\lbrace (a_s, r_s, r'_s)\rbrace_{t=1}^{t - 1}$.
\item [3.] The stopping rule and recommendation rule only depend on the history $\mathcal{H}_{t - 1}$, the stopping and recommendation rules ignore $\{r'_s\}_{s = 1}^{t - 1}$
\end{enumerate}

The distribution of the augmented history induced by the interaction of $\pol$ and the coupled environment can be defined as
\begin{align*}
     p_{\gamma \pol}(\tau, a_1 , r_1 , r'_1 \dots , a_\tau , r_\tau, r'_\tau, \hat a ) &\deffn \operatorname{Rec}_{\tau+1}\left(\widehat{a} \mid \mathcal{H}_{\tau}\right) \mathrm{S}_{\tau+1}\left(\top \mid \mathcal{H}_{\tau} \right) \prod_{t=1}^\tau \mathrm{~S}_t\left(a_t \mid \mathcal{H}_{t - 1}\right) c_{a_t} (r_t, r'_t) \\
     &= \pol(\underline{a}^\tau, \hat{a}, \tau \mid  \underline{r}^\tau ) \prod_{t = 1}^\tau c_{a_t} (r_t, r'_t)
\end{align*}
where $\tau \in \mathbb{N}$, $a_t \in [K]$ and $r_t \in \real$ are fixed.

To simplify the notation, let $\underline{r}^\tau \deffn (r_1, \dots, r_\tau)$, $\underline{r'}^\tau \deffn (r'_1, \dots, r'_\tau)$ and $c_{\underline{a}^\tau}(\underline{r}^\tau, \underline{r'}^\tau) \deffn \prod_{t = 1}^\tau c_{a_t}(r_t, r'_t)$. We put $\textbf{h} \deffn (\tau, \underline{a}^\tau, \underline{r}^\tau , \underline{r'}^\tau, \hat a)$.

With the new notation 
\begin{equation*}
 p_{\gamma \pol}(\tau, \underline{a}^\tau, \underline{r}^\tau , \underline{r'}^\tau, \hat a) \deffn \pol(\underline{a}^\tau, \hat{a}, \tau \mid  \underline{r}^\tau ) c_{\underline{a}^\tau}(\underline{r}^\tau, \underline{r'}^\tau)
\end{equation*}

Similarly, we define
\begin{equation*}
  q_{\gamma \pol}(\tau, \underline{a}^\tau, \underline{r}^\tau , \underline{r'}^\tau, \hat a) \deffn \pol(\underline{a}^\tau, \hat{a}, \tau \mid  \underline{r'}^\tau ) c_{\underline{a}^\tau}(\underline{r}^\tau, \underline{r'}^\tau)
\end{equation*}
which corresponds to the augmented history interaction, where the policy ignores the rewards $(r_1, \dots, r_t, \dots)$ in the interaction.

It follows that $m_{\nu \pol}$ and $m_{\lambda \pol}$ are the marginals of $p_{\gamma \pol}$ and $q_{\gamma \pol}$ when integrated over the rewards, \ie 
\begin{equation*}
 m_{\nu \pol}( \tau, \underline{a}^\tau,  \hat a ) = \int_{\underline{r}^\tau , \underline{r'}^\tau} p_{\gamma \pol}(\underline{r}^\tau , \underline{r'}^\tau) \prod_{t = 1}^\tau \dd r_t \dd r'_t 
\end{equation*}
and
\begin{equation*}
 m_{\lambda \pol}( \tau, \underline{a}^\tau,  \hat a ) = \int_{\underline{r}^\tau , \underline{r'}^\tau} q_{\gamma \pol}(\underline{r}^\tau , \underline{r'}^\tau) \prod_{t = 1}^\tau \dd r_t \dd r'_t 
\end{equation*}

By the data-processing inequality, we get that
\begin{equation*}
 \KLL{\mathbb{M}_{\model \pi}}{\mathbb{M}_{\lambda \pi}} \leq \KLL{p_{\gamma \pol}}{q_{\gamma \pol}}
\end{equation*}

On the other hand, we have
\begin{align*}
 \KLL{p_{\gamma \pol}}{q_{\gamma \pol}} & \stackrel{(a)}{=} \expect_{\textbf{h} \deffn (\tau, \underline{a}^\tau, \underline{r}^\tau , \underline{r'}^\tau, \hat a) \sim p_{\gamma \pol} } \left[ \log \left( \frac{\pol(\underline{a}^\tau, \hat{a}, \tau \mid  \underline{r}^\tau ) c_{\underline{a}^\tau}(\underline{r}^\tau, \underline{r'}^\tau)}{\pol(\underline{a}^\tau, \hat{a}, \tau \mid  \underline{r'}^\tau ) c_{\underline{a}^\tau}(\underline{r}^\tau, \underline{r'}^\tau)} \right) \right]\\
    &= \expect_{\textbf{h} \deffn (\tau, \underline{a}^\tau, \underline{r}^\tau , \underline{r'}^\tau, \hat a) \sim p_{\gamma \pol} } \left[ \log \left( \frac{\pol(\underline{a}^\tau, \hat{a}, \tau \mid  \underline{r}^\tau )}{\pol(\underline{a}^\tau, \hat{a}, \tau \mid  \underline{r'}^\tau )} \right) \right]\\
    &\stackrel{(b)}{\leq} \expect_{\textbf{h} \deffn (\tau, \underline{a}^\tau, \underline{r}^\tau , \underline{r'}^\tau, \hat a) \sim p_{\gamma \pol} } \left[ \epsilon \dham(\underline{r}^\tau , \underline{r'}^\tau) \right]\\
    &\stackrel{(c)}{=} \epsilon \expect_{\textbf{h} \deffn (\tau, \underline{a}^\tau, \underline{r}^\tau , \underline{r'}^\tau, \hat a) \sim p_{\gamma \pol} } \left[ \sum_{t = 1}^\tau   \ind{r_t \neq r'_t} \right]\\
    &= \epsilon \expect_{\textbf{h} \deffn (\tau, \underline{a}^\tau, \underline{r}^\tau , \underline{r'}^\tau, \hat a) \sim p_{\gamma \pol} } \left[ \sum_{a = 1}^K \sum_{s = 1}^{N_a(\tau)} \ind{r_{s,a} \neq r'_{s,a}} \right]\\
    &\stackrel{(d)}{=} \epsilon  \sum_{a = 1}^K \expect_{\textbf{h} \deffn (\tau, \underline{a}^\tau, \underline{r}^\tau , \underline{r'}^\tau, \hat a) \sim p_{\gamma \pol} } [N_a(\tau)] \; \TV{P_{a}}{P'_{a}} \\
    & \stackrel{(e)}{=} \epsilon \sum_{a = 1}^\arms \expect_{\boldsymbol{\nu}, \pi}  \left [  N_a(\tau) \right] \TV{P_{a}}{P'_{a}} 
    \end{align*}

where:
(a): by definition of $p_{\gamma \pol}$, $q_{\gamma \pol} $ and the KL divergence
 (b) is by group privacy, (c) is by the definition of the hamming distance, (d) is using Wald's lemma and the definition of the maximal coupling, (e) is because $N_a(\tau)$ does not depend on the sequence of rewards $(r'_t)$, when $\textbf{h}$ is generated through $p_{\gamma, \pol}$.

\end{proof}

\subsubsection{Transportation Lemma under \texorpdfstring{$\epsilon$}{}-global DP: Proof of Lemma~\ref{lem:chg_env_dp}}\label{app:chg_env_proof}

\begin{replemma}{lem:chg_env_dp}[Transportation lemma under $\epsilon$-global DP]
Let $\delta \in (0,1) $ and $\epsilon>0$. Let $\boldsymbol{\nu}$ be a bandit instance and $\lambda \in \operatorname{Alt}(\boldsymbol{\nu})$. For any $\delta$-correct $\epsilon$-global DP BAI strategy, we have that
    \begin{equation*}
        \epsilon  \sum_{a = 1}^\arms \expect_{\boldsymbol{\nu}, \pi} \left [N_a(\tau) \right] \TV{\nu_{a}}{\lambda_{a}}  \geq \mathrm{kl}(1-\delta, \delta),
    \end{equation*}
    where $\mathrm{kl}(x, y) \defn x \log\frac{x}{y} + (1 - x) \log \frac{1 - x}{1 - y} \;$ for $x,y \in (0,1)$.
\end{replemma}

\begin{proof}
\textit{Step 1: Distinguishability due to $\delta$-correctness.} Let $\pi$ be a $\delta$-correct $\epsilon$-global DP BAI strategy. Let $\boldsymbol{\nu}$ be a bandit instance and $\lambda \in \operatorname{Alt}(\boldsymbol{\nu})$.

Let $\mathbb{M}_{\boldsymbol{\nu}, \pi}$ denote the probability distribution of $(\underline{A}, \widehat{A}, \tau)$ when the BAI strategy $\pi$ interacts with $\nu$. For any alternative instance $\lambda \in \operatorname{Alt}(\boldsymbol{\nu})$, the data-processing inequality gives that
\begin{align}\label{ineq:data}
\KL{\mathbb{M}_{\boldsymbol{\nu}, \pi}}{\mathbb{M}_{\boldsymbol{\lambda}, \pi}} & \geq \mathrm{kl}\left(\mathbb{M}_{\boldsymbol{\nu}, \pi}\left(\widehat{A}=a^{\star}(\boldsymbol{\nu})\right), \mathbb{M}_{\boldsymbol{\lambda}, \pi}\left(\widehat{A}=a^{\star}(\boldsymbol{\nu})\right)\right) \notag\\
& \geq \mathrm{kl}(1-\delta, \delta) .
\end{align}
where the second inequality is because $\pi$ is $\delta$-correct i.e. $\mathbb{M}_{\boldsymbol{\nu}, \pi}\left(\widehat{A}=a^{\star}(\boldsymbol{\nu})\right) \geq 1 - \delta$ and $\mathbb{M}_{\boldsymbol{\lambda}, \pi}\left(\widehat{A}=a^{\star}(\boldsymbol{\nu})\right) \leq \delta$, and the monotonicity of the $\mathrm{kl}$.

\textit{Step 2: Using the KL decomposition under global DP.} By Remark~\ref{rmk:kl_bai}, we have
\begin{align}\label{ineq:karwa}
    \KL{\mathbb{M}_{\boldsymbol{\nu}, \pi}}{\mathbb{M}_{\boldsymbol{\lambda}, \pi}} \leq \epsilon  \sum_{a = 1}^\arms \expect_{\boldsymbol{\nu}, \pi} \left [N_a(\tau) \right] \TV{\nu_{a}}{\lambda_{a}}.
\end{align}
Combining Inequalities~\ref{ineq:data} and~\ref{ineq:karwa} concludes the proof.
\end{proof}

\subsubsection{Proof of Theorem~\ref{thm:global_lower_bound}}
\begin{reptheorem}{thm:global_lower_bound}
    Let $\delta \in (0,1) $ and $\epsilon>0$. 
For any $\delta$-correct and $\epsilon$-global DP FC-BAI algorithm, we have that $\mathbb{E}_{\boldsymbol{\nu}}[\tau_{\delta}] \geq T^{\star}_g\left(\boldsymbol{\nu} ; \epsilon\right) \log (1/(2.4 \delta))$ with
\[
	 T^{\star}_g\left(\boldsymbol{\nu} ; \epsilon\right)^{-1 } \defn \sup\limits_{\omega \in \Sigma_K} \inf\limits_{\boldsymbol{\lambda} \in \operatorname{Alt}(\boldsymbol{\nu})} \min  \bigg\{\sum_{a \in [K]} \omega_a \KL{\nu_a}{\lambda_a},  \epsilon \sum_{a \in [K]} \omega_a  \TV{\nu_a}{\lambda_a} \bigg\} \: .
\]
\end{reptheorem}

\begin{proof}
Let $\pi$ be a $\delta$-correct $\epsilon$-global DP BAI strategy. Let $\boldsymbol{\nu}$ be a bandit instance and $\lambda \in \operatorname{Alt}(\boldsymbol{\nu})$.

Let $\expect$ denote the expectation under $\mathbb{P}_{\boldsymbol{\nu}, \pi}$, ie $\expect \defn \expect_{\boldsymbol{\nu}, \pi}$.

By Lemma~\ref{lem:chg_env_dp}, we have that
$\epsilon  \sum_{a = 1}^\arms \expect \left [N_a(\tau) \right] \TV{\nu_{a}}{\lambda_{a}}  \geq \mathrm{kl}(1-\delta, \delta).$

Lemma 1 from~\cite{kaufmann2016complexity} gives that $\sum_{a = 1}^\arms \expect \left [N_a(\tau) \right] \KL{\nu_{a}}{\lambda_{a}} \geq \mathrm{kl}(1-\delta, \delta).$

Since these two inequalities hold for all $\boldsymbol{\lambda} \in \operatorname{Alt}(\boldsymbol{\nu})$, we get

\begin{align*}
    \mathrm{kl}(1- &\delta, \delta)  \leq \inf _{\boldsymbol{\lambda} \in \operatorname{Alt}(\boldsymbol{\nu})} \min \left(\epsilon  \sum_{a = 1}^\arms \expect \left [N_a(\tau) \right] \TV{\nu_{a}}{\lambda_{a}} ,\sum_{a = 1}^\arms \expect \left [N_a(\tau) \right] \KL{\nu_{a}}{\lambda_{a}}\right) \\
    & \stackrel{(a)}{=} \expect[\tau] \inf _{\boldsymbol{\lambda} \in \operatorname{Alt}(\boldsymbol{\nu})} \min \left(\epsilon \sum_{a=1}^K \frac{\expect\left[N_a(\tau)\right]}{\expect[\tau]} \TV{\nu_{a}}{\lambda_{a}}, \sum_{a=1}^K \frac{\expect\left[N_a(\tau)\right]}{\expect[\tau]} \KL{\nu_{a}}{\lambda_{a}}\right) \\
    & \stackrel{(b)}{\leq} \expect[\tau]\left(\sup _{\omega \in \Sigma_K} \inf _{\boldsymbol{\lambda} \in \operatorname{Alt}(\boldsymbol{\nu})} \min \left(\epsilon \sum_{a=1}^K \omega_a  \TV{\nu_a}{\lambda_a}, \sum_{a=1}^K \omega_a \KL{\nu_a}{\lambda_a} \right) \right) \: .
\end{align*}

(a) is due to the fact that $\expect[\tau]$ does not depend on $\boldsymbol{\lambda}$. (b) is obtained by noting that the vector $\left(\omega_a\right)_{a \in[K]} \triangleq \left(\frac{\mathbb{E}_{\nu, \pi}\left[N_a(\tau)\right]}{\mathbb{E}_{\nu, \pi}[\tau]}\right)_{a \in[K]}$ belongs to the simplex $\Sigma_\arms$. 

The theorem follows by noting that for $\delta \in(0,1), \mathrm{kl}(1-\delta, \delta) \geq \log (1 / (2.4 \delta))$.
\end{proof}

\subsubsection{TV Characteristic Time for Bernoulli Instances: Proof of Corollary~\ref{lem:more_explicit_lower_bound}}
\label{app:ssec_proof_prop_tv}

\begin{proposition}[TV characteristic time for Bernoulli instances] Let $\nu$ be a bandit instance, i.e. such that $\nu_a = \text{Bernoulli}(\mu_a)$ and $\mu_1 > \mu_2 \geq \dots \geq \mu_\arms$. Let $\Delta_a \defn \mu_1 - \mu_a$ and $\Delta_\text{min} \defn \min_{a \neq 1} \Delta_a $. We have that
\begin{align*}
    T^{\star}_{\mathrm{TV}}(\boldsymbol{\nu}) = \frac{1}{\Delta_\text{min}} + \sum_{a=2}^\arms \frac{1}{\Delta_a}, &&\text{and}  &&\frac{1}{\Delta_{\text{min}}} \leq T^{\star}_{\mathrm{TV}}(\boldsymbol{\nu}) \leq  \frac{\arms}{\Delta_{\text{min}}}.
\end{align*}
\end{proposition}

\begin{proof}
\textit{Step 1:}    Let $\nu$ be a bandit instance, i.e. such that $\nu_a \defn \text{Bernoulli}(\mu_a)$ and $\mu_1 > \mu_2 \geq \dots \geq \mu_\arms$.

    For the alternative bandit instance $\boldsymbol{\lambda}$, we refer to the mean of arm $a$ as $\rho_a$, i.e. $\lambda_a \defn \text{Bernoulli}(\rho_a)$.
    
    By the definition of $T^{\star}_{\mathrm{TV}} $, we have that
    \begin{align*}
        \left(T^{\star}_{\mathrm{TV}}(\boldsymbol{\nu}) \right)^{-1} &= \sup _{\omega \in \Sigma_K} \inf _{\boldsymbol{\lambda} \in \operatorname{Alt}(\boldsymbol{\nu})}  \sum_{a=1}^K \omega_a \TV{\nu_a}{ \lambda_a}\\
        &\stackrel{(a)}{=} \sup _{\omega \in \Sigma_K} \min_{a \neq 1} \inf _{\boldsymbol{\lambda} : \rho_a > \rho_1} \omega_1 \left | \mu_1 - \rho_1 \right | + \omega_a \left | \mu_a - \rho_a  \right | \\
        &\stackrel{(b)}{=} \sup _{\omega \in \Sigma_K} \min_{a \neq 1} \min(\omega_1, \omega_a) \Delta_a \\
        &\stackrel{(c)}{=} \sup _{\omega \in \Sigma_K} \omega_1 \min_{a \neq 1} \min(1, \frac{\omega_a}{\omega_1}) \Delta_a \\
        &\stackrel{(d)}{=} \sup_{ (x_2, \dots, x_\arms) \in (\real ^+)^{\arms - 1} } \frac{\min_{a \neq 1} g_a(x_a)}{1 + x_2 + \dots + x_\arms} \: ,
    \end{align*}

where $g_a(x_a) \defn \min(1, x_a) \Delta_a$. 

Equality (a) is obtained due to the fact that $\operatorname{Alt}(\boldsymbol{\nu}) = \bigcup_{a \neq 1} \{ \boldsymbol{\lambda}: \rho_a > \rho_1 \}$, and for Bernoullis, $\TV{\nu_a}{\lambda_a} = | \mu_a - \rho_a |$.

Equality (b) is true, since $\inf _{\boldsymbol{\lambda} : \rho_a > \rho_1} \omega_1 \left | \mu_1 - \rho_1 \right | + \omega_a \left | \mu_a - \rho_a  \right | = \min(\omega_1, \omega_a) \Delta_a $.

Equality (c) holds true, since $\omega_1 \neq 1$ (if $\omega_1 = 0$, the value of the objective is 0).

Equality (d) is obtained by the change of variable $x_a \defn \frac{\omega_a}{\omega_1}$

\textit{Step 2:} Let $(x_2, \dots, x_\arms) \in (\real ^+)^{\arms - 1}$. By the definition of $g_a$, we have that
\[g_a(x_a) \leq x_a \Delta_a \quad \text{and} \quad g_a(x_a) \leq \Delta_a.\]

This leads to the inequalities 
\[\min_{a \neq 1} g_a(x_a) \leq g_a(x_a) \leq x_a \Delta_a\quad \text{and} \quad \min_{a \neq 1} g_a(x_a) \leq \Delta_\text{min}.\]

Thus, 
\begin{align*}
    \left (\min_{a \neq 1} g_a(x_a) \right)\left ( \frac{1}{\Delta_\text{min}} + \sum_{a=2}^\arms \frac{1}{\Delta_a} \right) &= \frac{\min_{a \neq 1} g_a(x_a)}{\Delta_\text{min}} + \sum_{a=2}^\arms \frac{\min_{a \neq 1} g_a(x_a)}{\Delta_a}\\
    &\leq 1 + \sum_{a=2}^\arms x_a \: .
\end{align*}

This means that for every $(x_2, \dots, x_\arms) \in (\real ^+)^{\arms - 1}$,  
\begin{equation*}
    \frac{\min_{a \neq 1} g_a(x_a)}{1 + x_2 + \dots + x_\arms} \leq \frac{1}{\frac{1}{\Delta_\text{min}} + \sum_{a=2}^\arms \frac{1}{\Delta_a}}.
\end{equation*}

Here, the upper bound is achievable for $x^\star_a = \frac{\Delta_\text{min}}{\Delta_a}$, since $g_a(x^\star_a) = \Delta_\text{min} $ for all $a \neq 1$.

This concludes that 
\begin{align*}
    T^{\star}_{\mathrm{TV}}(\boldsymbol{\nu})^{-1} = \frac{1}{\frac{1}{\Delta_\text{min}} + \sum_{a=2}^\arms \frac{1}{\Delta_a}} &&\implies T^{\star}_{\mathrm{TV}}(\boldsymbol{\nu})  = \frac{1}{\Delta_\text{min}} + \sum_{a=2}^\arms \frac{1}{\Delta_a} \: .
\end{align*}

\textit{Step 3:} The lower and upper bounds on $T^{\star}_{\mathrm{TV}}(\boldsymbol{\nu}) $ follow from the fact that $\frac{1}{\Delta_a} \geq 0$  for all $a$, and $\frac{1}{\Delta_a} \leq \frac{1}{\Delta_{\min}}$ for all $a \neq 1$.

Hence, we conclude the proof.
\end{proof}

\subsubsection{On the Total Variation Distance and the Hardness of Privacy}

Our lower bound suggests that the hardness of the DP-FC-BAI problem is characterized by $T^\star_{\mathrm{TV}}(\boldsymbol{\nu})$, which is a total variation counterpart of the classic KL-based characteristic time $T^\star_{\mathrm{KL}}(\boldsymbol{\nu})$ in FC-BAI~\cite{garivier2016optimal}. The total variation distance appears to be the natural measure to quantify the hardness of privacy in other settings such as regret minimization~\cite{azize2022privacy}, Karwa-Vadhan lemma~\cite{KarwaVadhan} and Differentially Private Assouad, Fano, and Le Cam~\cite{acharya2021differentially}. The high-level intuition is that: Pure DP can be seen as a multiplicative stability constraint of $e^\epsilon$ when one data point changes. With group privacy, if two data sets differ in $d_{ham}$ points, then one incurs a factor $e^ {d_{ham}~\epsilon}$. Now, by sampling $n$ i.i.d points from a distribution $P$ and $n$ i.i.d points from a distribution $Q$, the Karwa-Vadhan lemma states that the incurred factor is $e^{(nTV(P,Q)) ~ \epsilon}$. This is proved by building a maximal coupling, which is the coupling that minimizes the Hamming distance in expectation. In brief, \textit{the total variation naturally appears in lower bounds since it is the quantity that characterises the hardness of the optimal transport problem minimizing the hamming distance}, i.e $\mathrm{TV}(P, Q) = \inf_{(X,Y)\sim (P,Q)} E(1_{X \neq Y})$. However, it is possible that the problem can be characterized by other f-divergences.  Finally, one can always go from TV to KL using Pinsker's inequality, though that would always be less tight than the TV-based lower bound. 

\textit{On the relation between $T^\star_{\mathrm{TV}}(\boldsymbol{\nu})$ and $T^\star_{\mathrm{KL}}(\boldsymbol{\nu})$.} 
A direct application of Pinsker's inequality gives that $T^\star_{\mathrm{TV}}(\boldsymbol{\nu}) \geq \sqrt{2 T^\star_{\mathrm{KL}}(\boldsymbol{\nu})}$. For completeness, we present here the exact calculations:

For every alternative mean parameter $\lambda$ and every arm $a$, using Pinkser's inequality, we have that $d_{TV} ( \mu_a, \lambda_a) \leq \sqrt{\frac{1}{2} d_{KL} ( \mu_a, \lambda_a)}$.
Therefore, for every allocation over arms $\omega$, we have
\[
\sum_a \omega_a d_{TV} ( \mu_a, \lambda_a) \leq \sum_a \omega_a \sqrt{\frac{1}{2} d_{KL} ( \mu_a, \lambda_a)} \leq \sqrt{\frac{1}{2}\sum_a \omega_a d_{KL} ( \mu_a, \lambda_a)} \: .
\]
Taking the supremum over the simplex and the infimum over the set of alternative mean parameters yields $T^\star_{\mathrm{TV}}(\boldsymbol{\nu})^{-1} \leq \sqrt{\frac{1}{2} T^\star_{\mathrm{KL}}(\boldsymbol{\nu})^{-1} }$.
This concludes the proof.

\section{Privacy analysis}\label{app:privacy_proof}

We prove that \adaptt{} and \adapttt{} satisfy $\epsilon$-global DP. We first provide the privacy lemma that justifies using doubling and forgetting. Using the privacy lemma and the post-processing property of DP, we conclude the privacy analysis of \adaptt{} and \adapttt{}.

\subsection{Privacy Lemma for Non-overlapping Sequences}

\begin{lemma}[Privacy of non-overlapping sequence of empirical means]\label{lem:privacy}
Let $\mathcal M$ be a mechanism that takes a \textbf{set} as input and outputs the private empirical mean, i.e., 
\[
    \mathcal{M}(\{ r_i, \dots, r_j \}) \defn \frac{1}{j -i} \sum_{t=i}^{j} r_t + Lap\left ( \frac{1}{(j-i) \epsilon} \right ) \: .
\]
Let $\ell < \horizon$ and $t_1, \ldots t_{\ell}, t_{\ell + 1}$ be in $[1, \horizon]$ such that $1 = t_1  < \cdots < t_\ell < t_{\ell+1}  - 1 = T$.\\
Let's define the following mechanism
\begin{equation}\label{eq:priv_mean_non_over}
    \mathcal G : \{ r_1, \dots, r_\horizon \} \rightarrow \bigotimes_{i=1}^\ell \mathcal{M}_{ \{r_{t_i}, \ldots, r_{t_{i + 1} - 1} \}}
\end{equation}

In other words, $\mathcal{G}$ is the mechanism we get by applying $\mathcal{M}$ to the non-overlapping partition of the sequence $\{r_1, \dots, r_\horizon\}$ according to $t_1  < \cdots < t_\ell < t_{\ell+1}$, i.e.
\begin{align*}
    \begin{pmatrix}
r_1\\ 
r_2\\ 
\vdots \\
r_T
\end{pmatrix} \overset{\mathcal{G}}{\rightarrow}  \begin{pmatrix}
\mu_1\\ 
\vdots \\
\mu_\ell
\end{pmatrix}
\end{align*}
where $\mu_i \sim \mathcal{M}_{ \{r_{t_i}, \ldots, r_{t_{i + 1} - 1} \}}$.

For $r_t \in [0,1]$, the mechanism $\mathcal{G}$ is $\epsilon$-DP.
\end{lemma}
\begin{proof}
Let $r^\horizon \defn (r_1, \dots, r_\horizon)$ and $r'^\horizon \defn (r_1', \dots, r_\horizon')$ be two neighbouring reward sequences in [0,1]. This implies that $\exists j \in [1, \horizon]$ such that $r_j \neq r_j'$ and $\forall t \neq j$, $r_t = r_t'$. 

Let $\episode'$ be such that $t_{\episode'} \leq j \leq t_{\episode'+1}-1$, and follows the convention that $t_0 = 1$ and $t_{\episode + 1} = T + 1$.

Let $\mu \defn (\mu_1, \dots, \mu_\episode)$ a fixed sequence of outcomes. Then,
\begin{align*}
    \frac{\mathbb P ( \mathcal{G} (r^\horizon) = \mu  )}{\mathbb P ( \mathcal{G} (r'^\horizon) = \mu  )} = \frac{\mathbb P \left( \mathcal{M}(\{ r_{{t_{\ell'}}}, \dots, r_{t_{\ell'+1} - 1} \}) = \mu_{\ell'}  \right)}{\mathbb P \left( \mathcal{M}(\{ r_{{t_{\ell'}}}, \dots, r_{t_{\ell'+1} - 1} \}) = \mu_{\ell'}  \right)} \leq e^\epsilon,
\end{align*}
where the last inequality holds true because $\mathcal{M}$ satisfies $\epsilon$-DP following Theorem~\ref{thm:laplace}.
\end{proof}

\subsection{Privacy Analysis of \adaptt{} and \texorpdfstring{\adapttt{}}{}}

\begin{theorem}[Privacy analysis]
For rewards in $[0,1]$, \adaptt{} and \adapttt{} satisfy $\epsilon$-global DP.
\end{theorem}

\begin{remark}
    The following proof is valid for any BAI strategy that only uses the DAF($\epsilon$) to estimate the means.
\end{remark}

\begin{proof}
Let $\horizon \geq 1$. Let $\underline{\textbf{d}}^\horizon = \{\textbf{x}_1, \dots, \textbf{x}_\horizon\}$ and $\underline{\textbf{d}'}^\horizon = \{\textbf{d}'_1, \dots, \textbf{d}'_\horizon\}$ two neighbouring reward tables in $(\real^\arms)^\horizon$. Let $j \in [1, \horizon]$ such that, for all $t \neq j$, $d_t = d'_t$.

We also fix a sequence of sampled actions $\underline{a}^T = \{a_1, \dots, a_T\} \in [\arms]^\horizon$ and a recommended action $\hat{a}\in \arms$.

Let $\pi$ be a BAI strategy that only uses DAF($\epsilon$) to estimate the means, i.e. either \adaptt{} or \adapttt{}.

We want to show that:
$\pi(\underline{a}^\horizon, \widehat{a}, \horizon \mid \underline{\textbf{d}}^\horizon)  \leq e^\epsilon \pi(\underline{a}^\horizon, \widehat{a}, \horizon \mid \underline{\textbf{d}'}^\horizon)$.

The main idea is that the change of reward in the $j$-th reward only affects the empirical mean computed in one episode, which is made private using the Laplace Mechanism and Lemma \ref{lem:privacy}.

\noindent \textit{Step 1. Sequential decomposition of the output probability}

    We observe that due to the sequential nature of the interaction, the output probability can be decomposed to a part that depends on $\underline{\textbf{d}}^{j - 1} \defn \{\textbf{x}_1, \dots, \textbf{x}_{j - 1}\}$, which is identical for both $\underline{\textbf{d}}^\horizon$ and $\underline{\textbf{d}'}^\horizon$ and a second conditional part on the history.

    Specifically, we have that
    \begin{align*}
        \pi(\underline{a}^\horizon, \widehat{a}, \horizon \mid \underline{\textbf{d}}^\horizon) &\defn \operatorname{Rec}_{T+1}\left(\widehat{a} \mid \mathcal{H}_{\horizon}\right) \mathrm{S}_{T+1}\left(\top \mid \mathcal{H}_{\horizon} \right) \prod_{t=1}^T \mathrm{~S}_t\left(a_t \mid \mathcal{H}_{t - 1}\right)\\
        &\defn \mathcal{P}^{\pol}_{ \underline{\textbf{d}}^{j - 1} } (\underline{a}^j) \mathcal{P}^{\pol}_{\underline{\textbf{d}} }(a_{> j}, \hat{a}, T \mid \underline{a}^j) 
    \end{align*}
    where 
    \begin{itemize}
        \item $ a_{> j} \defn (a_{j + 1}, \dots, a_\horizon)$
        \item $\mathcal{P}^{\pol}_{ \underline{\textbf{d}}^{j - 1} } (\underline{a}^j) \defn \prod_{t = 1}^j \mathrm{~S}_t\left(a_t \mid \mathcal{H}_{t - 1}\right) $
        \item $\mathcal{P}^{\pol}_{\underline{\textbf{d}} }(a_{> j}, \hat{a}, T \mid \underline{a}^j)  \defn \operatorname{Rec}_{T+1}\left(\widehat{a} \mid \mathcal{H}_{\horizon}\right) \mathrm{S}_{T+1}\left(\top \mid \mathcal{H}_{\horizon} \right) \prod_{t = j+1}^\horizon \mathrm{~S}_t\left(a_t \mid \mathcal{H}_{t - 1}\right) $
    \end{itemize}

    Similarly
    \begin{equation*}
        \pi(\underline{a}^\horizon, \widehat{a}, \horizon \mid \underline{\textbf{d}'}^\horizon) \defn \mathcal{P}^{\pol}_{ \underline{\textbf{d}}^{j - 1} } (\underline{a}^j) \mathcal{P}^{\pol}_{\underline{\textbf{d}'} }(a_{> j}, \hat{a}, T \mid \underline{a}^j) 
    \end{equation*}
    since $\underline{\textbf{d}'}^{j - 1} = \underline{\textbf{d}}^{j - 1}$.

    Which means that
    \begin{equation}\label{eq:deco_frac}
        \frac{\pi(\underline{a}^\horizon, \widehat{a}, \horizon \mid \underline{\textbf{d}}^\horizon)}{\pi(\underline{a}^\horizon, \widehat{a}, \horizon \mid \underline{\textbf{d}'}^\horizon)} = \frac{\mathcal{P}^{\pol}_{\underline{\textbf{d}} }(a_{> j}, \hat{a}, T \mid \underline{a}^j)}{\mathcal{P}^{\pol}_{\underline{\textbf{d}'} }(a_{> j}, \hat{a}, T \mid \underline{a}^j)}
    \end{equation}

    \noindent \textit{Step 2. The adaptive episodes are the same, before step $j$}

    Let $\ell$ such that $t_\ell \leq j < t_{\ell + 1}$ when $\pi$ interacts with $\underline{\textbf{d}}^\horizon$. Let us call it $\psi^\pol_{\underline{\textbf{d}}^\horizon}(j) \defn \ell$.\\ Similarly, let $\ell'$ such that $t_{\ell'} \leq j < t_{\ell' + 1}$ when $\pi$ interacts with $\underline{\textbf{d}'}^\horizon$. Let us call it $\psi^\pol_{\underline{\textbf{d}'}^\horizon}(j) \defn \ell'$.

    Since $\psi^\pol_{\underline{\textbf{d}}^\horizon}(j)$ only depends on $\underline{\textbf{d}}^{j - 1}$, which is identical for $\underline{\textbf{d}}^\horizon$ and $\underline{\textbf{d}'}^\horizon$, we have that $\psi^\pol_{\underline{\textbf{d}}^\horizon}(j) = \psi^\pol_{\underline{\textbf{d}'}^\horizon}(j)$ with probability $1$.

    We call $\xi_j$ the last \textbf{time-step} of the episode $\psi^\pol_{\underline{\textbf{d}}^\horizon}(j)$, i.e $\xi_j \defn t_{\psi^\pol_{\underline{\textbf{d}}^\horizon}(j) + 1} -1$.

    \noindent \textit{Step 3. Private sufficient statistics}
    
    Let $r_t \defn \underline{\textbf{d}}^\horizon_{t, a_t}$, be the reward corresponding to the action $a_t$ in the table $\underline{\textbf{d}}^\horizon$. Similarly, $r'_t \defn \underline{\textbf{d}'}^\horizon_{t, a_t}$ for $\underline{\textbf{d}'}^\horizon$.
    
    Let us define $L_j \defn \mathcal{G}_{\{r_1, \dots, r_{\xi_j} \} }$ and $L'_j \defn \mathcal{G}_{\{r'_1, \dots, r'_{\xi_j} \} }$, where $\mathcal{G}$ is defined as in Eq.~\ref{eq:priv_mean_non_over}, using the same episodes for $d$ and $d'$. In other words, $L_j$ is the list of private empirical means computed on a non-overlapping sequence of rewards before step $\xi_j$.
    
    Using the forgetting structure of $\pi$, there exists a randomised mapping $f_{\underline{\textbf{d}}_{> \xi_j}}$ such that $\mathcal{P}^{\pol}_{\underline{\textbf{d}} }(. \mid \underline{a}^j) = f_{\underline{\textbf{d}}_{> \xi_j}} \circ   L_j $ and $\mathcal{P}^{\pol}_{\underline{\textbf{d}'} }(. \mid \underline{a}^j) = f_{\underline{\textbf{d}}_{> \xi_j}} \circ   L_j' $.

    In other words, the interaction of $\pi$ with $\underline{\textbf{d}}$ and $\underline{\textbf{d}'}$ from step $\xi_j + 1$ until $\horizon$ only depends on the sufficient statistics $L_j$, which summarises what happened before $\xi_j$, and the new inputs $\underline{\textbf{d}}_{> \xi_j}$, which are the same for $\underline{\textbf{d}}$ and $\underline{\textbf{d}'}$.

   \noindent \textit{Step 4. Concluding with Lemma~\ref{lem:privacy} and the post-processing lemma}
   
   Since rewards are in $[0,1]$, using Lemma~\ref{lem:privacy}, we have that $\mathcal{G}$ is $\epsilon$-DP.

    Since $\mathcal{P}^{\pol}_{\underline{\textbf{d}} }(. \mid \underline{a}^j)$ is just a post-processing of the output of $\mathcal{G}$, we have that
   \begin{equation*}
       \frac{\mathcal{P}^{\pol}_{\underline{\textbf{d}} }(a_{> j}, \hat{a}, T \mid \underline{a}^j)}{\mathcal{P}^{\pol}_{\underline{\textbf{d}'} }(a_{> j}, \hat{a}, T \mid \underline{a}^j)} \leq e^\epsilon \: ,
   \end{equation*}
   and Eq.~\eqref{eq:deco_frac} concludes the proof.
\end{proof}

\section{Globally Differentially Private GLR Stopping Rules}
\label{app:private_stopping_rule}

After studying the non-private GLR stopping rule with phases (Appendix~\ref{app:ssec_non_private_GLR_stopping}), we study the private GLR stopping rule with non-private transportation costs $W^{G}_{a,b}$ in Appendix~\ref{app:ssec_proof_delta_correct_threshold_TTUCB_global_v1} (Lemma~\ref{lem:delta_correct_threshold_TTUCB_global_v1}) and with adapted transportation costs $W^{G,\epsilon}_{a,b}$ in Appendix~\ref{app:ssec_proof_delta_correct_threshold_TTUCB_global_v2} (Lemma~\ref{lem:delta_correct_threshold_TTUCB_global_v2}).

\subsection{Non-private GLR Stopping Rule with Per-arm Phases}
\label{app:ssec_non_private_GLR_stopping}

Before accounting for the privacy (i.e. Laplace noise), we first highlight the price of \hyperlink{DAF}{DAF}$(\epsilon)$ for $\epsilon = + \infty$.
This stopping condition is only evaluated at the beginning of each phase for each arm since it involves quantities that are fixed until we switch phase again, and it recommends $\hat a_{n} = \argmax_{a \in [K]} \hat \mu_{k_{n,a},a}$ which is the best arm for the non-private empirical means.
Lemma~\ref{lem:delta_correct_threshold_TTUCB_non_private} yields a threshold function ensuring $\delta$-correctness.
\begin{lemma}\label{lem:delta_correct_threshold_TTUCB_non_private} 
    Let $\delta \in (0,1)$.
    Let $s > 1$, $\zeta$ be the Riemann $\zeta$ function, $c^{G}_{a,b}$ as in Eq.~\eqref{eq:threshold_non_private} and $k(x) = \log_2 x + 2 $.
    Combining the \hyperlink{DAF}{DAF}$(\epsilon)$ estimator for $\epsilon = + \infty$ with the GLR stopping rule with $W^{G}_{a,b}$ as in Eq.~\eqref{eq:TTUCB_Gaussian} and the stopping threshold $c^{G}_{a,b}(\omega,\delta (\zeta(s)^2 k(\omega_{a})^s k(\omega_{b})^s)^{-1})$ yields a $\delta$-correct algorithm for $\sigma$-sub-Gaussian distributions regardless of the sampling rule. 
\end{lemma}  
\begin{proof}
    The non-private GLR stopping rule matches the one used for Gaussian bandits.
    Proving $\delta$-correctness of a GLR stopping rule is done by leveraging concentration results.
\begin{lemma}[Theorem 9 in \citealt{KK18Mixtures}] \label{lem:theorem_9_KK18Mixtures}
    Let $\boldsymbol{\nu}$ be a sub-Gaussian bandit with means $\boldsymbol{\mu} \in \R^{K}$ and variance proxy $\sigma$. Let $S \subseteq [K]$ and $ x > 0$.
    \begin{align*}
        \bP_{\bnu} \left( \exists n \in \N, \: \sum_{a \in S} \frac{N_{n,a}}{2\sigma^2} (\mu_{n, a} - \mu_{a})^2 >  \sum_{a \in S} 2 \log \left(4 + \log \left(N_{n,a}\right)\right) + |S| \cC_{G}\left(\frac{x}{|S|}\right) \right) \leq e^{-x} \: ,
    \end{align*}
    where $\cC_{G}$ is defined in \cite{KK18Mixtures} as 
    \begin{equation} \label{eq:def_C_gaussian_KK18Mixtures}
        \cC_{G}(x) \defn \min_{\lambda \in ]1/2,1]} \frac{g_G(\lambda) + x}{\lambda} \text{ and } g_G (\lambda) \defn 2\lambda - 2\lambda \log (4 \lambda) + \log \zeta(2\lambda) - \frac{1}{2} \log(1-\lambda) \: .
    \end{equation}
    Here, $\zeta$ is the Riemann $\zeta$ function and $\cC_{G}(x) \approx x + \log(x)$.
\end{lemma}

We consider the concentration event $\cE^{(1)}_{\delta} = \bigcap_{a \ne a^\star} \bigcap_{n \in \N} \cE^{(1)}_{\delta}(a,n)$ with $\cE^{(1)}_{\delta}(a,n) = $
\begin{equation} \label{eq:event_non_private}
    \left\{ \frac{\tilde N_{k_{n,a}, a}}{2\sigma^2}(\hat \mu_{k_{n,a},a}  -  \mu_{a})^2 +  \frac{\tilde N_{k_{n,a^\star},a^\star}}{2\sigma^2}(\hat \mu_{k_{n,a^\star},a^\star}  -  \mu_{a^\star})^2 < c^{G}_{a,a^\star}(\tilde N_{k_n}, \frac{\delta}{\zeta(s)^2 k_{n,a}^s k_{n,a^\star}^s}) \right\} \: .
\end{equation}
For all $a \ne a^\star$ and all $(k_{a}, k_{a^\star}) \in \N^2$, the estimators $(\hat \mu_{k_{c},c})_{c \in \{a,a^\star\}}$ are based solely on the observations collected for arm $a$ (resp. arm $a^\star$) between times $n \in \{ T_{k_{a}-1}(a), \cdots, T_{k_{a}}(a) -1 \}$ (resp. $n \in \{ T_{k_{a^\star}-1}(a^\star), \cdots, T_{k_{a^\star}}(a^\star)-1\}$) with local counts $(\tilde N_{k_{c},c})_{c \in \{a,a^\star\}}$, i.e. dropping past observations.
Using a direct union bound, we obtain that $\bP_{\nu}((\cE^{(1)}_{\delta})^{\complement})$ is smaller than
\begin{align*}
  & \sum_{a \ne a^\star} \sum_{k_{a}, k_{a^\star} \in \N}  \bP_{\bnu}\left( \frac{\tilde N_{k_{a}, a}}{2\sigma^2}(\hat \mu_{k_{a},a}  -  \mu_{a})^2 +  \frac{\tilde N_{k_{a^\star},a^\star}}{2\sigma^2}(\hat \mu_{k_{a^\star},a^\star}  -  \mu_{a^\star})^2  \ge c^{G}_{a,a^\star}(\tilde N_{k_n}, \frac{\delta}{\zeta(s)^2 k_{a}^s k_{a^\star}^s}) \right) \\
  &\le \frac{\delta}{K-1} \frac{1}{\zeta (s)^2 }  \sum_{a \ne a^\star} \sum_{(k_{a}, k_{a^\star}) \in \N^2} \frac{1}{(k_{a} k_{a^\star})^s} =  \delta \: .
\end{align*}  
where the last inequality uses Lemma~\ref{lem:theorem_9_KK18Mixtures} for all $a \ne a^\star$ and all $(k_{a}, k_{a^\star}) \in \N^2$.
Therefore,
\[
    \bP_{\nu}(\tau_{\delta} < + \infty, \: \hat a_{\tau_{\delta}} \ne a^\star ) \le \delta + \bP_{\nu}(\cE^{(1)}_{\delta} \cap \{\tau_{\delta} < + \infty, \: \hat a_{\tau_{\delta}} \ne a^\star  \}) \: .
\]
Under $\cE^{(1)}_{\delta}\cap \{\tau_{\delta} < + \infty, \: \hat a_{\tau_{\delta}} \ne a^\star  \}$, we have $\hat a_{\tau_{\delta}} = \argmax_{b \in [K]} \hat \mu_{k_{\tau_{\delta},a},a} \ne a^\star$ and
\begin{align*}
    &c^{G}_{\hat a_{\tau_{\delta}},a^\star}(\tilde N_{k_{\tau_{\delta}}}, \frac{\delta}{\zeta(s)^2 k_{\tau_{\delta},a}^s k_{\tau_{\delta},a^\star}^s}) \le \frac{(\hat \mu_{k_{\tau_{\delta},\hat a_{\tau_{\delta}}},\hat a_{\tau_{\delta}}}  - \hat \mu_{k_{\tau_{\delta},a^\star},a^\star})^2}{2\sigma^2(1/\tilde N_{k_{\tau_{\delta},\hat a_{\tau_{\delta}}},\hat a_{\tau_{\delta}}}+ 1/\tilde N_{k_{\tau_{\delta},a^\star},a^\star})} \\
    &\quad =  \inf_{y \ge x} \left\{ \frac{\tilde N_{k_{\tau_{\delta},\hat a_{\tau_{\delta}}}, \hat a_{\tau_{\delta}}}}{2\sigma^2}( \hat \mu_{k_{\tau_{\delta},\hat a_{\tau_{\delta}}}, \hat a_{\tau_{\delta}}} - x)^2 + \frac{\tilde N_{k_{\tau_{\delta},a^\star}, a^\star}}{2\sigma^2} ( \hat \mu_{k_{\tau_{\delta},a^\star}, a^\star} - y)^2\right\} \\
    &\quad \le \frac{\tilde N_{k_{\tau_{\delta},\hat a_{\tau_{\delta}}}, \hat a_{\tau_{\delta}}}}{2\sigma^2}( \hat \mu_{k_{\tau_{\delta},\hat a_{\tau_{\delta}}}, \hat a_{\tau_{\delta}}} - \mu_{\hat a_{\tau_{\delta}}})^2 + \frac{\tilde N_{k_{\tau_{\delta},a^\star}, a^\star}}{2\sigma^2} ( \hat \mu_{k_{\tau_{\delta},a^\star}, a^\star} - \mu_{a^\star})^2 \\
    &\quad < c^{G}_{\hat a_{\tau_{\delta}},a^\star}(\tilde N_{k_{\tau_{\delta}}}, \frac{\delta}{\zeta(s)^2 k_{\tau_{\delta},a}^s k_{\tau_{\delta},a^\star}^s}) \: .
\end{align*}
This is a contradiction, hence $\cE^{(1)}_{\delta}\cap \{\tau_{\delta} < + \infty, \: \hat a_{\tau_{\delta}} \ne a^\star  \} = \emptyset$. 
This concludes the proof. 
\end{proof}
    
\subsection{Private GLR with Non-private Transportation Cost: Proof of Lemma~\ref{lem:delta_correct_threshold_TTUCB_global_v1}}
\label{app:ssec_proof_delta_correct_threshold_TTUCB_global_v1}

The proof of Lemma~\ref{lem:delta_correct_threshold_TTUCB_global_v1} is similar as the one detailed in Appendix~\ref{app:ssec_non_private_GLR_stopping}, with the added difficulty of controlling the Laplace noise.
We consider the concentration event $\cE_{\delta} = \cE^{(1)}_{\delta/2} \cap \cE^{(2)}_{\delta/2}$ where $\cE^{(1)}_{\delta}$ as in Eq.~\eqref{eq:event_non_private} and $\cE^{(2)}_{\delta} = \bigcap_{a \in [K]} \bigcap_{n \in \N} \cE^{(2)}_{\delta}(a,n)$ with
\begin{equation} \label{eq:event_laplace_noise}
    \cE^{(2)}_{\delta}(a,n) = \left\{ \epsilon \tilde N_{k_{n,a}, a} |Y_{k_{n,a},a}| < \log \left(\frac{K \zeta(s) k_{n,a}^{s} }{\delta} \right) \right\} \: .
\end{equation}
Since $Y_{k_{n,a},a} \sim \text{Lap}\left((\epsilon\tilde N_{k_{n,a},a})^{-1} \right)$, we have that $\tilde N_{k_{n,a}, a} |Y_{k_{n,a},a}| \sim \cE(\epsilon)$ for all $a \in [K]$ and all $n \in \N$, where $\cE(\cdot)$ denotes the exponential distribution.
Using concentration results for exponential distribution, a direct union bound yields that $\bP_{\nu}((\cE^{(2)}_{\delta})^{\complement}) \le \delta$, hence
\[
    \bP_{\nu}(\tau_{\delta} < + \infty, \: \hat a_{\tau_{\delta}} \ne a^\star ) \le \delta + \bP_{\nu}(\cE_{\delta} \cap \{\tau_{\delta} < + \infty, \: \hat a_{\tau_{\delta}} \ne a^\star  \}) \: .
\]
Under $\cE_{\delta}\cap \{\tau_{\delta} < + \infty, \: \hat a_{\tau_{\delta}} \ne a^\star  \}$, we have $\hat a_{\tau_{\delta}} = \argmax_{b \in [K]} \tilde \mu_{k_{\tau_{\delta},a},a} \ne a^\star$ and
\begin{align*}
    &c^{G,\epsilon}_{\hat a_{\tau_{\delta}},a^\star}(\tilde N_{k_{\tau_{\delta}}},\delta) \le \frac{\tilde N_{k_{\tau_{\delta},\hat a_{\tau_{\delta}}}, \hat a_{\tau_{\delta}}}}{2\sigma^2}( \tilde \mu_{k_{\tau_{\delta},\hat a_{\tau_{\delta}}}, \hat a_{\tau_{\delta}}} - \mu_{\hat a_{\tau_{\delta}}})^2 + \frac{\tilde N_{k_{\tau_{\delta},a^\star}, a^\star}}{2\sigma^2} ( \tilde \mu_{k_{\tau_{\delta},a^\star}, a^\star} - \mu_{a^\star})^2 \\
    &\quad \le \frac{\tilde N_{k_{\tau_{\delta},\hat a_{\tau_{\delta}}}, \hat a_{\tau_{\delta}}}}{\sigma^2}( \hat \mu_{k_{\tau_{\delta},\hat a_{\tau_{\delta}}}, \hat a_{\tau_{\delta}}} - \mu_{\hat a_{\tau_{\delta}}})^2 + \frac{\tilde N_{k_{\tau_{\delta},a^\star}, a^\star}}{\sigma^2} ( \hat \mu_{k_{\tau_{\delta},a^\star}, a^\star} - \mu_{a^\star})^2\\
    &\qquad + \frac{\tilde N_{k_{\tau_{\delta},\hat a_{\tau_{\delta}}}, \hat a_{\tau_{\delta}}}}{\sigma^2}Y_{k_{\tau_{\delta},\hat a_{\tau_{\delta}}}, \hat a_{\tau_{\delta}}} ^2 + \frac{\tilde N_{k_{\tau_{\delta},a^\star}, a^\star}}{\sigma^2} Y_{k_{\tau_{\delta},a^\star}, a^\star} ^2\\
    &\quad < 2c^{G}_{\hat a_{\tau_{\delta}},a^\star}(\tilde N_{k_{\tau_{\delta}}}, \frac{\delta}{2\zeta(s)^2 k_{\tau_{\delta},a}^s k_{\tau_{\delta},a^\star}^s}) + \frac{1}{\epsilon^2\sigma^2} \sum_{c \in \{a,a^\star\}} \frac{1}{\tilde N_{k_{\tau_{\delta},c},c}} \left(\log \frac{2 K \zeta(s) k_{\tau_{\delta},c}^{s} }{\delta} \right)^2  \: .
\end{align*}
where we used that $\tilde \mu_{k_{\tau_{\delta},a}, a} = \tilde \mu_{k_{\tau_{\delta},a}, a} + Y_{k_{\tau_{\delta},a},a}$ and $(x - y)^2 \le 2x^2 + 2y^2$.
This is a contradiction, hence $\cE_{\delta}\cap \{\tau_{\delta} < + \infty, \: \hat a_{\tau_{\delta}} \ne a^\star  \} = \emptyset$. 
To conclude the proof, \Modif{we take as a specific parameter $s=2$, hence $\zeta(2) = \pi^2/6$.}

\subsection{Private GLR with Adapted Transportation Cost: Proof of Lemma~\ref{lem:delta_correct_threshold_TTUCB_global_v2}}
\label{app:ssec_proof_delta_correct_threshold_TTUCB_global_v2}

The proof of Lemma~\ref{lem:delta_correct_threshold_TTUCB_global_v2} is similar as the one detailed in Appendix~\ref{app:ssec_proof_delta_correct_threshold_TTUCB_global_v1}.
The main difference lies in the considered transportation costs, namely $W^{G,\epsilon}_{a,b}$ instead of $W^{G}_{a,b}$.

\begin{lemma}[Lemma 28 in \citealt{jourdan_2022_DealingUnknownVariance}] \label{lem:uniform_upper_lower_tails_concentration_mean}
    Let $\delta \in (0,1)$.
    For all $x \ge 1$, let $\overline{W}_{-1}(x)  = - W_{-1}(-e^{-x})$  (see Lemma~\ref{lem:property_W_lambert}), where $W_{-1}$ is the negative branch of the Lambert $W$ function.
    Let $c(x,\delta) = \frac{1}{2}\overline{W}_{-1} \left(2\ln \left( K/\delta\right) +  4 \ln(4 + \ln x ) + 1/2  \right)$.
    Consider $\sigma$-sub-Gaussian bandits with means $\mu \in \R^{K}$.
    Then,
    \begin{align*}
        \bP \left( \exists n \in \N, \: \exists a \in [K], \: \frac{N_{n,a}}{2\sigma^2}(\mu_{n, a} - \mu_{a})^2 > c(N_{n,a}, \delta) \right) \le \delta \: .
    \end{align*}
\end{lemma}
Recall that $h(\tilde N_{k_{n,a}, a}, \delta) = c(\tilde N_{k_{n,a}, a}, \frac{\delta}{3\zeta(s)k_{n,a}^s})$.
We use the concentration event $\cE_{\delta} = \cE^{(1)}_{\delta/3} \cap \cE^{(2)}_{\delta/3} \cap \cE^{(3)}_{\delta/3}$ where $\cE^{(1)}_{\delta}$ as in Eq.~\eqref{eq:event_non_private}, $\cE^{(2)}_{\delta}$ as in Eq.~\eqref{eq:event_laplace_noise} and $\cE^{(3)}_{\delta} = \bigcap_{a \in [K]} \bigcap_{n \in \N} \cE^{(3)}_{\delta}(a,n)$ with
\begin{equation*}
    \cE^{(3)}_{\delta}(a,n) = \left\{ \frac{\tilde N_{k_{n,a}, a}}{2\sigma^2} (\hat \mu_{k_{n,a},a}  -  \mu_{a})^2 < h(\tilde N_{k_{n,a}, a}, 3\delta) \right\} \: .
\end{equation*}

Using Lemma~\ref{lem:uniform_upper_lower_tails_concentration_mean}, a direct union bound yields that $\bP_{\nu}((\cE^{(3)}_{\delta})^{\complement}) \le \delta$, hence
\[
    \bP_{\nu}(\tau_{\delta} < + \infty, \: \hat a_{\tau_{\delta}} \ne a^\star ) \le \delta + \bP_{\nu}(\cE_{\delta} \cap \{\tau_{\delta} < + \infty, \: \hat a_{\tau_{\delta}} \ne a^\star  \}) \: .
\]

Under $\cE_{\delta}\cap \{\tau_{\delta} < + \infty, \: \hat a_{\tau_{\delta}} \ne a^\star  \}$, we have $\hat a_{\tau_{\delta}} = \argmax_{b \in [K]} \tilde \mu_{k_{\tau_{\delta},a},a} \ne a^\star$.

{\bf Case 1.}
Under $\cE_{\delta}\cap \{\tau_{\delta} < + \infty, \: \hat a_{\tau_{\delta}} \ne a^\star, \: (\hat \mu_{k_{\tau_{\delta},\hat a_{\tau_{\delta}}},\hat a_{\tau_{\delta}}}  - \hat \mu_{k_{\tau_{\delta},a^\star},a^\star})_{+} < \epsilon/2  \}$, we have
\begin{align*}   
    &\frac{1}{2} c^{G,\epsilon}_{\hat a_{\tau_{\delta}},a^\star}(\tilde N_{k_{\tau_{\delta}}},\delta) + \frac{\sqrt{2}}{\epsilon \sigma}  \sum_{c \in \{\hat a_{\tau_{\delta}},a^\star\}} \sqrt{\frac{h(\tilde N_{k_{\tau_{\delta},c}},\delta)}{\tilde N_{k_{\tau_{\delta},c}}}} \log \left(\frac{3 K \zeta(s) k_{\tau_{\delta},c}^{s} }{\delta} \right) \\
    &\quad \le \frac{\tilde N_{k_{\tau_{\delta},\hat a_{\tau_{\delta}}}, \hat a_{\tau_{\delta}}}}{2\sigma^2}( \tilde \mu_{k_{\tau_{\delta},\hat a_{\tau_{\delta}}}, \hat a_{\tau_{\delta}}} - \mu_{\hat a_{\tau_{\delta}}})^2 + \frac{\tilde N_{k_{\tau_{\delta},a^\star}, a^\star}}{2\sigma^2} ( \tilde \mu_{k_{\tau_{\delta},a^\star}, a^\star} - \mu_{a^\star})^2 \\
    &\quad = \frac{\tilde N_{k_{\tau_{\delta},\hat a_{\tau_{\delta}}}, \hat a_{\tau_{\delta}}}}{2\sigma^2}( \hat \mu_{k_{\tau_{\delta},\hat a_{\tau_{\delta}}}, \hat a_{\tau_{\delta}}} - \mu_{\hat a_{\tau_{\delta}}})^2 + \frac{\tilde N_{k_{\tau_{\delta},a^\star}, a^\star}}{2\sigma^2} ( \hat \mu_{k_{\tau_{\delta},a^\star}, a^\star} - \mu_{a^\star})^2\\
    &\qquad + \frac{\tilde N_{k_{\tau_{\delta},\hat a_{\tau_{\delta}}}, \hat a_{\tau_{\delta}}}}{2\sigma^2}Y_{k_{\tau_{\delta},\hat a_{\tau_{\delta}}}, \hat a_{\tau_{\delta}}} ^2 + \frac{\tilde N_{k_{\tau_{\delta},a^\star}, a^\star}}{2\sigma^2} Y_{k_{\tau_{\delta},a^\star}, a^\star} ^2 \\
    &\qquad + \frac{\tilde N_{k_{\tau_{\delta},\hat a_{\tau_{\delta}}}, \hat a_{\tau_{\delta}}}}{\sigma^2}Y_{k_{\tau_{\delta},\hat a_{\tau_{\delta}}}, \hat a_{\tau_{\delta}}} ( \hat \mu_{k_{\tau_{\delta},\hat a_{\tau_{\delta}}}, \hat a_{\tau_{\delta}}} - \mu_{\hat a_{\tau_{\delta}}}) + \frac{\tilde N_{k_{\tau_{\delta},a^\star}, a^\star}}{\sigma^2} Y_{k_{\tau_{\delta},a^\star}, a^\star} ( \hat \mu_{k_{\tau_{\delta},a^\star}, a^\star} - \mu_{a^\star}) \\
    &\quad < \frac{1}{2} c^{G,\epsilon}_{\hat a_{\tau_{\delta}},a^\star}(\tilde N_{k_{\tau_{\delta}}},\delta) + \frac{\sqrt{2}}{\epsilon \sigma} \sum_{c \in \{\hat a_{\tau_{\delta}},a^\star\}} \sqrt{\frac{h(\tilde N_{k_{\tau_{\delta},c}},\delta)}{\tilde N_{k_{\tau_{\delta},c}}}} \log \left(\frac{3 K \zeta(s) k_{\tau_{\delta},c}^{s} }{\delta} \right) \: .
\end{align*}
This is a contradiction, hence $\cE_{\delta}\cap \{\tau_{\delta} < + \infty, \: \hat a_{\tau_{\delta}} \ne a^\star  \: (\hat \mu_{k_{\tau_{\delta},\hat a_{\tau_{\delta}}},\hat a_{\tau_{\delta}}}  - \hat \mu_{k_{\tau_{\delta},a^\star},a^\star})_{+} < \epsilon/2  \} = \emptyset$.

{\bf Case 2.}
Under $\cE_{\delta}\cap \{\tau_{\delta} < + \infty, \: \hat a_{\tau_{\delta}} \ne a^\star, \: (\hat \mu_{k_{\tau_{\delta},\hat a_{\tau_{\delta}}},\hat a_{\tau_{\delta}}}  - \hat \mu_{k_{\tau_{\delta},a^\star},a^\star})_{+} \ge \epsilon/2  \}$, we have
\begin{align*}
    &\frac{1}{2\sigma^2}  \log \left( \frac{3 K \zeta(s) \max_{c \in \{\hat a_{\tau_{\delta}},a^\star\}}k_{\tau_{\delta},c}}{\delta}\right) + \frac{\epsilon}{2\sqrt{2}\sigma} \sum_{c \in \{\hat a_{\tau_{\delta}},a^\star\}} \sqrt{\tilde N_{k_{\tau_{\delta},c}} h(\tilde N_{k_{\tau_{\delta},c}},\delta)} \\
    &\quad \le \frac{\epsilon(\tilde \mu_{k_{\tau_{\delta},\hat a_{\tau_{\delta}}},\hat a_{\tau_{\delta}}}  - \tilde \mu_{k_{\tau_{\delta},a^\star},a^\star})}{4\sigma^2(1/\tilde N_{k_{\tau_{\delta},\hat a_{\tau_{\delta}}},\hat a_{\tau_{\delta}}}+ 1/\tilde N_{k_{\tau_{\delta},a^\star},a^\star})} \\
    &\quad\le \frac{\epsilon}{4\sigma^2} \min\{\tilde N_{k_{\tau_{\delta},\hat a_{\tau_{\delta}}},\hat a_{\tau_{\delta}}}, \tilde N_{k_{\tau_{\delta},a^\star},a^\star})\} (\tilde \mu_{k_{\tau_{\delta},\hat a_{\tau_{\delta}}},\hat a_{\tau_{\delta}}}  - \tilde \mu_{k_{\tau_{\delta},a^\star},a^\star}) \\
    &\quad = \frac{\epsilon}{4\sigma^2}  \inf_{y \ge x} \left\{ \tilde N_{k_{\tau_{\delta},\hat a_{\tau_{\delta}}}, \hat a_{\tau_{\delta}}} | \tilde \mu_{k_{\tau_{\delta},\hat a_{\tau_{\delta}}}, \hat a_{\tau_{\delta}}} - x| + \tilde N_{k_{\tau_{\delta},a^\star}, a^\star} | \tilde \mu_{k_{\tau_{\delta},a^\star}, a^\star} - y|\right\} \\
    &\quad \le \frac{\epsilon}{4\sigma^2} \tilde N_{k_{\tau_{\delta},\hat a_{\tau_{\delta}}}, \hat a_{\tau_{\delta}}} | \tilde \mu_{k_{\tau_{\delta},\hat a_{\tau_{\delta}}}, \hat a_{\tau_{\delta}}} - \mu_{\hat a_{\tau_{\delta}}}| +  \frac{\epsilon}{4\sigma^2}  \tilde N_{k_{\tau_{\delta},a^\star}, a^\star} | \tilde \mu_{k_{\tau_{\delta},a^\star}, a^\star} - \mu_{a^\star}| \\
    &\quad \le \frac{\epsilon}{4\sigma^2} \tilde N_{k_{\tau_{\delta},\hat a_{\tau_{\delta}}}, \hat a_{\tau_{\delta}}} | \hat \mu_{k_{\tau_{\delta},\hat a_{\tau_{\delta}}}, \hat a_{\tau_{\delta}}} - \mu_{\hat a_{\tau_{\delta}}}| +  \frac{\epsilon}{4\sigma^2}  \tilde N_{k_{\tau_{\delta},a^\star}, a^\star} | \hat \mu_{k_{\tau_{\delta},a^\star}, a^\star} - \mu_{a^\star}| \\
    &\qquad + \frac{\epsilon}{4\sigma^2} \tilde N_{k_{\tau_{\delta},\hat a_{\tau_{\delta}}}, \hat a_{\tau_{\delta}}} |Y_{k_{\tau_{\delta},\hat a_{\tau_{\delta}}}, \hat a_{\tau_{\delta}}}| +  \frac{\epsilon}{4\sigma^2}  \tilde N_{k_{\tau_{\delta},a^\star}, a^\star} |Y_{k_{\tau_{\delta},a^\star}, a^\star} | \\
    &\quad < \frac{1}{2\sigma^2}  \log \left( \frac{3 K \zeta(s) \max_{c \in \{\hat a_{\tau_{\delta}},a^\star\}}k_{\tau_{\delta},c}}{\delta}\right) + \frac{\epsilon}{2\sqrt{2}\sigma} \sum_{c \in \{\hat a_{\tau_{\delta}},a^\star\}} \sqrt{\tilde N_{k_{\tau_{\delta},c}} h(\tilde N_{k_{\tau_{\delta},c}},\delta)} \: .
\end{align*}
This is a contradiction, hence $\cE_{\delta}\cap \{\tau_{\delta} < + \infty, \: \hat a_{\tau_{\delta}} \ne a^\star  \: (\hat \mu_{k_{\tau_{\delta},\hat a_{\tau_{\delta}}},\hat a_{\tau_{\delta}}}  - \hat \mu_{k_{\tau_{\delta},a^\star},a^\star})_{+} \ge \epsilon/2  \} = \emptyset$.

{\bf Summary.} Putting both cases together \Modif{and take as a specific parameter $s=2$} yields the result.

\section{Expected Sample Complexity of \adaptt{} and \texorpdfstring{\adapttt{}}{}}
\label{app:TTUCB_global_upper_bounds}

Let $\beta \in (0,1)$, $\epsilon \in \R^{\star}_{+},$ and $\bnu$ be a bandit instance consisting of $\sigma$-sub-Gaussian distributions with distinct means $\boldsymbol{\mu} \in \R^{K}$, i.e. $\min_{a \ne b}|\mu_{a} - \mu_{b}| > 0$.
For conciseness, we denote $\Delta_{a} \defn \mu_{a^\star} - \mu_{a}$, $\Delta_{\min} \defn \min_{a \ne a^\star} \Delta_{a}$, 
and $\Delta_{\max} \defn \max_{a \ne a^\star} \Delta_{a}$.
For Gaussian, we define the unique $\beta$-optimal allocations $\omega^{\star}_{\mathrm{KL},\beta}(\boldsymbol{\nu}) = \{(\omega^{\star}_{\beta, a})_{a \in [K]}\}$ and $\omega^\star_{\mathrm{KL}, \beta}(\bm{\nu}_{G,\epsilon})= \{(\omega^{\star}_{\epsilon,\beta, a})_{a \in [K]}\}$ as
\begin{equation*} 
	\omega^{\star}_{\mathrm{KL},\beta}(\boldsymbol{\nu}) \defn \argmax_{\omega \in \Sigma_K, \omega_{a^\star} = \beta}  \min_{a \ne a^\star} \frac{\Delta_{a}^2}{1/\beta + 1/\omega_{a}} \: \text{ , } \: \omega^\star_{\mathrm{KL}, \beta}(\bm{\nu}_{G,\epsilon}) \defn \argmax_{\omega \in \Sigma_K, \omega_{a^\star} = \beta}  \min_{a \ne a^\star} \frac{\Delta_{a} \min\{\epsilon/2, \Delta_{a}\}}{1/\beta + 1/\omega_{a}}  \: .
\end{equation*}
At equilibrium, we have equality of the transportation costs (see~\cite{jourdan2022non} for example), namely
\begin{equation} \label{eq:equality_equilibrium}
	\forall a \ne a^\star, \quad \frac{\Delta_{a}^2}{1/\beta + 1/\omega^{\star}_{\beta,a}} = 2 \sigma^2 T^{\star}_{\mathrm{KL},\beta}(\boldsymbol{\nu})^{-1} \quad \text{,} \quad \frac{\Delta_{a} \min\{\epsilon/2, \Delta_{a}\}}{1/\beta + 1/\omega^{\star}_{\epsilon,\beta,a}} = 2 \sigma^2 T^\star_{\mathrm{KL}, \beta}(\bm{\nu}_{G,\epsilon})^{-1} \: .
\end{equation}
Our proof follows the unified sample complexity analysis of Top Two algorithms from~\cite{jourdan2022top}.

Let $\gamma> 0$.
Let $\omega \in \Sigma_K$ be any allocation over arms such that $\min_{a} \omega_{a} > 0$.
We denote by $T_{\boldsymbol{\mu}, \gamma}(\omega)$ the \emph{convergence time} towards $\omega$, which is a random variable quantifying the number of samples required for the global empirical allocations $N_n/(n-1)$ to be $\gamma$-close to $\omega$ for any subsequent time, namely
\begin{equation} \label{eq:rv_T_eps_alloc}
    T_{\boldsymbol{\mu}, \gamma}(\omega) \defn \inf \left\{ T \ge 1 \mid \forall n \geq T, \: \left\| \frac{N_{n}}{n-1} - \omega \right\|_{\infty} \leq \gamma \right\}  \: .
\end{equation}

As the \adaptt{} and \adapttt{} algorithms share the same leader, all results solely on the leader applies to both of them.
As the \adaptt{} and \adapttt{} algorithms consider a TC challenger with different transportation costs, all results involving the challenger should be slightly modified.
Except if specified otherwise, all the results presented in the following hold for both algorithms.

The rest of Appendix~\ref{app:TTUCB_global_upper_bounds} is organised as follows.
After recalling some technical results (Appendix~\ref{app:ssec_technical_result}), we prove sufficient exploration (Appendix~\ref{app:ssec_sufficient_exploration})
Second, we prove that convergence towards the $\beta$-optimal allocation (Appendix~\ref{app:ssec_convergence_beta_optimal_allocation}) in finite time.
Third, we explicit the cost of doubling and forgetting (Appendix~\ref{app:ssec_cost_doubling_and_forgetting}).
Finally, we conclude the proof of Theorems~\ref{thm:sample_complexity_TTUCB_global_v1} and~\ref{thm:sample_complexity_TTUCB_global_v2} (Appendix~\ref{app:ssec_asymptotic_upper_bound_sample_complexity}).

\subsection{Technical Results}
\label{app:ssec_technical_result}

Before delving into the proofs, we first recall some useful technical results.

\textit{Doubling trick.} 
Due to the doubling, the growth of the counts is exponential (Lemma~\ref{lem:counts_at_phase_switch}).
\begin{lemma} \label{lem:counts_at_phase_switch}
	For all $(a,k) \in [K] \times \N$ s.t. $\bE_{\bnu}[T_{k}(a)] < +\infty$, $N_{T_{k}(a), a} = 2^{k-1}$ and $\tilde N_{k,a} = 2^{k-2}$.
\end{lemma}
\begin{proof}
	Let $a \in [K]$.
	After initialisation, we have $k = 1$, $T_{1}(a) = \arms+1$ and $N_{T_{1}(a),a} =1$.
	Using the definition of the phase switch, it is direct to see that $N_{T_{2}(a), a} = 2$ and $\tilde N_{2,a} = 1$ when $\bE_{\bnu}[T_{2}(a)] < +\infty$.

	Now, we proceed by recurrence.
	Suppose that $N_{T_{k}(a), a} = 2^{k-1}$ and $\tilde N_{k,a} = 2^{k-2}$ when $\bE_{\bnu}[T_{k}(a)] < +\infty$.
	If $\bE_{\bnu}[T_{k+1}(a)] < +\infty$, then it means that the phase $k$ ends for arm $a$ almost surely.
	Since we sample only one arm at each round, at the beginning of phase $k+1$ for arm $a$, we have $N_{T_{k+1}(a), a} = 2 N_{T_{k}(a), a} = 2^{k}$ by using the definition of the phase switch.
	Then, we have directly that	$\tilde N_{k+1,a} = N_{T_{k+1}(a),a} -  N_{T_{k}(a),a} = 2^{k} - 2^{k-1} = 2^{k-1}$.
\end{proof}

\textit{Tracking.}
We denote by $N^{a}_{n,b} \defn \sum_{t \in [n-1]} \indi{B_t = a, \: a_t = C_t  = b}$ the number of times the arm $b$ was pulled while the arm $a$ was the leader, and by $L_{n,a} \defn \sum_{t \in [n-1]} \indi{B_t = a}$ the number of times arm $a$ was the leader.
\begin{lemma}[Lemma 2.2 in~\citealt{jourdan2022non}] \label{lem:tracking_guaranty_light}
	For all $n > K$ and all $a \in [K]$, we have $-1/2 \le N_{n,a}^{a} - \beta L_{n,a}  \le 1$.
\end{lemma}

\textit{Concentration results.}
In order to control the randomness of $(\tilde \mu_{k_{a},a})_{a \in [K]}$, we use a standard concentration result on the empirical mean of sub-Gaussian random variables and on sub-exponential observations (Lemma~\ref{lem:W_concentration_gaussian}).
Since Bernoulli distributions are $1/2$-sub-Gaussian and the absolute value of a Laplace is an exponential distribution, Lemma~\ref{lem:W_concentration_gaussian} applies to our setting.
\begin{lemma}\label{lem:W_concentration_gaussian}
There exists a sub-Gaussian random variable $W_\mu$ such that, almost surely,
\begin{align*}
 \forall a \in [K], \: \forall k_{a} \in \N, \quad |\hat \mu_{k_{a},a} - \mu_a| \le W_\mu \sqrt{\frac{\log (e + \tilde N_{k_{a},a})}{\tilde N_{k_{a},a}}} \: .
\end{align*}
There exists a sub-exponential random variable $W_{\epsilon}$ such that, almost surely,
\begin{align*}
 \forall a \in [K], \: \forall k_{a} \in \N, \quad |Y_{k_{a},a}| \le W_{\epsilon} \frac{\log (e + k_{a})}{\tilde N_{k_{a}, a}} \: .
\end{align*}
In particular, any random variable which is polynomial in $(W_{\epsilon},W_{\mu})$ has a finite expectation.
\end{lemma}
\begin{proof}
The first part is a known result, e.g. Appendix E.2 in~\cite{jourdan2022top}.
Let 
\[
	W_{\epsilon} \defn \sup_{a \in [K]}\sup_{k_{a} \in \N} \frac{\tilde N_{k_{a},a}|Y_{k_{a},a}|}{\log(e + k_{a})} \: .
\]
By definition, we have that, almost surely,
\begin{align*}
 \forall a \in [K], \: \forall k_{a} \in \N, \quad |Y_{k_{a},a}| \le W_{\epsilon} \frac{\log (e + k_{a})}{\tilde N_{k_{a}, a}} \: .
\end{align*}
Since $\tilde N_{k,i}|Y_{k,i}| \sim \cE(\epsilon )$, Lemma 72 in~\cite{jourdan2022top} yields that $W_{\epsilon}$ is a sub-exponential random variable.
Since $W_\mu$ is sub-Gaussian and $W_{\epsilon}$ is a sub-exponential, any random variable which is polynomial in $(W_{\epsilon},W_{\mu})$ has a finite expectation.
\end{proof}

\textit{Inversion results.}
Lemma~\ref{lem:property_W_lambert} gathers properties on the function $\overline{W}_{-1}$, which is used in the literature to obtain concentration results.
\begin{lemma}[\citealt{jourdan_2022_DealingUnknownVariance}] \label{lem:property_W_lambert}
	Let $\overline{W}_{-1}(x) \defn - W_{-1}(-e^{-x})$ for all $x \ge 1$, where $W_{-1}$ is the negative branch of the Lambert $W$ function.
	The function $\overline{W}_{-1}$ is increasing on $(1, +\infty)$ and strictly concave on $(1, + \infty)$.
	In particular, $\overline{W}_{-1}'(x) = \left(1-\frac{1}{\overline{W}_{-1}(x)} \right)^{-1}$ for all $x > 1$.
	Then, for all $y \ge 1$ and $x \ge 1$,
	\[
	 	\overline{W}_{-1}(y) \le x \quad \iff \quad y \le x - \log(x) \: .
	\]
	Moreover, for all $x > 1$,
	\[
	 x + \log(x) \le \overline{W}_{-1}(x) \le x + \log(x) + \min \left\{ \frac{1}{2}, \frac{1}{\sqrt{x}} \right\} \: .
	\]
\end{lemma}

Lemma~\ref{lem:inversion_upper_bound} is an inversion result to upper bound a time, which is implicitly defined.
It is a direct consequence of Lemma~\ref{lem:property_W_lambert}.
\begin{lemma} \label{lem:inversion_upper_bound}
	Let $\overline{W}_{-1}$ defined in Lemma~\ref{lem:property_W_lambert}.
	Let $A > 0$, $B > 0$ such that $B/A + \log A >  1$ and
	\begin{align*}
	&C(A, B) = \sup \left\{ x \mid \: x < A \log x + B \right\} \: .
	\end{align*}
	Then, $C(A,B) < h_{1}(A,B)$ with $h_{1}(z,y) = z \overline{W}_{-1} \left(y/z  + \log z\right)$.
\end{lemma}
\begin{proof}
	Since $B/A + \log A >  1$, we have $C(A, B) \ge A$, hence
	\[
	C(A, B) = \sup \left\{ x \mid \: x < A \log (x) + B \right\} = \sup \left\{ x \ge A \mid \: x < A \log (x) + B \right\} \: .
	\]
	Using Lemma~\ref{lem:property_W_lambert} yields that
	\begin{align*}
		x \ge A \log x + B  \: \iff \: \frac{x}{A} - \log \left( \frac{x}{A} \right) \ge \frac{B}{A} + \log A \: \iff \: x \ge A \overline{W}_{-1} \left( \frac{B}{A} + \log A \right) \: .
	\end{align*}
\end{proof}

\subsection{Sufficient Exploration}
\label{app:ssec_sufficient_exploration}

The first step of in the generic analysis of Top Two algorithms~\cite{jourdan2022top} consists in showing sufficient exploration.
The main idea is that, if there are still undersampled arms, either the leader or the challenger will be among them.
Therefore, after a long enough time, no arm can still be undersampled.
We emphasise that there are multiple ways to select the leader/challenger pair in order to ensure sufficient exploration.
Therefore, other choices of leader/challenger pair would yield similar results.

Given an arbitrary phase $p \in \N$, we define the sampled enough set, i.e. the arms having reached phase $p$, and the arm with highest mean in this set (when not empty) as
\begin{equation} \label{eq:def_sampled_enough_sets}
	S_{n}^{p} = \{a \in [K] \mid N_{n,a} \ge 2^{p-1} \} \quad \text{and} \quad a_n^\star = \argmax_{a \in S_{n}^{p}} \mu_{a} \: .
\end{equation}
Since $\min_{a \neq b}|\mu_a - \mu_b| > 0$, $a_n^\star$ is unique.
Let $p \in \N$ such that $(p-1)/4 \in \N$.
We define the highly and the mildly under-sampled sets as
\begin{equation} \label{eq:def_undersampled_sets}
	U_n^p \defn \{a \in [K]\mid N_{n,a} < 2^{(p-1)/2} \} \quad \text{and} \quad V_n^p \defn \{a \in [K] \mid N_{n,a} < 2^{3(p-1)/4}\} \: .
\end{equation}
Those arms have not reached phase $(p-1)/2$ and phase $3(p-1)/4$, respectively.

\textit{Lemma~\ref{lem:UCB_ensures_suff_explo} shows that, when the leader is sampled enough, it is the arm with highest true mean among the sampled enough arms.}
\begin{lemma} \label{lem:UCB_ensures_suff_explo}
 Let $S_{n}^{p}$ and $a_n^\star$ as in (\ref{eq:def_sampled_enough_sets}).
	There exists $p_0$ with $\bE_{\bnu}[\exp(\alpha p_0)] < +\infty$ for all $\alpha > 0$ such that if $p \ge p_0$, for all $n$ such that $S_{n}^{p} \neq \emptyset$, $B_{n} \in S_{n}^{p}$ implies that $B_{n} = a_n^\star = \argmax_{a \in S_{n}^{p}} \tilde \mu_{k_{n,a},a} $.
\end{lemma}
\begin{proof}
	Let $p_{0}$ to be specified later.
	Let $p \ge p_0$.
	Let $n \in \N$ such that $S_{n}^{p} \ne \emptyset$, where $S_{n}^{p}$ and $a_n^\star$ as in Equation~\eqref{eq:def_sampled_enough_sets}.
	Let $(k_{n,a})_{a \in [K]}$ be the phases indices for all arms.
	Since $N_{n,a} \ge 2^{p-1}$ for all $a \in S_{n}^{p}$, we have $k_{n,a} \ge p$ and $\tilde N_{k_{n,a},a} \ge 2^{p-2}$ by using Lemma~\ref{lem:counts_at_phase_switch}. Using	Lemma~\ref{lem:W_concentration_gaussian}, we obtain that
\begin{align*}
	\tilde \mu_{k_{n,a_n^\star}, a_n^\star} &\ge \mu_{a_n^\star} - W_{\mu} \sqrt{\frac{\log(e+2^{p-2})}{2^{p-2}}} - W_{\epsilon}  \frac{\log(e+p)}{2^{p-2}} \: , \\
	\tilde \mu_{k_{n,a},a}  &\le  \mu_{a} + W_{\mu} \sqrt{\frac{\log(e+2^{p-2})}{2^{p-2}}} + W_{\epsilon}  \frac{\log(e+p)}{2^{p-2}} \: ,\quad \forall a \in S_{n}^{p} \setminus \{ a_n^\star \}.
\end{align*}
Here, we use that $x \to \log(e + x)/x$ is decreasing.

Let $\overline \Delta_{\min} = \min_{a \ne b} |\mu_{a} - \mu_{b}|$. By assumption on the considered instances, we know that $\overline \Delta_{\min} > 0$.
Let $p_{1} = \lceil \log_2 (X_{1}-e) \rceil + 2$ and $ p_2 = \lceil \log_2 ( (X_{2} - e  -2)\log 2  +  1) \rceil + 2$ with
\begin{align*}
	X_{1} &= \sup \left\{ x > 1 \mid \: x \le 64 \overline\Delta_{\min}^{-2} W_{\mu}^2 \log x + e \right\} \le h_{1} ( 64 \overline\Delta_{\min}^{-2} W_{\mu}^2 ,\: e) \: , \\
	X_{2} &= \sup \left\{ x > 1 \mid \: x \le \frac{8}{\log 2} \overline\Delta_{\min}^{-1} W_{\epsilon} \log x  + e + 2 - 1/\log 2 \right\} \le h_{1} ( 8 \overline\Delta_{\min}^{-1} W_{\epsilon} /\log 2 ,\: 4) \: ,
\end{align*}
where we used Lemma~\ref{lem:inversion_upper_bound}, and $h_{1}$ defined therein.
Then, for all $p \in \N $ such that $p \ge \max\{p_1, p_{2}\} + 1$ and all $n \in \N$ such that $S_{n}^{p} \ne \emptyset$, we have $\tilde \mu_{k_{n,a_n^\star}, a_n^\star} \ge \mu_{a_n^\star} - \overline\Delta_{\min}/4$ and $\tilde \mu_{k_{n,a},a}  \le  \mu_{a} + \overline\Delta_{\min}/4$ for all $a \in S_{n}^{p} \setminus \{ a_n^\star \}$, hence $a_n^\star = \argmax_{a \in [K]} \tilde \mu_{k_{n,a},a}$.

We have, for all $\alpha \in \R_{+}$,
\[
	\exp(\alpha p_{1}) \le e^{3\alpha} (X_{1}-e)^{\alpha/\log 2} \quad \text{hence} \quad \bE_{\bnu}[\exp(\alpha p_{1})] < + \infty \: ,
\]
where we used Lemma~\ref{lem:W_concentration_gaussian} and $h_{1}(x,e) \sim_{x \to +\infty} x \log x$ to obtain that $\exp(\alpha p_{1}) $ is at most polynomial in $W_{\mu}$.
Likewise, we obtain that $\bE_{\bnu}[\exp(\alpha p_{2})] < + \infty$ for all $\alpha \in \R_{+}$.

Let us define the UCB indices by $I_{k_{n,a},a} = \tilde \mu_{k_{n,a}, a} + \sqrt{k_{n,a}/\tilde N_{k_{n,a},a}} + k_{n,a}/(\epsilon \tilde N_{k_{n,a},a})$.
Using the above, we have
\begin{align*}
	& I_{k_{n,a_n^\star},a_n^\star} \ge \mu_{a_n^\star} - W_{\mu} \sqrt{\frac{\log(e+2^{p-2})}{2^{p-2}}} - W_{\epsilon}  \frac{\log(e+p)}{2^{p-2}}  \: , \\
\forall a \in S_{n}^{p} \setminus \{ a_n^\star \}, \quad  & I_{k_{n,a},a}  \le  \mu_{a} + W_{\mu} \sqrt{\frac{\log(e+2^{p-2})}{2^{p-2}}} + W_{\epsilon}  \frac{\log(e+p)}{2^{p-2}} + \sqrt{\frac{p}{2^{p-2}}} + \frac{p}{\epsilon 2^{p-2}} \: ,
\end{align*}
where we used Lemma~\ref{lem:counts_at_phase_switch} and the fact that $x \to \log(e + x)/x$ and $x \to x 2^{2-x}$ are decreasing function for $x \ge 2$.
Let $p_{3} = \lceil \log_2 X_{3} \rceil + 2$ and $ p_{4} = \lceil \log_2 X_{4} \rceil + 2$ with
\begin{align*}
	X_{3} &= \sup \left\{ x > 1 \mid \: x \le 64 \overline\Delta_{\min}^{-2} ( \log_{2} x + 2) \right\} \le h_{1} ( 64 \overline\Delta_{\min}^{-2}/\log 2 ,\: 128 \overline\Delta_{\min}^{-2}) \: , \\
	X_{4} &= \sup \left\{ x > 1 \mid \: x \le 8 \epsilon^{-1}\overline\Delta_{\min}^{-1}   ( \log_{2} x + 2) \right\} \le h_{1} ( 8 \epsilon^{-1}\overline\Delta_{\min}^{-1} / \log 2,\:  16 \overline\Delta_{\min}^{-1} \epsilon^{-1}) \: ,
\end{align*}
where we used Lemma~\ref{lem:inversion_upper_bound}, and $h_{1}$ defined therein.
We highlight that $(p_{3}, p_{4})$ are deterministic values, hence their expectation is finite.
Then, for all $p \in \N $ such that $p \ge p_{0} = \max\{p_{1}, p_{2}, p_{3}, p_{4}\} + 1$ and all $n \in \N$ such that $S_{n}^{p} \ne \emptyset$, we have $I_{k_{n,a_n^\star},a_n^\star} \ge \mu_{a_n^\star} - \overline\Delta_{\min}/4$ and $I_{k_{n,a},a}   \le  \mu_{a} + \overline\Delta_{\min}/2$ for all $a \in S_{n}^{p} \setminus \{ a_n^\star \}$, hence $a_n^\star = B_n$ since we have $B_n = \argmax_{a \in [K]} I_{k_{n,a},a} $.

Since we have $\bE_{\bnu}[\exp(\alpha p_{0})] < + \infty$ for all $\alpha \in \R_{+}$, this concludes the proof.
\end{proof}

\textit{Lemma~\ref{lem:fast_rate_emp_tc_aeps} shows that the transportation costs between the sampled enough arms with largest true means and the other sampled enough arms are increasing fast enough.}
\begin{lemma} \label{lem:fast_rate_emp_tc_aeps}
 Let $S_{n}^{p}$ and $a_n^\star$ as in Eq.~\eqref{eq:def_sampled_enough_sets}.
 There exists $p_1$ with $\bE_{\bnu}[\exp(\alpha p_1)] < +\infty$ for all $\alpha > 0$ such that if $p \ge p_1$, for all $n$ such that $S_{n}^{p} \neq \emptyset$, for all $b \in  S_{n}^{p} \setminus \{a_n^\star\} $, we have
 \begin{align*}
 &\text{[\adaptt{}]} \quad \frac{\tilde \mu_{k_{n,a_n^\star},a_n^\star} - \tilde \mu_{k_{n,b},b}}{\sqrt{1/\tilde N_{k_{n,a_n^\star}, a_n^\star}+ 1/\tilde N_{k_{n,b},b}}} \geq  2^{p/2} C_{\mu} \: , \\
&\text{[\adapttt{}]} \quad \frac{(\tilde \mu_{k_{n,a_n^\star},a_n^\star} - \tilde \mu_{k_{n,b},b}) \min\{\epsilon/2, \tilde \mu_{k_{n,a_n^\star},a_n^\star} - \tilde \mu_{k_{n,b},b}\} }{1/\tilde N_{k_{n,a_n^\star}, a_n^\star}+ 1/\tilde N_{k_{n,b},b}} \geq  2^{p} C_{\mu} \: ,
 \end{align*}
 where $C_{\mu} > 0$ is a problem dependent constant.
 \end{lemma}
 \begin{proof}
 	Let $p_{1}$ to be specified later.
 	Let $p \ge p_1$.
 	Let $n \in \N$ such that $S_{n}^{p} \ne \emptyset$, where $S_{n}^{p}$ and $a_n^\star$ as in Equation~\eqref{eq:def_sampled_enough_sets}.
 	Let $(k_{n,a})_{a \in [K]}$ be the phases indices for all arms.
 	Since $N_{n,a} \ge 2^{p-1}$ for all $a \in S_{n}^{p}$, we have $k_{n,a} \ge p$ and $\tilde N_{k_{n,a},a} \ge 2^{p-2}$ by using Lemma~\ref{lem:counts_at_phase_switch}.
	Let $\overline \Delta_{\min} = \min_{a \ne b} |\mu_{a} - \mu_{b}|$, which satisfies $\overline \Delta_{\min} > 0$ by assumption on the instance considered.
	
Using Lemma~\ref{lem:W_concentration_gaussian}, for all $b \in S_{n}^{p} \setminus \{ a_n^\star \}$, we obtain
 \begin{align*}
 	&\tilde \mu_{k_{n,a_n^\star}, a_n^\star} - \tilde \mu_{k_{n,b},b} \ge \overline \Delta_{\min} - W_{\mu} \sqrt{\frac{\log(e+2^{p-2})}{2^{p-4}}} - W_{\epsilon}  \frac{\log(e+p)}{2^{p-3}} \: .
 \end{align*}
 Let $p_{3} = \lceil \log_2 ((X_{3}-e)/4) \rceil + 4$ and $ p_2 = \lceil \log_2 ( (X_{2} - e  -3)\log 2  +  1) \rceil + 3$ with
 \begin{align*}
 	X_{3} &= \sup \left\{ x > 1 \mid \: x \le 64 \overline\Delta_{\min}^{-2} W_{\mu}^2 \log x + e \right\} \le h_{1} ( 64 \overline\Delta_{\min}^{-2} W_{\mu}^2 ,\: e ) \: , \\
 	X_{2} &= \sup \left\{ x > 1 \mid \: x \le 4 \overline\Delta_{\min}^{-1} W_{\epsilon} \log x + e + 3 -1/\log 2 \right\} \le h_{1} ( 4 \overline\Delta_{\min}^{-1} W_{\epsilon}  ,\: 5) \: ,
 \end{align*}
 where we used Lemma~\ref{lem:inversion_upper_bound}, and $h_{1}$ defined therein.
 Then, for all $p \in \N $ such that $p \ge p_{1} = \max\{p_{3}, p_{2}\} + 1$ and all $n \in \N$ such that $S_{n}^{p} \ne \emptyset$, we have, for all $b \in S_{n}^{p} \setminus \{ a_n^\star \}$,
\begin{align*}
 &\tilde \mu_{k_{n,a_n^\star}, a_n^\star} - \tilde \mu_{k_{n,b},b} \ge  \overline \Delta_{\min} /2\: .
\end{align*}
As in the proof of Lemma~\ref{lem:UCB_ensures_suff_explo}, we obtain that $\bE_{\bnu}[\exp(\alpha p_{1})] < + \infty$ for all $\alpha \in \R_{+}$.

Then, for all $b \in S_{n}^{p} \setminus \{ a_n^\star \}$, we have
\[
	  \frac{\tilde \mu_{k_{n,a_n^\star},a_n^\star} - \tilde \mu_{k_{n,b},b}}{\sqrt{1/\tilde N_{k_{n,a_n^\star}, a_n^\star}+ 1/\tilde N_{k_{n,b},b}}} \ge 2^{p/2} \frac{\overline\Delta_{\min}}{2^{5/2}}   \: ,
\]
where we used that $\min\{\tilde N_{k_{n,a_n^\star}, \tilde N_{k_{n,b},b}}\} \ge 2^{p-2}$.
 Setting $C_{\mu} = \overline \Delta_{\min} /2^{5/2}$ yields the first result.

The second result is obtained similarly by taking $C_{\mu} = \frac{\overline\Delta_{\min}}{16} \min\{\epsilon/2, \frac{\overline\Delta_{\min}}{2}\}$
 \end{proof}

\textit{Lemma~\ref{lem:small_tc_undersampled_arms_aeps} shows that the transportation costs between sampled enough arms and undersampled arms are not increasing too fast.}
 \begin{lemma} \label{lem:small_tc_undersampled_arms_aeps}
 	Let $S_{n}^{p}$ be as in Eq.~\eqref{eq:def_sampled_enough_sets}.
  For all $p \ge 1$ and all $n$ such that $S_{n}^{p} \neq \emptyset$, for all $a \in S_{n}^{p}$ and $b \notin  S_{n}^{p}$,
  \begin{align*}
 &\text{[\adaptt{}]} \quad \frac{\tilde \mu_{k_{n,a},a} - \tilde \mu_{k_{n,b},b}}{\sqrt{1/\tilde N_{k_{n,a}, a}+ 1/\tilde N_{k_{n,b},b}}} \leq  2^{p/2} D_{\mu} +  2W_{\mu}\sqrt{\log(e + 2^{p-2})} + 2W_{\epsilon} \log(e+p) \: , \:   \\
&\text{[\adapttt{}]} \quad  \frac{(\tilde \mu_{k_{n,a},a} - \tilde \mu_{k_{n,b},b}) \min\{ \epsilon/2, \tilde \mu_{k_{n,a},a} - \tilde \mu_{k_{n,b},b}\}}{1/\tilde N_{k_{n,a}, a}+ 1/\tilde N_{k_{n,b},b}} \\
& \qquad \qquad \qquad\qquad \qquad \qquad \leq  2^{p} D_{\mu} +  8 W_{\mu}^2 \log(e + 2^{p-2}) + 8 W_{\epsilon}^2 \log(e+p)^2 \: ,
 \end{align*}
 where $D_{\mu} > 0$ is a problem dependent constant and $(W_\mu, W_{\epsilon})$ are the random variables defined in Lemma~\ref{lem:W_concentration_gaussian}.
 \end{lemma}
 \begin{proof}
	 	Let $p \ge 1$.
	 	Let $n \in \N$ such that $S_{n}^{p} \ne \emptyset$, where $S_{n}^{p}$ as in Equation~\eqref{eq:def_sampled_enough_sets}.
	 	Let $(k_{n,a})_{a \in [K]}$ be the phases indices for all arms.
	 	Since $N_{n,a} \ge 2^{p-1}$ for all $a \in S_{n}^{p}$, we have $k_{n,a} \ge p$ and $\tilde N_{k_{n,a},a} \ge 2^{p-2}$ by using Lemma~\ref{lem:counts_at_phase_switch}.
		Likewise, $N_{n,a} < 2^{p-1}$ for all $a \notin S_{n}^{p}$, we have $k_{n,a} < p$ and $\tilde N_{k_{n,a},a} < 2^{p-2}$.
		Let $\overline \Delta_{\max} = \min_{a \ne b} |\mu_{a} - \mu_{b}|$, which satisfies $\overline \Delta_{\max} > 0$ by assumption on the instance considered.
		Using	Lemma~\ref{lem:W_concentration_gaussian}, for all $a \in S_{n}^{p}$ and $b \notin S_{n}^{p}$, we obtain
		\begin{align*}
			\frac{\tilde \mu_{k_{n,a},a} - \tilde \mu_{k_{n,b},b}}{\sqrt{1/\tilde N_{k_{n,a}, a}+ 1/\tilde N_{k_{n,b},b}}} &\le \sqrt{\tilde N_{k_{n,b},b}}(\tilde \mu_{k_{n,a},a} - \tilde \mu_{k_{n,b},b}) \\
			&\le \sqrt{\tilde N_{k_{n,b},b}}(\mu_{a} - \mu_{b}) + 2W_{\mu}\sqrt{\log(e + \tilde N_{k_{n,b},b})} + 2W_{\epsilon} \frac{\log(e+k_{n,b})}{\sqrt{\tilde N_{k_{n,b},b}}} \\
			&\le 2^{(p-2)/2} \overline \Delta_{\max} +  2W_{\mu}\sqrt{\log(e + 2^{p-2})} + 2W_{\epsilon} \log(e+p)
		\end{align*}
		where we used that $\tilde N_{k_{n,b},b} \ge 1$, $k_{n,b} < p$, $\tilde N_{k_{n,b},b} < 2^{p-2} \le \tilde N_{k_{n,a},a}$ and $x \to \log(e + x)/x$ is decreasing.
		Taking $D_{\mu} = \overline \Delta_{\max}/2$ yields the first result.

		The proof of the second result follows along the same line by noting that this transportation cost is lower than the other:
		\begin{align*}
				&\frac{(\tilde \mu_{k_{n,a},a} - \tilde \mu_{k_{n,b},b}) \min\{ \epsilon/2, \tilde \mu_{k_{n,a},a} - \tilde \mu_{k_{n,b},b}\}}{1/\tilde N_{k_{n,a}, a}+ 1/\tilde N_{k_{n,b},b}} \\
				&\le 2\tilde N_{k_{n,b},b}(\hat \mu_{k_{n,a},a} - \hat \mu_{k_{n,b},b})^2 + 2\tilde N_{k_{n,b},b}(Y_{k_{n,a},a} - Y_{k_{n,b},b})^2\\
			&\le 2^{p} \overline \Delta_{\max}^2 +  8 W_{\mu}^2\log(e + 2^{p-2}) + 8 W_{\epsilon}^2 \log(e+p)^2 \: .
		\end{align*}
		Taking $D_{\mu} = \overline \Delta_{\max}^2$ yields the result.
 \end{proof}

\textit{Lemma~\ref{lem:TC_ensures_suff_explo} shows that the challenger is mildly undersampled if the leader is not mildly undersampled.}
\begin{lemma} \label{lem:TC_ensures_suff_explo}
	Let $V_{n}^{p}$ be as in Equation~\eqref{eq:def_undersampled_sets}.
	There exists $p_2$ with $\bE_{\bnu}[\exp(\alpha p_2)] < +\infty$ for all $\alpha > 0$ such that if $p \ge p_2$, for all $n$ such that $U_n^p \neq \emptyset$, $B_{n} \notin V_{n}^{p}$ implies $C_{n} \in V_{n}^{p}$.
\end{lemma}
\begin{proof}
	Let $p_{2}$ to be specified later.
	Let $p \ge p_{2}$.
	Let $n \in \N$ such that $U_n^p \neq \emptyset$ and $V_{n}^{p} \ne [K]$, where $U_{n}^{p} \subseteq V_{n}^{p}$ are defined in Equation~\eqref{eq:def_undersampled_sets}.
	In the following, we suppose that $B_{n} \notin V_{n}^{p}$.

		Let $(k_{n,a})_{a \in [K]}$ be the phases indices for all arms.
		Let $p_{0}$ as in Lemma~\ref{lem:UCB_ensures_suff_explo}.
	Let $b_n^\star = \argmax_{ b \notin V_{n}^{p}} \mu_{b}$.
	Then, for all $p \ge 4p_{0}/3 -1/3$ and all $n$ such that $B_{n} \notin V_{n}^{p}$, Lemma~\ref{lem:UCB_ensures_suff_explo} yields that $B_{n} =  b_n^\star = \argmax_{a \notin V_{n}^{p} } \tilde \mu_{k_{n,a},a} $.

	Let $p_{1}$ and $C_{\mu}$ as in Lemma~\ref{lem:fast_rate_emp_tc_aeps}, and $D_{\mu}$ as in Lemma~\ref{lem:small_tc_undersampled_arms_aeps}.
	Then, for all $p \ge \frac{4}{3}\max\{p_{0}, p_{1}\} -1/3$ and all $n$ such that $B_{n} \notin V_{n}^{p}$, we have $B_{n} =  b_n^\star $ and
	\begin{align*}
		  \forall b \notin V_{n}^{p}, \quad \frac{\tilde \mu_{k_{n,b_n^\star},b_n^\star} - \tilde \mu_{k_{n,b},b}}{\sqrt{1/\tilde N_{k_{n,b_n^\star}, b_n^\star}+ 1/\tilde N_{k_{n,b},b}}} &\ge  2^{(3p+1)/8} C_{\mu} \: , \\
			\forall b \in U_{n}^{p}, \quad \frac{\tilde \mu_{k_{n,b_n^\star},b_n^\star} - \tilde \mu_{k_{n,b},b}}{\sqrt{1/\tilde N_{k_{n,b_n^\star}, b_n^\star}+ 1/\tilde N_{k_{n,b},b}}} &\le 2^{(p+1)/4} D_{\mu}  +2W_{\mu}\sqrt{\log(e + 2^{(p+1)/2-2})}   \\
			&\quad + 2W_{\epsilon} \log(e+(p+1)/2) \: ,
	\end{align*}
	where we used the first results of Lemmas~\ref{lem:fast_rate_emp_tc_aeps} and~\ref{lem:small_tc_undersampled_arms_aeps}.
	Let $p_{3} = 16 \lceil \log_{2} (2D_{\mu}/C_{\mu}) \rceil + 1$, then we have $2^{(p-1)/16} > \frac{D_{\mu}}{C_{\mu}}  $ for all $p \ge p_{3}$.
	Let $p_{4} = \frac{16}{9} \lceil \log_{2} X_{4} \rceil + 25$ and $p_{5} = \frac{32}{9} \lceil \log_{2} X_{5} \rceil + 7$ where
	\begin{align*}
		X_{4} &= \sup \left\{ x > 1 \mid x \le  \frac{W_{\mu}^2}{C_{\mu}^2} \log(e + x^{8/9} 2^{25/18-3/4}) \right\}\: ,\\
		X_{5} &= \sup \left\{ x > 1 \mid x \le   \frac{2W_{\epsilon}}{C_{\mu}} \log(e+ 4 + 32\log_{2} (x)/18) \right\}\: .
	\end{align*}
As in the proof of Lemma~\ref{lem:UCB_ensures_suff_explo}, using Lemma~\ref{lem:W_concentration_gaussian} yields that $\bE_{\bnu}[\exp(\alpha p_{4})] < + \infty$ and $\bE_{\bnu}[\exp(\alpha p_{5})] < + \infty$ for all $\alpha \in \R_{+}$.
Let $p_{2} = \max\{p_3, p_4, p_5, 4\max\{p_{0}, p_{1}\}/3 -1/3\} + 1$.
Then, we have shown that for all $p \ge p_{2}$, for all $n$ such that $B_{n} \notin V_{n}^{p}$, we have $B_{n} =  b_n^\star $ and
\[
	 \min_{b \notin V_{n}^{p}}\frac{\tilde \mu_{k_{n,b_n^\star},b_n^\star} - \tilde \mu_{k_{n,b},b}}{\sqrt{1/\tilde N_{k_{n,b_n^\star}, b_n^\star}+ 1/\tilde N_{k_{n,b},b}}} > \max_{b \in U_{n}^{p}}  \frac{\tilde \mu_{k_{n,b_n^\star},b_n^\star} - \tilde \mu_{k_{n,b},b}}{\sqrt{1/\tilde N_{k_{n,b_n^\star}, b_n^\star}+ 1/\tilde N_{k_{n,b},b}}}  \: ,
\]
Therefore, by definition of the TC challenger $C_{n} = \argmin_{b \ne b_n^\star}\frac{\tilde \mu_{k_{n,b_n^\star},b_n^\star} - \tilde \mu_{k_{n,b},b}}{\sqrt{1/\tilde N_{k_{n,b_n^\star}, b_n^\star}+ 1/\tilde N_{k_{n,b},b}}} $, we obtain that $C_{n} \in V_{n}^{p}$.
Otherwise, there would be a contradiction given that we assumed that $U_{n}^{p} \ne \emptyset$.
Given all the condition exhibited above, it is direct to see that $\bE_{\bnu}[\exp(\alpha p_2)] < +\infty$ for all $\alpha > 0$. 
This concludes the proof for the \adaptt{} algorithm.

For the \adapttt{} algorithm, the proof is done similarly based on the second results of Lemmas~\ref{lem:fast_rate_emp_tc_aeps} and~\ref{lem:small_tc_undersampled_arms_aeps}.
As above, we can construct $\tilde p_{3}$, with $\bE_{\nu}[\exp(\alpha \tilde p_{3})] < + \infty$ for all $\alpha \in \R_{+}$, such that for all $p \ge \tilde p_{3}$, we have
	\[
	2^{(3p+1)/4} C_{\mu} > 2^{(p+1)/2} D_{\mu} + 8 W_{\mu}^2  \log(e + 2^{(p+1)/2-2})  + 8 W_{\epsilon}^2 \log(e+(p+1)/2)^2 \: .
	\]
Let $\tilde p_{2} = \max\{\tilde p_3, 4\max\{p_{0}, p_{1}\}/3 -1/3\} + 1$.
Then, we have shown that for all $p \ge \tilde p_{2}$, for all $n$ such that $B_{n} \notin V_{n}^{p}$, we have $B_{n} =  b_n^\star $ and
\begin{align*}
		 &\min_{b \notin V_{n}^{p}}\frac{(\tilde \mu_{k_{n,b_n^\star},b_n^\star} - \tilde \mu_{k_{n,b},b}) \min\{ \epsilon/2, \tilde \mu_{k_{n,b_n^\star},b_n^\star} - \tilde \mu_{k_{n,b},b}\}}{1/\tilde N_{k_{n,b_n^\star}, b_n^\star}+ 1/\tilde N_{k_{n,b},b}} \\
		 &> \max_{b \in U_{n}^{p}} \frac{(\tilde \mu_{k_{n,b_n^\star},b_n^\star} - \tilde \mu_{k_{n,b},b}) \min\{ \epsilon/2, \tilde \mu_{k_{n,b_n^\star},b_n^\star} - \tilde \mu_{k_{n,b},b}\}}{1/\tilde N_{k_{n,b_n^\star}, b_n^\star}+ 1/\tilde N_{k_{n,b},b}}  \: .
\end{align*}
Then, we conclude similarly by using the definition of the TC challenger.
\end{proof}

\noindent\textit{Lemma~\ref{lem:suff_exploration} shows that all the arms are sufficient explored for large enough $n$.}
\begin{lemma} \label{lem:suff_exploration}
	There exists $N_0$ with $\bE_{\bnu}[N_0] < + \infty$ such that for all $ n \geq N_0$ and all $a \in [K]$,
	\[
		N_{n,a} \geq \sqrt{n/K} \quad \text{and} \quad k_{n,a} \ge  \frac{\log (n/K)}{2 \log 2} + 1 \: .
	\]
\end{lemma}
\begin{proof}
		Let $p_0$ and $p_2$ as in Lemmas~\ref{lem:UCB_ensures_suff_explo} and~\ref{lem:TC_ensures_suff_explo}.
		Combining Lemmas~\ref{lem:UCB_ensures_suff_explo} and~\ref{lem:TC_ensures_suff_explo} yields that, for all $p \ge p_{3} = \max\{p_2, 4p_{0}/3 -1/3\}$ and all $n$ such that $U_{n}^{p} \ne \emptyset$, we have $B_n \in V_{n}^{p}$ or $C_n \in V_{n}^{p}$.
		We have $\bE_{\bnu}[2^{p_2}] < + \infty$.
		We have $2^{p-1} \ge K 2^{3(p-1)/4}$ for all $p \ge p_{4} = 4 \lceil \log_2 K \rceil + 1$.
		Let $p \ge \max \{p_3 , p_4\}$.

			Suppose towards contradiction that $U_{ K 2^{p-1}}^{p}$ is not empty.
			Then, for any $1 \leq t \leq K 2^{p-1}$, $U_{t}^{p}$ and $V_{t}^{p}$ are non empty as well.
			Using the pigeonhole principle, there exists some $a \in [K]$ such that $N_{ 2^{p-1}, a} \geq 2^{3(p-1)/4}$.
			Thus, we have $\left|V_{ 2^{p-1}}^{p}\right| \leq K-1$.
			Our goal is to show that $\left|V_{ 2^{p}}^{p}\right| \leq K-2$.
			A sufficient condition is that one arm in $V_{ 2^{p-1}}^{p}$ is pulled at least $2^{3(p-1)/4}$ times between $ 2^{p-1}$ and $ 2^{p}-1$.

			{\textit{Case 1.}} Suppose there exists $a \in V_{ 2^{p-1} }^{p}$ such that $L_{ 2^{p},a} - L_{ 2^{p-1},a} \ge \frac{2^{3(p-1)/4}}{\beta} + 3/(2\beta)$.
			Using Lemma~\ref{lem:tracking_guaranty_light}, we obtain
			\[
			N^{a}_{ 2^{p}, a} - N^{a}_{ 2^{p-1} , a}  \ge \beta(L_{ 2^{p},a} - L_{ 2^{p-1},a}) - 3/2 \ge 2^{3(p-1)/4} \: ,
			\]
			hence $a$ is sampled $2^{3(p-1)/4}$ times between $ 2^{p-1}$ and $ 2^{p}-1$.

			{\textit{Case 2.}} Suppose that for all $a \in V_{ 2^{p-1} }^{p}$, we have $L_{ 2^{p},a} - L_{ 2^{p-1},a} < 2^{3(p-1)/4}/\beta + 3/(2\beta)$.
			Then,
			\[
			\sum_{a \notin V_{ 2^{p-1} }^{p} }(L_{ 2^{p},a} - L_{ 2^{p-1},a}) \ge 2^{p-1} - K \left( 2^{3(p-1)/4}/\beta + 3/(2\beta) \right)
			\]
			Using Lemma~\ref{lem:tracking_guaranty_light}, we obtain
			\begin{align*}
			\left|\sum_{a \notin V_{ 2^{p-1}}^{p} }(N^{a}_{ 2^{p},a} - N^{a}_{ 2^{p-1},a}) - \beta \sum_{a \notin V_{ 2^{p-1}}^{p} }(L_{ 2^{p},a} - L_{ 2^{p-1},a}) \right| \le  3(K-1)/2  \: .
			\end{align*}
			Combining all the above, we obtain
			\begin{align*}
			&\sum_{a \notin V_{ 2^{p-1} }^{p} }(L_{ 2^{p},a} - L_{ 2^{p-1},a}) - \sum_{a \notin V_{ 2^{p-1}}^{p} }(N^{a}_{ 2^{p},a} - N^{a}_{ 2^{p-1},a}) \\
			&\ge (1-\beta) \sum_{a \notin V_{ 2^{p-1} }^{p} }(L_{ 2^{p},a} - L_{ 2^{p-1},a}) - 3(K-1)/2 \\
			&\ge (1-\beta) \left(  2^{p-1} - K \left( 2^{3(p-1)/4}/\beta + 3/(2\beta) \right) \right) - 3(K-1)/2 \ge K 2^{3(p-1)/4} \: ,
			\end{align*}
			where the last inequality is obtained for $p \ge p_{5}$ with
			\[
			p_{5} = \sup\left\{ p \in \N \mid (1-\beta) \left(  2^{p-1} - K \left( 2^{3(p-1)/4}/\beta + \frac{3}{2\beta} \right) \right) - \frac{3}{2}(K-1) < K 2^{3(p-1)/4} \right\}\: .
			\]
			The LHS summation is exactly the number of times where an arm $a \notin V_{ 2^{p-1}}^{p}$ was leader but wasn't sampled, hence
			\[
			\sum_{t = 2^{p-1}}^{2^{p} - 1} \indi{B_{t} \notin V_{2^{p-1}}^{p}, \: a_t = C_{t}} \ge K 2^{3(p-1)/4}
			\]

			For any $2^{p-1} \leq t \leq 2^{p} - 1$, $U_{t}^{p}$ is non-empty, hence we have $B_{t} \notin V_{2^{p-1}}^{p}$ (hence $B_{t} \notin V_{t}^{p} $) implies $C_{t}  \in V_{t}^{p} \subseteq V_{2^{p-1}}^{p} $.
			Therefore, we have shown that
			\[
			\sum_{t = 2^{p-1}}^{2^{p} - 1} \indi{ a_t \in V_{2^{p-1}}^{p}} \ge \sum_{t = 2^{p-1}}^{2^{p} - 1} \indi{B_{t} \notin V_{2^{p-1}}^{p}, \: a_t = C_{t}} \ge K 2^{3(p-1)/4} \: .
			\]
			Therefore, there is at least one arm in $V_{ 2^{p-1}}^{p}$ that is sampled $2^{3(p-1)/4}$ times between $ 2^{p-1}$ and $ 2^{p}-1$.

			In summary, we have shown $\left|V_{ 2^{p} }^{p}\right| \leq K-2$ for all $p \ge p_{6} = \max \{p_3 , p_4, p_{5}\}$.
			By induction, for any $1 \leq k \leq K$, we have $\left|V_{ k 2^{p-1}}^{p}\right| \leq K-k$, and finally $U_{ K 2^{p-1}}^{p} = \emptyset$ for all $p \ge p_{6}$.
		Defining $N_{0} = K 2^{p_{6} - 1}$, we have $\bE_{\bnu}[N_{0}] < +\infty$ by using Lemmas~\ref{lem:UCB_ensures_suff_explo} and~\ref{lem:TC_ensures_suff_explo} for $p_{3} = \max\{p_2, 4p_{0}/3 -1/3\}$ and $p_{4}$ and $p_{5}$ are deterministic.
			For all $n \ge N_{0}$, we let $2^{p-1} = \frac{n}{K}$.
			Then, by applying the above, we have $U_{ K 2^{p-1}}^{p} = U^{\log_{2}(n/K) + 1}_{n}$ is empty, which shows that $N_{n,a} \geq \sqrt{n/K}$ for all $a \in [K]$.
			Using Lemma~\ref{lem:counts_at_phase_switch}, we obtain that $k_{n,a} \ge \frac{\log (n/K)}{2 \log 2} + 1$ for all $a \in [K]$.
			This concludes the proof.
\end{proof}

\subsection{Convergence Towards \texorpdfstring{$\beta$}{}-optimal Allocation}
\label{app:ssec_convergence_beta_optimal_allocation}

The second step of in the generic analysis of Top Two algorithms~\cite{jourdan2022top} is to show the convergence of the empirical proportions towards the $\beta$-optimal allocation.
First, we show that the leader coincides with the best arm. Hence, the tracking procedure will ensure that the empirical proportion of time we sample it is exactly $\beta$.
Second, we show that a sub-optimal arm whose empirical proportion overshoots its $\beta$-optimal allocation will not be sampled next as challenger.
Therefore, this ``overshoots implies not sampled'' mechanism will ensure the convergence towards the $\beta$-optimal allocation.
We emphasise that there are multiple ways to select the leader/challenger pair in order to ensure convergence towards the $\beta$-optimal allocation.
Therefore, other choices of leader/challenger pair would yield similar results.
Note that our results heavily rely on having obtained sufficient exploration first.

\textit{Convergence for the best arm.}
Lemma~\ref{lem:starting_phase_best_arm_identified} exhibits a random phase which ensures that the leader and the candidate answer are equal to the best arm  for large enough $n$.
\begin{lemma} \label{lem:starting_phase_best_arm_identified}
	Let $N_{0}$ be as in Lemma~\ref{lem:suff_exploration}.
	There exists $N_1 \ge N_{0}$ with $\bE_{\bnu}[N_1] < + \infty$ such that, for all $ n \geq N_1$, we have $\hat a_{n} = B_{n} = a^\star$.
\end{lemma}
\begin{proof}
	Let $k \ge 1$.
	Suppose that $\bE_{\bnu}[\max_{a \in [K]} T_{k}(a)] < +\infty$.
	Then, Lemma~\ref{lem:counts_at_phase_switch} yields that $N_{T_{k}(a), a} = 2^{k-1}$ and $\tilde N_{k,a} = 2^{k-2}$.
	Using Lemma~\ref{lem:W_concentration_gaussian}, we obtain that
\begin{align*}
	&\tilde \mu_{k, a^\star} \ge \mu_{a^\star} - W_{\mu} \sqrt{\frac{\log(e+2^{k-2})}{2^{k-2}}} - W_{\epsilon}  \frac{\log(e+k)}{2^{k-2}} \: , \\
	\forall a \ne a^\star, \quad &\tilde \mu_{k,a}  \le  \mu_{a} + W_{\mu} \sqrt{\frac{\log(e+2^{k-2})}{2^{k-2}}} + W_{\epsilon}  \frac{\log(e+k)}{2^{k-2}} \: .
\end{align*}
Let $p_{1} = \lceil \log_2 (X_{1}-e) \rceil + 2$ and $ p_2 = \lceil \log_2 ( X_{2} - e - 1) \rceil + 2$ with
\begin{align*}
	X_{1} &= \sup \left\{ x > 1 \mid \: x \le 64 \Delta_{\min}^{-2} W_{\mu}^2 \log x + e \right\} \le h_{1} ( 64 \Delta_{\min}^{-2} W_{\mu}^2 ,\: e) \: , \\
	X_{2} &= \sup \left\{ x > 1 \mid \: x \le 8 \Delta_{\min}^{-1} W_{\epsilon} \log x + e + 1 \right\} \le h_{1} ( 8 \Delta_{\min}^{-1} W_{\epsilon}  ,\: e+1) \: , \\
	X_{2} &\ge \sup \left\{ x > 1 \mid \: x \le 8 \Delta_{\min}^{-1} W_{\epsilon} \log (e + 2 + \log x )  \right\} \: ,
\end{align*}
where we used Lemma~\ref{lem:inversion_upper_bound}, and $h_{1}$ defined therein.
Then, for all $k\in \N^{K}$ such that $\min_{a \in [K]} k_{a} > p_{0} = \max\{p_1, p_{2}\}$ such that $\bE_{\bnu}[\max_{a \in [K]} T_{k_{a}}(a)] < +\infty$, we have $\tilde \mu_{k, a^\star} \ge \mu_{a^\star} - \Delta_{\min}/4$ and $\tilde \mu_{k,a}  \le  \mu_{a} + \Delta_{\min}/4$ for all $a \ne a^\star$, hence $a^\star = \argmax_{a \in [K]} \tilde \mu_{k,a}$.
We have, for all $\alpha \in \R_{+}$,
\[
	\exp(\alpha p_{1}) \le e^{3\alpha} (X_{1}-e)^{\alpha/\log 2} \quad \text{hence} \quad \bE_{\bnu}[\exp(\alpha p_{1})] < + \infty \: ,
\]
where we used Lemma~\ref{lem:W_concentration_gaussian} and $h_{1}(x,e) \sim_{x \to +\infty} x \log x$ to obtain that $\exp(\alpha p_{1}) $ is at most polynomial in $W_{\mu}$.
Likewise, we obtain that $\bE_{\bnu}[\exp(\alpha p_{2})] < + \infty$ for all $\alpha \in \R_{+}$.
Therefore, we have $\bE_{\bnu}[\exp(\alpha p_0)] < + \infty$ for all $\alpha \in \R_{+}$.

Let us define the UCB indices by $I_{k,a} = \tilde \mu_{k, a} + \sqrt{k/\tilde N_{k,a}} + k/(\epsilon \tilde N_{k,a})$.
Using the above, we have
\begin{align*}
	& I_{k,a^\star} \ge \mu_{a^\star} - W_{\mu} \sqrt{\frac{\log(e+2^{k-2})}{2^{k-2}}} - W_{\epsilon}  \frac{\log(e+k)}{2^{k-2}} + \frac{k}{\epsilon 2^{k-2}} \: , \\
\forall a \ne a^\star, \quad & I_{k,a}  \le  \mu_{a} + W_{\mu} \sqrt{\frac{\log(e+2^{k-2})}{2^{k-2}}} + W_{\epsilon}  \frac{\log(e+k)}{2^{k-2}} + \frac{k}{\epsilon 2^{k-2}}\: .
\end{align*}
Therefore, we have $a^\star = \argmax_{a \in [K]} I_{k,a}$ for all $k\in \N^{K}$ such that $\min_{a } k_{a} > \max\{p_1, p_{2}\}$ such that $\bE_{\bnu}[\max_{a \in [K]} T_{k_{a}}(a)] < +\infty$.

Let $N_{0}$ as in Lemma~\ref{lem:suff_exploration}.
Using Lemma~\ref{lem:suff_exploration}, we obtain that, for all $n \ge N_{0}$ and all $a \in [K]$, $k_{n,a} \ge \log_{2}(n/K)/2 + 1$.
Therefore, we obtain $\min_{a \in [K]} k_{n,a} > \max\{ p_{1}, p_{2}\}$ is implied by $n \ge N_{1} = \max\{ K 4^{\max\{ p_{1}, p_{2}\}}, N_{0}\}$.
Using the above, we conclude that $\bE_{\bnu}[N_{1}] < + \infty$ and $\hat a_{n} = B_{n} = a^\star$ for all $n \ge N_{1}$.
\end{proof}

\textit{Lemma~\ref{lem:convergence_best_arm_to_beta} shows that that the pulling proportion of the best arm converges towards $\beta$, provided the phase defined in Lemma~\ref{lem:starting_phase_best_arm_identified} is reached in finite time for all arms.}
\begin{lemma} \label{lem:convergence_best_arm_to_beta}
	Let $\gamma > 0$, and $N_1$ be as in Lemma~\ref{lem:starting_phase_best_arm_identified}.
	There exists a deterministic constant $C_{0} \ge 1$ such that, for all $n \ge C_0 N_{1}$,
	\[
		\left| \frac{N_{n,a^\star}}{n-1} - \beta\right| \le \gamma \: .
	\]
\end{lemma}
\begin{proof}
	Let $\gamma > 0$.
Let $N_1$ as in Lemma~\ref{lem:starting_phase_best_arm_identified}.
Let $M \ge N_1$.
Using Lemma~\ref{lem:starting_phase_best_arm_identified}, we obtain $B_{n} = a^\star$ for all $n \ge M$.
Therefore, we obtain $L_{n,a^\star} \ge n - M$ and $\sum_{a \neq a^\star} N^{a}_{n,a^\star} \le M $ for all $n \ge M $.
Using Lemma~\ref{lem:tracking_guaranty_light} yields that
\begin{align*}
	\left| \frac{N_{n, a^\star}}{n-1} - \beta \right| &\le \frac{|N^{a^\star}_{n, a^\star} - \beta L_{n,a^\star}|}{n-1} + \beta \left|\frac{L_{n,a^\star}}{n-1} - 1 \right| + \frac{1}{n-1} \sum_{a \neq a^\star} N^{a}_{n,a^\star} \\
	&\le \frac{1}{2(n-1)} + \beta \frac{2(M-1)}{n-1} \le \gamma \: ,
\end{align*}
where the last inequality is obtained by taking $n \ge \max\{M,(1/2+2\beta (M-1))/\gamma + 1\}$.
\end{proof}

\textit{Convergence for the sub-optimal arms.}
Lemma~\ref{lem:starting_phase_overshooting_challenger_implies_not_sampled_next} exhibits a random phase which ensures that if a sub-optimal arm overshoots its $\beta$-optimal allocation then it cannot be selected as challenger for large enough $n$.
\begin{lemma}\label{lem:starting_phase_overshooting_challenger_implies_not_sampled_next}
	Let $\gamma > 0$. Let $N_{1}$ and $C_{0}$ be as in Lemma~\ref{lem:starting_phase_best_arm_identified} and~\ref{lem:convergence_best_arm_to_beta}.
	There exists $N_{2} \ge C_{0}N_{1}$ with $\bE_{\bnu}[N_{2}] < + \infty$ such that, for all $n \ge N_{2}$,
	\begin{align*}
 & \quad \exists a \ne a^\star, \quad \frac{N_{n,a}}{n-1} \ge \gamma + \begin{cases}
 \omega^{\star}_{\beta,a} & \text{[\adaptt{}]} \\
 \omega^{\star}_{\epsilon,\beta,a} & \text{[\adapttt{}]} 
 \end{cases} \quad \implies \quad C_{n} \ne a  \: , 
 \end{align*}
\end{lemma}
\begin{proof}
	Let $\gamma > 0$ and $\tilde \gamma > 0$.
	Let $N_1$ as in Lemma~\ref{lem:starting_phase_best_arm_identified} and $C_{0}$ as in Lemma~\ref{lem:convergence_best_arm_to_beta} for $\tilde \gamma$.
	Let $n \ge C_0 N_{1}$.

	Let $a \neq a^\star$ such that $\frac{N_{n,a}}{n-1} \ge  \omega^{\star}_{\beta,a} + \gamma$.
	Suppose towards contradiction that $\frac{N_{n,b}}{n-1} > \omega^{\star}_{\beta,a}$ for all $b \notin \{ a^\star, a\}$.
	Then, for all $n \ge C_0 N_{1}$, we have
	\begin{align*}
		1 - \beta + \tilde \gamma \ge 1 - \frac{N_{n,a^\star}}{n-1} = \sum_{b \ne a^\star} \frac{N_{n,b}}{n-1} > \gamma + \sum_{b \ne a^\star} \omega^{\star}_{\beta,b} = 1-\beta + \gamma\: ,
	\end{align*}
	which yields a contradiction for $\tilde \gamma \le \gamma$.
	Therefore, for all $n \ge C_{0} N_{1}$, we have
	\[
		\exists a \ne a^\star, \quad \frac{N_{n,a}}{n-1} \ge  \omega^{\star}_{\beta,a} + \gamma	 \quad \implies \quad  \exists b \notin \{ a^\star, a\}, \quad \frac{N_{n,b}}{n-1} \le \omega^{\star}_{\beta,b} \: .
	\]
	Then, we have
	\[
		\sqrt{ \frac{1 + N_{n,a^\star}/N_{n,b}}{1+N_{n,a^\star}/N_{n,a}}}  \ge \sqrt{ \frac{1 + (\beta - \tilde \gamma)/\omega^{\star}_{\beta,b}}{1+(\beta + \tilde \gamma)/(\omega^{\star}_{\beta,a} +  \gamma)}} \: .
	\]
	In the following, we use Lemma~\ref{lem:W_concentration_gaussian} and similar manipulations as in the proof of Lemma~\ref{lem:starting_phase_best_arm_identified}.
	Therefore, we obtain that, for all $c \ne a^\star$,
	\begin{align*}
	\left|\tilde \mu_{k_{n,a^\star},a^\star} - \tilde \mu_{k_{n,c},c} - \Delta_{c} \right| &\le W_\mu \left( \sqrt{\frac{\log (e + 2^{k_{n,a^\star}-2})}{2^{k_{n,a^\star}-2}}} +  \sqrt{\frac{\log (e + 2^{k_{n,c}-2})}{2^{k_{n,c}-2}}} \right) \\
	& \quad +  W_{\epsilon} \left( \frac{\log(e+k_{n,a^\star})}{2^{k_{n,a^\star}-2}} + \frac{\log(e+k_{n,c})}{2^{k_{n,c}-2}} \right) \: .
	\end{align*}
	Let $p_{3} = \lceil \log_2 (X_{1}-e) \rceil + 2$ and $ p_2 = \lceil \log_2 ( X_{2} - e - 1) \rceil + 2$ with
	\begin{align*}
		X_{3} &= \sup \left\{ x > 1 \mid \: x \le 16 \eta^{-2} W_{\mu}^2 \log x + e \right\} \le h_{1} ( 16 \eta^{-2} W_{\mu}^2 ,\: e) \: , \\
		X_{2} &= \sup \left\{ x > 1 \mid \: x \le 4 \eta^{-1} W_{\epsilon} \log x + e + 1 \right\} \le h_{1} ( 4 \eta^{-1} W_{\epsilon}  ,\: e+1) \: , \\
		X_{2} &\ge \sup \left\{ x > 1 \mid \: x \le 4 \eta^{-1} W_{\epsilon} \log (e + 2 + \log x )  \right\} \: ,
	\end{align*}
	where we used Lemma~\ref{lem:inversion_upper_bound}, and $h_{1}$ defined therein.
		We have, for all $\alpha \in \R_{+}$,
		\[
			\exp(\alpha p_{3}) \le e^{3\alpha} (X_{3}-e)^{\alpha/\log 2} \quad \text{hence} \quad \bE_{\bnu}[\exp(\alpha p_{3})] < + \infty \: ,
		\]
		where we used Lemma~\ref{lem:W_concentration_gaussian} and $h_{1}(x,e) \sim_{x \to +\infty} x \log x$ to obtain that $\exp(\alpha p_{3}) $ is at most polynomial in $W_{\mu}$.
		Likewise, we obtain that $\bE_{\bnu}[\exp(\alpha p_{2})] < + \infty$ for all $\alpha \in \R_{+}$.

		Using Lemma~\ref{lem:suff_exploration} (with $C_{0}N_{1}\ge N_{1} \ge N_{0}$), we obtain that, for all $n \ge C_{0}N_{1}$ and all $a \in [K]$, $k_{n,a} \ge \log_{2}(n/K)/2 + 1$.
		Therefore, we obtain $\min_{a \in [K]} k_{n,a} > \max\{ p_{2}, p_{3}\}$ is implied by $n \ge N_{2} = \max\{ K 4^{\max\{ p_{3}, p_{2}\}}, C_{0}N_{1}\}$.
		Using the above, we conclude that $\bE_{\bnu}[N_{2}] < + \infty$ and $\max_{c \ne a^\star} |\tilde \mu_{k_{n,a^\star},a^\star} - \tilde \mu_{k_{n,c},c} - \Delta_{c} | \le \eta$ for all $n \ge N_{2}$.

	Then, for all $n \ge N_{2}$, we have $B_n =a^\star$ and
	\begin{align*}
		\frac{\tilde \mu_{k_{n,a^\star},a^\star} - \tilde \mu_{k_{n,a},a} }{\tilde \mu_{k_{n,a^\star},a^\star} - \tilde \mu_{k_{n,b},b}}  \sqrt{ \frac{1 + N_{n,a^\star}/N_{n,b}}{1+N_{n,a^\star}/N_{n,a}}}  \ge \frac{ \Delta_{a} - \eta}{ \Delta_{b} + \eta } \sqrt{ \frac{1 + (\beta - \tilde \gamma)/\omega^{\star}_{\beta,b}}{1+(\beta + \tilde \gamma)/(\omega^{\star}_{\beta,a} +  \gamma)}} > 1 \: ,
	\end{align*}
	where the last inequality is obtained by taking $\eta$ and $\tilde \gamma$ sufficiently small and by using~\eqref{eq:equality_equilibrium}
	\[
	 \frac{ \Delta_{a}}{ \Delta_{b} } \sqrt{ \frac{1 + \beta /\omega^{\star}_{\beta,b}}{1+\beta /\omega^{\star}_{\beta,a} }} = 1 \: .
	\]
	Therefore, we have shown that $B_n =a^\star$ and
	\[
	\frac{\tilde \mu_{k_{n,a^\star},a^\star} - \tilde \mu_{k_{n,a},a} }{\sqrt{1/N_{n,a^\star} + 1/N_{n,a}} }   > \frac{\tilde \mu_{k_{n,a^\star},a^\star} - \tilde \mu_{k_{n,b},b}}{\sqrt{1/N_{n,a^\star} + 1/N_{n,b}}}\quad \text{hence} \quad C_{n} \neq a \: .
	\]
	This concludes the proof of the first result. 

	For the \adapttt{} algorithm, the proof is done similarly.
	As above, we can construct $\tilde N_2$ with $\bE_{\bnu}[\tilde N_{2}] < + \infty$ such that, for all $n \ge \tilde N_2$, we have $B_n =a^\star$ and
	\begin{align*}
		&\frac{(\tilde \mu_{k_{n,a^\star},a^\star} - \tilde \mu_{k_{n,a},a}) \min\{\epsilon/2, \tilde \mu_{k_{n,a^\star},a^\star} - \tilde \mu_{k_{n,a},a}\}}{(\tilde \mu_{k_{n,a^\star},a^\star} - \tilde \mu_{k_{n,b},b}) \min\{\epsilon/2, \tilde \mu_{k_{n,a^\star},a^\star} - \tilde \mu_{k_{n,b},b}\}}   \frac{1 + N_{n,a^\star}/N_{n,b}}{1+N_{n,a^\star}/N_{n,a}}  \\
		& \ge \frac{ (\Delta_{a} - \eta) \min\{\epsilon/2, \Delta_{a} - \eta\} }{(\Delta_{b} + \eta) \min \{\epsilon/2, \Delta_{b} + \eta\} } \frac{1 + (\beta - \tilde \gamma)/\omega^{\star}_{\epsilon,\beta,b}}{1+(\beta + \tilde \gamma)/(\omega^{\star}_{\epsilon,\beta,a} +  \gamma)} > 1 \: ,
	\end{align*}
	where the last inequality is obtained by taking $\eta$ and $\tilde \gamma$ sufficiently small and by using~\eqref{eq:equality_equilibrium}
	\[
	 \frac{ \Delta_{a}\min \{\epsilon/2, \Delta_{a} \}}{ \Delta_{b} \min \{\epsilon/2, \Delta_{b} \}} \sqrt{ \frac{1 + \beta /\omega^{\star}_{\epsilon,\beta,b}}{1+\beta /\omega^{\star}_{\epsilon,\beta,a} }} = 1 \: .
	\]
	Then, we conclude similarly by using the definition of the TC challenger.
\end{proof}

\textit{Lemma~\ref{lem:convergence_others_arms_to_beta_allocation} shows that that the pulling proportion of the best arm converges towards $\beta$ for large enough $n$.}
\begin{lemma} \label{lem:convergence_others_arms_to_beta_allocation}
	Let $\gamma > 0$ and $T_{\boldsymbol{\mu}, \gamma}(w)$ as in Eq.~\eqref{eq:rv_T_eps_alloc}.
	Then, we have $\bE_{\bnu} [T_{\boldsymbol{\mu}, \gamma}(\omega^\star_{\beta})]  < + \infty$ (\adaptt{}) and $\bE_{\bnu} [T_{\boldsymbol{\mu}, \gamma}(\omega^\star_{\epsilon,\beta})]  < + \infty$ (\adapttt{}).
\end{lemma}
\begin{proof}
		Let $\gamma > 0$ and $\tilde \gamma > 0$.
		Let $N_2$ as in Lemma~\ref{lem:starting_phase_overshooting_challenger_implies_not_sampled_next} for $\tilde \gamma$.
		Let $M \ge N_{2}$.
		Using Lemmas~\ref{lem:starting_phase_best_arm_identified},~\ref{lem:convergence_best_arm_to_beta} and~\ref{lem:starting_phase_overshooting_challenger_implies_not_sampled_next} for all $n \ge M$, we obtain that $B_{n} = a^\star$, $\left| \frac{N_{n,a^\star}}{n-1} - \beta\right| \le \tilde \gamma$ and
		\[
		\exists a \ne a^\star, \quad \frac{N_{n,a}}{n-1} \ge  \omega^{\star}_{\beta,a} + \tilde \gamma \quad \implies \quad C_{n} \ne a \: .
		\]
		For all $a \ne a^\star$, let us define $t_{n,a}(\tilde \gamma)=\max \left\{t \mid M \le t \le n, \: N_{t, a}/(n-1) < \omega^{\star}_{\beta,a}+ \tilde \gamma \right\}$.
		Since $N_{t, a}/(n-1) \le N_{t, a}/(t-1)$ for $t\leq n$, we have
		\begin{align*}
		\frac{N_{n,a}}{n-1} &\le \frac{M-1}{n-1} + \frac{1}{n-1} \sum_{t=M}^{n} \indi{a_t = C_{t} = a} \\
		&\le \frac{M-1}{n-1} + \frac{1}{n-1} \sum_{t=M}^{n} \indi{\frac{N_{t, a}}{n-1} < \omega^{\star}_{\beta,a}+ \tilde \gamma, \: a_t = C_{t} = a} \\
		&\le \frac{M-1}{n-1} + \frac{N_{t_{n,a}(\tilde \gamma),a}}{n-1} < \frac{M-1}{n-1} + \omega^{\star}_{\beta,a} + \tilde \gamma \: .
		\end{align*}
		The second inequality uses Lemma~\ref{lem:starting_phase_overshooting_challenger_implies_not_sampled_next}, and the two last inequalities use the definition of $t_{n,a}(\tilde \gamma)$.
		Using that $\sum_{a \in [K]}\frac{N_{n,a}}{n-1} = \sum_{a \in [K]} \omega^{\star}_{\beta,a} = 1$, we obtain
		\begin{align*}
		\frac{N_{n,a}}{ n - 1} &= 1 - \sum_{b \neq a} \frac{N_{n,a}}{ n -1}
		\geq 1 - \sum_{b \neq a} \left(  \omega^{\star}_{\beta,b} + \tilde \gamma + \frac{M-1}{n-1}\right)
		=  \omega^{\star}_{\beta,a} - (K-1) \left( \tilde \gamma  + \frac{M-1}{n-1} \right) \: .
		\end{align*}
		Taking $\tilde \gamma \le  \gamma/(2(K-1))$ and $n \ge \max\{M, 2(K-1)(M-1)/\gamma + 1\}$ yields that
            \[
            \left\| \frac{N_{n}}{n-1} - \omega^{\star}_{\beta}\right\|_{\infty} \le \gamma \: .
            \]
            Let $T_{\boldsymbol{\mu}, \gamma}(w)$ as in Eq.~\eqref{eq:rv_T_eps_alloc}. Then, we showed that $T_{\boldsymbol{\mu}, \gamma}(\omega^\star_{\beta}) \le \max\{M, 2(K-1)(M-1)/\gamma + 1\}$.
	   Therefore, we have
    \[
    \bE_{\bnu} [T_{\boldsymbol{\mu}, \gamma}(\omega^\star_{\beta})] \le  \bE_{\bnu} [\max\{M, 2(K-1)(M-1)/\gamma + 1\}] < + \infty \: ,
    \]
    which concludes the proof of the first result.

	For the \adapttt{} algorithm, the proof is exactly the same by replacing $\omega^\star_{\beta}$ by $\omega^\star_{\epsilon,\beta}$.
\end{proof}

\subsection{Cost of Doubling and Forgetting}
\label{app:ssec_cost_doubling_and_forgetting}

Compared to the generic analysis of Top Two algorithms~\cite{jourdan2022top}, we need to control the sample complexity cost of the \hyperlink{DAF}{DAF}$(\epsilon)$ update (Algorithm~\ref{algo:doubling_forgetting}).
Due to this reason, we have to pay a multiplicative four-factor: one two-factor due to doubling, and another two-factor due to forgetting.
It is possible to show that this cost exists when adapting any BAI algorithm in which the empirical proportions are converging towards an allocation $\omega$ such that $\min_{a} \omega_{a} > 0$, i.e. there exists $\omega$ such that $\bE_{\bnu} [T_{\boldsymbol{\mu}, \gamma}(\omega)] < + \infty$.
As shown in Lemma~\ref{lem:convergence_others_arms_to_beta_allocation}, this is the case for the \adaptt{} and \adapttt{} algorithms.

\textit{Lemma~\ref{lem:bounded_phase_length_ratio} shows that the phase switches of the arms happen in a round-robin fashion, which means that an arm switches phase for a second time after all other arms first switch their own phases.}
\begin{lemma} \label{lem:bounded_phase_length_ratio}
	Let $\omega \in \Sigma_K$ such that $\min_{a} \omega_{a} > 0$. 
        Assume that there exists $\gamma_{\boldsymbol{\mu}} > 0$ such that for $\bE_{\bnu} [T_{\boldsymbol{\mu}, \gamma}(\omega)]  < + \infty$ for all $\gamma \in (0, \gamma_{\boldsymbol{\mu}})$, where $T_{\boldsymbol{\mu}, \gamma}(\omega)$ is defined in Equation~\eqref{eq:rv_T_eps_alloc}.
        Let $\eta > 0$.
        There exists $\tilde \gamma_{\boldsymbol{\mu}} \in (0, \gamma_{\boldsymbol{\mu}})$ such that, for all $\gamma \in (0, \tilde \gamma_{\boldsymbol{\mu}})$, there exists $N_{3} \ge T_{\boldsymbol{\mu}, \gamma}(\omega)$ with $\bE_{\bnu} [N_{3}]  < + \infty$ which satisfies
	\[
	\forall n \ge N_{3}, \quad \frac{\max_{a \in [K]}T_{k_{n,a}}(a) - 1}{\min_{a \in [K]}T_{k_{n,a}}(a) - 1} \le 2+\eta \: .
	\]
\end{lemma}
\begin{proof}
	Let $\eta > 0$.
	Let $\tilde \gamma_{\boldsymbol{\mu}} \in (0, \gamma_{\boldsymbol{\mu}})$ such that $2\max_{a \in [K]}(\omega_{a} + \gamma)/(\omega_{a} - \gamma) \le 2+\eta$, which is possible since $\min_{a} \omega_{a} > 0$.
        Let $\gamma \in (0, \tilde \gamma_{\boldsymbol{\mu}})$.
        By assumption, we have $\bE_{\bnu} [T_{\boldsymbol{\mu}, \gamma}(\omega)]  < + \infty$.
	Then, for all $n \ge T_{\boldsymbol{\mu}, \gamma}(\omega)$,
	\[
		\left\| \frac{N_{n}}{n-1} - \omega\right\|_{\infty} \le \gamma \: .
	\]
	Let $M \ge T_{\boldsymbol{\mu}, \gamma}(\omega)$.
	Let use denote by $k_{M} = (k_{M,a})_{a \in [K]}$ the current phases for all arms $a \in [K]$ at time $M$.
	Then, for all $n \ge M$ and all $a \in [K]$, we have $N_{n,a} \ge (n-1)(\omega_{a} - \gamma)$.
	Therefore, taking $n \ge \max_{a \in [K]} 2^{k_{M,a}} (\omega_{a} - \gamma)^{-1} + 1$, we obtain that $N_{n,a} \ge 2^{k_{M,a}}$ for all $a \in [K]$, hence we have $\max_{a \in [K]} T_{k_{M,a} + 1}(a) \le n$.
	Since $\min_{a \in [K]} T_{k_{M,a} + 1}(a) \ge M$, we have
	\[
		\max_{a \in [K]}\left| \frac{N_{T_{k_{M,a} + 1}(a), a}}{n-1} - \omega_{a}\right| \le \gamma \: .
	\]
	Likewise, taking $n \ge \max_{a \in [K]} 2^{k_{M,a}+1} (\omega_{a} - \gamma)^{-1} + 1$, we obtain that $N_{n,a} \ge 2^{k_{M,a}+1}$ for all $a \in [K]$, hence we have $\max_{a \in [K]} T_{k_{M,a} + 2}(a) \le n$.
	Let $a_{1} = \argmin_{a \in [K]} T_{k_{M,a} + 2}(a)$.
	By definition and using Lemma~\ref{lem:counts_at_phase_switch}, we have
	\begin{align*}
		& 2^{k_{M,a_{1}}+1} = N_{T_{k_{M,a_{1}} + 2}(a_{1}), a_{1}}  \le (T_{k_{M,a_{1}} + 2}(a_{1}) - 1)( \omega_{a_{1}} + \gamma) \: , \\
		\forall a \ne a_{1}, \quad & 2^{k_{M,a}} \le N_{T_{k_{M,a_{1}} + 2}(a_{1}), a}  \le (T_{k_{M,a_{1}} + 2}(a_{1}) - 1)( \omega_{a} + \gamma) \: .
	\end{align*}
	Let $a_{2} = \argmax_{a \in [K]} T_{k_{M,a} + 2}(a)$.
	By definition and using Lemma~\ref{lem:counts_at_phase_switch}, we have
	\begin{align*}
		& 2^{k_{M,a_{2}}+1} = N_{T_{k_{M,a_{2}} + 2}(a_{2}), a_{2}}  \ge (T_{k_{M,a_{2}} + 2}(a_{2}) - 1)( \omega_{a_{2}} - \gamma) \: ,
	\end{align*}
	Therefore, combining the above yields
	\begin{align*}
		(T_{k_{M,a_{2}} + 2}(a_{2}) - 1) \le (T_{k_{M,a_{1}} + 2}(a_{1}) - 1 ) 2\frac{\omega_{a_{2}} + \gamma}{\omega_{a_{2}} - \gamma} \le (T_{k_{M,a_{2}} + 2}(a_{2}) - 1 )(2+\eta) \: ,
	\end{align*}
	where the last inequality uses that $\gamma \in (0, \tilde \gamma_{\boldsymbol{\mu}})$ and $\tilde \gamma_{\boldsymbol{\mu}} \in (0, \gamma_{\boldsymbol{\mu}})$ is such that $2\max_{a \in [K]}(\omega_{a} + \gamma)/(\omega_{a} - \gamma) \le 2+\eta$.
	We take $n \ge N_{3} = \max_{a \in [K]} T_{k_{M,a} + 2}(a)$, hence we have $k_{n,a} \ge k_{M,a} + 2$ for all $a \in [K]$.
	Since $\bE_{\bnu} [T_{\boldsymbol{\mu}, \gamma}(\omega)]  < + \infty$ (i.e. arms are sampled linearly), it is direct to see that $\bE_{\bnu}[\max_{a \in [K]} T_{k_{M,a} + 2}(a)] < +\infty$.
        This concludes the proof.
\end{proof}

\subsection{Asymptotic Upper Bound on the Expected Sample Complexity}
\label{app:ssec_asymptotic_upper_bound_sample_complexity}

The final step of the generic analysis of Top Two algorithms~\citep{jourdan2022top} is to invert the private GLR stopping rule by leveraging the convergence of the empirical proportions towards the $\beta$-optimal allocation.
Provided this convergence is shown, the asymptotic upper bound on the expected sample complexity only depends on the dependence in $\log(1/\delta)$ of the threshold that ensures $\delta$-correctness.
Compared to the non-private GLR stopping rule, the private GLR stopping rules pay an extra cost to ensure privacy.
In Section~\ref{ssec:plug_in_based_TTUCB_global}, the stopping threshold is adapted with an additive term in $\cO(\log(1/\delta)^2)$.
In Section~\ref{ssec:lower_bound_based_TTUCB_global}, both the stopping threshold and the transportation costs are modified.

\begin{lemma} \label{lem:inversion_GLR_stopping_rule}
    Let $(\delta, \beta) \in (0,1)^2$.
    Assume that there exists $\gamma_{\boldsymbol{\mu}} > 0$ such that $\bE_{\bnu} [T_{\boldsymbol{\mu}, \gamma}(\omega^\star_{\beta})]  < + \infty$ for all $\gamma \in (0, \gamma_{\boldsymbol{\mu}})$, where $T_{\boldsymbol{\mu}, \gamma}(w)$ is defined in Eq.~\eqref{eq:rv_T_eps_alloc}.
    Combining such a sampling rule, using the \hyperlink{DAF}{DAF}$(\epsilon)$ update, with the GLR stopping rule with $W^{G}_{a,b}$ as in Eq.~\eqref{eq:TTUCB_Gaussian} and the stopping threshold $c^{G,\epsilon}_{a,b}$ as in Eq.~\eqref{eq:threshold_TTUCB_global_v1} yields a $\delta$-correct algorithm which satisfies that, for all $\bnu$ with mean $\boldsymbol{\mu}$ such that $|a^\star(\boldsymbol{\mu})| = 1$,
    \[
    \limsup_{\delta \to 0} \frac{\bE_{\bnu}\left[\tau_{\delta}\right]}{\log (1 / \delta)}
\le 4T^{\star}_{\mathrm{KL},\beta}(\boldsymbol{\nu}) \left( 1 + \sqrt{1 +\frac{\Delta_{\max}^2}{2\epsilon^2\sigma^4}}  \right) \: .
    \]
	where $ T^{\star}_{\mathrm{KL},\beta}(\boldsymbol{\nu})$ as in Eq.~\eqref{eq:T_KL_Gaussian} with $\sigma = 1/2$.

    Assume that there exists $\gamma_{\boldsymbol{\mu}} > 0$ such that $\bE_{\bnu} [T_{\boldsymbol{\mu}, \gamma}(\omega^\star_{\epsilon,\beta})]  < + \infty$.
    Combining such a sampling rule, using the \hyperlink{DAF}{DAF}$(\epsilon)$ update, with the GLR stopping rule with $W^{G,\epsilon}_{a,b}$ as in Eq.~\eqref{eq:TTUCB_global_v2} and the stopping threshold $\tilde c^{G,\epsilon}_{a,b}$ as in Eq.~\eqref{eq:threshold_TTUCB_global_v2} yields a $\delta$-correct algorithm which satisfies that, for all $\bnu$ with mean $\boldsymbol{\mu}$ such that $|a^\star(\boldsymbol{\mu})| = 1$,    
\begin{align*}
		\limsup_{\delta \to 0} \frac{\bE_{\boldsymbol{\nu}}\left[\tau_{\delta}\right]}{\log (1 / \delta)}
&\le  \begin{cases}
	4 T^{\star}_{\mathrm{KL},\beta}(\boldsymbol{\nu})g_{1}\left(\Delta_{\max}/(\sigma^2 \epsilon) \right) & \text{if } \Delta_{\max} < 3 \epsilon \\
	 12 T^\star_{\mathrm{KL}, \beta}(\bm{\nu}_{G,\epsilon}) g_{2}(3\epsilon^2 T^\star_{\mathrm{KL}, \beta}(\bm{\nu}_{G,\epsilon}) \max\{\beta, 1-\beta\}/2)/\sigma^2 & \text{otherwise}
\end{cases} \: ,
	\end{align*}
	where $T^\star_{\mathrm{KL}, \beta}(\bm{\nu}_{G,\epsilon})$ as in Eq.~\eqref{eq:relaxed_characteristic_time}.
	The function $g_{1}(y) =  \sup \left\{ x \mid  x^2 <  x +  y \sqrt{2x}  +  \frac{y^2}{4} \right\}$ is increasing on $[0,12]$ and satisfies that $g_{1}(0) = 1$ and $g_{1}(12) \le 10$.
	The function $g_{2}(y) = 1 + 2(\sqrt{1+1/y}-1)^{-1}$ is increasing on $\R^{\star}_+$ and satisfies that $\lim_{y \to 0} g_{2}(y) = 1$.
\end{lemma}
\begin{proof}
Lemma~\ref{lem:delta_correct_threshold_TTUCB_global_v1} and Lemma~\ref{lem:delta_correct_threshold_TTUCB_global_v2} yields the $\delta$-correctness of both algorithms.

\textit{\adaptt{} algorithm.}
Let $\zeta > 0$, $a^\star$ be the unique best arm.
Using~\eqref{eq:equality_equilibrium} and the continuity of
\[
	(\boldsymbol{\mu}, w) \mapsto \min_{a \ne a^\star(\boldsymbol{\mu})} \frac{(\mu_{a^\star(\boldsymbol{\mu})} - \mu_{a})^2}{2\sigma^2(1/w_{a^\star(\boldsymbol{\mu})} + 1/w_{a})}
\]
yields that there exists $\gamma_\zeta > 0$ such that $\left\| \frac{N_{n}}{n-1} - \omega^{\star}_{\beta} \right\|_{\infty} \le \gamma_\zeta$ and $\max_{a \in [K]}| \tilde \mu_{k_{n,a}+1, a} - \mu_{a} | \le \gamma_\zeta$ implies that
\begin{align*}
	\forall a \ne a^\star , \quad &\frac{(\tilde \mu_{k_{n,a^\star}+1, a^\star} - \tilde \mu_{k_{n,a}+1, a})^2 }{(n-1)/ N_{n, a^\star}+ (n-1)/ N_{n,a}} \ge \frac{2\sigma^2(1-\zeta)}{T^{\star}_{\mathrm{KL},\beta}(\boldsymbol{\nu})} \: , \\
	&\frac{n-1}{N_{n, a^\star}}+ \frac{n-1}{N_{n,a}} \le \frac{\Delta_{a}^2}{2\sigma^2} (1+\zeta) T^{\star}_{\mathrm{KL},\beta}(\boldsymbol{\nu}) \: .
\end{align*}
We choose such a $\gamma_\zeta$.
Let $\gamma_{\boldsymbol{\mu}} > 0$ be such that for $\bE_{\bnu} [T_{\boldsymbol{\mu}, \gamma}(\omega^\star_{\beta})]  < + \infty$ for all $\gamma \in (0, \gamma_{\boldsymbol{\mu}})$, where $T_{\boldsymbol{\mu}, \gamma}(\omega)$ is defined in Eq.~\eqref{eq:rv_T_eps_alloc}.
Let $\eta > 0$.
Let $\tilde \gamma_{\boldsymbol{\mu}} \in (0, \gamma_{\boldsymbol{\mu}})$ as in Lemma~\ref{lem:bounded_phase_length_ratio} for this $\eta$.
In the following, let us consider $\gamma \in (0, \min\{\tilde \gamma_{\boldsymbol{\mu}}, \gamma_\zeta, \beta/4, \Delta_{\min}/4\} )$.

Let $N_{3} \ge T_{\boldsymbol{\mu}, \gamma}(\omega^\star_{\beta})$ with $\bE_{\bnu} [N_{3}]  < + \infty$ as Lemma~\ref{lem:bounded_phase_length_ratio} for those $(\gamma, \eta)$.
Then, we have $\bE_{\bnu} [T_{\boldsymbol{\mu}, \gamma}(\omega^\star_{\beta})]  < + \infty$ and
	\[
	\forall n \ge N_{3}, \quad \frac{\max_{a \in [K]}T_{k_{n,a}}(a) - 1}{\min_{a \in [K]}T_{k_{n,a}}(a) - 1} \le 2+\eta \: .
	\]
Since arms are sampled linearly, it is direct to construct $N_{4} \ge N_{3}$ with $\bE_{\bnu}[N_{4}] < + \infty$ such that, for all $n \ge N_{4}$, we have $\max_{a \in [K]} \max_{k \in \{k_{n,a}, k_{n,a}+1\}}| \tilde \mu_{k, a} - \mu_{a} | \le \gamma$, 
Therefore, we have $\hat a_n = a^\star$.

Let $\kappa \in (0,1)$.
Let $n \ge N_{4}/\kappa$ and $(k_{n,a})_{a \in [K]}$ be the current phases at time $n$.
Combining the above, we have $\hat a_n = a^\star$ and
\begin{align*}
	&\max_{a \in [K]}| \tilde \mu_{k_{n,a}+1, a} - \mu_{a} | \le \gamma  \quad ,  \quad \left\| \frac{N_{n}}{n-1} - \omega^{\star}_{\beta}\right\|_{\infty} \le \gamma   \quad \text{and} \quad \frac{\max_{a \in [K]}T_{k_{n,a}}(a) - 1}{\min_{a \in [K]}T_{k_{n,a}}(a) - 1} \le 2+\eta \: .
\end{align*}
Let $a_{1} = \argmin_{a \in [K]}T_{k_{n,a}}(a)$ and $a_{2} = \argmax_{a \in [K]}T_{k_{n,a}}(a)$.
Therefore, we obtain
\begin{align*}
		\forall a \ne a^\star, \quad \frac{(\tilde \mu_{k_{n,\hat a_{n}}+1,\hat a_{n}}  - \tilde \mu_{k_{n,a}+1, a})^2}{1/\tilde N_{k_{n,\hat a_{n}}+1, \hat a_{n}}+ 1/\tilde N_{k_{n,a}+1,a}} &= \frac{(\tilde \mu_{k_{n,a^\star}+1,a^\star}  - \tilde \mu_{k_{n,a}+1, a})^2}{1/N_{T_{k_{n,a^\star}}(a^\star), a^\star}+ 1/N_{T_{k_{n,a}}(a),a}} \\
		&\ge \frac{(\tilde \mu_{k_{n,a^\star}+1,a^\star}  - \tilde \mu_{k_{n,a}+1, a})^2}{1/N_{T_{k_{n,a_{1}}}(a_{1}), a^\star}+ 1/N_{T_{k_{n,a_{1}}}(a_{1}),a}} \\
		&\ge (\min_{a \in [K]} T_{k_{n,a}}(a) -1)\frac{2\sigma^2(1-\zeta)}{T^{\star}_{\mathrm{KL},\beta}(\boldsymbol{\nu})}  \: .
\end{align*}
Similarly, we can show that, for all $a \ne a^\star$,
\begin{align*}
	\frac{1}{\tilde N_{k_{n,a^\star}+1,a^\star} } + \frac{1}{\tilde N_{k_{n,a}+1,a}} &= \frac{1}{N_{T_{k_{n,a^\star}}(a^\star),a^\star} } + \frac{1}{ N_{T_{k_{n,a}}(a),a}} \\
	&\le \frac{1}{N_{T_{k_{n,a_{1}}}(a_{1}),a^\star} } + \frac{1}{ N_{T_{k_{n,a_{1}}}(a_{1}),a}} \\
	&\le \frac{1}{\min_{a \in [K]} T_{k_{n,a}}(a) - 1} \frac{\Delta_{a}^2}{2\sigma^2} (1+\zeta) T^{\star}_{\mathrm{KL},\beta}(\boldsymbol{\nu}) \\
	&\le \frac{1}{\min_{a \in [K]} T_{k_{n,a}}(a) - 1} \frac{\Delta_{\max}^2}{2\sigma^2} (1+\zeta) T^{\star}_{\mathrm{KL},\beta}(\boldsymbol{\nu})  \: .
\end{align*}
Let $c^{G}_{a,b}$ as in Eq.~\eqref{eq:threshold_non_private}.
Using Lemma~\ref{lem:counts_at_phase_switch}, we obtain, for all $a \ne a^\star$,
\begin{align*}
	& c^{G}_{a^\star,a}(\tilde N_{k_n+1}, \delta (2 \zeta(s)^2 (k_{n,a^\star}+1)^s (k_{n,a}+1)^s)^{-1}) \le 4 \log(4 + (\max_{b \in [K]}k_{n,b}-1) \log 2 )  \\
	&\quad + 2 \cC_{G} \left(\log(1/\delta)/2 +  s\log(\max_{b \in [K]}k_{n,b}-1) +  \log(2(K-1)\zeta (s)^2)/2\right)
\end{align*}
Likewise, we obtain, for all $a \in [K]$,
\begin{align*}
	&\frac{1}{\epsilon^2\sigma^2} \sum_{c \in \{a^\star,a\}} \frac{1}{\tilde N_{k_{n,c},c}} \left(\log \frac{2 K \zeta(s) (k_{n,c}+1)^{s} }{\delta} \right)^2  \\
	&\le \frac{\Delta_{\max}^2}{2\epsilon^2\sigma^4} \frac{(1+\zeta) T^{\star}_{\mathrm{KL},\beta}(\boldsymbol{\nu})}{\min_{b \in [K]} T_{k_{n,b}}(b) - 1}  \left( \log(1/\delta) + s \log (\max_{b \in [K]} k_{n,b}+1) + \log (2 K \zeta(s) ) \right)^2
\end{align*}

Let us denote by $T^{+}_{k_n+1} = \max_{b \in [K]} T_{k_{n,b}+1}(b)$, $T^{+}_{k_n+2} = \max_{b \in [K]} T_{k_{n,b}+2}(b)$, $T^{-}_{k_n+1} = \min_{b \in [K]} T_{k_{n,b}+1}(b)$, $T^{-}_{k_n} = \min_{b \in [K]} T_{k_{n,b}}(b)$.
Let $T$ be a time such that $T \ge T^{+}_{k_n+1} \ge \kappa T$.
Using Lemmas~\ref{lem:counts_at_phase_switch} and~\ref{lem:bounded_phase_length_ratio}, we have
\begin{align*}
	(k_{n,b} - 1) \log 2 = \log N_{ T_{k_{n,b}}(b), b} \le \log T_{k_{n,b}}(b) \le \log T^{+}_{k_n} \le \log T^{-}_{k_n} + \log (2+\eta)\: .
\end{align*}

Using the \hyperlink{DAF}{DAF}$(\epsilon)$ update with the GLR stopping rule with $W^{G}_{a,b}$ as in Eq.~\eqref{eq:TTUCB_Gaussian} and the stopping threshold $c^{G,\epsilon}_{a,b}$ as in Eq.~\eqref{eq:threshold_TTUCB_global_v1}, we have
\begin{align*}
&\min \left\{\tau_{\delta}, T\right\} -\kappa T \leq   \sum_{T \ge T^{+}_{k_n} \ge \kappa T} (T^{+}_{k_n+2} - T^{+}_{k_n+1}) \indi{\tau_{\delta} > T^{+}_{k_n+1}} \\
 &\leq  \sum_{T^{+}_{k_n} = \kappa T}^{T} (T^{+}_{k_n+2} - T^{+}_{k_n+1})  \mathds{1}\left( \exists a \ne a^\star,\frac{(\tilde \mu_{k_{n,a^\star}+1,a^\star}  - \tilde \mu_{k_{n,a}+1, a})^2}{2\sigma^2\left(\frac{1}{\tilde N_{k_{n,a^\star}+1, a^\star}}+ \frac{1}{\tilde N_{k_{n,a}+1,a}}\right)} <  c^{G,\epsilon}_{a^\star,a}(\tilde N_{k_n+1}, \delta)  \right)\\
&\leq \sum_{T \ge T^{+}_{k_n} \ge \kappa T} (T^{+}_{k_n+2} - T^{+}_{k_n+1}) \mathds{1}\left( (T^{-}_{k_{n}} -1)\frac{1-\zeta}{T^{\star}_{\mathrm{KL},\beta}(\boldsymbol{\nu})}  < 8 \log( 4 + \log T^{-}_{k_{n}}+ \log(2+\eta)  ) \right.\\
& + 4 \cC_{G} \left(\log(1/\delta)/2 +  s\log(2 + \log_{2} T^{-}_{k_{n}} + \log_{2} (2+\eta)) + \log(2(K-1)\zeta (s)^2)/2\right)   \\
& \left. +  \frac{\Delta_{\max}^2}{2\epsilon^2\sigma^4} \frac{(1+\zeta) T^{\star}_{\mathrm{KL},\beta}(\boldsymbol{\nu})}{T^{-}_{k_{n}} - 1}  \left( \log(1/\delta) + s \log ( 2 + \log_{2} T^{-}_{k_{n}} + \log_{2} (2+\eta) ) + \log (2 K \zeta(s) ) \right)^2 \right) \: ,
\end{align*}
Let $T_{\zeta}(\delta)$ defined as the largest deterministic time such that the above condition is satisfied when replacing $T^{-}_{k_n}$ by $(1-\kappa)T$.
Let $k_{\delta}$ be the largest random vector of phases such that that $T^{+}_{k_{\delta}+1} \le T_{\zeta}(\delta)$ almost surely, hence $T^{+}_{k_{\delta}+2} > T_{\zeta}(\delta)$ almost surely.
Then, using the above yields that $\tau_{\delta} \le T^{+}_{k_{\delta}+2} $ almost surely, hence
\[
	\limsup_{\delta \to 0} \frac{\bE_{\bnu}[\tau_{\delta}]}{\log(1/\delta)} \le \limsup_{\delta \to 0} \frac{\bE_{\bnu}[T^{+}_{k_{\delta}+2}]}{\log(1/\delta)} \le (2+\eta)^2\limsup_{\delta \to 0} \frac{\bE_{\bnu}[T^{+}_{k_{\delta}+1}]}{\log(1/\delta)} \le (2+\eta)^2\limsup_{\delta \to 0} \frac{ T_{\zeta}(\delta)}{\log(1/\delta)} 
\]
where the second inequality uses Lemma~\ref{lem:bounded_phase_length_ratio} twice, i.e. $T^{+}_{k_{\delta}+2} \le (2+\eta) T^{-}_{k_{\delta}+2} \le (2+\eta)^2 T^{+}_{k_{\delta}+1}$, and the last one used the definition of $k_{\delta}$ and that $T_{\zeta}(\delta)$ is deterministic.

Since we are only interested in upper bounding $\limsup_{\delta \to 0} \frac{ T_{\zeta}(\delta)}{\log(1/\delta)}$, we can safely drop the second orders terms in $T$ and $\log(1/\delta)$.
This allows us to remove the terms in $\cO(\log \log T)$ and in $\cO(\log \log (1/\delta))$.
Using that $ \cC_{G}(x) = x + \cO(\log x)$, tedious manipulations yields that
\[
	\limsup_{\delta \to 0} \frac{ T_{\zeta}(\delta)}{\log(1/\delta)} \le \frac{T^{\star}_{\mathrm{KL},\beta}(\boldsymbol{\nu})}{1-\kappa} D_{\zeta}(\mu, \epsilon) \: ,
\]
where
\begin{align*}
	D_{\zeta}(\mu, \epsilon) = \sup \left\{ x \mid  x^2 < \frac{2}{1-\zeta} x + \frac{1+\zeta}{1-\zeta}\frac{\Delta_{\max}^2}{2\epsilon^2\sigma^4} \right\} \le \frac{1}{1-\zeta} \left( 1 + \sqrt{1 + (1-\zeta^2)\frac{\Delta_{\max}^2}{2\epsilon^2\sigma^4}}  \right) \: .
\end{align*}
The last inequality uses that $x^2 - 2bx - c < 0$ for all $x \in [0, b(1 + \sqrt{1 + c/b^2}))$.
Therefore, we have shown that
\begin{align*}
	\limsup_{\delta \to 0} \frac{\bE_{\bnu}[\tau_{\delta}]}{\log(1/\delta)} \le (2+\eta)^2 \frac{T^{\star}_{\mathrm{KL},\beta}(\boldsymbol{\nu})}{(1-\kappa)(1-\zeta)} \left( 1 + \sqrt{1 + (1-\zeta^2)\frac{\Delta_{\max}^2}{2\epsilon^2\sigma^4}}  \right) \: .
\end{align*}
Letting $\kappa$, $\eta$ and $\zeta$ goes to zero concludes the proof of the first result. 

\textit{\adapttt{} algorithm.}
For the \adapttt{} algorithm, the proof is done with similar arguments.
Using~\eqref{eq:equality_equilibrium} and the continuity of $(\boldsymbol{\mu},w) \to \min_{a \ne a^\star}W^{G,\epsilon}_{a^\star,a}(\boldsymbol{\mu},w)$, defined in Eq.~\eqref{eq:TTUCB_global_v2}, we obtain another $\gamma_{\zeta} > 0$ such that $\left\| \frac{N_{n}}{n-1} - \omega^{\star}_{\epsilon,\beta} \right\|_{\infty} \le \gamma_\zeta$ and $\max_{a \in [K]}| \tilde \mu_{k_{n,a}+1, a} - \mu_{a} | \le \gamma_\zeta$ implies that
\begin{align*}
	\forall a \ne a^\star ,\quad& \frac{(\tilde \mu_{k_{n,a^\star}+1, a^\star} - \tilde \mu_{k_{n,a}+1, a}) \min\{\epsilon/2, \tilde \mu_{k_{n,a^\star}+1, a^\star} - \tilde \mu_{k_{n,a}+1, a}\} }{(n-1)/ N_{n, a^\star}+ (n-1)/ N_{n,a}} \ge \frac{2\sigma^2(1-\zeta)}{T^\star_{\mathrm{KL}, \beta}(\bm{\nu}_{G,\epsilon})} \: , \\
	&\frac{n-1}{N_{n, a^\star}}+ \frac{n-1}{N_{n,a}} \le \frac{\Delta_{a}\min\{\epsilon/2,\Delta_{a}\}}{2\sigma^2} (1+\zeta)T^\star_{\mathrm{KL}, \beta}(\bm{\nu}_{G,\epsilon}) \: . 
\end{align*}
We choose such a $\gamma_\zeta$.
Let $\gamma_{\boldsymbol{\mu}} > 0$ be such that for $\bE_{\bnu} [T_{\boldsymbol{\mu}, \gamma}(\omega^\star_{\epsilon,\beta})]  < + \infty$ for all $\gamma \in (0, \gamma_{\boldsymbol{\mu}})$.
Let $\eta > 0$.
Let $\tilde \gamma_{\boldsymbol{\mu}} \in (0, \gamma_{\boldsymbol{\mu}})$ as in Lemma~\ref{lem:bounded_phase_length_ratio} for this $\eta$.
In the following, let us consider $\gamma \in (0, \min\{\tilde \gamma_{\mu}, \gamma_\zeta, \beta/4, \Delta_{\min}/4, (\epsilon/2 - \max_{a, \Delta_{a} < \epsilon/2} \Delta_{a})/2\} )$.

Let $\kappa \in (0,1)$. 
As above, we can construct $N_3$ with Lemma~\ref{lem:bounded_phase_length_ratio} and $N_4 \ge N_3$ such that $\bE_{\nu}[N_{4}] < + \infty$.
Let $n \ge N_4/\kappa$ and $(k_{n,a})_{a \in [K]}$ the current phases. 
Then, we have $\hat a_n = a^\star$,
\begin{align*}
	&\max_{a \in [K]}| \tilde \mu_{k_{n,a}+1, a} - \mu_{a} | \le \gamma  \: ,  \quad \left\| \frac{N_{n}}{n-1} - \omega^{\star}_{\epsilon,\beta}\right\|_{\infty} \le \gamma   \: \text{and} \: \frac{\max_{a \in [K]}T_{k_{n,a}}(a) - 1}{\min_{a \in [K]}T_{k_{n,a}}(a) - 1} \le 2+\eta \: .
\end{align*}

Depending on the value of the private empirical gap, the stopping condition that is checked is different.
For all $a \ne a^\star$ such that $\Delta_{a} < \epsilon/2$, we have $\tilde \mu_{k_{n,a^\star}+1, a^\star} - \tilde \mu_{k_{n,a}+1, a} \le \Delta_{a} + 2\gamma < \epsilon/2$.
For all $a \ne a^\star$ such that $\Delta_{a} \ge \epsilon/2$, we have either $\tilde \mu_{k_{n,a^\star}+1, a^\star} - \tilde \mu_{k_{n,a}+1, a} \ge \epsilon/2 $ or $\tilde \mu_{k_{n,a^\star}+1, a^\star} - \tilde \mu_{k_{n,a}+1, a} \le \epsilon/2 $ and $\Delta_{a} \le \epsilon/2 + 2\gamma$.
Let $a_{1} = \argmin_{a \in [K]}T_{k_{n,a}}(a)$ and $a_{2} = \argmax_{a \in [K]}T_{k_{n,a}}(a)$.
Therefore, we obtain similarly that, for all $a \ne a^\star$,
\begin{align*}
		&\frac{(\tilde \mu_{k_{n,a^\star}+1, a^\star} - \tilde \mu_{k_{n,a}+1, a})\min\{\epsilon/2, \tilde \mu_{k_{n,a^\star}+1, a^\star} - \tilde \mu_{k_{n,a}+1, a}\}}{1/\tilde N_{k_{n,a^\star}+1, a^\star}+ 1/\tilde N_{k_{n,a}+1,a}} \ge (\min_{b \in [K]} T_{k_{n,b}}(b) -1)\frac{2 \sigma^2(1-\zeta)}{T^\star_{\mathrm{KL}, \beta}(\bm{\nu}_{G,\epsilon})}  \: .
\end{align*}
Similarly, for all $a \ne a^\star$ such that $\tilde \mu_{k_{n,a^\star}+1, a^\star} - \tilde \mu_{k_{n,a}+1, a} \le \epsilon/2 $, hence $\Delta_{a} \le \epsilon/2 + 2\gamma$, we have
\begin{align*}
	&\frac{1}{\tilde N_{k_{n,a^\star}+1,a^\star} } + \frac{1}{\tilde N_{k_{n,a}+1,a}} \le \frac{1}{\min_{b \in [K]} T_{k_{n,b}}(b) - 1} \frac{  \Delta_{a} \min\{\epsilon/2,\Delta_{a}\}}{2 \sigma^2} (1+\zeta)  T^\star_{\mathrm{KL}, \beta}(\bm{\nu}_{G,\epsilon}) \: ,\\
	&\frac{1}{\sqrt{\tilde N_{k_{n,a^\star}+1,a^\star}} } + \frac{1}{\sqrt{\tilde N_{k_{n,a}+1,a}}} \\ 
    & \le \sqrt{\frac{1}{(\min_{b \in [K]} T_{k_{n,b}}(b) - 1)} \frac{  \Delta_{a} \min\{\epsilon/2,\Delta_{a}\}}{ \sigma^2} (1+\zeta)  T^\star_{\mathrm{KL}, \beta}(\bm{\nu}_{G,\epsilon})} , \\	
	& \frac{1}{2\epsilon^2\sigma^2}\sum_{c \in \{a^\star,a\}}\frac{1}{\tilde N_{k_{n,c}+1,c}} \left(\log \frac{3 K (k_{n,c}+1)^{s} \zeta(s)}{\delta} \right)^2 \\
	& \le \frac{  \Delta_{a} \min\{\epsilon/2,\Delta_{a}\}(1+\zeta)  T^\star_{\mathrm{KL}, \beta}(\bm{\nu}_{G,\epsilon})}{4\epsilon^2 \sigma^4(\min_{b \in [K]} T_{k_{n,b}}(b) - 1)}   \left( \log(1/\delta) + s \log (\max_{b \in [K]} k_{n,b}+1) + \log (3 K \zeta(s) ) \right)^2 \: ,  \\
	&\frac{\sqrt{2}}{\epsilon \sigma}\sum_{c \in \{a^\star,a\}} \sqrt{\frac{h(\tilde N_{k_{n,c}+1,c},\delta)}{\tilde N_{k_{n,c}+1,c}}} \log  \left(\frac{3K\zeta(s) (k_{n,c}+1)^{s}}{\delta}  \right) \\ 
    & \le \sqrt{\frac{2  \Delta_{a} \min\{\epsilon/2,\Delta_{a}\} (1+\zeta)  T^\star_{\mathrm{KL}, \beta}(\bm{\nu}_{G,\epsilon})}{ \epsilon^2 \sigma^4 (\min_{b \in [K]} T_{k_{n,b}}(b) - 1) } } \\
	&\quad   \sqrt{h(2^{\max_{b \in [K]}k_{n,b}-1}, \delta)}  \left( \log(1/\delta) + s \log (\max_{b \in [K]}k_{n,b}+1) + \log(3K\zeta(s)) \right) \: .
\end{align*}
Moreover, for all $a \ne a^\star$ such that $\tilde \mu_{k_{n,a^\star}+1, a^\star} - \tilde \mu_{k_{n,a}+1, a} \ge 3\epsilon $, hence $\Delta_{a} \ge 3\epsilon$, we have
\begin{align*}
	&\sqrt{\tilde N_{k_{n,a}+1,a}} + \sqrt{\tilde N_{k_{n,a^\star}+1,a^\star}}   \le \sqrt{T_{k_{n,a_2}}(a_2) - 1} \left( \sqrt{\beta + \gamma} + \sqrt{\omega^{\star}_{\epsilon,\beta,a} + \gamma} \right)\\
	&\quad \le \sqrt{\min_{b \in [K]}T_{k_{n,b}}(b) - 1} \sqrt{2(2+\eta)(\max\{\beta, 1-\beta\} + \gamma)}  \: , \\
	&\frac{\epsilon}{2\sqrt{2\sigma^2}} \sum_{c \in \{a^\star,a\}} \sqrt{ \tilde N_{k_{n,c}+1,c} h(\tilde N_{k_{n,c}+1,c},\delta) } \\
	&\quad\le \sqrt{h(2^{\max_{b \in [K]}k_{n,b}-1}, \delta)}  \frac{\epsilon}{2\sigma} \sqrt{\min_{b \in [K]}T_{k_{n,b}}(b) - 1} \sqrt{(2+\eta)(\max\{\beta, 1-\beta\} + \gamma)}  \: .
\end{align*}

Let $T^{+}_{k_n+1} = \max_{b } T_{k_{n,b}+1}(b)$, $T^{+}_{k_n+2} = \max_{b } T_{k_{n,b}+2}(b)$, $T^{-}_{k_n+1} = \min_{b } T_{k_{n,b}+1}(b)$, $T^{-}_{k_n} = \min_{b } T_{k_{n,b}}(b)$.
Let $T$ be a time such that $T \ge T^{+}_{k_n+1} \ge \kappa T$.
Then, $(\max_{b} k_{n,b} - 1) \log 2 \le \log T^{-}_{k_n} + \log (2+\eta)$.
As above, using the \hyperlink{DAF}{DAF}$(\epsilon)$ update with the GLR stopping rule with $W^{G,\epsilon}_{a,b}$ as in Eq.~\eqref{eq:TTUCB_global_v2} and the stopping threshold $\tilde c^{G,\epsilon}_{a,b}$ as in Eq.~\eqref{eq:threshold_TTUCB_global_v2}, we have
\begin{align*}
&\min \left\{\tau_{\delta}, T\right\} -\kappa T \leq   \sum_{T \ge T^{+}_{k_n} \ge \kappa T} (T^{+}_{k_n+2} - T^{+}_{k_n+1}) \indi{\tau_{\delta} > T^{+}_{k_n+1}} \\
 &\leq  \sum_{T \ge T^{+}_{k_n} \ge \kappa T} (T^{+}_{k_n+2} - T^{+}_{k_n+1})  \mathds{1}\left( \exists a \ne a^\star, \:   \right. \\
 &\quad \left( \tilde \mu_{k_{n,a^\star}+1,a^\star}  - \tilde \mu_{k_{n,a}+1, a} < \epsilon/2,  \frac{(\tilde \mu_{k_{n,a^\star}+1,a^\star}  - \tilde \mu_{k_{n,a}+1, a})^2}{2\sigma^2(1/\tilde N_{k_{n,a^\star}+1, a^\star}+ 1/\tilde N_{k_{n,a}+1,a})}  < \tilde c^{G,\epsilon}_{a^\star,a}(\tilde N_{k_n+1}, \delta) \right)  \lor \\
 &\quad \left. \left( \tilde \mu_{k_{n,a^\star}+1,a^\star}  - \tilde \mu_{k_{n,a}+1, a} \ge \epsilon/2, \frac{\epsilon (\tilde \mu_{k_{n,a^\star}+1,a^\star}  - \tilde \mu_{k_{n,a}+1, a})}{4\sigma^2(1/\tilde N_{k_{n,a^\star}+1, a^\star}+ 1/\tilde N_{k_{n,a}+1,a})}  < \tilde c^{G,\epsilon}_{a^\star,a}(\tilde N_{k_n+1}, \delta) \right) \right) \: .
\end{align*}
Leveraging the inequalities explicited above, we can upper bound it by a condition which only involves $T^{-}_{k_n}$ and problem dependent quantities (in a highly convoluted fashion).
As above, we define $T_{\zeta}(\delta)$ as the largest deterministic time such that the above condition is satisfied when replacing $T^{-}_{k_n}$ by $(1-\kappa)T$.
Then, we obtain similarly that
\[
	\limsup_{\delta \to 0} \frac{\bE_{\nu}[\tau_{\delta}]}{\log(1/\delta)}  \le (2+\eta)^2\limsup_{\delta \to 0} \frac{ T_{\zeta}(\delta)}{\log(1/\delta)} \: .
\]
Droping the second orders terms in $T$ and $\log(1/\delta)$ and using that $ \cC_{G}(x) = x + \cO(\log x)$ and $ \overline{W}_{-1}(x) = x + \cO(\log x)$, tedious manipulations yields that
\[
	\limsup_{\delta \to 0} \frac{ T_{\zeta}(\delta)}{\log(1/\delta)} \le \frac{T^\star_{\mathrm{KL}, \beta}(\bm{\nu}_{G,\epsilon})}{1-\kappa} \max\{ D^{(1)}_{\gamma,\zeta}(\mu, \epsilon) , \indi{\Delta_{\max} \ge \epsilon/2} D^{(2)}_{\gamma,\zeta, \eta}(\mu, \epsilon)  \}\: ,
\]
where $g(x,y) = \sqrt{2xy} + y/4$ and
\begin{align*}
	&D^{(1)}_{\gamma,\zeta}(\mu, \epsilon) = \sup \left\{ x \mid  x^2(1-\zeta) <  x + g\left(x,\frac{1+\zeta}{\epsilon^2 \sigma^4} \max_{\Delta_{a} \le \epsilon/2 + 2\gamma} \Delta_{a}\min\{\Delta_{a}, \epsilon/2\} \right) \right\} \: , \\
	&D^{(2)}_{\gamma,\zeta, \eta}(\mu, \epsilon) = \sup \left\{ x \mid  2x\sigma(1-\zeta) < \epsilon\sqrt{x} \sqrt{(2+\eta) T^\star_{\mathrm{KL}, \beta}(\bm{\nu}_{G,\epsilon})( \max\{\beta, 1-\beta\} + \gamma)} + \frac{1}{\sigma} \right\} \: .
\end{align*}
Letting $\kappa$, $\gamma$, $\eta$ and $\zeta$ goes to zero yields that
\[
\limsup_{\delta \to 0} \frac{\bE_{\nu}\left[\tau_{\delta}\right]}{\log (1 / \delta)}
\le 4T^\star_{\mathrm{KL}, \beta}(\bm{\nu}_{G,\epsilon}) \max\{ D^{(1)}_{0,0}(\mu, \epsilon) , \indi{\Delta_{\max} \ge \epsilon/2} D^{(2)}_{0,0,0}(\mu, \epsilon)  \} \: .
\]
Using that $x^2 - 2bx - c < 0$ for all $x \in [0, b(1 + \sqrt{1 + c/b^2}))$, we obtain that
\begin{align*}
	D^{(2)}_{0,0,0}(\mu, \epsilon) &= \left( \sup \left\{ x \mid  x^2 <   \sqrt{ \epsilon^2 \sigma^{-2} T^\star_{\mathrm{KL}, \beta}(\bm{\nu}_{G,\epsilon}) \max\{\beta, 1-\beta\}/2} x + \sigma^{-2}/2 \right\} \right)^2 \\
	&\le \frac{\epsilon^2}{8\sigma^2} T^\star_{\mathrm{KL}, \beta}(\bm{\nu}_{G,\epsilon}) \max\{\beta, 1-\beta\} \left( 1  + \sqrt{1 + \frac{4}{\epsilon^2 T^\star_{\mathrm{KL}, \beta}(\bm{\nu}_{G,\epsilon}) \max\{\beta, 1-\beta\}}} \right)^2 \: , \\
	D^{(1)}_{0,0}(\mu, \epsilon) &= g_{1} \left(\max_{\Delta_{a} \le 3\epsilon} \frac{\Delta_{a}}{\sigma^2\epsilon} \right) \quad \text{with} \quad g_{1}(y) =  \sup \left\{ x \mid  x^2 <  x +  y \sqrt{2x}  +  y^2/4 \right\} \: .
\end{align*}
In more details, we have
\begin{align*}
	g_{1}(0) = \sup \left\{ x \mid  x^2 <  x\right\} = 1 \quad \text{and} \quad g_{1}(12) = \sup \left\{ x \mid  x^2 <  x + 12 \sqrt{2x}  + 36 \right\} \le 10 \: ,
\end{align*}
where the last inequality is obtained by numerical analysis.
The function $g_{2}$ is obtained by noting that $y (1+\sqrt{1+1/y})^2 = 1 + 2(\sqrt{1+1/y}-1)^{-1}$.
When $\Delta_{\max} < \epsilon/2$, we have $T^\star_{\mathrm{KL}, \beta}(\bm{\nu}_{G,\epsilon}) = T^{\star}_{\mathrm{KL},\beta}(\boldsymbol{\nu})$ where $T^{\star}_{\mathrm{KL},\beta}(\boldsymbol{\nu})$ as in Eq.~\eqref{eq:T_KL_Gaussian} with $\sigma = 1/2$.
This concludes the proof of the second result.
\end{proof}

\textit{Concluding the proof of Theorems~\ref{thm:sample_complexity_TTUCB_global_v1} and~\ref{thm:sample_complexity_TTUCB_global_v2}.} Combining Lemmas~\ref{lem:suff_exploration},~\ref{lem:convergence_others_arms_to_beta_allocation},~\ref{lem:bounded_phase_length_ratio} and~\ref{lem:inversion_GLR_stopping_rule} concludes the proof of Theorems~\ref{thm:sample_complexity_TTUCB_global_v1} and~\ref{thm:sample_complexity_TTUCB_global_v2}.
We restrict the result to instances such that $\min_{a \ne b}|\mu_{a} - \mu_{b}| > 0$ in order for Lemma~\ref{lem:suff_exploration} to hold.
Note that this is an artifact of the asymptotic proof which could be alleviated with more careful considerations.
$\qed$

\section{On the Number of Rounds of Adaptivity}
\label{app:rounds_adaptivity}

Due to its generality, using the \hyperlink{DAF}{DAF} update yields a batched version of any existing FC-BAI algorithm, which satisfies $\epsilon$-global DP.
At the end of the episode of arm $a$ (after updating its mean), it is possible to compute the sequence of all the arms to be pulled before the end of the next episode (for another arm), without taking the collected observations into account.
In contrast to the classical batched setting where the batch size is fixed, the size of the resulting batches is adaptive and data-dependent.

Let $C(\tau_{\delta}) = \sum_{a \in [K]} k_{\tau_{\delta},a}$ be the number of rounds of adaptivity, where $k_{\tau_{\delta},a}$ denotes the number of episodes of arm $a \in [K]$ at stopping time.
Using Jensen's inequality, the number of rounds of adaptivity is upper bounded by $\bE_{\boldsymbol{\nu}}\left[C(\tau_{\delta}) \right]\le K \log_2 \bE_{\boldsymbol{\nu}}\left[\tau_{\delta}\right]$.
Therefore, any upper bound on the expected sample complexity directly implies an upper bound on the number of rounds of adaptivity.

\textit{One global episode.}
	The multiplicative factor $K$ is incurred because \hyperlink{DAF}{DAF} maintains one episode per arm.
	Alternatively, one can consider one global episode $k_n$.
	Formally, we switch phase as soon as all the arms have doubled their empirical counts, i.e. $N_{n,a} \ge 2 N_{T_{k_{n}},a}$ for all $a \in [K]$.
	This modification allows to shave the $K$ factor since $\bE_{\boldsymbol{\nu}}\left[C(\tau_{\delta}) \right]\le \log_2 \bE_{\boldsymbol{\nu}}\left[\tau_{\delta}\right]$.
	When using one global episode, one can show the same asymptotic upper bound as when we used one episode per arm.

	Empirically, the performance is worsen by considering one global episode, hence we recommend to use one episode per arm.
	A sub-optimal arm $a$ might be sampled more than $a^\star$ in early stage due to unlucky first draws.
	When there is only one global episode, the learner will always have to double the counts of this sub-optimal arm before updating its estimators of the other arms.
	After realizing that this arm is sub-optimal, it won't be sampled frequently, hence many samples should be collected before ending the episode.

\textit{Batched best arm identification}
In the non-private setting ($\epsilon = + \infty$), we recover Batched Best-Arm Identification (BBAI) in the fixed-confidence setting.
One of the question arising in this setting is the following: \textit{Can we solve the BBAI problem with asymptotically optimal sample complexity (up to a constant factor) and a small number of batches?}
A slight modification of the above result provides a positive answer.

\begin{algorithm}[t]
         \caption{Doubling-Per-Arm (\protect\hypertarget{DPA}{DPA})}
         \label{algo:doubling_per_arm}
	\begin{algorithmic}
 	      \State {\bfseries Input:} History $\mathcal H_{n}$, arm $a \in [K]$.
		\State {\bfseries Initialization:} For all $a \in [K]$, $T_{1}(a) = \arms+1$ and $k_{K+1,a}=1$; 
		\If{$N_{n,a} \ge 2 N_{T_{k_{n,a}}(a),a}$}{

			 Change phase $k_{n,a} \gets k_{n,a} + 1$ for this arm $a$;
                
             Set $T_{k_{n,a}}(a) = n$ and $\hat \mu_{k_{n,a},a} =  N_{T_{k_{n,a}}(a), a}^{-1} \sum_{t \in [T_{k_{n,a}}(a) - 1]} r_{t} \ind{a_t = a}$;
		}
		\EndIf
		\State \textbf{Return} $(\hat \mu_{n,a}, N_{n,a})$;
	\end{algorithmic}
\end{algorithm}

Without the privacy constraint, there is no need to forget about past observations or to add Laplacian noise.
Therefore, the \hyperlink{DPA}{DPA} update is better to suited for BBAI than the \hyperlink{DAF}{DAF} one.
Using the \hyperlink{DPA}{DPA} update yields an adaptive batched version of any existing FC-BAI algorithm.
It is direct to see that the same analysis can be used to study \hyperlink{TTUCB}{TTUCB} with \hyperlink{DPA}{DPA} update.
Namely, it yields a $\delta$-correct algorithm such that, for all $\mu$ with distinct means, 
\begin{align*}
	\limsup_{\delta \to 0} \frac{\bE_{\boldsymbol{\nu}}\left[\tau_{\delta}\right]}{\log (1 / \delta)}
&\le 2 T^{\star}_{\mathrm{KL},\beta}(\boldsymbol{\nu}) \: \text{,} \: \limsup_{\delta \to 0} \left( \bE_{\boldsymbol{\nu}}\left[C(\tau_{\delta}) \right] - K\log_2 \log(1/\delta) \right) \le K \log_2 (2 T^\star_{\mathrm{KL},\beta}(\boldsymbol{\nu})) \: .
\end{align*}
For $\beta=1/2$, the algorithm is asymptotically optimal (up to a multiplicative factor $4$) with solely $\mathcal O \left( K \log_2 (T^\star_{\mathrm{KL}}(\boldsymbol{\nu})\log(1/\delta)) \right)$ rounds of adaptivity.

There are already several works studying BBAI~\citep{pmlrv28karnin13,jin2019efficient,jin2023optimal}, see Table 1 in~\citet{jin2023optimal} for a detailed comparison.  
Building on the Exponential-Gap Elimination algorithm~\citep{pmlrv28karnin13},~\citet{jin2019efficient} proposed an algorithm achieving an expected sample complexity of the order of $\cO(\sum_{a \ne a^\star}\Delta_{a}^{-2} \log(\log(\Delta_{a}^{-1})/\delta))$ with $\cO(\log^{\star}_{1/\delta}(K) \log(\Delta_{\min}^{-1}))$ batches, where $\log^{\star}_{1/\delta}$ is the iterated logarithm function with base $1/\delta$. 
To the best of our knowledge, existing lower bound on the number of rounds are worst-case bounds.
For constant $\delta \in (0,1)$,~\citet{8948669} proved that for certain bandit instances, any algorithm that achieves the sample complexity bound obtained in~\citet{jin2019efficient} requires at least $\Omega(\log(\Delta_{\min}^{-1})/\log \log \Delta_{\min}^{-1})$ batches.
~\citet{jin2023optimal} proposed the Tri-BBAI algorithm which achieves asymptotic optimality with two rounds of adaptivity (i.e. three phases).
An important remark here is that the analysis of Tri-BBAI is purely asymptotic, and it is only $\delta$-correct for sufficiently small $\delta$.
As an improvement with similar asymptotic guarantees as well as non-asymptotic ones, they propose Opt-BBAI which uses the same first two phases as Tri-BBAI, then uses successive elimination and checks for best arm elimination.

\section{Extended Experimental Analysis}
\label{app:experiment}

For both local DP and global DP, we perform additional experiments on six bandit environments with Bernoulli distributions, as defined by~\citep{dpseOrSheffet}, namely
\begin{align*}
&\mu_1 = (0.95, 0.9, 0.9, 0.9, 0.5),~~~&&\mu_2 = (0.75, 0.7, 0.7, 0.7, 0.7),\\
&\mu_3 = (0, 0.25, 0.5, 0.75, 1),~~~&&\mu_4 = (0.75, 0.625, 0.5, 0.375, 0.25)  \},\\
&\mu_5 = (0.75, 0.53125, 0.375, 0.28125, 0.25),~~~&&\mu_6 = (0.75, 0.71875, 0.625, 0.46875, 0.25)  \}.
\end{align*}
For each Bernoulli instance, we implement the algorithms with \[\epsilon \in \{0.001, 0.005, 0.01, 0.05, 0.1, 0.2, 0.3, 0.4, 0.5, 0.6, 0.7, 0.8, 0.9, 1, 10, 100, 1000\}, \] for global DP, and  \[\epsilon \in \{0.01, 0.05, 0.1, 0.2, 0.3, 0.4, 0.5, 0.6, 0.7, 0.8, 0.9, 1, 10, 100\}, \] for local DP.

The risk level is set at $\delta = 0.01$. We verify empirically that the algorithms are $\delta$-correct by running each algorithm $1000$ times.

The additional results for local DP are presented in Figure~\ref{fig:ext_experiments_local}. For global DP, the additional results are provided in Figure~\ref{fig:ext_experiments}. To show the difference between \adaptt{} and \adapttt{}, we plot the stopping time not in a logarithmic scale in Figure~\ref{fig:ext_experiments_no_log}.
The additional experiments validate the same conclusions as the ones reached in Section~\ref{sec:experiments}.

\begin{remark}
    To implement the thresholds of \adaptt{} and \adapttt{}, we use empirical thresholds that we get by approximating the theoretical thresholds. The expressions of the empirical thresholds used can be found in the code at \url{https://github.com/achraf-azize/DP-BAI}. 
\end{remark}

\begin{figure}[p]
    \centering
    \begin{subfigure}[b]{0.48\textwidth}
        \centering
        \includegraphics[width=0.8\textwidth]{img/env1_local.pdf}
        \caption{$\mu_1$}    
    \end{subfigure}
    \hfill
    \begin{subfigure}[b]{0.48\textwidth}  
        \centering
        \includegraphics[width=0.8\textwidth]{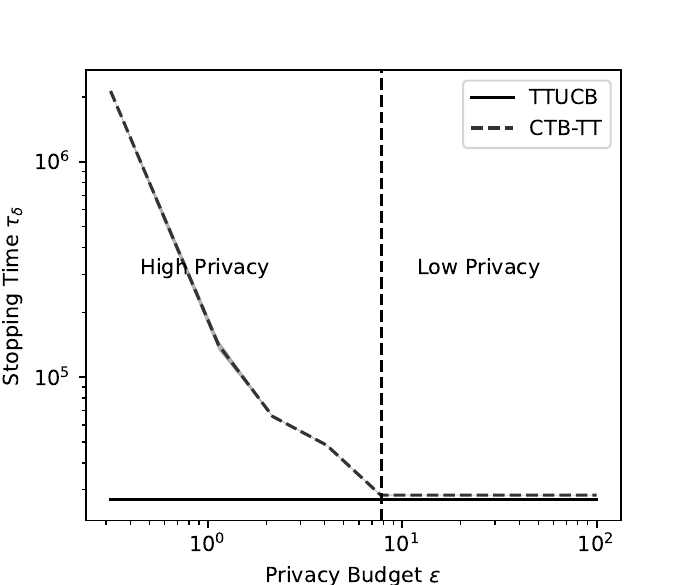}
        \caption{$\mu_2$}    
    \end{subfigure}
    \vskip\baselineskip
    \begin{subfigure}[b]{0.48\textwidth}
        \centering
        \includegraphics[width=0.8\textwidth]{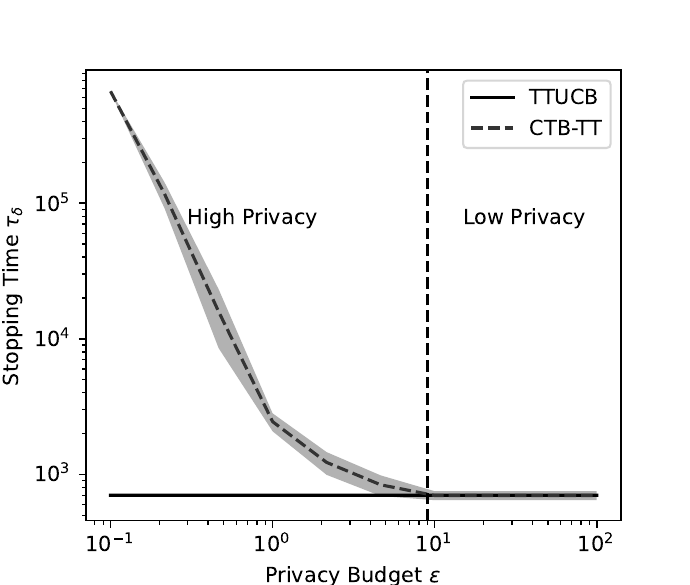}
        \caption{$\mu_3$}    
    \end{subfigure}
    \hfill
    \begin{subfigure}[b]{0.48\textwidth}  
        \centering
        \includegraphics[width=0.8\textwidth]{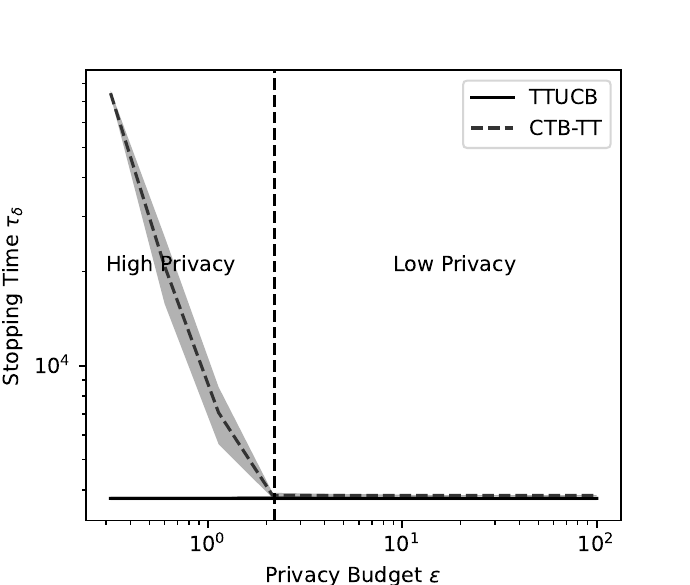}
        \caption{$\mu_4$}    
    \end{subfigure}
    \vskip\baselineskip
    \begin{subfigure}[b]{0.48\textwidth}
        \centering
        \includegraphics[width=0.8\textwidth]{img/env4_local.pdf}
        \caption{$\mu_5$}    
    \end{subfigure}
    \hfill
    \begin{subfigure}[b]{0.48\textwidth}  
        \centering
        \includegraphics[width=0.8\textwidth]{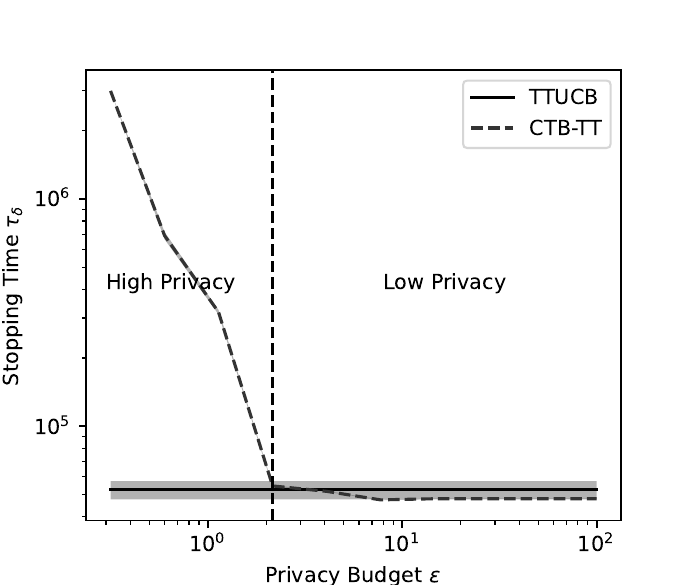}
        \caption{$\mu_6$}    
    \end{subfigure}
    \caption{Evolution of the stopping time $\tau$ (mean $\pm$ std. over 1000 runs) of CTB-TT and TTUCB with respect to the privacy budget $\epsilon$ for $\delta = 10^{-2}$ on different Bernoulli instances. The shaded vertical line separates the two privacy regimes.}  \label{fig:ext_experiments_local}
\end{figure}

\begin{figure}[p]
    \centering
    \begin{subfigure}[b]{0.48\textwidth}
        \centering
        \includegraphics[width=0.8\textwidth]{img/env1_global_log.pdf}
        \caption{$\mu_1$}    
    \end{subfigure}
    \hfill
    \begin{subfigure}[b]{0.48\textwidth}  
        \centering
        \includegraphics[width=0.8\textwidth]{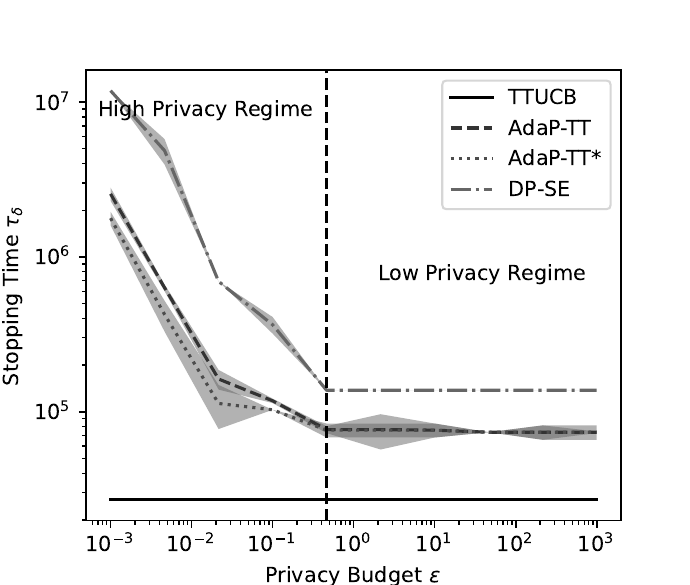}
        \caption{$\mu_2$}    
    \end{subfigure}
    \vskip\baselineskip
    \begin{subfigure}[b]{0.48\textwidth}
        \centering
        \includegraphics[width=0.8\textwidth]{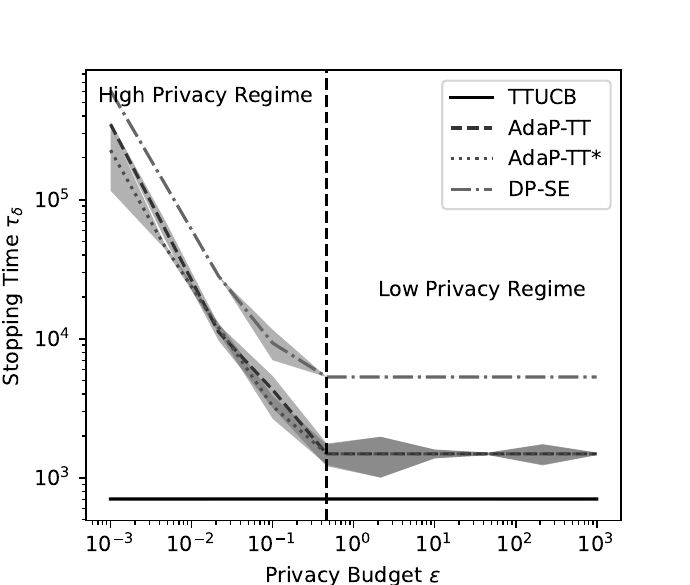}
        \caption{$\mu_3$}    
    \end{subfigure}
    \hfill
    \begin{subfigure}[b]{0.48\textwidth}  
        \centering
        \includegraphics[width=0.8\textwidth]{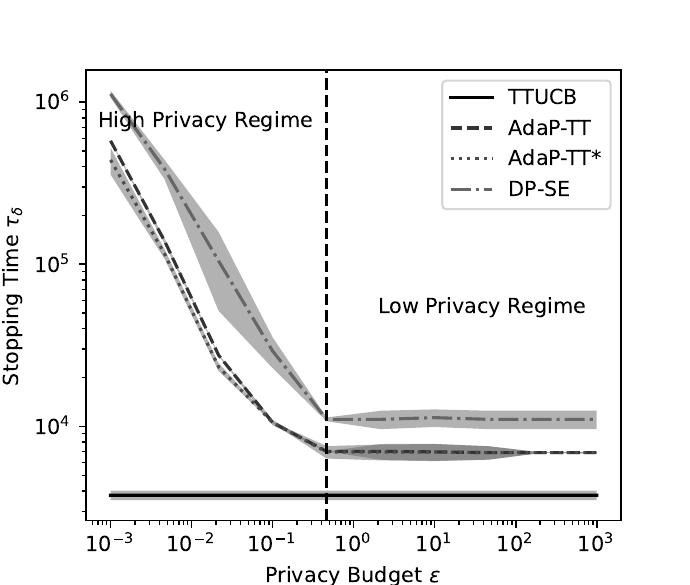}
        \caption{$\mu_4$}    
    \end{subfigure}
    \vskip\baselineskip
    \begin{subfigure}[b]{0.48\textwidth}
        \centering
        \includegraphics[width=0.8\textwidth]{img/env4_global_log.pdf}
        \caption{$\mu_5$}    
    \end{subfigure}
    \hfill
    \begin{subfigure}[b]{0.48\textwidth}  
        \centering
        \includegraphics[width=0.8\textwidth]{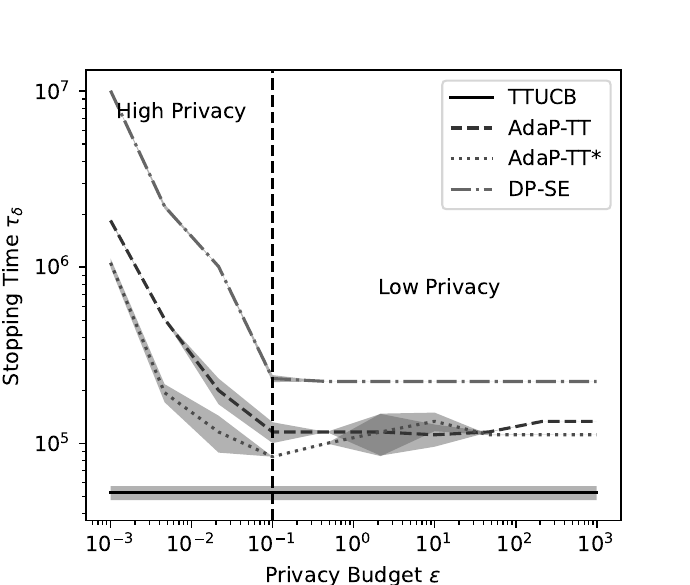}
        \caption{$\mu_6$}    
    \end{subfigure}
    \caption{Evolution of the stopping time $\tau$ (mean $\pm$ std. over 1000 runs) of Imp-\adaptt{}, \adaptt{}, DP-SE, and TTUCB with respect to the privacy budget $\epsilon$ for $\delta = 10^{-2}$ on different Bernoulli instances. The shaded vertical line separates the two privacy regimes. Both the $x$-axis and $y$-axis are in logarithmic scale.}  \label{fig:ext_experiments}
\end{figure}

\begin{figure}[p]
    \centering
    \begin{subfigure}[b]{0.48\textwidth}
        \centering
        \includegraphics[width=0.8\textwidth]{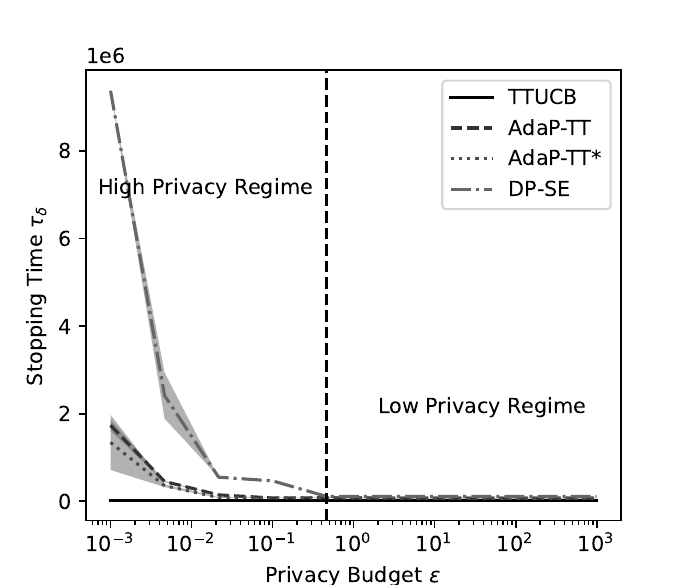}
        \caption{$\mu_1$}    
    \end{subfigure}
    \hfill
    \begin{subfigure}[b]{0.48\textwidth}  
        \centering
        \includegraphics[width=0.8\textwidth]{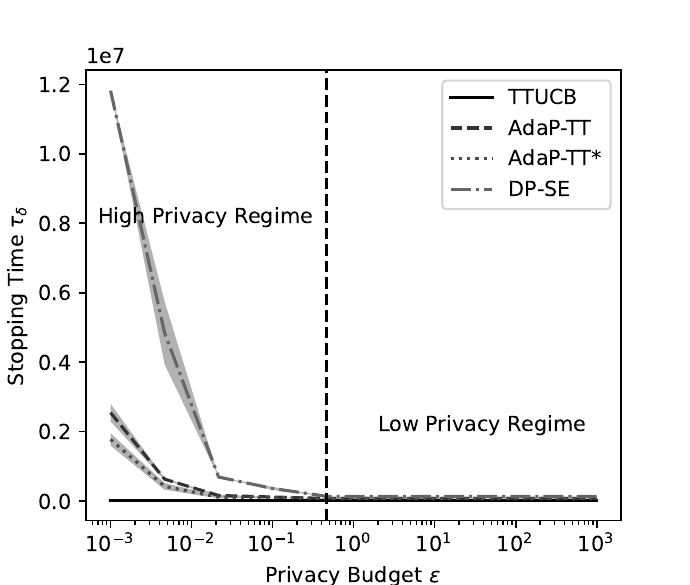}
        \caption{$\mu_2$}    
    \end{subfigure}
    \vskip\baselineskip
    \begin{subfigure}[b]{0.48\textwidth}
        \centering
        \includegraphics[width=0.8\textwidth]{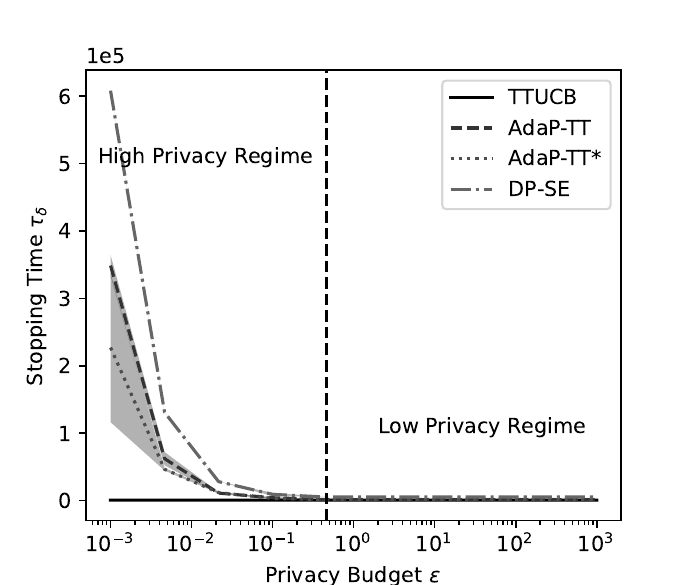}
        \caption{$\mu_3$}    
    \end{subfigure}
    \hfill
    \begin{subfigure}[b]{0.48\textwidth}  
        \centering
        \includegraphics[width=0.8\textwidth]{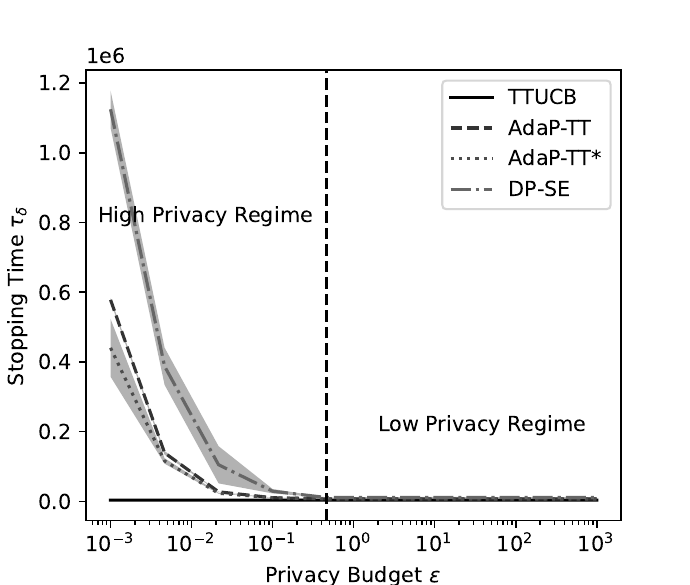}
        \caption{$\mu_4$}    
    \end{subfigure}
    \vskip\baselineskip
    \begin{subfigure}[b]{0.48\textwidth}
        \centering
        \includegraphics[width=0.8\textwidth]{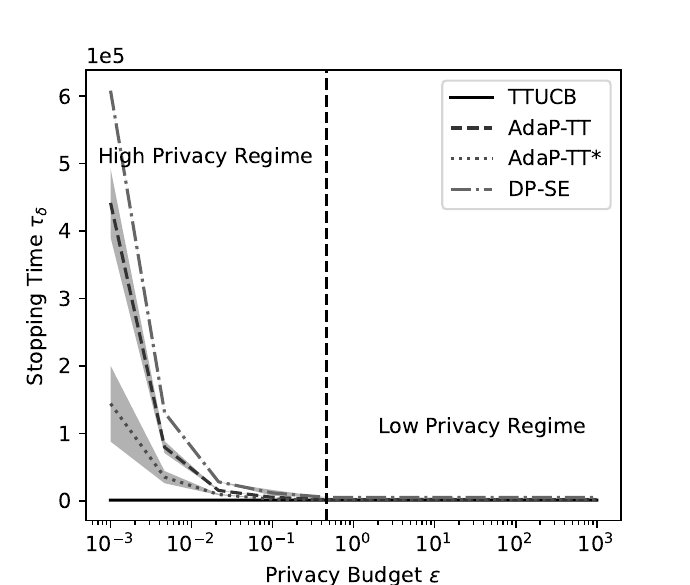}
        \caption{$\mu_5$}    
    \end{subfigure}
    \hfill
    \begin{subfigure}[b]{0.48\textwidth}  
        \centering
        \includegraphics[width=0.8\textwidth]{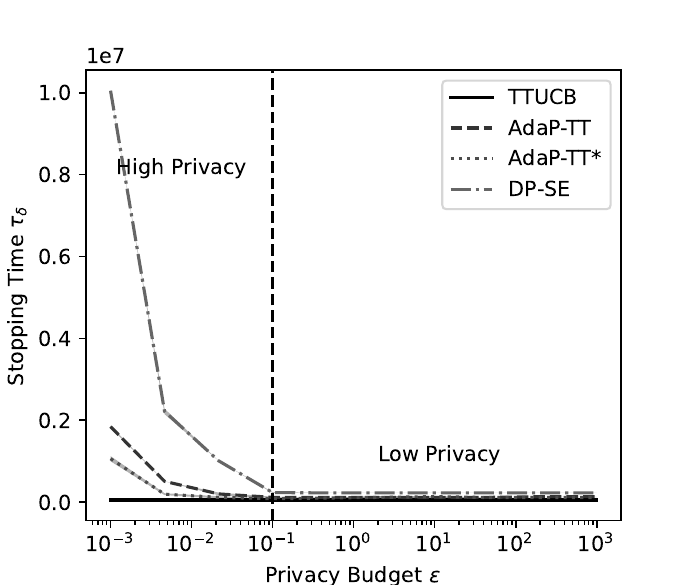}
        \caption{$\mu_6$}    
    \end{subfigure}
    \caption{Evolution of the stopping time $\tau$ (mean $\pm$ std. over 1000 runs) of \adapttt{}, \adaptt{}, DP-SE, and TTUCB with respect to the privacy budget $\epsilon$ for $\delta = 10^{-2}$ on different Bernoulli instances. The shaded vertical line separates the two privacy regimes. Only the $x$-axis is in logarithmic scale.}  \label{fig:ext_experiments_no_log}
\end{figure}

\newpage
\bibliography{main}

\end{document}